\documentclass{article}

\usepackage[margin=1in]{geometry}
\usepackage[colorlinks=true,linkcolor=blue!60!black,citecolor=blue!60!black,urlcolor=blue!60!black]{hyperref}
\usepackage{natbib}

\usepackage{graphicx}
\usepackage{booktabs}
\usepackage{amsmath}
\usepackage{amssymb}
\usepackage{mathtools}
\usepackage{amsthm}
\usepackage{algorithm}
\usepackage[noend]{algpseudocode}
\usepackage{enumitem}
\usepackage{microtype}
\usepackage{xcolor}
\usepackage{placeins}

\theoremstyle{plain}
\newtheorem{theorem}{Theorem}[section]
\newtheorem{proposition}[theorem]{Proposition}
\newtheorem{lemma}[theorem]{Lemma}
\newtheorem{corollary}[theorem]{Corollary}

\theoremstyle{definition}

\newtheorem{assumption}[theorem]{Assumption}
\theoremstyle{remark}
\newtheorem{remark}[theorem]{Remark}

\newcommand{\R}{\mathbb{R}}
\newcommand{\norm}[1]{\left\|#1\right\|}
\newcommand{\inner}[2]{\left\langle #1,\, #2\right\rangle}
\newcommand{\Hb}{\mathbf{H}}
\newcommand{\Ib}{\mathbf{I}}
\newcommand{\Qb}{\mathbf{Q}}
\newcommand{\Tb}{\mathbf{T}}
\newcommand{\Vb}{\mathbf{V}}
\newcommand{\Ub}{\mathbf{U}}
\newcommand{\eps}{\varepsilon}
\DeclareMathOperator*{\argmin}{arg\,min}

\title{Blockwise Stabilized Adaptive Cubic Regularization with
Subsolvers via Recurrence}
\author{Rodion M.~Podorozhny\\
Department of Computer Science, Texas State University\\
San Marcos, TX, USA\\
\texttt{rp31@txstate.edu}}
\date{}

\begin{document}
\maketitle

\begin{abstract}
Overparameterized networks often have ill-conditioned, saddle rich
loss landscapes. They are usually trained with first-order methods
that struggle on those landscapes, while many second-order methods
stay infeasible at scale. Dense Hessian variants in the cubic Newton
family stop scaling near $15$k parameters. Some scalable cubic
methods replace the Hessian by a cheap approximation and drop the
exact block curvature. This paper introduces ARC-block, a blockwise adaptive regularization
with cubics (ARC) optimizer. For each parameter tensor it minimizes a
cubic model over the true block Hessian, with a per-block adaptive
cubic constant and a monotone guard on the full loss. Large tensors
are handled matrix-free in a Lanczos-built Krylov subspace.
ARC-block's CubicKrylov subsolver's descent is monotone with a
per-block cubic decrease bound. For exact steps, and for Krylov steps
that meet the Cartis--Gould--Toint residual test, we provide an
$\mathcal{O}(\sqrt{B}\varepsilon^{-3/2})$ sweep bound above a
cross-block coupling floor. A second subsolver, ARC-$\varphi_1$, uses
a trust-controlled exponential relaxation step: Newton on
high curvature modes, gradient descent on flat modes, and a clamped
exponential escape on negative modes, at the same per-trial cost.
Among the cubic Newton methods evaluated here, these are the only
ones whose steps stay exact on every parameter block (within the
Krylov subspace), including on a $91.4$M-parameter implicit neural
representation (INR).
Stopped after $400$ sweeps without a $0.05$\,dB gain on FINER 2D
image-fitting INR network, ARC-$\varphi_1$ reaches $133.5$\,dB peak
signal-to-noise ratio (PSNR), with no post peak collapse. Tuned Adam
peaks at $78.2$\,dB at $2{,}629$\,s (step ${\sim}625$k, about $44$
minutes), further steps degrade its PSNR. ARC-$\varphi_1$ reaches
that $78.2$\,dB sooner, at $1{,}927$\,s, and by Adam's peak time is
already at $84.2$\,dB. By Adam's $10^6$-step budget (about $70$
minutes) ARC-$\varphi_1$ reaches $95.6$\,dB. On the same benchmark,
SOAP, L-BFGS, and the scalable diagonal cubic method AdaCubic stay at
or below Adam's $78.2$\,dB peak. Over the benchmark's 20 images,
run to a plateau rule, ARC-$\varphi_1$'s per-image PSNR has median
$132.4$ and mean $131.5\pm3.3$\,dB against tuned Adam's $72.8\pm4.5$\,dB.
The paper provides the FINER network loss landscape analysis
explaining the behavior of these optimizers.
\end{abstract}

\section{Introduction}
\label{sec:intro}

First-order optimizers (Adam, AdamW,
Muon: \citealp{kingma2015adam,loshchilov2019,Jordan2024Muon}) use the
gradient alone, with no curvature, and therefore exhibit \emph{spectral bias}.
Low frequency components of the target are fit quickly, while
high frequency components lag by orders of magnitude~\citep{rahaman2019}.
In neural tangent kernel terms, error along each kernel eigenmode
decays at a rate proportional to its eigenvalue, so the kernel
spectrum determines the order in which frequencies
converge~\citep{jacot2018ntk,lee2019wide} (Appendix~\ref{app:ntk}).
This work develops a second-order optimizer that mitigates spectral
bias in training implicit neural representations (INRs), although it can be applied to 
other kinds of neural networks.

\par\noindent\textbf{Spectral bias mitigation.}
The INR literature describes approaches to mitigate spectral bias architecturally, through
Fourier feature embeddings~\citep{tancik2020fourier}, wavelet activations
(WIRE:~\citealp{saragadam2023wire}), variable periodic activations
(FINER:~\citealp{liu2024finer}), and embeddings whose initial spectrum
is aligned with the target's~\citep{kania2025fresh}, and via custom training through diagonal preconditioners~\citep{chng2025precond},
tangent kernel gradient transformations~\citep{shi2025iga}, and
stochastic preconditioning~\citep{ling2025stochprecond}. Most use
first-order or diagonal curvature models, as does
AdaCubic~\citep{tsingalis2026adacubic}, a scalable cubic regularized optimizer
with a Hutchinson diagonal Hessian approximation.

\par\noindent\textbf{Comparative analysis of Adam's performance.}
Adam \citep{kingma2015adam} trains many networks across problem
domains, image processing INRs included, with success, and their
authors often report that the more computationally expensive
second-order methods cannot compete with it. Our results on networks
such as FINER (\citealp{liu2024finer}), whose loss landscape is more
complex, show otherwise. One line of research on the reasons for
Adam's performance attributes its advantage over stochastic gradient
descent (SGD) to coordinate-aligned landscape structure
\citep{xie2024linf,zhang2024transformers,dong2025hessianstruct,jiang2023geometry,cohen2022adaptive,zhangmaes2025rotation},
and on quadratics \citet{das2024precond} show that Adam performs well on
near diagonal Hessians and can lose to gradient descent for
non-diagonal ones. Section~\ref{sec:fingerprint} relates our
benchmark results to this line of work.

\par\noindent\textbf{Cubic Newton methods are often not scalable.}
Cubic regularized Newton
methods~\citep{griewank1981modification,nesterov2006cubic}, including
the adaptive regularization with cubics (ARC) of
\citet{cartis2011a,cartis2011b},
achieve the optimal $\mathcal{O}(\eps^{-3/2})$ rate, finding
approximate stationary points of non-convex objectives, but each step
usually needs a full eigendecomposition or factorization at
$\mathcal{O}(n^3)$ time and $\mathcal{O}(n^2)$ storage
\citep{nesterov2006cubic,cartis2011a,doikov2023fo,doikov_thesis}. The
refinements of this family
\citep{doikov2023lazy,doikov2023fo,doikov2024spectral,mishchenko2023,agafonov2024inexact,chayti2025momentum,chayti2024unified}
mitigate this cubic cost through subspace methods, low-rank
approximations, or approximate Hessians, while keeping a full space
linear solve as the basic step, and are often validated at only
$n\lesssim10^3$ scale. In our implementations and on our hardware, all of
them and the subspace cubic Newton SSCN
\citep{zhao2025sscn} remained impractical beyond roughly 15k--20k
parameters (Appendix~\ref{sec:feasibility}). The original ARC with its
Lanczos subsolver \citep{cartis2011a} is matrix-free, yet it stalls at
$4.0$\,dB peak signal-to-noise ratio (PSNR) on the $91.4$M ViSIR while
placing first on the $15$k SIREN-like replica (\citealp{siren2020}),
evidence that a Krylov subspace subsolver without the blockwise
organization may not scale.

\par\noindent\textbf{Inspiration for ARC-block: Chebyshev second kind acceleration and related work.}
\label{sec:dso}
ARC-block descends from the second-order Chebyshev second kind
accelerated optimizers such as the scaled up Dynamic Spectral
Optimizer (DSO) of~\citet{podorozh2026siam}, whose construction is
restated in full in Appendix~\ref{app:dso}. The ARC-block outer
scheme admits that recurrence as a subsolver (CubicCheby-DSO,
Appendix~\ref{sec:coststructure}), and the Lanczos basis is the
preferred one for the reasons of Section~\ref{sec:krylovtheory}. The Krylov
cubic step of \citet{gould1999gltr,cartis2011a,carmon2018krylov,jiang2024}
reaches a similar subsolver in the trust region line of research. Most
Chebyshev accelerated optimizers are first-order. The schemes surveyed
by~\citet{daspremont2021acceleration} and the polynomial
preconditioning of~\citet{doikov2023poly} (Chebyshev polynomials of the
first kind, $T_{\tau+1}$ on $[\lambda_n,\lambda_1]$) are gradient recurrences
with no explicit curvature. The closest second-order relative is the rank-$\tau$ spectral
preconditioning of~\citet{doikov2024spectral}, which takes no polynomial step: it keeps the top
$\tau$ Hessian eigenvectors $V_k\in\R^{n\times\tau}$ ($\tau n$ floats),
refreshes them each iteration by $T_k$ orthogonal iteration steps
($\tau$ HVPs and an $\mathcal{O}(\tau^2 n)$ QR each), and applies the
preconditioner through the Woodbury
identity at $\mathcal{O}(\tau^2 n)$. DSO instead applies Chebyshev
polynomials of the \emph{second kind} $U_n$ through a three term
recurrence of Hessian--vector products (HVPs), $L$ HVPs and three
vectors per block with no stored matrix, with a worst case residual of
${\approx}0.22$ per sweep at degree $L\approx1.28\sqrt{\kappa}$
(Appendix~\ref{app:dso}). The recurrence replaces the eigenvector
refresh and the Woodbury solve: $\mathcal{O}(Ln)$ computational time and $3n$
storage per sweep against $\mathcal{O}(\tau^2 n)$ and $\tau n$ per
iteration for the rank-$\tau$ spectral preconditioner of \citet{doikov2024spectral}, a factor $L$ against
$\tau^2$ and $3$ against $\tau$. The larger difference is in the step.
The rank $\tau$ preconditioner inverts the Hessian on $\tau$
directions and takes a gradient step on the other $n-\tau$, while the
degree $L$ polynomial contracts every direction in the spectral
interval. DSO trains the $91.4$M parameter ViSIR
(Appendix~\ref{sec:coststructure}), and on the $15$k SIREN-like
replica, whose spectrum has no drop after its first ten eigenvalues,
the rank-$\tau$ preconditioner at $\tau=10$ saturates at $27.7\pm3.3$\,dB
against $30.0$--$33.1$\,dB for DSO (Tables~\ref{tab:nanoesm20},
\ref{tab:nanobudget}). The drawbacks are that the polynomial degree
$L$ grows with $\kappa$ and can become too large, especially when the
estimate of the Hessian condition number is inaccurate, and that the
magnitude of the search step on indefinite curvature needs clamping,
because $U_n$ grows hyperbolically for $\lambda<0$ (the mechanism of
its saddle escape). In ARC-block, the cubic shift mitigates both of
these drawbacks (Section~\ref{sec:theory}).

\par\noindent\textbf{Approach.}
The proposed ARC-block outer scheme is blockwise adaptive cubic regularization.
Parameters are partitioned by tensor, each block minimizes a
cubic model with its own adapted cubic constant $M_b$, and
each block step is accepted against a monotone guard on the full loss
(Proposition~\ref{prop:perblock}). The subsolver is interchangeable.
The cubic shift $\lambda^\star=\tfrac{M}{2}\norm{s}$ bounds the
effective condition number of the shifted Hessian operator
(Proposition~\ref{prop:degree}) and makes it
positive semidefinite \emph{before} the polynomial processes it
(Lemma~\ref{lem:pd}). The full space cubic step has the global
$\mathcal{O}(\eps^{-3/2})$ rate to approximate second-order stationary
points (Theorem~\ref{thm:rate}, Corollary~\ref{cor:sosp}). The
blockwise sweep keeps the exponent, $\mathcal{O}(\sqrt B\,\eps^{-3/2})$
sweeps down to a coupling floor determined by the off-block Hessian, with
curvature certified per block on its Krylov subspace
(Theorem~\ref{thm:blockrate}). A second
subsolver keeps the outer scheme and replaces the cubic minimizer by
the exponential relaxation ($\varphi_1$) step: Newton-like on
high curvature modes, gradient-like on flat ones, an exponential
escape on negative ones (Section~\ref{sec:phi1}).

\subsection{Contributions}
\label{sec:contributions}

\begin{itemize}[leftmargin=2em,itemsep=2pt]
  \item \textbf{A blockwise outer scheme with per-block adaptive
        regularization (ARC-block).} Each parameter tensor has its
        own regularization constant, ratio test, and acceptance
        decision against a monotone guard on the full loss
        (Algorithm~\ref{alg:blockcn}, Proposition~\ref{prop:perblock}).
  \item \textbf{Proofs.} Proposition~\ref{prop:degree} bounds the
        subsolver's effective condition number by the cubic shift.
        Proposition~\ref{prop:perblock} gives the per-block cubic
        decrease, the Cartis--Gould--Toint acceptance argument applied
        per block. Theorems~\ref{thm:phaseA}--\ref{thm:blockrate} give
        the Phase~A, Phase~B, and joint sweep rates
        $\mathcal{O}(\sqrt B\,\eps^{-3/2})$ down to a measurable
        coupling floor (Appendix~\ref{app:remarks}). The full space
        results are restated with proofs in our notation
        (Lemma~\ref{lem:krylov} for the Gould--Lucidi--Roma--Toint
        reduction, Theorem~\ref{thm:rate} and Corollary~\ref{cor:sosp}
        for Nesterov and Polyak's rate). The algebraic cores of the
        blockwise, CubicKrylov, and $\varphi_1$ arguments were checked
        by the Lean~4 kernel (Appendix~\ref{app:leancheck}), with the
        cited arguments and Assumption~\ref{ass:lip} as hypotheses.
  \item \textbf{A trust-controlled $\varphi_1$ subsolver
        (ARC-$\varphi_1$).} The cubic minimizer is replaced by the
        exponential relaxation ($\varphi_1$) step rule
        (Proposition~\ref{prop:erlaw}): Newton on high curvature modes,
        gradient descent on flat modes, and a clamped exponential escape along
        negative modes (Lemma~\ref{lem:clamp}, Section~\ref{sec:phi1}).
  \item \textbf{Empirical results and a dispatch rule.} The 91.4M
        parameter ViSIR network is trained with the exact block Hessian
        through HVPs (Section~\ref{sec:experiments}). On
        FINER, ARC-$\varphi_1$, run to its plateau, reaches
        133.5\,dB against tuned Adam's 78.2\,dB on the reference
        image and a mean $131.5\pm3.3$ against $72.8\pm4.5$\,dB over
        20 images (Section~\ref{sec:finer}), and the landscape fingerprint
        (Section~\ref{sec:fingerprint}) distinguishes axis-aligned
        ill-conditioning, where tuned first-order methods perform well,
        from cross-coupled or saddle dominated landscapes, where the
        blockwise cubic step's cost is justified.
\end{itemize}

\section{Background}
\label{sec:background}
This section fixes the notation and recalls the cubic regularized
step that ARC-block applies per block.

\subsection{Problem setting}

We minimize a twice-differentiable, generally non-convex training loss
\begin{equation}
  \min_{\theta\in\R^n} f(\theta), \qquad
  f(\theta)=\frac1N\sum_{i=1}^N \tfrac12\norm{\Phi_\theta(x_i)-y_i}^2 ,
\label{eq:problem}
\end{equation}
where $\Phi_\theta$ is a neural network with weights $\theta$. We write
$g=\nabla f(\theta)$, $\Hb=\nabla^2 f(\theta)$ with eigenvalues
$\lambda_1\ge\dots\ge\lambda_n$ and orthonormal eigenvectors
$u_1,\dots,u_n$. Following~\citet{doikov2024spectral} we assume:

\begin{assumption}[Lipschitz--Hessian]
\label{ass:lip}
There exists $L_H\ge0$ such that
$\norm{\nabla^2 f(x)-\nabla^2 f(y)}\le L_H\norm{x-y}$ for all $x,y$.
\end{assumption}

Assumption~\ref{ass:lip} yields the standard cubic upper bound
\begin{equation}
  f(x+s)\;\le\; f(x)+\inner{g}{s}+\tfrac12\inner{\Hb s}{s}
  +\tfrac{L_H}{6}\norm{s}^3 .
\label{eq:cubicbound}
\end{equation}

\subsection{Cubic regularization as a shifted linear solve}
\label{sec:cubic}

The cubic regularized Newton step \citep{nesterov2006cubic} minimizes
\begin{equation}
  m(s)\;=\;\inner{g}{s}+\tfrac12\inner{\Hb s}{s}+\tfrac{M}{6}\norm{s}^3 .
\label{eq:cubicmodel}
\end{equation}
Its global minimizer is characterized (Lemma~\ref{lem:secular}) by the
stationarity system
\begin{equation}
  \Bigl(\Hb+\tfrac{M r}{2}\,\Ib\Bigr)s=-g,\qquad r=\norm{s},
  \qquad \Hb+\tfrac{Mr}{2}\Ib\succeq 0,
\label{eq:shifted}
\end{equation}
a damped Newton solve whose damping $\lambda^\star=Mr/2$ is determined
by a one-dimensional secular equation. Usually, the solver diagonalizes $\Hb$
\citep{nesterov2006cubic,doikov2023fo}. On large blocks we build a
Lanczos basis instead (Section~\ref{sec:krylovtheory}). The Krylov
cubic Newton method of \citet{jiang2024} is a different approach, one
Krylov subspace of the \emph{whole} parameter space with a dimension-free
rate for \emph{convex} objectives. Here the objective is non-convex,
the subspace is built and regularized \emph{per block} with its own
$M_b$ (Appendix~\ref{sec:feasibility}), and the depth is tied to the cubic
shift by Proposition~\ref{prop:degree}.

\section{The Per-Block Cubic Step and its Krylov Subsolver}
\label{sec:theory}

This section develops the theory of the per-block step, with all
proofs in Appendix~\ref{app:proofs}. The results concern the cubic step
on a single block (identically, on the whole parameter vector), hold
for a generic symmetric $\Hb$ with eigenvalues
$\lambda_1\ge\dots\ge\lambda_n$, $g\ne0$, and $M>0$, and for any
subsolver.

\subsection{The secular equation, the shift, and the degree bound}
\label{sec:seculartheory}

Lemma~\ref{lem:secular} (Appendix~\ref{app:proofs}) restates Section~5
of~\citet{nesterov2006cubic} (see also \citealp{cartis2011a}) in our
notation. If $\inner{u_n}{g}\ne0$ or $\lambda_n\ge0$, the secular
equation $\lambda=\tfrac{M}{2}\varphi(\lambda)$ with
$\varphi(\lambda)=\norm{(\Hb+\lambda\Ib)^{-1}g}$ has a unique solution
$\lambda^\star>\max\{0,-\lambda_n\}$, and
$s^\star=-(\Hb+\lambda^\star\Ib)^{-1}g$ is the global minimizer of
\eqref{eq:cubicmodel}. Lemma~\ref{lem:pd} (Appendix~\ref{app:proofs})
adds the floor $r^\star\ge2\max\{0,-\lambda_n\}/M$, so the shifted
operator $\Hb+\tfrac{Mr^\star}{2}\Ib$ is always positive semidefinite,
and positive definite in the non-degenerate case. Negative curvature
is thus handled \emph{before} the polynomial operates on the system.

Proposition~\ref{prop:degree} (Appendix~\ref{app:proofs}) makes the
self-adaptive depth schedule precise. The shift obeys
$\lambda^\star\ge\lambda_{\mathrm{lo}}
:=\tfrac12\big(\sqrt{\lambda_1^2+2M\norm{g}}-\lambda_1\big)$, and
whenever $\lambda_{\mathrm{lo}}\ge2\max\{0,-\lambda_n\}$ the shifted
operator has effective condition number
$\kappa_{\mathrm{eff}}:=\frac{\lambda_1+\lambda^\star}
{\lambda_n+\lambda^\star}\le
\frac{2(\lambda_1+\lambda_{\mathrm{lo}})}
{\max\{\lambda_n,0\}+\lambda_{\mathrm{lo}}}$, so a Chebyshev or
Lanczos solver needs depth
$L=\mathcal{O}\big(\sqrt{\kappa_{\mathrm{eff}}}\,
\ln\tfrac1{\eps_{\mathrm{rel}}}\big)$ for a relative accuracy
$\eps_{\mathrm{rel}}$. Far from stationarity
($2M\norm{g}\ge\lambda_1^2$)
$\kappa_{\mathrm{eff}}\le2+\mathcal{O}(\lambda_1/\sqrt{M\norm{g}})$,
a constant independent of $\kappa=\lambda_1/\lambda_n$. Near a strict
second-order stationary point the shift vanishes and
$\kappa_{\mathrm{eff}}\to\lambda_1/\lambda_n$. The Krylov subproblem
rates of \citet{carmon2018krylov} also depend on the shifted spectrum.

\subsection{Krylov subspace solver and its optimality}
\label{sec:krylovtheory}

We solve the cubic subproblem over the Krylov
subspace, following the Lanczos-based trust region and cubic
subproblem solvers of \citet{gould1999gltr,cartis2011a}. Run $L$
steps of the Lanczos process~\citep{golub2013} on $(\Hb,g)$ with full
reorthogonalization (one HVP per step), producing
$\Qb_L\in\R^{n\times L}$ whose columns span
$\mathcal{K}_L(\Hb,g)=\mathrm{span}\{g,\Hb g,\dots,\Hb^{L-1}g\}$, with
$\Qb_L^\top\Hb\,\Qb_L=\Tb_L$ tridiagonal ($L\times L$) and
$\Qb_L^\top g=\norm{g}\,e_1$. The Lanczos depth $L$ grows until the Cartis--Gould--Toint
test~\eqref{eq:tc} is met,
$\norm{\nabla m(s_L)}\le\kappa_\theta\min(1,\norm{s_L})\norm{g}$,
which is the inexactness condition that transfers the
$\mathcal{O}(\eps^{-3/2})$ rate to a subspace step
(Remark~\ref{rem:inexact}, Lemma~\ref{lem:blockgrad}).
Proposition~\ref{prop:degree} gives a worst case ceiling on that depth.
The Lanczos process on $(\Hb_b,g_b)$ builds a basis of the Krylov
subspace by a three term recurrence whose coefficients are computed
from that pair, where the Chebyshev recurrence uses coefficients set
in advance for a spectral interval. It evaluates the $g_b$-weighted
spectrum of the tensor by Gauss quadrature (the Ritz values of $\Tb_L$,
Remark~\ref{rem:gaussquad}), and the secular solve then takes the
exact cubic minimizer in that space (Lemma~\ref{lem:krylov}).
Modes orthogonal to $g_b$ are not in the subspace
(Remark~\ref{rem:hardcase}). Corollary~\ref{cor:dominate} adds that the
depth-$L$ Krylov step attains, at most, the model value of the
Chebyshev second kind recurrence step (the step that motivated the
ARC-block design) in the same space, at one HVP per basis vector
(memory in Appendix~\ref{sec:coststructure}). \citet{carmon2018} describe
gradient-based (first-order) subsolvers for the same subproblem, and
\citet{zhu2022acrl} bound the Lanczos projection error through the
condition number of the shifted Hessian.

\subsection{One-step progress and global complexity}
\label{sec:complexitytheory}

The exact cubic step, and any Krylov step meeting~\eqref{eq:tc}, has
the standard global guarantee. The theorem below mirrors Lemma~4.1 and
Theorem~4.2 of~\citet{doikov2024spectral} and the telescoping argument
of~\citet{nesterov2006cubic}, Section~4, with Lemmas~\ref{lem:decrease}
and~\ref{lem:newgrad} in Appendix~\ref{app:proofs}.

\begin{theorem}[Global complexity]
\label{thm:rate}
Let $f$ satisfy Assumption~\ref{ass:lip} with
$f^\star:=\inf f>-\infty$, and let $\{x_k\}$ be generated by
$x_{k+1}=x_k+s_k^\star$ where $s_k^\star$ minimizes the cubic model at
$x_k$ with constant $M\ge L_H$. Then for any $\eps>0$,
$\min_{1\le i\le K}\norm{\nabla f(x_i)}\le\eps$ holds once the
iteration count reaches
\[
  K=\left\lceil
  \tfrac{12\,(f(x_0)-f^\star)}{M}
  \bigl(\tfrac{L_H+M}{2}\bigr)^{3/2}\eps^{-3/2}
  \right\rceil
  =\mathcal{O}\!\Bigl(\tfrac{\sqrt{M}\,(f(x_0)-f^\star)}
  {\eps^{3/2}}\Bigr).
\]
\end{theorem}

Corollary~\ref{cor:sosp} (Appendix~\ref{app:proofs}) extends
Theorem~\ref{thm:rate} to approximate second-order stationarity. For
any $\eps>0$, within the same iteration count $K$ some iterate $x_i$
satisfies both $\norm{\nabla f(x_i)}\le\eps$ and
$\lambda_n(\nabla^2 f(x_i))\ge
-\tfrac{M+2L_H}{\sqrt{2(L_H+M)}}\sqrt{\eps}$, so for small $\eps$ the
curvature condition rules out strict saddles as limit points. Inexactness of the subproblem solve and the adaptivity of
$M$ are discussed in Remark~\ref{rem:inexact}
(Appendix~\ref{app:remarks}). Remark~\ref{rem:global} delimits the
scope of the guarantee.

\subsection{Block-diagonal decomposition at scale}
\label{sec:blocktheory}

We now partition $\theta=(\theta_1,\dots,\theta_B)$ by tensors (each block
a single weight or bias tensor), let $\Hb_b$ denote the diagonal
blocks, and apply the cubic step \emph{per block}, sequentially, with
fresh gradients. Block $b$ minimizes
$m_b(s_b)=\inner{g_b}{s_b}+\tfrac12\inner{\Hb_b s_b}{s_b}
+\tfrac{M_b}{6}\norm{s_b}^3$ where $g_b$ is the gradient evaluated
\emph{after} blocks $1,\dots,b-1$ have been updated (a fresh gradient,
which the theorems below assume for every block).
Proposition~\ref{prop:perblock} (Appendix~\ref{app:proofs}) supplies
the guarantee. Under the per-block ratio test of
Algorithm~\ref{alg:blockcn} the iteration is monotone, and every
accepted block step, exact or Krylov at any depth, decreases $f$ by at
least $\eta_1\tfrac{M_b}{12}\norm{s_b}^3$.
Remark~\ref{rem:joint} explains why
joint acceptance over many blocks fails at scale. The cross-block
coupling, absent from every block model, appears only in the gradient at a sweep's end. Theorem~\ref{thm:phaseA} bounds the small block gradient
after Phase~A and Theorem~\ref{thm:phaseB} bounds the large block
gradient after Phase~B. Theorem~\ref{thm:blockrate}
(Appendix~\ref{app:remarks}) assembles both phases of one sweep,
Phase~B is retried until accept, and gives
$\mathcal{O}(\sqrt{B}\,\eps^{-3/2})$ sweeps to an
$\eps$-second-order point for every $\eps$ above a floor
$\eps_\star=2B\delta^2/C_1$ determined by the off-block curvature $\delta$,
which is measured by the loss landscape fingerprint. A Krylov step is successful when it meets
the termination test of \citet{cartis2011b}, checked per trial
(Appendix~\ref{app:certified}).

\section{ARC-block Algorithm}
\label{sec:algorithm}

We now assemble the step of Section~\ref{sec:theory} into the blockwise
optimizer, \emph{ARC-block}. In the experiments, the large block step
is called \emph{CubicKrylov}. The name is also used as an optimizer
label for the blockwise scheme before the two per-block stabilizations of
Section~\ref{sec:fullvisir91}. The Chebyshev second kind recurrence
of \citet{podorozh2026siam} is one of the interchangeable subsolvers
(CubicCheby-DSO, Appendix~\ref{sec:coststructure}). The Lanczos basis
is used for the reasons of Section~\ref{sec:krylovtheory}.

\subsection{The CubicKrylov step (single block)}

Given an HVP oracle $v\mapsto\Hb v$ for one block: \emph{(i)} a
Lanczos build ($L$ HVPs, float64 accumulation, full reorthogonalization)
gives $\Qb_L,\Tb_L$ with $\Qb_L^\top g=\norm{g}e_1$; \emph{(ii)} the
reduced secular equation~\eqref{eq:secular} is solved on the
$L\times L$ tridiagonal $\Tb_L=\Vb\,\mathrm{diag}(\nu)\,\Vb^\top$
by safeguarded bisection or Newton for the shift
$\lambda>\max\{0,-\nu_{\min}\}$ (Lemma~\ref{lem:pd}), at
$\mathcal{O}(L^3)$ cost, negligible for $L\ll n_b$; \emph{(iii)} the
step is lifted, $s_L=\Qb_L y^\star$. Every trial $\lambda$ reuses the
cached $\Qb_L,\Tb_L$, so the whole root finding loop costs $L$ HVPs in
total, independent of the number of trial shifts, with no
$\mathcal{O}(n_b^3)$ factorization and no $\mathcal{O}(n_b^2)$
storage~\citep{gould1999gltr}.

\subsection{The blockwise algorithm and its design decisions}

\begin{algorithm}[t]
\footnotesize
\caption{\textsc{ARC-block}: blockwise stabilized adaptive cubic
regularization. The large block subsolver \textsc{CubicKrylov} is
interchangeable with the Chebyshev second kind recurrence or the
$\varphi_1$ step (Appendix~\ref{sec:coststructure},
Section~\ref{sec:phi1}).}
\label{alg:blockcn}
\begin{algorithmic}[1]
\Require parameters $\theta=(\theta_1,\dots,\theta_B)$; size threshold
$n_{\max}$; Lanczos depth $L$; laziness $p_b$; Lipschitz estimate
$L_H$; floor $M_{\min}$ and cap $M_{\max}$; initialization
$M_{\mathrm{init}}=6L_H$; ratio thresholds $\eta_1=0.1$, $\eta_2=0.9$;
factor $\gamma=2$; tolerance $t=\tau_{\mathrm{rel}}|\ell_{\mathrm{cur}}|+\tau_{\mathrm{abs}}$;
rejection cap $n_{\mathrm{rej}}$
\State partition blocks: $\mathcal{S}=\{b: n_b\le n_{\max}\}$ (small),
       $\mathcal{L}=\{b: n_b>n_{\max}\}$ (large)
\For{$k=0,1,2,\dots$}
  \State $g\gets\nabla f(\theta)$;\quad $\ell_{\mathrm{cur}}\gets f(\theta)$
         \Comment{sweep start closure; 1 gradient, 1 forward}
  \Statex \hspace{1.2em}\textbf{Phase A (small blocks; lazy exact
          cubic from the sweep start gradient, guarded):}
  \For{$b\in\mathcal{S}$}
    \If{$k\equiv0 \pmod{p_b}$ or the last visit of $b$ was rejected}
       rebuild $\Hb_b$ exactly; cache $\mathrm{eigh}(\Hb_b)$
    \EndIf
    \State $s_b\gets$ exact cubic step for $(g_b,\Hb_b)$ from the cached
           eigendecomposition with $M=6\,p_b\,L_H$
           \Comment{\citealp{doikov2023lazy}; $s_b=0$ if $g_b=0$}
    \State $\theta_b\mathrel{+}= s_b$;\quad
           $\ell_{\mathrm{new}}\gets f(\theta)$ \Comment{1 forward}
    \If{$\ell_{\mathrm{new}}>\ell_{\mathrm{cur}}$ or not finite}
       revert $\theta_b$; mark $b$ rejected \Comment{descent guard}
    \Else{} $\ell_{\mathrm{cur}}\gets\ell_{\mathrm{new}}$
    \EndIf
  \EndFor
  \Statex \hspace{1.2em}\textbf{Phase B (large blocks; per-block
          CubicKrylov, retry; acceptance against $\ell_{\mathrm{cur}}$
          after Phase A):}
  \For{$b\in\mathcal{L}$}
    \State $g_b\gets\nabla_{\theta_b} f(\theta)$ with graph
           \Comment{\emph{fresh} gradient, post Phase A / prior blocks}
    \If{$\norm{g_b}=0$}
       \State \textbf{next} $b$ \Comment{skip; $r_b=0$}
    \Else
    \For{trial $=1,\dots,n_{\mathrm{rej}}+1$}
    \State $s_b\gets\textsc{CubicKrylov}(\mathrm{hvp}_b,\,g_b,\,M_b,\,L)$
           \Comment{Lemma~\ref{lem:krylov}; $L{+}3$ gevals per trial}
    \State $\theta_b\mathrel{+}= s_b$;\quad
           $\ell_{\mathrm{new}}\gets f(\theta)$ \Comment{1 forward}
    \State $\rho_b\gets(\ell_{\mathrm{cur}}-\ell_{\mathrm{new}}+t)/(-m_b(s_b)+t)$
           \Comment{model decrease from one HVP}
    \If{$\rho_b<\eta_1$ or $\ell_{\mathrm{new}}>\ell_{\mathrm{cur}}+t$ or not finite}
       \State revert $\theta_b$;\quad $M_b\gets\min(\gamma M_b,\,M_{\max})$
             \Comment{reject; same $g_b$, larger $M_b$}
    \Else
       \State $\ell_{\mathrm{cur}}\gets\ell_{\mathrm{new}}$;\quad
              \textbf{if} $\rho_b\ge\eta_2$: $M_b\gets\max(M_b/\gamma,\,M_{\min})$;\quad
              \textbf{break}
    \EndIf
    \EndFor
    \EndIf
  \EndFor
\EndFor
\end{algorithmic}
\end{algorithm}

Algorithm~\ref{alg:blockcn} embeds the step into the block optimizer.
Three of its design decisions, each a response to an observed failure,
deserve emphasis here, with the full narratives and the computational cost 
table in Appendix~\ref{app:design}. \emph{(a) Two phase update with guarded
small blocks.} Small blocks (biases and norms, $n_b\le512$) take lazy
exact cubic steps \citep{doikov2023lazy} from the sweep start
gradient, each kept only if the loss does not rise (a descent guard,
with no ratio test and no $M$ adaptation). Phase~B acceptance is
measured against the loss after Phase~A. \emph{(b) Per-block
acceptance and per-block $M_b$} (Remark~\ref{rem:joint}). A joint test
over large tensors may reject even when every block step is good, so
per-block testing is used, at one forward pass per large block. A
rejected Phase~B trial restores block $b$, raises $M_b$, and is retried
in that sweep until acceptance, which Proposition~\ref{prop:perblock}
guarantees once $M_b\ge L_{H,b}$.
\emph{(c) A nearly zero floor for
$M_b$.} A floor $M_{\min}=6L_H$ over-regularizes blocks whose local
Hessian Lipschitz constant is far below the global estimate (accepted
but microscopic steps). So we set $M_{\mathrm{init}}=6L_H$
and let the adaptation of Algorithm~\ref{alg:blockcn} ($M_b\gamma$ on
a rejected trial, $M_b/\gamma$ on a very successful one,
$\rho_b\ge\eta_2$) take $M_b$ down to a nearly zero $M_{\min}$.
Phase~B gradients are recomputed after Phase~A and after each
preceding large block. Phase~A reuses the gradient computed at the
sweep start. Its Hessians use lazy computation. Reusing that gradient in
Phase~B too lowers the final PSNR by ${\approx}8.5$\,dB on ViSIR-Nano
(Table~\ref{tab:nano}). The CubicKrylov step (e.g.\ with $L\approx10$) is the
only entry of Table~\ref{tab:cost} whose oracle cost is independent of
$n_b$, which makes cubic steps feasible at large scale, e.g.\ on the
88.5M parameter ViSIR decoder weight.

\section{The Exponential Relaxation ($\varphi_1$) Step Subsolver}
\label{sec:phi1}

The subsolver's \emph{step rule} is an interchangeable component of
the ARC-block outer scheme. This section introduces the
\emph{exponential relaxation} step, known in the exponential
integrator literature as a $\varphi_1$ step
\citep{hochbruck1997krylov,hochbruck2010exponential},
\begin{equation}
s \;=\; -H(G,h)\,g, \qquad
H(G,h) \;=\; \int_0^h e^{-G\tau}\,d\tau \;=\; h\,\varphi_1(-hG),
\label{eq:relex}
\end{equation}
the exact solution at time $h$ of the gradient flow
$\dot{x}=-g-G\,(x-x_0)$ of the same local quadratic model, integrated
in closed form up to a finite horizon $h$
(Remark~\ref{rem:relexderiv}, Appendix~\ref{app:remarks}).

Here $G$ plays the role of the block Hessian $\Hb_b$ and $g$ the
block gradient $g_b$ of Section~\ref{sec:algorithm}. The outer scheme is reused
unchanged: the per-block trust constant $\sigma_b$ (with $M_b=2\sigma_b$ as implemented), the ratio test, and the proportional offset monotone
guard of the stabilized block-$\sigma$ ARC ablation
(Appendix~\ref{sec:feasibility}),
\begin{equation}
f(x + s_b) \;\le\; f(x) + t, \qquad
t \;=\; \tau_{\mathrm{rel}}\,|f(x)| + \tau_{\mathrm{abs}},
\label{eq:guard}
\end{equation}
of Section~\ref{sec:fullvisir91}. We write $h_b=c_h/\sigma_b$ for the
relaxation horizon, so the
trust adaptation is also the $\varphi_1$ step size control. Rejection
contracts the horizon, and sustained success extends it toward the
Newton limit. The resulting method, \textbf{ARC-$\varphi_1$}, shares the
outer scheme with ARC-block, so any performance difference is due to
the step rule.

\subsection{The $\varphi_1$ step rule}
\label{sec:relexlaw}

\begin{proposition}[$\varphi_1$ step rule]
\label{prop:erlaw}
Let $G$ be symmetric with eigendecomposition
$G = \Vb \Lambda \Vb^\top$. Then
$H(G,h) = \Vb\,\mathrm{diag}\!\big(\eta_h(\lambda_i)\big)\,\Vb^\top$ with
\begin{equation}
\eta_h(\lambda) \;=\; \frac{1 - e^{-h\lambda}}{\lambda}
\quad (\lambda \ne 0), \qquad \eta_h(0) = h,
\label{eq:mult}
\end{equation}
and: (i) $\eta_h(\lambda) \to 1/\lambda$ as $h\lambda \to \infty$
(Newton limit); (ii) $\eta_h(\lambda) = h\,(1 - h\lambda/2 +
\mathcal{O}(h^2\lambda^2))$ as $\lambda \to 0$ (gradient limit);
(iii) for $\lambda < 0$, $\eta_h(\lambda) = (e^{h|\lambda|}-1)/
|\lambda|$ grows exponentially in $h$ (saddle amplification);
(iv) the doubling identity $H(G,2h) = H(G,h)\,(2\Ib - G\,H(G,h))$
holds, so a geometric horizon search costs one application per
candidate; (v) for every $h>0$ and $g\ne0$ the step $s=-H(G,h)\,g$
decreases the quadratic model, $m(s)<0$, whatever the signs of the
spectrum (equation~\eqref{eq:phi1decrease}). Under
Assumption~\ref{ass:lip} the decrease transfers to $f$ once the cubic
remainder $L_H\norm{s}^3/6$ is below the model decrease. For
$G\succeq0$ and $h\le0.1/\norm{G}$ that condition is
$L_H h^2\norm{g}/6<\varphi_1(-1/5)$,
so that the horizon bound alone does not cancel $L_H$
(Appendix~\ref{app:leancheck}).
\end{proposition}

The contrast with the CubicKrylov step of Section~\ref{sec:cubic} is
most noticeable on the flat modes, where no regularization is needed and the
cubic multiplier $1/(\lambda+M\norm{s}/2)$ (written
$1/(\lambda+\sigma\norm{s})$ with $M=2\sigma$) suppresses progress.
The $\varphi_1$ multiplier \eqref{eq:mult} decouples the three cases
per mode with no shift estimation and no secular solve. A drawback is
that the negative mode amplifier must be bounded explicitly
(Lemma~\ref{lem:clamp}) and that no analogue of Lemma~\ref{lem:secular}
is claimed. The search step magnitude is safeguarded only by the outer
acceptance test, whose rejections contract the horizon $h_b=c_h/\sigma_b$.

\subsection{Krylov evaluation, decrease law, and the negative mode
clamp}
\label{sec:phi1theory}
\label{sec:phi1krylov}

Proposition~\ref{prop:erlaw} defines the step through the eigenpairs
of $G$, which are not available on a large block, so the multiplier
\eqref{eq:mult} is applied on the CubicKrylov step's Lanczos basis,
and the two subsolvers differ only in the function of the Ritz values.

For a block with gradient $g_b$ and oracle $\Hb_b$, run $L$ steps of
the Lanczos process with full reorthogonalization from
$q_1=g_b/\norm{g_b}$, producing $\Qb_L\in\R^{n_b\times L}$ and
tridiagonal $\Tb_L=\Vb\,\mathrm{diag}(\nu)\,\Vb^\top$ with Ritz values
$\nu_i$, exactly as in
Section~\ref{sec:krylovtheory}. The search step is the subspace
$\varphi_1$ step
\begin{equation}
s_b \;=\; -\Qb_L\,\Vb\,\mathrm{diag}\!\big(
\tilde\eta_{h_b}(\nu_i)\big)\,\Vb^\top\,\Qb_L^\top g_b
\;=\; -\norm{g_b}\;\Qb_L \Vb\,\mathrm{diag}\!\big(
\tilde\eta_{h_b}(\nu_i)\big)\,\Vb^\top e_1,
\label{eq:substep}
\end{equation}
where $\tilde\eta$ is the clamped multiplier of Lemma~\ref{lem:clamp}.
A trial costs $L+3$ gradient equivalents (one gradient, $L$ Lanczos
HVPs, one HVP for the model decrease, one function evaluation), the
cost of the CubicKrylov step at equal $L$ (Table~\ref{tab:cost}), so
any difference in the results of Section~\ref{sec:experiments} is due
to the step rule. The Lanczos approximation converges superlinearly once $L$
exceeds $h\norm{\Hb_b}$ \citep{saad1992krylov}.
The truncation error is caught by the ratio test, which rejects a
trial whose measured decrease is less than $\eta_1$ (the acceptance threshold of Algorithm~\ref{alg:blockcn}) times the predicted
one and contracts the horizon.

Diagonalizing $G$ and summing the per-mode multipliers of
Proposition~\ref{prop:erlaw} gives an exact single-step decrease law on
the frozen quadratic $m(u)=\inner{g}{u}+\tfrac12\inner{u}{Gu}$:
\begin{equation}
m(0) - m(s_h) \;=\; h\,\inner{g}{\varphi_1(-2hG)\,g} \;>\; 0,
\label{eq:phi1decrease}
\end{equation}
strictly positive for every sign pattern of the spectrum. The hard case
of Lemma~\ref{lem:secular} has no analogue here.
It is bounded below by
the Cauchy type estimate
$m(0)-m(s_h)\ge\tfrac14\norm{g}^2\min(h,1/\norm{G})$, and an accepted
trial decreases $f$ by at least $\eta_1$ times the model decrease
$m(0)-m(s_h)$, up to the tolerance.
Under a uniform floor $\underline{h}>0$ on $\min(h_b,1/\norm{\Hb_b})$ over accepted
steps, and the transfer of the model decrease to $f$ of
Proposition~\ref{prop:erlaw}(v), ARC-$\varphi_1$ reaches an
$\eps$-first-order point within $4(f(x_0)-f^\star)/(\eta_1 \underline{h}\,\eps^2)$
accepted steps, the $\mathcal{O}(\eps^{-2})$ class of steepest descent
\citep{cartis2010steepest}. The
$\mathcal{O}(\eps^{-3/2})$ cubic rate does not transfer, and
second-order stationarity (Corollary~\ref{cor:sosp}) transfers only
generically because the clamped escape of Lemma~\ref{lem:clamp} is
proportional to $\inner{u_n}{g}$. In a Newton limit basin the
complexity improves to $\mathcal{O}(\eps^{-1})$ sweeps, with a
$Q$-linear (quotient-linear) loss contraction,
$f(x_{k+1})-f^\star\le \vartheta\,(f(x_k)-f^\star)$ with $\vartheta<1$
(Appendix~\ref{app:phi1rate}; \citealp[Appendix~A.2]{nocedal2006}).

\par\noindent\textbf{The negative mode clamp.}
\label{sec:clamp}
Lemma~\ref{lem:clamp} (Appendix~\ref{app:proofs}) clamps the
unbounded negative mode amplifier $(e^{h|\lambda|}-1)/|\lambda|$ so
that, for $M_{\mathrm{amp}}\ge2$, one step amplifies any gradient
component by at most $(M_{\mathrm{amp}}-1)/|\lambda|$ and the step is
finite in fp32 for any horizon and spectrum.

\par\noindent\textbf{Trust-controlled horizons.}
The composite scheme sets $h_b=c_h/\sigma_b$ and evaluates the
acceptance ratio against the subspace quadratic model (one extra HVP).
Rejection contracts $h_b$ geometrically, and by
Proposition~\ref{prop:erlaw}(v) a small enough horizon brings the cubic
remainder below the model decrease, so acceptance follows after
finitely many rejections (Algorithm~\ref{alg:arcphi1},
Appendix~\ref{app:phi1detail}).

\section{Experimental Results}
\label{sec:experiments}
\label{sec:fullvisir}

We evaluate ARC-block and ARC-$\varphi_1$ at $91.4$M parameters
(Section~\ref{sec:fullvisir91}), measure the landscape features that
decide where their use is justified (Section~\ref{sec:fingerprint}), and run FINER
to full convergence (Section~\ref{sec:finer}). All runs use PyTorch on
one RTX PRO 6000 GPU (96\,GB). The full scale model is
ViSIR~\citep{zeraatkar2025} ($91{,}383{,}042$ parameters), a vision
transformer encoder with SIREN feed-forward blocks and decoder
whose final layer holds 97\% of the model.
ViSIR-Nano (Appendix~\ref{sec:nano}, Table~\ref{tab:nano}) is the same
architecture at $15{,}235$ parameters. We report PSNR in dB, and
cost in HVPs, gradient equivalents (gevals, one per
gradient, backward pass, or function evaluation), and wall clock
seconds.

\subsection{The 91.4M parameter single image ViSIR Benchmark}
\label{sec:fullvisir91}

The main experiment extends the Chebyshev second kind ablation
(ESM image \#1, $60\times60\to240\times240$ image super-resolution task, 100 steps)
with Algorithm~\ref{alg:blockcn}, $L=10$, $p_b=n_b$ (lazy),
$L_H=10$ (Table~\ref{tab:fullvisir}, Appendix~\ref{app:expdetail}).
Scaling to the full model exposed the two failure modes of
Section~\ref{sec:blocktheory}. With the per-block corrections of
Section~\ref{sec:algorithm}, the CubicKrylov optimizer reaches \textbf{51.65\,dB} at $36{,}836$
gevals in 148\,s, $+9.90$\,dB over the best Chebyshev-ON
result and $+29.4$\,dB over damped Newton large block steps. With the
two per-block stabilizations it reaches
\textbf{64.4\,dB} best PSNR over 150 sweeps (\textbf{64.2\,dB} with
the Chebyshev three term recurrence, Appendix~\ref{sec:memmatched}).

\subsection{Landscape Fingerprints Along
the Adam Trajectory}
\label{sec:fingerprint}

Adam divides each gradient coordinate by $\sqrt{\hat v_t}+\epsilon$, so
the landscape it sees has the curvature of $D^{-1/2}\Hb D^{-1/2}$ with
$D=\mathrm{diag}(\sqrt{\hat v_t}+\epsilon)$.
At log-spaced checkpoints of an Adam run we measure matrix-free the
raw and Adam preconditioned conditionings $\kappa_{\mathrm{raw}}$ and
$\kappa_{\mathrm{Adam}}$, the diagonal mass
$\rho=\|\mathrm{diag}\,\Hb\|^2/\|\Hb\|_F^2$, the negative spectral mass
\citep{ghorbani2019hessian,dauphin2014saddle}, and the gradient energy
split over high curvature, flat, and negative curvature classes
\citep{gurari2018subspace}. When $\rho$ is small, high curvature and flat
directions are mixtures of coordinates, and Jacobi scaling barely
changes $\kappa$. Since for a positive definite matrix no diagonal
scaling improves on Jacobi scaling by more than a factor $n$
\citep{vandersluis1969}, the per-coordinate class, Adam included,
faces the raw ill-conditioning, and second-order methods are more
effective there.

Table~\ref{tab:fingerprint} (Appendix~\ref{app:expdetail})
distinguishes three cases. \emph{(i) Axis-aligned} (35k SIREN-like,
SIREN 2D and signed distance function (SDF), WIRE 2D). The rescaling leaves a
small $\kappa_{\mathrm{Adam}}$, the gradient never enters flat or
negative directions, and tuned Adam reaches parity with second-order methods (Appendix~\ref{app:adam35k}).
\emph{(ii) Coupled.} Multisaddle and
Chebyshev--Rosenbrock~\citep{jarre2011chebyrosen} retain
$\kappa_{\mathrm{Adam}}$ up to $1.5\times10^{6}$ at $\rho\le0.04$, the
landscapes where the first-order methods evaluated here stall at any
tuning (ViSIR is in this class).
\emph{(iii) Saddle dominated.} FINER (2D and SDF)
has small $\kappa_{\mathrm{Adam}}$ (61 and 6.9), so conditioning does
\emph{not} limit Adam, but its negative spectral mass is $0.81$ at
initialization ($0.57$--$0.89$ over the 20 Earth System Model
(ESM) images).
Adam converges to a saddle adjacent point~\citep{sagun2018hessian} at
65\,dB with a low frequency dominated residual, while the
cubic shift keeps every shifted subproblem positive semidefinite and
ARC-block continues to 120--129\,dB with a band uniform residual. The same loss landscape fingerprint appears in 3D
(SDF-FINER, negative mass $0.68$, Adam chamfer $4.6$--$5.5$ times
worse) and on small hard landscapes (Appendix~\ref{sec:small}).
\emph{(iv) Coupled, with a piecewise quadratic penalty and a short
horizon.} Audio declipping with 1D FINER and SIREN networks
(Appendix~\ref{app:declip}) has the coupled fingerprint,
$\kappa_{\mathrm{Adam}}/\kappa_{\mathrm{raw}}=0.3$--$0.9$ at
$\rho\le10^{-3}$, and yet tuned Adam reaches the task accuracy at
$1/46$ of the gradient equivalents ARC-block needs to tie it, and
ARC-$\varphi_1$ remains below the identity reconstruction. The
fingerprint is therefore a necessary condition for a second-order
advantage and the appendix measures the two features that remove it.
The clipping penalty changes the Hessian by a rank one term at every
accepted step, which the cubic model cannot predict (24--30\,\% of
the trials are rejected), and the accuracy the task requires is
reached inside Adam's stable phase, before the slow phase in which the
cubic step gains.

\subsection{FINER 2D Image Fitting with the $\varphi_1$ Step at Full Convergence}
\label{sec:finer}

FINER~\citep{liu2024finer}, the saddle dominated case of
Section~\ref{sec:fingerprint}, is a coordinate multilayer perceptron (MLP) with the variable
periodic activation $\sin((|x|+1)x)$ in place of SIREN's fixed period
$\sin(\omega_0 x)$. Three hidden layers of width 256
(${\sim}199$k parameters) map pixel coordinates to
RGB on the 20 $240\times240$ ESM images, the hardest initialization of
our benchmarks.

\par\noindent\textbf{Run to full convergence.}
\label{sec:finerconv}
We run the sweep winner ($c_h=1.0$, $\sigma_0=0.3$, $L_{\max}=163$,
Table~\ref{tab:finerband}) and tuned
Adam on image~0 to their plateaus. Adam peaks at
$\mathbf{78.2}$\,dB at $2{,}629$\,s and then degrades. ARC-$\varphi_1$ reaches
$\mathbf{133.5}$\,dB at $17{,}173$\,s. On all 20 ESM images, run to a plateau
(Appendix~\ref{sec:esm20}), ARC-$\varphi_1$'s
per-image best has median $132.4$ and mean $131.5\pm3.3$\,dB against
Adam's $72.8\pm4.5$\,dB and AdaCubic's $46.7\pm0.3$\,dB, a lead over
Adam of $42.1$--$66.4$\,dB on every image.

\begin{figure}[h]
\centering
\begin{minipage}[c]{0.54\textwidth}
\includegraphics[width=\linewidth]{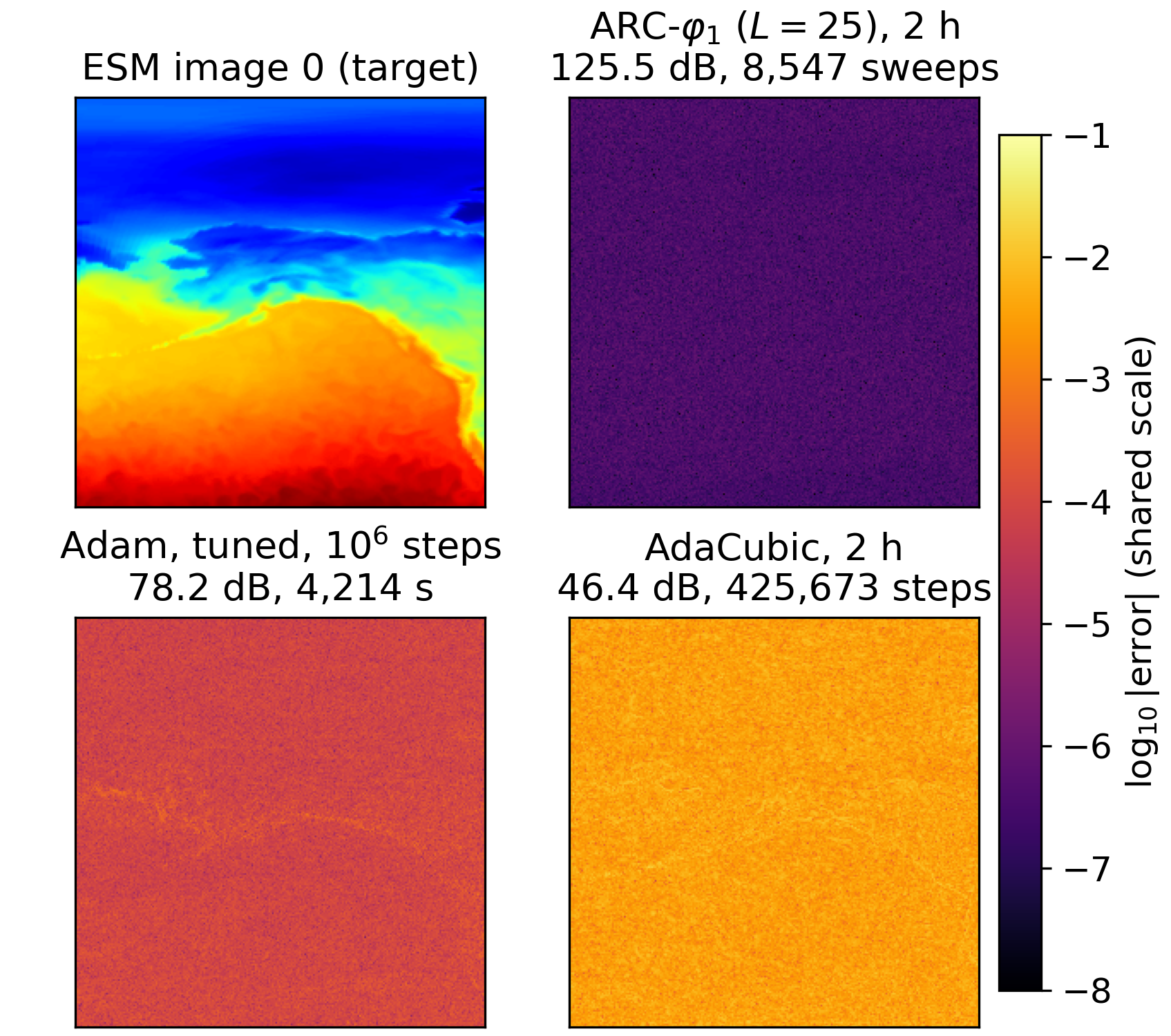}
\end{minipage}\hfill
\begin{minipage}[c]{0.45\textwidth}
\caption{FINER on ESM image~0 in equal wall clock time of $2$~h,
$\log_{10}$ absolute error on a shared scale. ARC-$\varphi_1$ at
Lanczos depth $L{=}25$ reaches $125.5$\,dB in $8{,}547$ sweeps. Tuned
Adam converges to $78.2$\,dB ($10^6$ steps, $4{,}214$\,s). AdaCubic
\citep{tsingalis2026adacubic} with its authors' universal
hyperparameters reaches $46.4$\,dB at $6{,}328$\,s and then oscillates
(Appendix~\ref{sec:adacubic}). SOAP \citep{vyas2024soap} converges
at $77.65$\,dB and L-BFGS stays at $7.3$\,dB at the
initialization saddle (Table~\ref{tab:isowall}). At its tuned depth
$L{=}163$, ARC-$\varphi_1$ converges to $133.5$\,dB ($2{,}824$ sweeps,
$17{,}173$\,s). The faint bright filament along the warm/cold front
boundary in the Adam panel is signal dependent error.}
\label{fig:finerconv}
\end{minipage}
\end{figure}

\par\noindent\textbf{Iso-time comparison.}
Table~\ref{tab:isowall} and Figure~\ref{fig:finerconv} provide an equal wall clock time comparison of the optimizers' results on Earth System Model (ESM) image~0.
ARC-$\varphi_1$ passes Adam's $78.2$\,dB at $1{,}927$\,s, shows
$84.2$\,dB at Adam's peak time ($2{,}629$\,s) and $\mathbf{95.6}$\,dB
at $4{,}207$\,s (the time it takes Adam to perform $10^6$ steps in this experiment on this hardware). AdaCubic
\citep{tsingalis2026adacubic}, whose authors call it ``the first
optimizer to leverage cubic regularization in scalable deep learning
applications'', reaches $46.4$\,dB in $2$~h, $32$\,dB below Adam.

\par\noindent\textbf{Practical consequences.}
PSNR converts to bit depth at $6.02$\,dB per bit.
Converged Adam at $73$--$78$\,dB ($\approx$12--13 bits) discards
$3$--$4$ bits of a $16$-bit scientific raster as a \emph{spatially
coherent} bias field (in Figure~\ref{fig:finerconv}, the faint bright
filament along the warm/cold front boundary in the Adam panel is
signal dependent error). It is harmful for downstream climate statistics
because that bias field does not average out over regions and shifts
regional means, anomalies, and gradients.
The precision of a $16$-bit discrete representation corresponds to
$16\times 6.02 = 96.3$\,dB in the continuous representation.
The PSNR of the ARC-$\varphi_1$ reconstruction is greater than
$96.3$\,dB ($131$--$134$\,dB, $\approx$22 bits), so the resulting
continuous representation can replace the raster without measurable
error in resampling onto other grids, analytic differentiation for
gradient and Laplacian fields, and pointwise queries off the native
grid.

\section{Conclusion and Limitations}
\label{sec:limitations}

We introduced ARC-block, a blockwise stabilized ARC optimizer with an
independent cubic constant and acceptance test per tensor and an
interchangeable subsolver. Its $\varphi_1$ subsolver, at the
CubicKrylov step's per-trial cost, reaches $133.5$\,dB against tuned
Adam's $78.2$\,dB at full convergence. We proved approximate
second-order stationarity (Remark~\ref{rem:global}), an
$\mathcal{O}(\sqrt{B}\eps^{-3/2})$ sweep bound above the cross-block
coupling floor (Theorem~\ref{thm:blockrate}), and the Cauchy type
$\mathcal{O}(\eps^{-2})$ guarantee of the $\varphi_1$ rule
(Section~\ref{sec:phi1theory}). On the landscapes with the features
where Adam and its variants struggle (cross-coupled ill-conditioning
and persistent negative spectral mass), ARC-block outperforms every
optimizer evaluated here.
A major limitation is that the cross-block curvature is lost in the
blockwise scheme. Another limitation is that 
the optimizers' performance evaluation on the 20 ESM images was done using 
one seed so far. 
There is also room for tighter theoretical guarantees.
The fingerprint identifies where a second-order advantage is possible,
and Appendix~\ref{app:declip} shows a coupled landscape on which it
does not materialize. A piecewise quadratic penalty whose active set
changes at every step defeats the cubic model, and a task whose
accuracy target is met inside Adam's stable phase leaves the cubic
step no accuracy to gain, so Adam is the better method at $1/46$ of
the gradient equivalents.

Directions for future work are as follows: 
(i) achieve the full-space $\mathcal{O}(\eps^{-3/2})$ rate when the
cross-block coupling vanishes, that is, when the Hessian has no
curvature between tensor blocks. A possible way is per-block
Hessian--vector probes that measure the block-diagonal mass $\beta$,
and from it the coupling bound $\delta$ and the floor $\eps_\star$
the sweep rate is stated against (Remark~\ref{rem:krylovso}), so the
rate can be claimed on the runs where that bound is zero;
(ii) predict the first-order loss curve. A possible way is a per-block
fingerprint from the Gauss--Newton decomposition of the Hessian, whose
kernel spectrum, the neural tangent kernel, provides the curve and whose
residual term holds all negative curvature
(Appendix~\ref{app:ntk}; \citealp{jacot2018ntk}; for Newton's method
\citealp{riedl2026newtonntk});
and (iii) other step rules and larger networks in other domains.

\section*{Reproducibility Statement}
Appendix~\ref{app:proofs} restates every theorem with a full proof and
Appendix~\ref{app:leancheck} lists the machine-checked cores and the
hypotheses outside them. Those hypotheses include community accepted published results, 
assumptions on the objective, and the mechanics of the optimization algorithms.
The assumptions on the objective are reasonable in this case because
the networks use smooth activations, so a Lipschitz Hessian on the
region of the iterates, and a local basin near a minimizer, are reasonable assumptions for the cubic rate and the fast local rate.
The assumptions tied to the algorithm's mechanics are based on the computations done in a single parameter block update.
That is, these assumptions reflect the fact that the gradient of a large block is refreshed after the preceding blocks have been processed, and a rejected trial doubles the cubic constant. They also reflect the fact that the Cartis–Gould–Toint termination test is recorded on every CubicKrylov step and once a step meets the test then the stated joint sweep rate is ensured (i.e. ($\mathcal{O}(\sqrt{B}\varepsilon^{-3/2})$) sweeps to an approximate stationary point).
It is important to note that this rate is a worst-case upper bound. During the performed
experiments a faster rate was observed (e.g. $-0.9$ to $-1.7$, steeper than the $-2/3 \approx -0.67$ guaranteed by the bound
(Table~\ref{tab:certified})).

Appendix~\ref{sec:esm20} describes the 20 ESM
image experiment (architecture, seeds, stopping rules, per-image runs
and wall clock), Appendix~\ref{app:expdetail} describes the remaining
experimental details and iso-time tables, and
Appendix~\ref{sec:coststructure} describes the computational cost
evaluation in gradient equivalents. 

\FloatBarrier

\clearpage
\appendix
\begin{center}
{\LARGE\bfseries Appendix}\\[0.4em]
{\large Supplementary Analyses and Experimental Evaluations}
\end{center}
\vspace{0.8em}

\section{Deferred Proofs}
\label{app:proofs}

Statements summarized in prose in Sections~\ref{sec:theory}
and~\ref{sec:phi1} are restated here in full, each followed by its
proof. Algebraic cores of the \emph{introduced} blockwise arguments
(Proposition~\ref{prop:perblock}, Theorems~\ref{thm:phaseA}--\ref{thm:blockrate},
Proposition~\ref{prop:erlaw}, and the Cauchy estimate of
Section~\ref{sec:phi1theory}, not the restated full space cubic lemmas)
were checked by the Lean~4 kernel as described in Appendix~\ref{app:leancheck}.

\begin{lemma}[Secular equation]
\label{lem:secular}
Define, for $\lambda>\max\{0,-\lambda_n\}$,
\[
  s(\lambda):=-(\Hb+\lambda\Ib)^{-1}g, \qquad
  \varphi(\lambda):=\norm{s(\lambda)}
  =\Bigl(\sum_{i=1}^n \frac{c_i^2}{(\lambda_i+\lambda)^2}\Bigr)^{1/2},
  \quad c_i:=\inner{u_i}{g}.
\]
Then $\varphi$ is continuous, strictly decreasing, and
$\varphi(\lambda)\to0$ as $\lambda\to\infty$. If $g$ has a nonzero
component in the eigenspace of $\lambda_n$ (for a simple $\lambda_n$,
$c_n\neq0$) or $\lambda_n\ge0$, the equation
\begin{equation}
  \lambda=\tfrac{M}{2}\,\varphi(\lambda)
\label{eq:secular}
\end{equation}
has a unique solution $\lambda^\star>\max\{0,-\lambda_n\}$, and
$s^\star=s(\lambda^\star)$ with $r^\star=\varphi(\lambda^\star)$ is the
global minimizer of the cubic model Eq.~\eqref{eq:cubicmodel}.
\end{lemma}
\begin{proof}
Each term $c_i^2/(\lambda_i+\lambda)^2$ with $c_i\ne0$ is positive and
strictly decreasing in $\lambda$ on the admissible domain, and at
least one such term exists since $g\ne0$. Hence $\varphi$ is strictly
decreasing, and continuity and the limit are immediate. Consider
$\psi(\lambda):=\tfrac{M}{2}\varphi(\lambda)-\lambda$, which is strictly
decreasing with $\psi(\lambda)\to-\infty$. If $\lambda_n<0$ and
$g$ has a nonzero component in its eigenspace then $\varphi(\lambda)\to\infty$ as
$\lambda\downarrow-\lambda_n$, so $\psi\to+\infty$ and a unique root
exists by the intermediate value theorem. If $\lambda_n\ge0$, then at
$\lambda\downarrow0$ either $\varphi(0^+)>0$ (so $\psi(0^+)>0$, giving a
root) or $g=0$, excluded by assumption. Global optimality of the
resulting pair $(s^\star,\lambda^\star)$ for Eq.~\eqref{eq:cubicmodel} is
the characterization in Section~5 in~\citet{nesterov2006cubic}, that
any $s$ with $g+\Hb s+\tfrac{M}{2}\norm{s}s=0$ and
$\Hb+\tfrac{M}{2}\norm{s}\Ib\succeq0$ globally minimizes $m$. Our
$\lambda^\star>-\lambda_n$ ensures the second condition with strict
inequality.
\end{proof}

\begin{lemma}[Positive definiteness before the polynomial]
\label{lem:pd}
The solution of Eq.~\eqref{eq:secular} satisfies
\begin{equation}
  r^\star \;\ge\; \frac{2\max\{0,-\lambda_n\}}{M},
\label{eq:rlo}
\end{equation}
and consequently, the shifted operator
$\Hb+\tfrac{Mr^\star}{2}\Ib$ is positive semidefinite. It is positive
definite whenever $\lambda^\star>-\lambda_n$, which holds in the
non-degenerate case of Lemma~\ref{lem:secular}.
\end{lemma}
\begin{proof}
By Lemma~\ref{lem:secular}, $\lambda^\star=\tfrac{M}{2}r^\star$ lies in
the admissible domain $\lambda^\star>\max\{0,-\lambda_n\}$, which is
Eq.~\eqref{eq:rlo} after rearranging. The smallest eigenvalue of the
shifted operator is $\lambda_n+\lambda^\star> 0$.
\end{proof}

\begin{proposition}[Degree bound]
\label{prop:degree}
Let $\lambda^\star=\tfrac{M}{2}r^\star$ solve Eq.~\eqref{eq:secular}. Then
\begin{equation}
  \lambda^\star \;\ge\; \lambda_{\mathrm{lo}}
  := \frac{\sqrt{\lambda_1^2+2M\norm{g}}-\lambda_1}{2},
\label{eq:shiftlo}
\end{equation}
and, whenever the gradient driven shift also dominates the negative
curvature, in the sense that
\begin{equation}
  \lambda_{\mathrm{lo}} \;\ge\; 2\max\{0,-\lambda_n\},
\label{eq:shiftdom}
\end{equation}
the effective condition number of the shifted operator obeys
\begin{equation}
  \kappa_{\mathrm{eff}}
  := \frac{\lambda_1+\lambda^\star}{\lambda_n+\lambda^\star}
  \;\le\;
  \frac{2\bigl(\lambda_1+\lambda_{\mathrm{lo}}\bigr)}
       {\max\{\lambda_n,0\}+\lambda_{\mathrm{lo}}}
  \qquad(\text{the factor $2$ is not needed when }\lambda_n\ge0).
\label{eq:kappaeff}
\end{equation}
In particular, a Chebyshev (or Lanczos) solver needs a degree
$L=\mathcal{O}\bigl(\sqrt{\kappa_{\mathrm{eff}}}\,
\ln\tfrac1{\eps_{\mathrm{rel}}}\bigr)$ for a relative-accuracy
$\eps_{\mathrm{rel}}$ solve, and:
\begin{enumerate}[leftmargin=2em,itemsep=1pt]
\item (far from stationarity) if $2M\norm{g}\ge\lambda_1^2$ and
      Eq.~\eqref{eq:shiftdom} holds, for graded non-convex $f$ the
      latter is implied by
      $\sqrt{M\norm{g}/2}\ge\tfrac{2\sigma_\tau}{\sqrt2-1}$, since
      $-\lambda_n\le\sigma_\tau$ (Remark~\ref{rem:graded}), then
      $\lambda_{\mathrm{lo}}\ge
      (\sqrt2-1)\,\sqrt{M\norm{g}/2}$ and
      $\kappa_{\mathrm{eff}}
      \le 2+2\lambda_1\big/\lambda_{\mathrm{lo}}
      = 2+\mathcal{O}\bigl(\lambda_1/\sqrt{M\norm{g}}\bigr)$,
      so the required degree is $\mathcal{O}(1)$ independent of
      $\kappa=\lambda_1/\lambda_n$;
\item (near stationarity) if $\lambda_n>0$ (a neighborhood of a
      strict second-order stationary point), then as $\norm{g}\to0$
      the shift vanishes and
      $\kappa_{\mathrm{eff}}\to\lambda_1/\lambda_n$, recovering the
      undamped Chebyshev--Newton case (the pure Newton step computed by
      a Chebyshev inner solve) with the Chebyshev polynomial decay bound
      $L\approx1.28\sqrt{\kappa}$ and its fast local convergence rate
      (residual contraction ${\le}0.22$ per application of the
      degree-$L$ polynomial, Remark~\ref{rem:cheb2bound}).
\end{enumerate}
\end{proposition}
\begin{proof}
From Eq.~\eqref{eq:shifted},
$\norm{g}=\norm{(\Hb+\lambda^\star\Ib)s^\star}
\le(\lambda_1+\lambda^\star)\,r^\star
=(\lambda_1+\lambda^\star)\cdot\tfrac{2\lambda^\star}{M}$.
Rearranging gives the quadratic inequality
$2(\lambda^\star)^2+2\lambda_1\lambda^\star-M\norm{g}\ge0$, whose
positive root yields
$\lambda^\star\ge
\tfrac{-\lambda_1+\sqrt{\lambda_1^2+2M\norm{g}}}{2}
=\lambda_{\mathrm{lo}}$, which is Eq.~\eqref{eq:shiftlo}. This
estimation argument follows Section~5 in \citet{nesterov2006cubic},
where it bounds the step norm. Here we use it to bound the shift. The map
$t\mapsto(\lambda_1+t)/(\lambda_n+t)$ is non-increasing for
$t>-\lambda_n$ (its derivative has the sign of
$\lambda_n-\lambda_1\le0$), and $\lambda^\star\ge\lambda_{\mathrm{lo}}$
with $\lambda^\star>-\lambda_n$ (Lemma~\ref{lem:pd}), so
\[
  \kappa_{\mathrm{eff}}
  \;\le\;\frac{\lambda_1+\lambda_{\mathrm{lo}}}
              {\lambda_n+\lambda_{\mathrm{lo}}}
  \qquad\text{whenever }\lambda_{\mathrm{lo}}>-\lambda_n .
\]
If $\lambda_n\ge0$ the denominator equals
$\max\{\lambda_n,0\}+\lambda_{\mathrm{lo}}$, which is
Eq.~\eqref{eq:kappaeff} without the factor $2$. If $\lambda_n<0$,
Eq.~\eqref{eq:shiftdom} gives $-\lambda_n\le\lambda_{\mathrm{lo}}/2$,
hence $\lambda_n+\lambda_{\mathrm{lo}}\ge\lambda_{\mathrm{lo}}/2
=\tfrac12\bigl(\max\{\lambda_n,0\}+\lambda_{\mathrm{lo}}\bigr)$, which
is Eq.~\eqref{eq:kappaeff}. Some domination condition of this kind is
necessary. With $\lambda_n<0$ fixed and $c_n\to0$, the secular
solution approaches $-\lambda_n$ from above and
$\kappa_{\mathrm{eff}}\to\infty$, the near-hard case, in which
Krylov subsolvers fall back to the shift-independent sublinear case
analyzed by~\citet{carmon2018krylov} and thus the augmentation of
Remark~\ref{rem:hardcase} applies. For
item~1, write $t:=2M\norm{g}$. The assumption $t\ge\lambda_1^2$ gives
\begin{align*}
  \lambda_{\mathrm{lo}}
  &=\frac{\sqrt{\lambda_1^2+t}-\lambda_1}{2}
  =\frac{t}{2\bigl(\sqrt{\lambda_1^2+t}+\lambda_1\bigr)}
  \;\ge\;\frac{t}{2\bigl(\sqrt{2t}+\sqrt{t}\bigr)}
  =\frac{\sqrt{t}}{2(\sqrt2+1)}\\
  &=\frac{\sqrt2-1}{2}\,\sqrt{2M\norm{g}}
  \;=\;(\sqrt2-1)\sqrt{M\norm{g}/2},
\end{align*}
so $\lambda_{\mathrm{lo}}=\Theta(\sqrt{M\norm{g}})$. The sufficient
condition quoted in item~1 follows because $-\lambda_n\le\sigma_\tau$
and $\lambda_{\mathrm{lo}}\ge(\sqrt2-1)\sqrt{M\norm{g}/2}\ge2\sigma_\tau$
together imply Eq.~\eqref{eq:shiftdom}. Dropping the
non-negative term $\max\{\lambda_n,0\}$ in Eq.~\eqref{eq:kappaeff} gives
$\kappa_{\mathrm{eff}}\le2+2\lambda_1/\lambda_{\mathrm{lo}}$. The
degree bound for Chebyshev iteration on a positive definite system
follows from the extremal property of Chebyshev polynomials. Among
degree-$L$ polynomials with $p(0)=1$, the shifted Chebyshev polynomial
minimizes the maximal residual on $[\lambda_{\min},\lambda_{\max}]$,
giving
$\|e_L\|\le2\bigl(\tfrac{\sqrt{\kappa}-1}{\sqrt{\kappa}+1}\bigr)^{L}\|e_0\|$,
so degree $L\ge\tfrac{\sqrt{\kappa_{\mathrm{eff}}}+1}{2}\ln\tfrac{2}{\eps}$
reduces the shifted residual by the factor
$\eps$~\citep{saad2003}. Item~2 concerns the local phase near a
strict second-order stationary point, where $\lambda_n>0$. Since
$\varphi(\lambda^\star)=\norm{(\Hb+\lambda^\star\Ib)^{-1}g}
\le\norm{g}/\lambda^\star$, the secular equation Eq.~\eqref{eq:secular}
gives $(\lambda^\star)^2\le\tfrac{M}{2}\norm{g}$, so
$\lambda^\star\to0$ as $g\to0$, while
$\lambda_{\mathrm{lo}}\to0$ by Eq.~\eqref{eq:shiftlo}. Hence
$\kappa_{\mathrm{eff}}\to\lambda_1/\lambda_n$. If $\lambda_n<0$
persists as $g\to0$, the shift remains bounded below by $-\lambda_n$
and the method does not reduce to the undamped Chebyshev--Newton case
of item~2. 

This is the intended behavior. The shift then keeps the
shifted operator positive definite in the non-degenerate case.
Lemma~\ref{lem:pd} gives positive semidefiniteness unconditionally.
Strict definiteness can fail only in the hard case of
Lemma~\ref{lem:secular}, where $\lambda_n<0$ and the gradient has no
component along the bottom eigenvector, $c_n=\inner{u_n}{g}=0$, so
that $\lambda^\star$ may equal $-\lambda_n$ and the shifted operator
becomes singular exactly on the eigenspace of $\lambda_n$.

\end{proof}

\begin{lemma}[Exact subspace reduction]
\label{lem:krylov}
For any $y\in\R^{L}$ and $s=\Qb_L y$,
\begin{equation}
  m(\Qb_L y)\;=\;\hat m(y):=\norm{g}\,e_1^\top y
  +\tfrac12\,y^\top\Tb_L\,y+\tfrac{M}{6}\norm{y}^3 .
\label{eq:reduced}
\end{equation}
Consequently, if $y^\star=\argmin_y \hat m(y)$ (computable by the
secular solver of Lemma~\ref{lem:secular} on the
$L\times L$ eigendecomposition of $\Tb_L$), then
$s_L:=\Qb_L y^\star$ is the \emph{exact} minimizer of the cubic model
$m$ over the Krylov subspace $\mathcal{K}_L(\Hb,g)$:
\[
  m(s_L)\;=\;\min_{s\in\mathcal{K}_L(\Hb,g)} m(s).
\]
\end{lemma}
\begin{proof}
Using $\Qb_L^\top\Qb_L=\Ib$ we have
$\inner{g}{\Qb_Ly}=(\Qb_L^\top g)^\top y=\norm{g}e_1^\top y$;
$\inner{\Hb\Qb_Ly}{\Qb_Ly}=y^\top(\Qb_L^\top\Hb\Qb_L)y=y^\top\Tb_Ly$;
and $\norm{\Qb_Ly}=\norm{y}$. Substituting into Eq.~\eqref{eq:cubicmodel}
gives Eq.~\eqref{eq:reduced} exactly (no truncation error appears because
all three terms involve $\Hb$ at most once, and the Lanczos relation
$\Qb_L^\top\Hb\Qb_L=\Tb_L$ is exact regardless of the residual term in
the three term recurrence). Since every $s\in\mathcal{K}_L$ is
$\Qb_Ly$ for a unique $y$, minimizing $\hat m$ over $y$ minimizes $m$
over $\mathcal{K}_L$.
\end{proof}
This is the reduction of \citet{gould1999gltr} for the trust region
subproblem, extended to the cubic model by \citet{cartis2011a}. It is
restated here because Corollary~\ref{cor:dominate} and
Proposition~\ref{prop:perblock} use it per block.

\begin{corollary}[CubicKrylov step dominates the fixed coefficient $U_n$ step]
\label{cor:dominate}
Let $d_L$ be any step produced by a polynomial of degree at most
$L-1$ in $\Hb$ applied to $g$, in particular the depth-$L$ Chebyshev
second kind recurrence Eq.~\eqref{eq:dsorec}, whose polynomial has
degree $L-2$.
Then $d_L\in\mathcal{K}_L(\Hb,g)$ and
\[
  m(s_L)\;\le\; m(d_L).
\]
\end{corollary}
\begin{proof}
A polynomial step of degree at most $L-1$ is by definition
$q(\Hb)g\in\mathcal{K}_L(\Hb,g)=\mathrm{span}\{g,\dots,\Hb^{L-1}g\}$,
and $s_L$ minimizes $m$ over that set by Lemma~\ref{lem:krylov}. (The
shifted operator $\Hb+\lambda\Ib$ generates the same Krylov space as
$\Hb$.)
\end{proof}

\begin{remark}[Lanczos as Gauss quadrature]
\label{rem:gaussquad}
Lanczos on $(\Hb_b,g_b)$ is a three term implementation of the same
Krylov sequence the Chebyshev $U_n$ recurrence generates, with
coefficients computed from $(\Hb_b,g_b)$ rather than set in advance
for a spectral interval. The Ritz values of $\Tb_L$ are the Gauss quadrature nodes
for the $g_b$-weighted spectral measure
$\mathrm d\mu(\lambda)=\sum_i(u_i^\top g_b)^2\delta(\lambda-\lambda_i)$
\citep{golub2013}, so they integrate polynomials of degree less than
$2L$ against $\mathrm d\mu$ exactly. The secular solve on $\Tb_L$
then takes the exact cubic minimizer in that space
(Lemma~\ref{lem:krylov}). The $U_n$ recurrence is the minimax residual
polynomial for an assumed interval (Remark~\ref{rem:cheb2bound}).
Lanczos is the sequence of orthogonal polynomials for this block's
measured $\mathrm d\mu$. Modes orthogonal to $g_b$ have mass zero in
$\mathrm d\mu$ and are not in the subspace (Remark~\ref{rem:hardcase}).
This is optimality over $\mathcal{K}_L(\Hb_b,g_b)$, not among
$L$-dimensional subspaces, and it is a statement about the block
model, not a trajectory of $f$.
\end{remark}

Two auxiliary lemmas support Theorem~\ref{thm:rate}. Both are the
one-step estimates of \citet{nesterov2006cubic}, Section~4, written in
our notation.

\begin{lemma}[Model decrease]
\label{lem:decrease}
Let $s^\star$ be the global minimizer of Eq.~\eqref{eq:cubicmodel} with
$r^\star=\norm{s^\star}$. Then
\begin{equation}
  m(s^\star)\;\le\;-\tfrac{M}{12}\,(r^\star)^3 .
\label{eq:modeldec}
\end{equation}
If in addition $M\ge L_H$, then
$f(x+s^\star)\le f(x)-\tfrac{M}{12}(r^\star)^3$.
\end{lemma}

\begin{proof}
Stationarity Eq.~\eqref{eq:shifted} gives
$\inner{g}{s^\star}=-\inner{\Hb s^\star}{s^\star}
-\tfrac{M}{2}(r^\star)^3$. Substituting,
\[
  m(s^\star)
  = -\tfrac12\inner{\Hb s^\star}{s^\star}
    -\tfrac{M}{2}(r^\star)^3+\tfrac{M}{6}(r^\star)^3
  = -\tfrac12\inner{\Hb s^\star}{s^\star}-\tfrac{M}{3}(r^\star)^3 .
\]
The second-order condition $\Hb+\tfrac{Mr^\star}{2}\Ib\succeq0$ implies
$\inner{\Hb s^\star}{s^\star}\ge-\tfrac{M}{2}(r^\star)^3$, hence
$m(s^\star)\le\tfrac{M}{4}(r^\star)^3-\tfrac{M}{3}(r^\star)^3
=-\tfrac{M}{12}(r^\star)^3$. For the second claim, by
Eq.~\eqref{eq:cubicbound} and $M\ge L_H$,
$f(x+s^\star)\le f(x)+m(s^\star)
+\tfrac{L_H-M}{6}(r^\star)^3\le f(x)+m(s^\star)$.
\end{proof}

\begin{lemma}[New gradient bound]
\label{lem:newgrad}
Under Assumption~\ref{ass:lip}, the point $x^+=x+s^\star$ satisfies
\begin{equation}
  \norm{\nabla f(x^+)}\;\le\;\frac{L_H+M}{2}\,(r^\star)^2 .
\label{eq:newgrad}
\end{equation}
\end{lemma}

\begin{proof}
Using $g+\Hb s^\star=-\tfrac{Mr^\star}{2}s^\star$ from
Eq.~\eqref{eq:shifted},
\[
  \nabla f(x^+)
  =\underbrace{\nabla f(x^+)-g-\Hb s^\star}_{\text{Taylor remainder}}
   \;-\;\tfrac{Mr^\star}{2}\,s^\star,
\]
and Assumption~\ref{ass:lip} bounds the remainder by
$\tfrac{L_H}{2}(r^\star)^2$ while the second term has norm
$\tfrac{M}{2}(r^\star)^2$.
\end{proof}

\begin{proof}[Proof of Theorem~\ref{thm:rate}]
The argument is the telescoping count of \citet{nesterov2006cubic},
Section~4.
Fix $K$ and suppose $\norm{\nabla f(x_i)}>\eps$ for all $1\le i\le K$.
By Lemma~\ref{lem:newgrad},
$r_i:=\norm{s_i^\star}\ge
\bigl(\tfrac{2\norm{\nabla f(x_{i+1})}}{L_H+M}\bigr)^{1/2}
>\bigl(\tfrac{2\eps}{L_H+M}\bigr)^{1/2}$ for $0\le i\le K-1$.
By Lemma~\ref{lem:decrease},
\[
  f(x_i)-f(x_{i+1})\;\ge\;\tfrac{M}{12}\,r_i^3
  \;>\;\tfrac{M}{12}\Bigl(\tfrac{2\eps}{L_H+M}\Bigr)^{3/2}.
\]
Telescoping over $i=0,\dots,K-1$ (the sequence $\{f(x_i)\}$ is
monotone by Lemma~\ref{lem:decrease}):
\[
  f(x_0)-f^\star\;\ge\;
  K\cdot\tfrac{M}{12}\Bigl(\tfrac{2\eps}{L_H+M}\Bigr)^{3/2},
\]
so $K<\tfrac{12(f(x_0)-f^\star)}{M}
\bigl(\tfrac{L_H+M}{2}\bigr)^{3/2}\eps^{-3/2}$. Taking the ceiling
gives the stated $K$.
\end{proof}

\begin{corollary}[Approximate second-order stationarity]
\label{cor:sosp}
Under the assumptions of Theorem~\ref{thm:rate}, for any $\eps>0$ there
is an iterate $x_i$, $1\le i\le K$, with $K$ exactly as in
Theorem~\ref{thm:rate}, satisfying \emph{both} (the second-order
statement of \citealp{nesterov2006cubic}, Section~4, and of
\citealp{cartis2011b}, in our constants)
\[
  \norm{\nabla f(x_i)}\le\eps
  \qquad\text{and}\qquad
  \lambda_n\bigl(\nabla^2 f(x_i)\bigr)\;\ge\;
  -\,\frac{M+2L_H}{\sqrt{2(L_H+M)}}\,\sqrt{\eps}\,,
\]
where $\lambda_n$ denotes the smallest eigenvalue. With the canonical
choice $M=2L_H$ the curvature threshold is
$-\tfrac{4}{\sqrt{6}}\sqrt{L_H\,\eps}$. 

The method therefore converges to approximate \emph{second-order}
stationary points. At a strict saddle the smallest Hessian eigenvalue
stays below some fixed $-c<0$ that does not shrink as $\eps\to0$, so
for $\eps$ small enough the curvature condition above rules the saddle
out as a stopping point. The method keeps moving and escapes the saddle instead
of stopping just because the gradient norm is small.

\end{corollary}
\begin{proof}
Set $\rho:=\bigl(\tfrac{2\eps}{L_H+M}\bigr)^{1/2}$ and suppose the
conclusion fails for every $1\le i\le K$, that is, either
$\norm{\nabla f(x_i)}>\eps$ or
$\lambda_n(\nabla^2 f(x_i))<-(\tfrac{M}{2}+L_H)\rho$, noting that
$(\tfrac{M}{2}+L_H)\rho=\tfrac{M+2L_H}{\sqrt{2(L_H+M)}}\sqrt{\eps}$. In
the first case Lemma~\ref{lem:newgrad} gives
$r_{i-1}\ge\bigl(\tfrac{2\norm{\nabla f(x_i)}}{L_H+M}\bigr)^{1/2}>\rho$.
In the second case, the second-order condition of the cubic subproblem
at $x_{i-1}$, $\nabla^2 f(x_{i-1})+\tfrac{M r_{i-1}}{2}\Ib\succeq0$,
combined with Assumption~\ref{ass:lip} and Weyl's inequality, yields
\[
  \lambda_n\bigl(\nabla^2 f(x_i)\bigr)
  \;\ge\;\lambda_n\bigl(\nabla^2 f(x_{i-1})\bigr)-L_H\,r_{i-1}
  \;\ge\;-\Bigl(\tfrac{M}{2}+L_H\Bigr) r_{i-1},
\]
so again $r_{i-1}>\rho$. In either case Lemma~\ref{lem:decrease} gives
$f(x_{i-1})-f(x_i)\ge\tfrac{M}{12}\rho^3$, and telescoping over
$i=1,\dots,K$ as in the proof of Theorem~\ref{thm:rate} forces
$K<\tfrac{12(f(x_0)-f^\star)}{M}\bigl(\tfrac{L_H+M}{2}\bigr)^{3/2}
\eps^{-3/2}$, contradicting the definition of $K$ as the ceiling of this
quantity.
\end{proof}

\begin{proposition}[Per-block cubic decrease]
\label{prop:perblock}
Let each block trial be accepted by the ratio test of
Algorithm~\ref{alg:blockcn}, phase~B, namely $\rho_b=(f(x^{(b-1)})-f(x^{(b-1)}+s_b))/(-m_b(s_b))\ge\eta_1$
with $\eta_1\in(0,1)$, where $x^{(b-1)}$ is the point after blocks
$1,\dots,b-1$ have been updated and $m_b$ is the block cubic model at that
point, and let $s_b$ be a global minimizer of $m_b$ over $\R^{n_b}$ (exact
step) or over $\mathcal K_L(\Hb_b,g_b)$ (CubicKrylov step, any $L\ge1$).
Then:
\begin{enumerate}[leftmargin=1.6em,itemsep=1pt,topsep=2pt]
\item[(i)] every accepted trial satisfies
  $f(x^{(b-1)})-f(x^{(b)})\ge\eta_1\,\tfrac{M_b}{12}\norm{s_b}^3$, so one
  sweep gives $f(x_{k+1})\le f(x_k)-\tfrac{\eta_1}{12}\sum_b M_b\norm{s_b}^3$
  and the sweep sequence is monotone;
\item[(ii)] a trial with $M_b\ge L_{H,b}$ (the Hessian Lipschitz constant
  of the block restriction $f_b(s):=f(\dots,\theta_b+s,\dots)$) is
  accepted; hence $M_b\le M_{\max}:=\gamma\max(L_{H,b},M_b^{\mathrm{init}})$
  throughout, and the rejected trials of block $b$ over a run number at
  most $A_b+\log_\gamma(M_{\max}/M_{\min})$, $A_b$ being its accepted
  trials.
\end{enumerate}
\end{proposition}
\begin{proof}
(i) The ratio test gives $f(x^{(b-1)})-f(x^{(b)})\ge\eta_1(-m_b(s_b))$.
The proof of Lemma~\ref{lem:decrease} uses only the stationarity and
second-order conditions of the minimizer,
$P(g_b+\Hb_b s_b+\lambda_b s_b)=0$ and $P\Hb_bP+\lambda_bP\succeq0$ with
$\lambda_b=\tfrac{M_b}{2}\norm{s_b}$ and $P$ the projector onto the search
space ($P=\Ib$ for the exact step, the projector onto $\mathcal K_L$ for
the Krylov step, Lemma~\ref{lem:krylov}). Taking the inner product of the
first with $s_b=Ps_b$ and inserting the second gives
$m_b(s_b)\le-\tfrac{M_b}{12}\norm{s_b}^3$ in both cases. Summing over the
blocks of a sweep gives the rest. (ii) This is the bounded regularizer
and unsuccessful iteration count of \citet{cartis2011a}, applied to the
block restriction. The cubic upper bound of
Assumption~\ref{ass:lip} for $f_b$ gives
$f(x^{(b-1)}+s_b)\le f(x^{(b-1)})+m_b(s_b)+\tfrac{L_{H,b}-M_b}{6}\norm{s_b}^3
\le f(x^{(b-1)})+m_b(s_b)$ when $M_b\ge L_{H,b}$, so $\rho_b\ge1\ge\eta_1$.
A rejection therefore happens only at $M_b<L_{H,b}$ and leaves
$M_b<\gamma L_{H,b}$, and acceptances never raise $M_b$. With $U_b$
rejections and $A_b$ acceptances, each rejection multiplying $M_b$ by
$\gamma$ and each acceptance dividing it by at most $\gamma$ inside
$[M_{\min},M_{\max}]$, $U_b\log\gamma-A_b\log\gamma\le\log(M_{\max}/M_{\min})$.
\end{proof}
Algorithm~\ref{alg:blockcn} evaluates the ratio with the stabilized
tolerance $t$ added to numerator and denominator, so an accepted trial
decreases $f$ by at least $\eta_1(-m_b(s_b))-(1-\eta_1)\,t$. The
additive term is below the fp32 resolution of the loss and is dropped
in (i). Phase~A applies the exact cubic step with $M_b=6p_b L_H$ from
the sweep start gradient and keeps it only if the loss does not rise,
so Phase~A is monotone by its guard. Its decrease bound is a
hypothesis. Under the fresh gradient assumption of
Section~\ref{sec:blocktheory}, and taking the lazy exact decrease of
\citet{doikov2023lazy} as given (Appendix~\ref{app:leancheck}), (ii)
would accept the step and (i) holds with $\eta_1=1$. Phase~B retries a
rejection in the same sweep, so every large block that is not skipped
with $g_b=0$ contributes an accepted summand.

\begin{lemma}[Per-step escape bound]
\label{lem:clamp}
Let $M_{\mathrm{amp}}\ge2$ and define the clamped multiplier by
replacing $x = h\lambda$ with
$\hat x = \max\!\big(h\lambda,\, -\ln M_{\mathrm{amp}}\big)$ in
\eqref{eq:mult}:
\begin{equation}
\tilde\eta_h(\lambda) \;=\; \frac{1 - e^{-\hat x}}{\lambda}
\quad (\lambda \ne 0), \qquad \tilde\eta_h(0) \;=\; h.
\label{eq:clampmult}
\end{equation}
Then for every mode, $|\tilde\eta_h(\lambda)\,\lambda| \le
M_{\mathrm{amp}} - 1$, so the amplification of any gradient component
in one step is at most $(M_{\mathrm{amp}}-1)/|\lambda|$, and the step
is finite in fp32 for any $h$ and any spectrum.
\end{lemma}
\begin{proof}
Monotonicity of $x \mapsto 1 - e^{-x}$ and the floor
$\hat x \ge -\ln M_{\mathrm{amp}}$ give $1 - M_{\mathrm{amp}} \le
1 - e^{-\hat x} \le 1$, and $1\le M_{\mathrm{amp}}-1$ because
$M_{\mathrm{amp}}\ge2$, so $|1-e^{-\hat x}| \le M_{\mathrm{amp}} - 1$
on every mode, and $\tilde\eta_h(\lambda)\,\lambda = 1-e^{-\hat x}$.
The default $M_{\mathrm{amp}}=10^6$ satisfies the hypothesis.
\end{proof}

The clamp trades per-step escape magnitude for compounded escape
across accepted steps. $k$ accepted steps amplify an escape direction
by up to $(M_{\mathrm{amp}}-1)^k$ while every intermediate iterate
passes the monotone guard. Empirically (Appendix~\ref{sec:small})
the clamp engages in the first sweeps on saddle rich initializations
and disengages once the trust adaptation settles.

\section{Deferred Remarks}
\label{app:remarks}

\begin{remark}[Graded non-convexity bounds the curvature part of the shift]
\label{rem:graded}
If $f\in\mathcal{F}_\tau$ is non-convex of grade $\tau$ in the sense
of~\citet{doikov2024spectral}, then $-\lambda_n(x)\le\sigma_\tau$
uniformly, so the part of the shift spent on enforcing positive
definiteness is at most $\sigma_\tau$, and the lower bound
Eq.~\eqref{eq:rlo} costs at most $r\ge 2\sigma_\tau/M$. On problem classes
with small spectral tails (large $\tau$-grade), the size of the cubic
shift is therefore set by the gradient driven term
$\Theta(\sqrt{M\norm{g}})$ of Proposition~\ref{prop:degree} (stated
below), and not by the positive definiteness requirement
$\lambda^\star>-\lambda_n$, mirroring how
$\sigma_\tau^+$ replaces $\sigma_\tau$ in Theorem~5.1
in~\citet{doikov2024spectral}.
\end{remark}

\begin{remark}[Hard case]
\label{rem:hardcase}
If $g\perp u_n$ and $\lambda_n<0$ (the \emph{hard case} of the
trust region subproblem, \citealp{cartis2011a}), no
polynomial in $\Hb$ applied to
$g$ contains a component along $u_n$, so the Krylov solution
approaches the hard-case boundary solution rather than the interior
escape direction. The usual remedy is to augment the subspace with
an approximate eigenvector of $\lambda_n$ (one extra power or Lanczos
probe, the same $-\lambda_n$ probe that the Chebyshev second kind
recurrence solver of Appendix~\ref{app:dso} implements). 
This restores the saddle escape mechanism. Our implementation exposes this as an
optional augmentation vector.
\end{remark}

\begin{remark}[Inexactness and adaptivity]
\label{rem:inexact}
(i) When $s_k$ only minimizes $m$ over $\mathcal{K}_L$
(Lemma~\ref{lem:krylov}) instead of $\R^n$, the rate is retained under
the ARC termination criteria~\citep{cartis2011a,cartis2011b},
i.e.\ the residual test~\eqref{eq:tc}, not the Chebyshev degree of
Proposition~\ref{prop:degree}. The implementation grows the Lanczos
basis until that test holds or $L_{\max}$ is reached. The $U_n$ bound
is a ceiling on how large that $L$ can be when $\kappa_{\mathrm{eff}}$
is controlled by the cubic shift. In the
extreme case $L=1$, $\mathcal{K}_1\ni$ the Cauchy point already yields
a (slower) guaranteed decrease. The same applies to
Corollary~\ref{cor:sosp}. Its proof uses the subproblem second-order
condition, which the Krylov step certifies only within
$\mathcal{K}_L$, and \citet{cartis2011b} show that the ARC criteria
(with a Lanczos-based subproblem solver, as here) preserve the
second-order complexity bound as well. (ii) $L_H$ is never known. As
in Algorithm~3 in \citet{doikov2023lazy} and
in the adaptive search ensuring Eq.~(20) in \citet{doikov2024spectral},
we run a multiplicative adaptation of $M$. Algorithm~\ref{alg:blockcn}
(the stabilized code used for the FINER experiments) doubles $M_b$ on a
rejected trial and halves it on a very successful one ($\gamma=2$).
The Block-CN-Lazy code of the ViSIR runs (Appendix~\ref{app:design})
quadruples on a rejected step and halves on every accepted one, with a
descent guard in place of the ratio test. In either case the accepted
steps satisfy the descent inequality of Lemma~\ref{lem:decrease} with
the current $M_k\le \gamma\max\{L_H,M_{\min}\}$, and the rejected steps
cost one function evaluation each, at most doubling the total
evaluation count. This multiplicative adaptation, applied \emph{per
block}, is the first of the two stabilization mechanisms of the
blockwise optimizer.
The second, the floor on $M_b$, is analyzed with the algorithm in
Section~\ref{sec:algorithm}. (iii) For \emph{lazy} Hessians reused over
$m$ steps the
appropriate constant is $M=6mL_H$~\citep{doikov2023lazy}, which we
exploit for the small block exact step of
Section~\ref{sec:algorithm}.
\end{remark}

\begin{remark}[Global minima and the scope of the guarantee]
\label{rem:global}
Theorem~\ref{thm:rate} and Corollary~\ref{cor:sosp} guarantee approximate
second-order stationarity and stop short of global minimality. This scope
cannot be improved by a better method. By the information-theoretic lower
bound of \citet{nemirovski1983}, any method that accesses a general
smooth non-convex function through local oracle calls (values,
gradients, Hessians, or higher derivatives) needs a number of calls
growing exponentially with the dimension to locate an approximate global
minimizer. On the class of Assumption~\ref{ass:lip} no method of any
order can therefore provide a global-minimum guarantee, and the
$\mathcal{O}(\eps^{-3/2})$ rate to approximate stationarity is itself
optimal for methods with a second-order
oracle~\citep{carmon2020lower}. Global minimum guarantees exist only
under additional structure that removes this obstacle. For
gradient dominated objectives (a class containing convex, star-convex,
and Polyak--{\L}ojasiewicz functions), every stationary point is a
global minimizer, and the same cubic method reaches an $\eps$-global
minimum with improved complexity~\citep{chayti2024unified,masiha2022}.
Sinusoidal INR landscapes are not gradient dominated. The saddle
structure documented in Sections~\ref{sec:intro}
and~\ref{sec:experiments} is exactly what such assumptions exclude. So
approximate second-order stationarity is the strongest guarantee any
method can offer on the problem class studied here.
\end{remark}

\begin{remark}[Failure of joint acceptance at scale]
\label{rem:joint}
Suppose instead that all $B$ block steps $s_1,\dots,s_B$ are computed from
the same base point and accepted or rejected \emph{jointly}. Each
block step $s_b$ is bounded by its own trust region
($\norm{s_b}=r_b\approx\sqrt{2\norm{g_b}/M_b}$), but the joint step
has $\norm{s}=(\sum_b r_b^2)^{1/2}\approx\sqrt{B}\,\bar r$, and the
cross-block second-order coupling terms
$\inner{\Hb_{bb'}s_b}{s_{b'}}$, which are absent from every block model, grow
with $B$ as well. The joint cubic decrease guarantee would require
the \emph{joint} radius to satisfy the secular equation, i.e.\ each
$M_b$ inflated by $\Theta(\sqrt{B})$. With $B\approx40$ large tensors
(as in the 91.4M ViSIR) a joint test rejects nearly every step, and a
shared multiplicative $M$ adaptation then drives $M\to M_{\max}$,
resulting in the stall observed experimentally
(Section~\ref{sec:fullvisir}). Per-block acceptance with per-block
$M_b$ (Proposition~\ref{prop:perblock}) costs one extra forward pass
per block and avoids this kind of failure due to joint acceptance.
\end{remark}

Proposition~\ref{prop:perblock} gives the decrease. The sweep end
gradient is where the cross-block coupling of Remark~\ref{rem:joint}
appears. Throughout this section every block gradient $g_b$ is
evaluated at $x^{(b-1)}$, the point after blocks $1,\dots,b-1$ have
been updated (the assumptions of Section~\ref{sec:blocktheory}). Let
$\delta\ge0$ bound the off-block curvature,
$\norm{\nabla^2 f-\mathrm{blkdiag}\,\nabla^2 f}\le\delta$ on the sublevel
set $\{f\le f(x_0)\}$, and let $\kappa_H\ge\norm{\Hb_b}$ there. Write
$r_b=\norm{s_b}$ for accepted steps ($r_b=0$ otherwise) and
$R_k=(\sum_b r_b^2)^{1/2}$ for the joint step norm of one sweep.
A claim on $\nabla f$ needs every visited block to be accepted or
skipped with $g_b=0$. Phase~A applies an exact step or skips at
$g_b=0$. Phase~B does the same once a rejection is retried until
accept. Theorems~\ref{thm:phaseA} and~\ref{thm:phaseB} bound
$g_{\mathcal S}$ and $g_{\mathcal L}$ separately, and
Theorem~\ref{thm:blockrate} assembles both at the end of the sweep.

\begin{lemma}[Block gradient after its own step]
\label{lem:blockgrad}
For an accepted exact step,
$\norm{g_b(x^{(b)})}\le c_b r_b^2$ with $c_b=\tfrac{L_{H,b}+M_b}{2}$. For an
accepted Krylov step that satisfies the termination criterion
\begin{equation}
  \norm{\nabla m_b(s_b)}\le\kappa_\theta\min(1,r_b)\norm{g_b(x^{(b-1)})},
  \qquad\kappa_\theta\in(0,1),
  \label{eq:tc}
\end{equation}
the same holds with
$c_b=\tfrac{L_{H,b}+M_b}{2}+\tfrac{\kappa_\theta}{1-\kappa_\theta}
\bigl(\kappa_H+\tfrac{M_b}{2}\bigr)$.
\end{lemma}
\begin{proof}
Write $g=g_b(x^{(b-1)})$, $\Hb=\Hb_b(x^{(b-1)})$, $s=s_b$. By
Assumption~\ref{ass:lip} on $f_b$,
$\norm{g_b(x^{(b)})-g-\Hb s}\le\tfrac{L_{H,b}}{2}r_b^2$, and
$\nabla m_b(s)=g+\Hb s+\lambda_b s$ with $\norm{\lambda_bs}=\tfrac{M_b}{2}r_b^2$,
so $\norm{g_b(x^{(b)})}\le\tfrac{L_{H,b}+M_b}{2}r_b^2+\norm{\nabla m_b(s)}$.
The exact step has $\nabla m_b(s)=0$ (Lemma~\ref{lem:newgrad}). For the
Krylov step, \eqref{eq:tc} and
$\norm{g}\le\norm{\nabla m_b(s)}+\kappa_Hr_b+\tfrac{M_b}{2}r_b^2$ give
$\norm{g}\le(\kappa_Hr_b+\tfrac{M_b}{2}r_b^2)/(1-\kappa_\theta)$, and
$\min(1,r_b)\le r_b$ against the first summand, $\min(1,r_b)\le1$ against
the second, give
$\norm{\nabla m_b(s)}\le\tfrac{\kappa_\theta}{1-\kappa_\theta}(\kappa_H+\tfrac{M_b}{2})r_b^2$.
\end{proof}
Criterion~\eqref{eq:tc} is the step form of the termination test of
\citet{cartis2011b}, and the Krylov case of the lemma is their new
gradient estimate under that test, written for the block restriction.
Its left side is free for the Lanczos step,
$\norm{\nabla m_b(s_L)}=\beta_{L+1}|e_L^\top y^\star|$, and the
implementation records whether each trial met it
(Appendix~\ref{app:certified}).

\begin{lemma}[Sweep end gradient]
\label{lem:sweepgrad}
Let $\mathcal J$ be the blocks of a phase (or of a full sweep),
$R_{\mathcal J}=(\sum_{b\in\mathcal J}r_b^2)^{1/2}$, and
$g_{\mathcal J}(x)$ the concatenation of $g_b(x)$ over $b\in\mathcal J$.
If every accepted step in $\mathcal J$ satisfies Lemma~\ref{lem:blockgrad}
with $c_b\le\bar C$ and every skipped block in $\mathcal J$ has
$\norm{g_b}\le\tau$ at its visit, then at the point $x^+$ after those
visits
\[
  \norm{g_{\mathcal J}(x^+)}
  \le C_1^{\mathcal J}R_{\mathcal J}^2
     +\sqrt{\lvert\mathcal J\rvert}\,\delta\,R_{\mathcal J}
     +\sqrt{\lvert\mathcal J\rvert}\,\tau,
  \qquad
  C_1^{\mathcal J}=\bar C+\sqrt{\lvert\mathcal J\rvert}\,L_H.
\]
In particular, after a full sweep with $\tau=0$ and
$\mathcal J=\{1,\dots,B\}$,
$\norm{\nabla f(x_{k+1})}\le C_1R_k^2+\sqrt B\,\delta\,R_k$
with $C_1=\bar C+\sqrt B\,L_H$.
\end{lemma}
\begin{proof}
Block-diagonal extraction is a spectral norm contraction, since
$\mathrm{blkdiag}\,A=\mathbb E[DAD]$ for a random diagonal $D$ with
entries $\pm1$ constant within blocks, so
$\norm{\mathrm{blkdiag}\,A}\le\norm{A}$. Fix $b\in\mathcal J$ and let
$d_b$ be the displacement after the visit of $b$ that is supported off
block $b$ (the later Gauss--Seidel steps of the sweep). Then
$\norm{d_b}\le R_{\mathcal J}$.
Write $x^{(b)}$ for the point at which Lemma~\ref{lem:blockgrad} applies
(or the visit point if $b$ is skipped). Then
$g_b(x^+)-g_b(x^{(b)})=\int_0^1\Ub_b^\top[\nabla^2f(y_t)-\mathrm{blkdiag}\,\nabla^2f(y_t)]\,d_b\,dt$
with $y_t=x^{(b)}+td_b$, because the block-diagonal part annihilates
$d_b$ in row block $b$. The integrand differs from its value at $x^{(b)}$
by at most $2L_Ht\norm{d_b}$ in norm (Assumption~\ref{ass:lip} and the
contraction), and $x^{(b)}$ lies in the sublevel set by
Proposition~\ref{prop:perblock}(i), so
$\norm{g_b(x^+)-g_b(x^{(b)})}\le\delta R_{\mathcal J}+L_HR_{\mathcal J}^2$.
With Lemma~\ref{lem:blockgrad},
$\norm{g_b(x^+)}\le c_br_b^2+\delta R_{\mathcal J}+L_HR_{\mathcal J}^2$
on an updated block, and the same with $c_br_b^2$ replaced by $\tau$ on a
skipped block. Summing the squares over $\mathcal J$ and using
$\sum_{b\in\mathcal J}r_b^4\le R_{\mathcal J}^4$ gives the claim.
\end{proof}

\begin{theorem}[Phase~A rate]
\label{thm:phaseA}
Let Assumption~\ref{ass:lip} hold, let every $g_b$, $b\in\mathcal S$,
be evaluated at $x^{(b-1)}$, and take the lazy exact step of
\citet{doikov2023lazy} as a hypothesis. An exact cubic step with
$M_b=6p_b L_H$ from a Hessian at most $p_b$ sweeps old satisfies
Proposition~\ref{prop:perblock}(i) with $\eta_1=1$ and the gradient
bound of Lemma~\ref{lem:blockgrad} with some $c_b\le\bar C$. For
$p_b=1$ the second part is Lemma~\ref{lem:blockgrad} itself. For
$p_b>1$ the stale Hessian adds the term
$L_H\,r_b\,\norm{x^{(b-1)}-x_{\mathrm{snap}}}$, with $x_{\mathrm{snap}}$
the point of the last Hessian rebuild, to the new block gradient.
\citet{doikov2023lazy} absorb that term in their potential, and it is
part of the hypothesis here. After Phase~A of
Algorithm~\ref{alg:blockcn}, every $b\in\mathcal S$ has taken such a
step or was skipped with $g_b=0$. Write
$x^{(A)}$ for the point after Phase~A,
$M_A=\min_{b\in\mathcal S}6p_b L_H$,
$C_1^{\mathcal S}$ as in Lemma~\ref{lem:sweepgrad} with
$\mathcal J=\mathcal S$, and
$\eps_\star^{\mathcal S}=2\lvert\mathcal S\rvert\delta^2/C_1^{\mathcal S}$.
Then
$f(x_k)-f(x^{(A)})\ge M_A R_{\mathcal S}^3/(12\sqrt{\lvert\mathcal S\rvert})$,
and for every $\eps\ge\eps_\star^{\mathcal S}$ some Phase~A index
$k<K_A(\eps)$,
\[
  K_A(\eps)=\Bigl\lceil\frac{12\sqrt{\lvert\mathcal S\rvert}\,
  (2C_1^{\mathcal S})^{3/2}\,(f(x_0)-f^\star)}
  {M_A\,\eps^{3/2}}\Bigr\rceil
  =\mathcal O\bigl(\sqrt{\lvert\mathcal S\rvert}\,\eps^{-3/2}\bigr),
\]
produces $\norm{g_{\mathcal S}(x^{(A)})}\le\eps$.
\end{theorem}
\begin{proof}
By the lazy decrease hypothesis,
$f(x_k)-f(x^{(A)})\ge(M_A/12)\sum_{b\in\mathcal S}r_b^3$. The power mean
over $\lvert\mathcal S\rvert$ blocks gives the $R_{\mathcal S}^3$ bound.
Summing over $k<K$ and using $f\ge f^\star$ (Phase~B decreases only
tighten the telescope) yields
$\min_{k<K}R_{\mathcal S}^3\le 12\sqrt{\lvert\mathcal S\rvert}
(f(x_0)-f^\star)/(M_A K)$, so at $K=K_A(\eps)$ some Phase~A has
$R_{\mathcal S}\le(\eps/2C_1^{\mathcal S})^{1/2}$. There
Lemma~\ref{lem:sweepgrad} with $\tau=0$ gives
$\norm{g_{\mathcal S}(x^{(A)})}
\le\eps/2+\sqrt{\lvert\mathcal S\rvert}\,\delta\,
(\eps/2C_1^{\mathcal S})^{1/2}\le\eps$,
the last step being $\eps\ge\eps_\star^{\mathcal S}$.
\end{proof}

\begin{theorem}[Phase~B rate]
\label{thm:phaseB}
Let Assumption~\ref{ass:lip} hold, let every accepted CubicKrylov step
satisfy Lemma~\ref{lem:blockgrad} with \eqref{eq:tc}, and let
$M_{\min}\le M_b\le M_{\max}$. After Phase~B of
Algorithm~\ref{alg:blockcn} with an unbounded rejection cap, every
$b\in\mathcal L$ has been accepted or skipped with $g_b=0$. A rejection
restores the block and the inner loop retries until
Proposition~\ref{prop:perblock}(ii) forces acceptance. Write
$C_1^{\mathcal L}$ as in Lemma~\ref{lem:sweepgrad}
with $\mathcal J=\mathcal L$ and
$\eps_\star^{\mathcal L}=2\lvert\mathcal L\rvert\delta^2/C_1^{\mathcal L}$.
Then
$f(x^{(A)})-f(x_{k+1})
\ge\eta_1 M_{\min} R_{\mathcal L}^3/(12\sqrt{\lvert\mathcal L\rvert})$,
and for every $\eps\ge\eps_\star^{\mathcal L}$ some Phase~B index
$k<K_B(\eps)$,
\[
  K_B(\eps)=\Bigl\lceil\frac{12\sqrt{\lvert\mathcal L\rvert}\,
  (2C_1^{\mathcal L})^{3/2}\,(f(x_0)-f^\star)}
  {\eta_1 M_{\min}\,\eps^{3/2}}\Bigr\rceil
  =\mathcal O\bigl(\sqrt{\lvert\mathcal L\rvert}\,\eps^{-3/2}\bigr),
\]
produces $\norm{g_{\mathcal L}(x_{k+1})}\le\eps$.
\end{theorem}
\begin{proof}
Rejected trials revert, so only the accepted step of each $b\in\mathcal L$
contributes to the decrease. Proposition~\ref{prop:perblock}(i) and the power
mean over $\lvert\mathcal L\rvert$ blocks give the $R_{\mathcal L}^3$
bound. The same telescope as in Theorem~\ref{thm:phaseA}, now on the
Phase~B summands, produces a sweep with
$R_{\mathcal L}\le(\eps/2C_1^{\mathcal L})^{1/2}$, and
Lemma~\ref{lem:sweepgrad} with $\tau=0$ closes as there.
\end{proof}

\begin{theorem}[Blockwise complexity with a coupling floor]
\label{thm:blockrate}
Let the hypotheses of Theorems~\ref{thm:phaseA} and~\ref{thm:phaseB}
hold (Assumption~\ref{ass:lip}, every $g_b$ evaluated at $x^{(b-1)}$,
the lazy exact step hypothesis in Phase~A, decrease and gradient bound,
Lemma~\ref{lem:blockgrad} for every accepted Phase~B step, exact or
CubicKrylov with \eqref{eq:tc},
$M_{\min}\le M_b\le M_{\max}$ throughout, and $M_A\ge\eta_1M_{\min}$).
After both phases of a sweep of Algorithm~\ref{alg:blockcn}, every
block is accepted or skipped with $g_b=0$, so Lemma~\ref{lem:sweepgrad}
applies to $\nabla f$ with $\tau=0$. With $C_1$ as in that lemma and
$\eps_\star=2B\delta^2/C_1$, for every $\eps\ge\eps_\star$ some sweep
$k<K(\eps)$,
\[
  K(\eps)=\Bigl\lceil\frac{12\sqrt B\,(2C_1)^{3/2}\,(f(x_0)-f^\star)}
  {\eta_1M_{\min}\,\eps^{3/2}}\Bigr\rceil
  =\mathcal O\bigl(\sqrt B\,\eps^{-3/2}\bigr),
\]
produces $x_{k+1}$ with $\norm{\nabla f(x_{k+1})}\le\eps$. If in
addition every accepted step is exact, or is a Krylov step with the
negative curvature check of \citet{cartis2011b}, then also
$\lambda_n(\nabla^2f(x_{k+1}))\ge-(\tfrac{M_{\max}}{2}+L_H)(\eps/2C_1)^{1/2}-\delta$.
For a Krylov step without that check the second-order condition is
certified on $\mathcal K_L$ only (Remark~\ref{rem:krylovso}(iii)).
At $B=1$, $\delta=0$ the theorem recovers the $\eps^{-3/2}$ exponent of
Theorem~\ref{thm:rate} and Corollary~\ref{cor:sosp}, with a looser
constant ($C_1$ still includes the within-sweep remainder term $L_H$).
Theorems~\ref{thm:phaseA} and~\ref{thm:phaseB}
do not by themselves give $\norm{\nabla f}\le\eps$ at
$k=\max(K_A,K_B)$. Phase~B displaces $g_{\mathcal S}$, and the two
phases may meet their tolerances at different sweeps.
\end{theorem}
\begin{proof}
Phase~A contributes at least $(M_A/12)\sum_{b\in\mathcal S}r_b^3$ and
Phase~B at least $(\eta_1 M_{\min}/12)\sum_{b\in\mathcal L}r_b^3$. By the
hypothesis $M_A\ge\eta_1 M_{\min}$,
Proposition~\ref{prop:perblock}(i) and the power mean
$\sum_br_b^3\ge B^{-1/2}(\sum_br_b^2)^{3/2}$ give
$f(x_k)-f(x_{k+1})\ge\tfrac{\eta_1M_{\min}}{12\sqrt B}R_k^3$. Summing over
$k<K$ and using $f\ge f^\star$,
$\min_{k<K}R_k^3\le12\sqrt B(f(x_0)-f^\star)/(\eta_1M_{\min}K)$, so at
$K=K(\eps)$ some sweep has $R_k\le(\eps/2C_1)^{1/2}$. There
Lemma~\ref{lem:sweepgrad} with $\tau=0$ gives
$\norm{\nabla f(x_{k+1})}\le\tfrac{\eps}{2}+\sqrt B\,\delta\,(\eps/2C_1)^{1/2}\le\eps$,
the last step being $\eps\ge\eps_\star=2B\delta^2/C_1$. For the
curvature clause, the second-order condition of each accepted step
gives $\lambda_n(\Hb_b(x^{(b-1)}))\ge-\tfrac{M_b}{2}r_b$ on the full
block (exact step, or Krylov step with the negative curvature check).
Assumption~\ref{ass:lip} extends it to $x^{(B)}$ at a cost $L_HR_k$, and
Weyl's inequality with the $\delta$ bound gives
$\lambda_n(\nabla^2f(x^{(B)}))\ge\min_b\lambda_n(\Hb_b(x^{(B)}))-\delta$.
\end{proof}

Algorithm~\ref{alg:blockcn} realizes the fresh gradient hypothesis in
Phase~B and, with $\mathcal S=\emptyset$ (the FINER configuration,
every tensor through Phase~B), for the whole sweep. With
$\mathcal S\ne\emptyset$ its Phase~A takes the small block gradients
from the sweep start closure and is monotone by its descent guard,
so the sweep stays monotone, while the Phase~A decrease bound of
Theorem~\ref{thm:phaseA} is proved for the fresh gradient variant.

\begin{remark}[Scope of the theorem]
\label{rem:krylovso}
(i) Below $\eps_\star$ the linear coupling term dominates the sweep end
gradient and the argument yields no further decrease. The block models
cannot see the off-block curvature that keeps $\nabla f$ away from zero.
The floor is measurable before training,
$\delta\le\norm{\nabla^2f}_F\sqrt{1-\beta}$ with
$\beta=\norm{\mathrm{blkdiag}\,\nabla^2f}_F^2/\norm{\nabla^2f}_F^2$
the tensor-block mass (Appendix~\ref{app:certified} reports the FINER
values). This $\beta$ is not the coordinate-diagonal ratio $\rho$ of
Section~\ref{sec:fingerprint}. (ii) The rigorous $C_1$ includes
the term $\sqrt BL_H$ from the second-order remainder along the
sweep. (iii) For the Krylov subsolver the second-order condition is
certified on $\mathcal K_L$ only, through the smallest Ritz value of
$\Tb_L$ against the shift. \citet{cartis2011b} recover the full space
clause for Lanczos-based ARC with an occasional negative curvature check,
which applies per block here. (iv) The argument is the inexact-Hessian
analysis of \citet{xu2020inexact} with the block-diagonal model as the
inexact Hessian and $\delta$ its error, plus the sequential
Gauss--Seidel sweep, which is the source of the $\sqrt B$ factors. The
published block cubic rates, $\mathcal O(\eps^{-2})$ under greedy block
selection \citep{cristofari2024block} and $\mathcal O(\eps^{-3})$ for
general block choice \citep{amaral2022highorder}, hold without a floor
but with the weaker exponent. (v) The $\varphi_1$ step is not a model
minimizer. Its decrease is the Cauchy type bound of
Section~\ref{sec:phi1theory}, and Appendix~\ref{app:certified} reports
the gradient-regularized horizon that places its positive-mode step in
the $\mathcal O(\eps^{-3/2})$ framework of regularized Newton methods
\citep{mishchenko2023,doikov2024gradreg,gratton2024yet}.
\end{remark}

\begin{remark}[Derivation of the step rule \eqref{eq:relex}]
\label{rem:relexderiv}
Write
$m(x) = f(x_0) + \inner{g}{x - x_0} + \tfrac12\inner{x - x_0}{G\,(x - x_0)}$
for the local quadratic model, so its gradient flow is
$\dot{x} = -\nabla m(x) = -g - G\,(x - x_0)$ with $x(0) = x_0$.
Substituting $u(\tau) = x(\tau) - x_0$ gives the linear
constant-coefficient system $\dot{u} = -g - G\,u$, $u(0) = 0$.
Multiplying by the integrating factor $e^{G\tau}$ and observing
$\tfrac{d}{d\tau}\big(e^{G\tau} u\big)
 = e^{G\tau}(\dot{u} + G u) = -e^{G\tau} g$,
integration over $[0,h]$ yields
\[
e^{Gh}\,u(h) \;=\; -\Big(\int_0^h e^{G\tau}\,d\tau\Big)\,g
\quad\Longrightarrow\quad
u(h) \;=\; -\int_0^h e^{-G(h-\tau)}\,d\tau\;g
 \;=\; -\int_0^h e^{-G\sigma}\,d\sigma\;g,
\]
after the change of variable $\sigma = h - \tau$ (all factors are
functions of the single symmetric matrix $G$ and commute). This is
exactly $s = -H(G,h)\,g$ with $H(G,h) = \int_0^h e^{-G\tau}\,d\tau$.
When $G$ is invertible the integral evaluates in closed form to
$H(G,h) = G^{-1}\big(\Ib - e^{-hG}\big)$, and comparison with
$\varphi_1(z) = (e^{z} - 1)/z$ gives $H = h\,\varphi_1(-hG)$. For
singular $G$ the same identity holds through the entire-function
series $\varphi_1(z) = \sum_{k\ge 0} z^{k}/(k+1)!$, so the step
requires no invertibility assumption. Projecting the same computation
onto an eigenpair $(\lambda_i, v_i)$ of $G$ reduces it to the scalar
ODE $\dot{u}_i = -\inner{v_i}{g} - \lambda_i u_i$ and produces the
per-mode multiplier $\eta_h(\lambda)$ of
Proposition~\ref{prop:erlaw}. 

Finally, the literal interpretation of infinitesimal learning rate.
Fix the horizon $h$, split it into $n$ explicit
gradient descent steps of size $\eta = h/n$ on the frozen quadratic,
$u_{k+1} = u_k - \eta\,(g + G u_k)$, $u_0 = 0$, and let $n$ grow so
that the step size shrinks toward zero at fixed total budget $h$. The
iteration unrolls to the geometric sum

\[
u_n \;=\; -\eta \sum_{j=0}^{n-1} (\Ib - \eta G)^{j}\, g
 \;=\; -G^{-1}\Big(\Ib - \big(\Ib - \tfrac{h}{n} G\big)^{n}\Big) g
 \;\xrightarrow{\;n\to\infty\;}\;
 -G^{-1}\big(\Ib - e^{-hG}\big)\, g,
\]
for invertible $G$ (the general case follows by continuity of both
sides in $G$), since
$(\Ib - \tfrac{h}{n} G)^{n} \to e^{-hG}$. The $\varphi_1$ step is the
limit of $n$ gradient descent steps of total budget $h$ on the
quadratic, integrated in closed form rather than iterated.
\end{remark}

\section{Chebyshev Second Kind Preconditioning in the DSO Construction}
\label{app:dso}

Chebyshev polynomials of the second kind are defined by
$U_n(\cos\zeta)=\sin((n{+}1)\zeta)/\sin\zeta$ with the recurrence
$U_0=1$, $U_1(x)=2x$, $U_{n+1}(x)=2xU_n(x)-U_{n-1}(x)$. The Dynamic
Spectral Optimizer (DSO) of \citet{podorozh2026siam} implements this
Chebyshev second kind preconditioning. The construction is restated
here in full because the OP26 presentation has no accompanying paper,
and every DSO result in this paper is a run of this restatement. DSO
normalizes the Hessian, $G=\Hb/\rho$ with
$\rho=\norm{\Hb}_2$, maps eigenvalues $\lambda\in[0,1]$ of $G$ to angles
via $\cos\zeta = 1-2\lambda$, and builds the \emph{relaxation function}
\begin{equation}
  R_L(\lambda)=\frac{P_L(\lambda)}{L},\qquad
  P_L(\lambda)=\frac{\sin(L\zeta)}{\sin\zeta},
\label{eq:relax}
\end{equation}
a degree-$(L{-}1)$ polynomial with $R_L(0)=1$. The induced
preconditioner
\begin{equation}
  H(\lambda)=\frac{1-R_L(\lambda)}{\lambda}
\label{eq:dsoprec}
\end{equation}
is a polynomial of degree $L-2$ that approximates $1/\lambda$ on most
of the normalized spectrum. Figure~\ref{fig:dsorelax} shows why. Near
$\lambda=0$ the angle is $\zeta\approx2\sqrt\lambda$, so
$R_L(\lambda)\approx\varphi(2L\sqrt\lambda)$ with the sinc function
$\varphi(x)=\sin(x)/x$, and $|R_L(1-\lambda)|=|R_L(\lambda)|$ mirrors
this behavior at the upper end. The sinc function satisfies
$|\varphi(x)|\le0.22$ for every $x\ge x_{\mathrm{cr}}\approx2.55$, its
first minimum being $\varphi(4.49)\approx-0.217$. On the band
$x_{\mathrm{cr}}^2/(4L^2)\le\lambda\le1-x_{\mathrm{cr}}^2/(4L^2)$ the
polynomial therefore stays small, with $\max|R_L|=0.33$, $0.27$,
$0.25$ for $L=3,4,5$ and $0.22$ from $L\approx15$ on. One application
of the degree-$L$ polynomial (one sweep of $L{-}2$ HVPs) contracts the
residual along every eigendirection in that band to at most this
fraction of its value, and $H(\lambda)$ is within the same fraction of
$1/\lambda$ there. This uniform contraction across the band, rather
than at the spectral extremes only, is the mechanism that equalizes
contraction rates across frequencies and counteracts spectral bias
when the recurrence is used as a standalone optimizer. The two ends of
the spectrum behave differently. As $\lambda\to0$ (flat directions,
the high frequency modes of an INR) $R_L\to1$ and
$H(\lambda)\to H_L(0)=2(L^2-1)/3$, so the gradient signal along a flat
direction is amplified by a factor that grows with $L^2$ and stays
finite, whereas Newton's method would divide by $\lambda$ itself. At
$\lambda=1$ (the highest curvature direction) $R_L(1)=(-1)^{L-1}$, so
$H(1)$ is $0$ for odd $L$ and $2$ for even $L$, a bounded overshoot
along the top eigenvector. For $\lambda<0$ (negative curvature)
$\cos\zeta>1$, the angle becomes imaginary, and $R_L$ grows like a
hyperbolic sine, which is the automatic saddle escape of the standalone
DSO and the reason its steps need a clamp.
The step $d_L = -H(G)\tilde g$ is computed by the three term recurrence
\begin{equation}
  d_{s+1}=\frac{2s}{s{+}1}\,(\Ib-2G)\,d_s
          -\frac{s{-}1}{s{+}1}\,d_{s-1}
          -\frac{4s}{s{+}1}\,\tilde g,
  \qquad s\ge2,\qquad d_1=0,\quad d_2=-2\tilde g,
\label{eq:dsorec}
\end{equation}
so that $d_s=-H_s(G)\,\tilde g$ for every $s$, where
$H_s(\lambda):=(1-R_s(\lambda))/\lambda$ is Eq.~\eqref{eq:dsoprec} at
degree $s$ (in particular $H_1\equiv0$ and $H_2\equiv2$, and starting the
same coefficients one index earlier, from $(0,-2\tilde g)$, would
leave the $U_n/n$ relaxation family and break the $R_L(0)=1$
normalization). The recurrence requires one HVP per degree and
$\mathcal{O}(1)$ auxiliary vectors.

\begin{figure}[t]
\centering
\begin{minipage}[c]{0.43\textwidth}
\includegraphics[width=\linewidth]{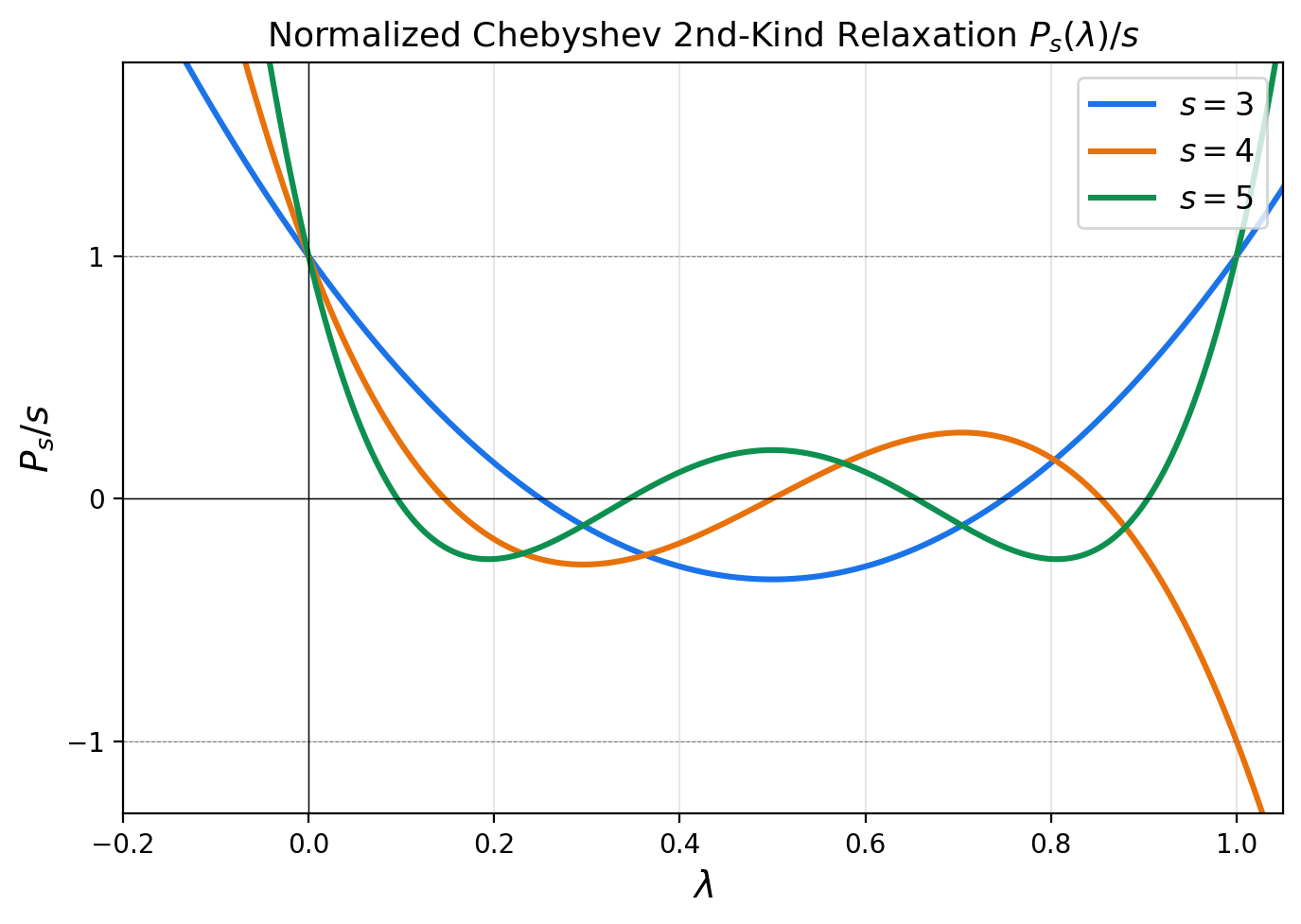}
\end{minipage}\hfill
\begin{minipage}[c]{0.55\textwidth}
\includegraphics[width=\linewidth]{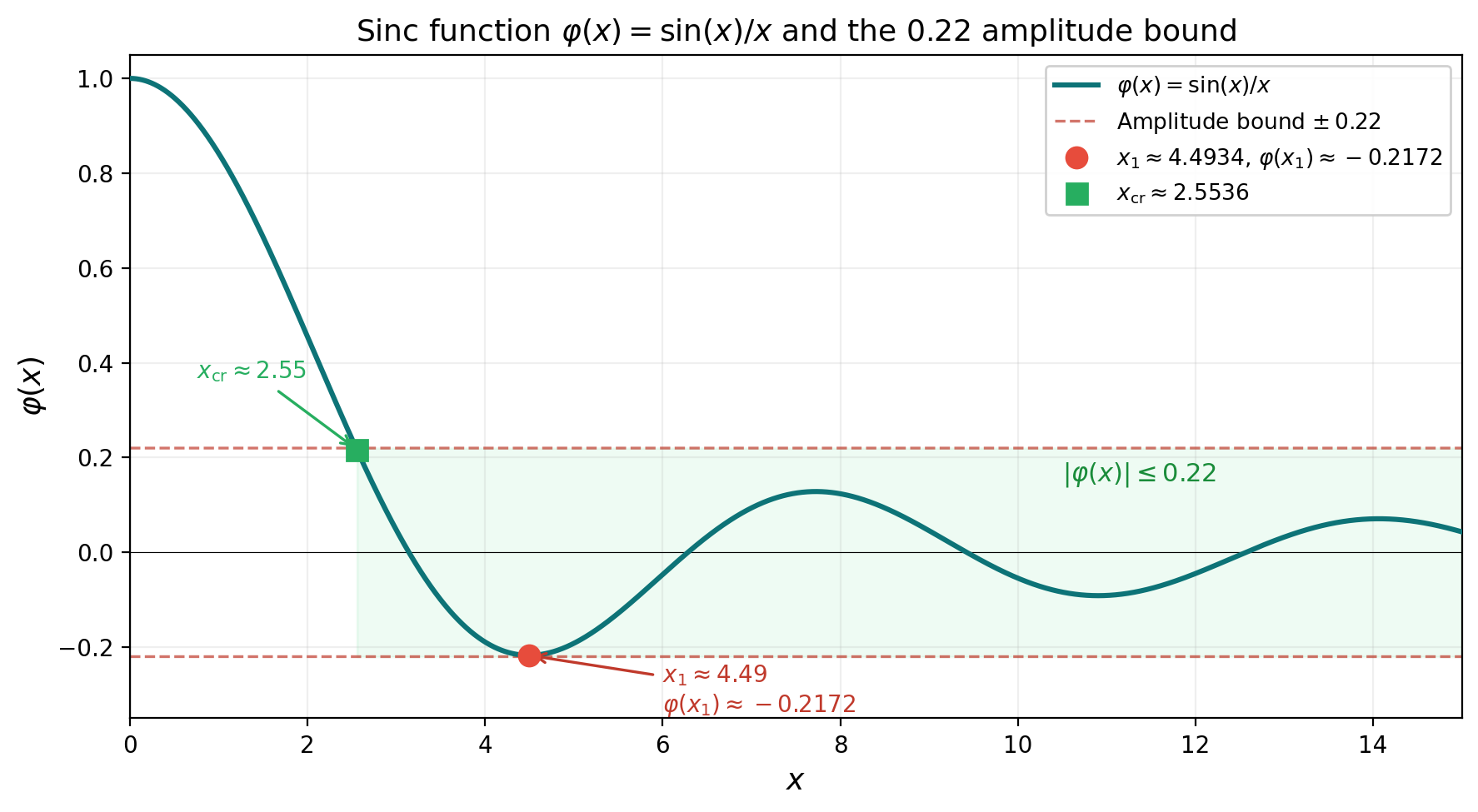}
\end{minipage}
\caption{The DSO relaxation function. Left, $R_s(\lambda)=P_s(\lambda)/s$
for $s=3,4,5$ on the normalized spectrum. Each curve has $R_s(0)=1$,
oscillates with small amplitude on the interior of $[0,1]$, returns to
$(-1)^{s-1}$ at $\lambda=1$, and grows without bound for $\lambda<0$.
Right, the sinc function $\varphi(x)=\sin(x)/x$ that governs $R_s$ near
both ends of the spectrum through $x=2s\sqrt\lambda$. Its first minimum
is $\varphi(x_1)\approx-0.217$ at $x_1\approx4.49$, and
$|\varphi(x)|\le0.22$ for all $x\ge x_{\mathrm{cr}}\approx2.55$ (shaded).
Reaching $\lambda_{\min}=1/\kappa$ with $2L\sqrt{\lambda_{\min}}\ge
x_{\mathrm{cr}}$ gives the degree rule
$L\ge(x_{\mathrm{cr}}/2)\sqrt\kappa\approx1.28\sqrt\kappa$ of
Remark~\ref{rem:cheb2bound}.}
\label{fig:dsorelax}
\end{figure}

Two structural costs of this construction motivate the present
work:
\begin{enumerate}[leftmargin=2em,itemsep=1pt]
  \item \textbf{Degree blow-up.} A uniform per-direction contraction
        over the normalized Hessian eigenvalue spectrum requires degree
        $L\approx 1.28\sqrt{\kappa}$, where $\kappa$ is the Hessian condition
        number, because the contraction band of Figure~\ref{fig:dsorelax}
        must reach $\lambda_{\min}=1/\kappa$ (Remark~\ref{rem:cheb2bound}).
        DSO implementation bounds $L$ heuristically when $\kappa$ is extreme (e.g. due to inaccuracy of eigenvalues estimation at scale).
  \item \textbf{Indefiniteness.} For $\lambda<0$ the recurrence enters
        hyperbolic growth. This provides automatic saddle escape but
        also requires clamping safeguards whose failure modes are
        documented in the DSO ablations.
\end{enumerate}

\begin{remark}[The Chebyshev second kind degree bound at a fixed
contraction]
\label{rem:cheb2bound}
The constant $1.28$ is $x_{\mathrm{cr}}/2$ of Figure~\ref{fig:dsorelax},
half the smallest argument beyond which the sinc envelope of $R_L$
stays inside $\pm0.22$, rounded up. It can also be
derived from the degree bound for Chebyshev iteration on a positive
definite system~\citep{saad2003}, instantiated at a fixed per-sweep
contraction. For a spectrum in $[\lambda_{\min},\lambda_{\max}]$ the
optimal degree-$L$ residual polynomial satisfies
$\|e_L\|\le2\bigl(\tfrac{\sqrt\kappa-1}{\sqrt\kappa+1}\bigr)^{L}\|e_0\|$,
and since $\ln\tfrac{\sqrt\kappa+1}{\sqrt\kappa-1}\ge\tfrac{2}{\sqrt\kappa}$,
the contraction target $\|e_L\|\le\varepsilon\,\|e_0\|$ is met by
\[
  L \;\ge\; \tfrac{1}{2}\,\ln\!\bigl(\tfrac{2}{\varepsilon}\bigr)\sqrt{\kappa}
  \;=\; c(\varepsilon)\,\sqrt{\kappa}.
\]
Setting $c=1.28$ gives $\varepsilon=2e^{-2c}\approx0.15$, so the Chebyshev
second kind degree bound $L\approx1.28\sqrt\kappa$ is exactly this bound
with the worst case residual contracted to ${\approx}15\%$ per polynomial
application. This is consistent with the ${\approx}0.22$ sinc envelope of
Eq.~\eqref{eq:relax}, which measures the same per-degree contraction in the
second kind normalization. By Proposition~\ref{prop:degree}, the cubic
shift replaces the $\kappa$ in this bound by the bounded
$\kappa_{\mathrm{eff}}$ of Eq.~\eqref{eq:kappaeff}. So the same bound
yields a degree that remains constant independent of $\kappa$ while the iterate is far
from stationarity and the gradient driven shift dominates the negative
curvature (Eq.~\eqref{eq:shiftdom}), where the large gradient keeps the
shift large and $\kappa_{\mathrm{eff}}$ bounded.

\end{remark}

\section{ARC-block Design Decisions and Computational Cost Evaluation}
\label{app:design}

Algorithm~\ref{alg:blockcn} embeds the step into the block optimizer
used for many experiments described in this paper (e.g. the ViSIR Chebyshev ablations). 
Three design decisions, each
in response to an observed failure, are emphasized below.

\par\noindent\textbf{(a) Two phase update with guarded small blocks.}
Small blocks (biases and norms, $n_b\le512$) take lazy exact cubic steps
\citep{doikov2023lazy} from the sweep start gradient. Each step is kept
only if the loss after it is finite and not above the loss before it.
A rejected small block is restored and its exact Hessian is rebuilt at
the next visit. Without this guard a single non-finite or overshooting
small block step (a stale lazy Hessian against a fresh gradient spike at
a batch size transition) poisons the parameters and every later step.
The per-block accept-or-reject of Phase~B is measured against the
loss \emph{after Phase~A}. Rejecting the whole step on a loss increase
lets one stale small block component veto good Krylov updates and
drives $M$ to the bound.

\par\noindent\textbf{(b) Per-block acceptance and per-block $M_b$
(Remark~\ref{rem:joint}).} A single joint test over large
tensors rejects even when every
individual block step is good 
(for the experiments in this paper,
this failure mostly was observed on $91.4$M ViSIR architecture with 45 tensors).
Per-block testing costs one forward
pass per a large block and restores monotone progress
(Proposition~\ref{prop:perblock}).

\par\noindent\textbf{(c) A nearly zero floor for $M_b$.}
The step radius scales as $r\approx\sqrt{2\norm{g_b}/M_b}$
(Proposition~\ref{prop:degree}). A floor $M_{\min}=6L_H$ over-regularizes
blocks whose true local Hessian Lipschitz constant is far below the
global estimate $L_H$. Steps are \emph{accepted but microscopic}, a
second, "silent"  stall mode (no prominent warning signs,
the optimizer making extremely slow progress). We set
$M_{\mathrm{init}}=6L_H$ but let the adaptation, upward by a factor
$\gamma$ on a rejected trial and downward by the same factor on a very
successful one, take $M_b$ all the way down to a nearly zero floor,
retained only to keep the step radius finite, so each block finds its
own curvature scale. Two
codes realize this. The Block-CN-Lazy code of the ViSIR and ViSIR-Nano
runs uses a descent guard in place of the ratio test, multiplies $M_b$
by $4$ on a rejected step and by $\tfrac12$ on every accepted one, and
has $M_{\min}=10^{-6}$. The stabilized code of Algorithm~\ref{alg:blockcn}
and the FINER experiments uses the ratio test with $\eta_1=0.1$,
$\eta_2=0.9$, multiplies $M_b=2\sigma_b$ by $2$ on a rejected trial and
by $\tfrac12$ on a very successful one, and has $M_{\min}=2\sigma_{\min}
=2\cdot10^{-8}$, $M_{\max}=2\sigma_{\max}=2\cdot10^{12}$. The asymmetric
factors of the first code keep the rejected step overhead bounded
(Remark~\ref{rem:inexact}(ii)).

\par\noindent\textbf{Fresh gradients.} Phase~B gradients are recomputed after
Phase~A and after each preceding large block
(Algorithm~\ref{alg:blockcn}, Phase~B). They come free
from the HVP graph construction. Phase~A keeps the sweep start
gradient for its small blocks, and only its Hessians are lazy. On the
ViSIR-Nano benchmark (small scale, around 15k parameters), reusing the
sweep start closure gradient in Phase~B as well (a single gradient
evaluated before Phase~A, stale for every subsequent block) lowers
the final peak signal-to-noise ratio (PSNR) by ${\approx}8.5$\,dB
(31.7 vs.\ 23.1\,dB in Table~\ref{tab:nano}).

\subsection{Computational cost evaluation}

\begin{table}[h]
\centering\small
\caption{Per-step, per-block cost of the step variants, block size
$n_b$, Lanczos depth $L$ (bounded by the Chebyshev degree analysis),
laziness $p_b$. The gradient equivalent count follows the
ViSIR ablation experiment. One gradient, one HVP, and one function
evaluation each cost one, and the CubicKrylov HVPs are $L$ for the
basis and one for the model decrease.}
\label{tab:cost}
\begin{tabular}{lccc}
\toprule
Step & Build & Solve & Cost \\
\midrule
Cheby-ON (DSO step) & full $\Hb_b$ every step ($n_b$ gevals) &
  degree-$L$ recurrence, $\mathcal{O}(Ln_b^2)$ & $n_b+1$ \\
Cheby-OFF (Newton) & full $\Hb_b$ every step &
  \texttt{linalg.solve}, $\mathcal{O}(n_b^3)$ & $n_b+1$ \\
Block-CN-Lazy (eigh) & full $\Hb_b$ every $p_b$ steps &
  cached eigh + secular, $\mathcal{O}(n_b^2)$ & $n_b/p_b+1$ \\
\textbf{CubicKrylov} & \textbf{none} &
  \textbf{$L{+}1$ HVPs + tridiagonal secular} & $\mathbf{L+3}$ \\
\bottomrule
\end{tabular}
\end{table}

For $L\approx10$, the Krylov row is the only entry whose oracle
cost is independent of $n_b$. Its wall clock and basis memory still
scale linearly with $n_b$ through the HVPs and the stored vectors. It is this independence that
makes it feasible for large block sizes. Cubic regularized second-order steps become
feasible on tensors where only HVPs are affordable, e.g., the SIREN
decoder of the 91.4M ViSIR.

\par\noindent\textbf{Cost per sweep.}
\label{sec:sweepcost}
One sweep visits $B$ blocks. Each block trial spends one gradient, at
most $L$ Hessian--vector products for the Lanczos basis, one for the
model decrease, and one function evaluation, that is $L{+}3$ gradient
equivalents, each Hessian--vector product costing one backward pass
over the whole model regardless of block size. A sweep therefore costs
$\mathcal{O}(BL)$ gradient equivalents and time linear in the parameter
count $n$ at fixed $B$ and $L$. Rejected trials are included at full cost.
Memory is $L{+}1$ block sized vectors for the Lanczos subsolver (the
basis and one work vector) and three in place buffers for the
Chebyshev and $\varphi_1$ recurrences, independent of the depth.
Measured cost follows this form. On FINER ($B{=}10$, $n{=}198{,}915$,
Section~\ref{sec:finer}) the $L{=}25$ ARC-$\varphi_1$ run on ESM
image~0 costs $276.6$ gradient equivalents per sweep ($2{,}366{,}167$
over $8{,}556$ sweeps) against the predicted $10\times(25+3)=280$. At
$L{=}163$ an instrumented run of the converged configuration costs
$1{,}710$ per sweep over its first $42$ sweeps ($71{,}840$ gradient
equivalents, $11\%$ of trials rejected, $1{,}630$ per sweep once the
first sweep's rejections are excluded) against the predicted
$10\times(163+3)=1{,}660$ at full depth. The adaptive stop of
Section~\ref{sec:krylovtheory} ends some builds below $L$. Per-sweep
wall clock rises $7.2\times$ between $L{=}25$ ($0.84$~s) and $L{=}163$
($6.08$~s) against a depth ratio of $6.5\times$ and a gradient
equivalent ratio of $6.2\times$. Because $B$ grows with architecture
($10$ tensors at $199$k parameters, $45$ at $91.4$M), $B$ is a
controlled quantity rather than a constant, and the comparisons of
Section~\ref{sec:experiments} report gradient equivalents beside
wall clock time.

\section{Trust-Controlled Horizons, Horizon Ranking, and the
ARC-$\varphi_1$ Algorithm}
\label{app:phi1detail}

\subsection{Trust-controlled horizons and inherited guarantees}
\label{sec:control}

The composite scheme sets $h_b = c_h/\sigma_b$ and evaluates the
acceptance ratio against the subspace quadratic model
$m_b(s) = \inner{g_b}{s} + \tfrac12\inner{s}{\Hb_b s}$ (one extra
Hessian--vector product), with the stabilized monotone guard of
Section~\ref{sec:algorithm} unchanged. Three consequences:

\begin{enumerate}[leftmargin=1.4em,itemsep=1pt,topsep=2pt]
\item \emph{Well posed adaptation.} Rejection multiplies $\sigma_b$
  by $\gamma_2$, hence contracts $h_b$ geometrically. By
  Proposition~\ref{prop:erlaw}(v) a sufficiently small horizon brings
  the cubic remainder below the model decrease, and the stabilized
  tolerance converts this into acceptance after finitely many
  rejections under the bounded Hessian and Lipschitz continuity
  assumptions of Section~\ref{sec:background}. The per-block rejection
  bound and $\sigma$-bounds of the block-$\sigma$ ablation
  (Appendix~\ref{sec:feasibility}) apply verbatim.
\item \emph{Newton limit under sustained acceptance.} Sustained acceptance drives
  $\sigma_b \to \sigma_{\min}$, hence $h_b \to c_h/\sigma_{\min}$, so
  high curvature modes receive Newton steps while flat modes receive steps of
  length $\le (c_h/\sigma_{\min})\,\norm{g}$, and acceptance is
  decided by the guard, not by a shift.
\item \emph{No claimed rate.} We do not transfer the
  $\mathcal{O}(\eps^{-3/2})$ complexity of the cubic model
  (Theorem~\ref{thm:rate}) to the $\varphi_1$ rule. That rate is a
  property of the cubic model's overestimation argument. What is
  retained is the acceptance driven global behavior (monotone
  full loss trajectory up to the stabilized tolerance) and the
  saddle escape mechanism, now explicit in the step rule.
\end{enumerate}

\subsection{Free per-trial horizon ranking}
\label{sec:hrank}

Proposition~\ref{prop:erlaw}(iv) is the $\varphi_1$ method's central
mechanism, doubling the horizon until the objective stops improving. A
grid search over horizons is affordable here for a stronger
reason. Once the Ritz decomposition of \eqref{eq:substep} exists, any
candidate $h$ costs one diagonal reweighting and one $L$-dimensional
model evaluation, with no Hessian--vector products and no function
evaluations. The \texttt{h\_rank} variant ranks
$\{h_b/2,\,h_b,\,2h_b\}$ by the subspace quadratic model and submits
the winner to the (unchanged) acceptance test. On a small MLP benchmark
the ranked variant selects the shorter horizon $68\%$ of the time and
the longer one $29\%$, confirming that the single-candidate scheme is
frequently off its own model's optimum. 

The extended hyperparameter sweep at the
scale of the FINER benchmark (Section~\ref{sec:finer}) is organized in
rungs of increasing step budget. At each rung all candidate
configurations are run at that budget, and only the best-ranked ones
advance to the next rung, up to a final
full budget confirmation run that decides the winner. The ranked
variant never advanced that far. At every rung, the configurations that
extended the Lanczos depth outranked the horizon-ranked variant.

\subsection{ARC-block with $\varphi_1$ subsolver algorithm}
\label{sec:phi1algo}

\begin{algorithm}[t]
\caption{ARC-$\varphi_1$, one Gauss--Seidel sweep}
\label{alg:arcphi1}
\begin{algorithmic}[1]
\State build loss and gradient graph; $f_0 \gets f(x)$
\For{each block $b = 1,\dots,B$ with $\norm{g_b} > \eps_g$}
  \State Lanczos: $(\Qb_L, \Tb_L) \gets \mathrm{Lanczos}(\Hb_b, g_b, L)$;
         $\Tb_L = \Vb\,\mathrm{diag}(\nu)\,\Vb^\top$
  \State $h_b \gets c_h / \sigma_b$
  \State candidates $\mathcal{H} \gets \{h_b\}$ (or
         $\{h_b/2, h_b, 2h_b\}$ if \texttt{h\_rank}$=3$)
  \State $s_b \gets$ step \eqref{eq:substep} for the
         model-best $h \in \mathcal{H}$, clamped per
         Lemma~\ref{lem:clamp}
  \State $\rho_b \gets \big(f_0 - f(x + s_b) + t\big) \big/
         \big({-}m_b(s_b) + t\big)$;\quad monotone guard \eqref{eq:guard}
  \If{$\rho_b \ge \eta_1$} accept; $\sigma_b \gets
      \max(\gamma_1\sigma_b, \sigma_{\min})$ if $\rho_b \ge \eta_2$;
      rebuild graph
  \Else{} restore block; $\sigma_b \gets \gamma_2\sigma_b$; retry, or
      skip after $n_{\mathrm{rej}}$ rejections
  \EndIf
\EndFor
\end{algorithmic}
\end{algorithm}

Algorithm~\ref{alg:arcphi1} follows the same per-block, Gauss--Seidel
structure as Phase B of Algorithm~\ref{alg:blockcn} (Krylov build,
search step, fresh-gradient accept-or-reject), replacing the
\textsc{CubicKrylov} step and its $M_b$-doubling rule with the
$\varphi_1$ step \eqref{eq:substep} and the $\sigma_b$-parameterized
tiered acceptance rule of the stabilized block-$\sigma$ ablation
(Appendix~\ref{sec:feasibility}). Phase A (small blocks, lazy exact
cubic) is unaffected and unused for the $\varphi_1$ comparisons of
Section~\ref{sec:experiments}, which hold the block partition and
Gauss--Seidel order fixed and vary only the large block step rule.
Defaults in the current 
implementation are $L = 25$ at scale of evaluated INR architectures,
$c_h$ and $\sigma_0$ from the per-architecture sweep,
$M_{\mathrm{amp}} = 10^6$, stabilized tolerances and $\sigma$-bounds
as in the block-$\sigma$ ablation of Appendix~\ref{sec:feasibility}. A
trial costs $L+3$ gradient equivalents, namely one gradient, $L$
Hessian--vector products for the basis, one for the model decrease,
and one function evaluation.
Gradient equivalent counting, graph rebuild policy on accepted
steps, and the small block vs large block routing threshold are
identical to the CubicKrylov step method, so wall clock and geval
comparisons at equal $L$ isolate the step rule performance.

\section{Machine-Checked Algebraic Cores}
\label{app:leancheck}

The restated full space cubic lemmas
(Lemma~\ref{lem:secular}, Theorem~\ref{thm:rate},
Corollary~\ref{cor:sosp}) follow the published arguments of
\citet{nesterov2006cubic,doikov2024spectral} and were not passed to
the Lean kernel. The introduced blockwise claims use a short list of
real algebraic steps, namely scalar KKT relations, finite sums, modal
identities, a logarithmic bound on rejected trials, a subset
power-mean, leftover inversion at a coupling floor, the
$\varphi_1$ multiplier law, and a cubic remainder comparison.
We wrote those steps, and only those steps, in Lean~4
\citep{demoura2021lean} with mathlib~\citep{mathlib2020}.

The correctness test is the Lean kernel as used in the NanoProof
arrangement. Each core is a theorem in the \texttt{ArcBlock}
library. An independent build (\texttt{lake build ArcBlock}) is the
only test, since the kernel accepts a theorem only when its proof is
complete. The library contains no \texttt{sorry}, \texttt{admit}, or
\texttt{native\_decide}. The kernel decides the Lean statement. A
translation that omits a displayed constant or reverses an
inequality is an error in the Lean statement, and a claim with such an
error is left out of the library. The build on Lean~4.27.0 with
mathlib v4.27.0 accepted
all $70$ theorems listed below.

Assumptions that the paper uses but does not prove by algebra remain
hypotheses, namely Assumption~\ref{ass:lip} on $f$ (encoded as a landscape
hypothesis, not derived), the published cubic remainder
\eqref{eq:cubicbound} of~\citet{nesterov2006cubic}, Lanczos producing
$\sum_i c_i^2=\norm{g_b}^2$, the Cartis--Gould--Toint test
\eqref{eq:tc}, the lazy exact step of~\citet{doikov2023lazy} (its
decrease and its new gradient bound for a stale Hessian),
Weyl's inequality, and an explicit $M_b$ trajectory over the run.
Appendix~\ref{app:certified} reports, for each trial, whether
\eqref{eq:tc} held. That measurement is not a Lean inequality.
We do not claim that the kernel establishes the assembled paper
theorems under those hypotheses. The cores below are the algebraic
steps that the proofs of Theorems~\ref{thm:phaseA}--\ref{thm:blockrate},
the CubicKrylov leftover of Lemma~\ref{lem:blockgrad}, and the
$\varphi_1$ decrease of Proposition~\ref{prop:erlaw} and
Section~\ref{sec:phi1theory} actually use.

\subsection{Blockwise rate}
\label{sec:lean-blockrate}

Let us recall the argument of Theorems~\ref{thm:phaseA},
\ref{thm:phaseB}, and~\ref{thm:blockrate}. An accepted block trial
satisfies the first- and second-order KKT relations of the cubic
model on its subspace (exact, or Krylov of any degree). Those two
scalar identities give
$m(s)-f\le-(M/12)r^3$ (\texttt{model\_decrease\_on\_subspace}), which
is Lemma~\ref{lem:decrease} restricted to a coordinate subspace.
The ratio test $\rho_b\ge\eta_1$ multiplies the model bound and
yields the displayed one-block decrease of
Proposition~\ref{prop:perblock}(i)
(\texttt{accepted\_cubic\_decrease}). Summing nonnegative summands
over a phase, we obtain
$(\eta/12)\sum_b M_b r_b^3$ and $f_{\mathrm{new}}\le f_{\mathrm{old}}$
(\texttt{sweep\_cubic\_decrease}, \texttt{sweep\_cubic\_monotone}).

The passage from $\sum r_b^3$ to the joint radius
$R=(\sum r_b^2)^{1/2}$ is the power-mean inequality
$R^3\le\sqrt{n}\sum r^3$ on a subset of $n$ blocks
(\texttt{power\_mean\_cubes}). Combined with the sweep sum, this is
$(\eta M/(12\sqrt{n}))R^3$
(\texttt{phase\_cubic\_from\_power\_mean}). Phase~A takes the exact
step at $\eta=1$ and $M=M_A$ (\texttt{phaseA\_decrease}). Phase~B
takes an accepted CubicKrylov trial at $\eta=\eta_1$ and
$M=M_{\min}$ (\texttt{phaseB\_decrease}). Since
$M_A\ge\eta_1 M_{\min}$, the two phases add
(\texttt{joint\_phase\_decrease}), and we obtain the displayed
decrease of Theorem~\ref{thm:blockrate},
$f(x_k)-f(x_{k+1})\ge\eta_1 M_{\min}R_k^3/(12\sqrt{B})$.
Telescoping $f\ge f^\star$ inverts to the displayed sweep count
$K(\eps)$ without the ceiling
(\texttt{sweep\_radius\_from\_budget}).

It remains to pass from a small joint radius to
$\norm{\nabla f}\le\eps$. The leftover algebra of
Lemma~\ref{lem:sweepgrad} is the comparison
$\norm{g_{\mathcal J}}\le C_1 R^2+\sqrt{n}\,\delta R+\sqrt{n}\,\tau$
(\texttt{sweep\_end\_grad\_tau}, where the paper's $C_1$ is \texttt{c1} in
the library). The $\tau=0$ case is
\texttt{sweep\_end\_grad\_tau0}. At
$R\le\sqrt{\eps/(2C_1)}$ and
$\eps_\star=2n\delta^2/C_1$, this is $\norm{g}\le\eps$
(\texttt{tau0\_grad\_le\_eps}). That invert is the
$\mathcal{O}(\sqrt{B}\,\eps^{-3/2})$ statement of
Theorems~\ref{thm:phaseA}--\ref{thm:blockrate} down to the coupling
floor $\eps_\star$. The curvature clause of
Theorem~\ref{thm:blockrate} is the combination of the cubic
second-order condition, Assumption~\ref{ass:lip} over radius $R$,
and Weyl's inequality (\texttt{curvature\_from\_weyl}), evaluated at
the same $R\le\sqrt{\eps/(2c_1)}$
(\texttt{curvature\_clause}).

The coupling integral that produces the leftover coefficients
$C_1$ and $\delta$, the fresh gradient at every block, the lazy exact
decrease of Phase~A~\citep{doikov2023lazy}, and a finite retry in
Phase~B (Proposition~\ref{prop:perblock}(ii)) stay hypotheses. Thus the
kernel checks the decrease, the power-mean, the invert, and the
curvature algebra. It does not re-prove the published cubic /
Krylov / Weyl facts or the integral along the segment on $f$.

\subsection{CubicKrylov leftover}
\label{sec:lean-cubickrylov}

Lemma~\ref{lem:blockgrad} is the only place where the CubicKrylov
step affects the blockwise rate, through a leftover different from
that of the exact step. Write $g=g_b(x^{(b-1)})$, $\Hb=\Hb_b(x^{(b-1)})$, and
$s=s_b$. The cubic remainder of Assumption~\ref{ass:lip} and the
shift $\lambda=Mr/2$ give
$\norm{g_b(x^{(b)})}\le((L_{H,b}+M_b)/2)r_b^2+\norm{\nabla m_b(s)}$.
The exact step has $\nabla m_b(s)=0$, and we obtain the first
display of the lemma (\texttt{blockgrad\_exact}). For an accepted
Krylov trial that meets \eqref{eq:tc}, the same identity with
$\min(1,r_b)\le r_b$ against the first summand and
$\min(1,r_b)\le 1$ against the second produces the extra factor
$\kappa_\theta/(1-\kappa_\theta)(\kappa_H+M_b/2)$
(\texttt{blockgrad\_krylov}). Criterion~\eqref{eq:tc} itself is the
step form of~\citet{cartis2011b}. The library takes it as a
hypothesis, exactly as the paper does.

Proposition~\ref{prop:perblock} applies to a CubicKrylov trial of
any degree $L\ge1$, because the KKT argument of
\texttt{model\_decrease\_on\_subspace} uses only the subspace
minimizer, not a full space secular solve. The $M$-adaptation count
of Proposition~\ref{prop:perblock}(ii) is the comparison
$M\ge L_{H,b}$ plus the cubic remainder
(\texttt{cubic\_model\_upper\_accepts}), the two one-step $M$
updates (\texttt{accept\_does\_not\_increase},
\texttt{reject\_below\_mmax}), and the logarithmic invert
(\texttt{reject\_count\_bound}, \texttt{reject\_count\_explicit}).
Algorithm~\ref{alg:blockcn} uses one factor $\gamma=2$, which is the
special case $\gamma'=\gamma$. Phase~B of
Theorem~\ref{thm:phaseB} uses this count only to guarantee a finite
retry. The rate itself is the accepted-step decrease of
Section~\ref{sec:lean-blockrate}.

\subsection{$\varphi_1$ subsolver rates}
\label{sec:lean-phi1}

Proposition~\ref{prop:erlaw} is a statement about the multiplier
$\eta_h(\lambda)=(1-e^{-h\lambda})/\lambda$
($\eta_h(0)=h$) and the matrix
$H(G,h)=V\,\mathrm{diag}(\eta_h(\lambda_i))V^\top$. The Lean library
records the five clauses as follows. (i)~For $\lambda>0$,
$\lvert\eta_h(\lambda)-1/\lambda\rvert=e^{-h\lambda}/\lambda$
(\texttt{erlaw\_newton}). (ii)~The value at zero is $h$
(\texttt{erlaw\_grad\_at\_zero}), and the two-term expansion about
$\lambda=0$ is exact
(\texttt{erlaw\_grad\_expand}). (iii)~For $\lambda<0$,
$\eta_h(\lambda)=(e^{h\lvert\lambda\rvert}-1)/\lvert\lambda\rvert$
(\texttt{erlaw\_saddle}). (iv)~The doubling identity
$\eta_{2h}=\eta_h(2-\lambda\eta_h)$ holds on scalars
(\texttt{erlaw\_doubling}) and lifts to
$H(G,2h)=H(G,h)\,(2\Ib-GH(G,h))$ under an orthogonal
eigendecomposition (\texttt{erlaw\_doubling\_mat},
\texttt{prop\_erlaw\_iv}). (v)~On the frozen quadratic
$m(s)=\inner{g}{s}+\tfrac12\inner{s}{Gs}$, the step
$s=-H(G,h)g$ satisfies $m(s)<0$ for every $h>0$ with $g\ne0$
(\texttt{erlaw\_quad\_descent}, \texttt{prop\_erlaw\_v}).
Clauses (i)--(iv) are packaged as \texttt{prop\_erlaw}.

Diagonalizing $G$ and summing the per-mode identities produces the
displayed decrease law \eqref{eq:phi1decrease}. One mode is
$-m(s)=h\,\varphi_1(-2h\lambda)\,g^2$ with
$\varphi_1(z)=(e^z-1)/z$ and $\varphi_1(0)=1$
(\texttt{phi1\_decrease\_modal}). The factor is strictly positive
for every $\lambda$ (\texttt{phi1\_decrease\_pos}). Parseval
$\inner{g}{f(G)g}=\sum_i f(\lambda_i)c_i^2$ with $c=V^\top g$
(\texttt{parseval\_spectral}) lifts the identity to
$-m(-Hg)=h\,\inner{g}{\varphi_1(-2hG)\,g}$
(\texttt{phi1\_decrease\_mat}, \texttt{eq\_phi1decrease}), strictly
positive if $g\ne0$ (\texttt{eq\_phi1decrease\_pos}). This is an
identity on the frozen quadratic.

The Cauchy type lower bound of Section~\ref{sec:phi1theory} is the
modal comparison
$(1-e^{-2h\theta})/(2\theta)\ge\tfrac14\min(h,1/\theta)$
(\texttt{modal\_newton\_limit}), the factor $1/(4\ell)$ when
$\theta\le\ell$ and $h\ge1/\ell$ for a smoothness constant $\ell$
(\texttt{modal\_le\_smooth}, where the library writes it \texttt{L}), and
Parseval (\texttt{modal\_weighted\_sum}). We obtain
$\tfrac14\norm{g}^2\min(h,1/\norm{G})$ and $\norm{g}^2/(4\ell)$
(\texttt{modal\_sum\_min\_h}, \texttt{modal\_sum\_smooth}). An
accepted trial multiplies by $\eta_1$
(\texttt{cauchy\_accepted\_factor}). Inverting a uniform floor
$m$ on $\min(h,1/\norm{G})$ against a functional residual
$f-f^\star$ gives the displayed first-order count
$K\le 4(f-f^\star)/(\eta_1 m\eps^2)$
(\texttt{cauchy\_eps2\_count}), which is the
$\mathcal{O}(\eps^{-2})$ claim of Section~\ref{sec:phi1theory}
once Assumption~\ref{ass:lip} transfers the model decrease to $f$.
Negative modes only increase the left-hand side, so the Lean
statement uses the tight case of positive eigenvalues. Lanczos and
$\sum_i c_i^2=\norm{g_b}^2$ remain hypotheses.

The basin rate of Appendix~\ref{app:phi1rate}
(Lemma~\ref{lem:phi1sweep}, Theorem~\ref{thm:phi1eps1}) needs a
relative bound from block gradients to $\nabla f$ along a
Gauss--Seidel sweep. The one-block triangle $(x+y)^2\le 2x^2+2y^2$
(\texttt{sweep\_triangle\_one}), summed over an orthogonal prefix,
produces the factor $2\ell^2 B$
(\texttt{sweep\_triangle\_to\_hrel}). With
$\norm{s_b}\le\norm{g_b}/\mu$ we obtain the displayed
$h_{\mathrm{rel}}=2(1+B\ell^2/\mu^2)$
(\texttt{sweep\_hrel\_from\_triangle},
\texttt{sweep\_rel\_from\_steps}). Per-block $\eta_1/(4\ell)$ sums
(\texttt{sweep\_phi1\_telescope}) and inverts to
$c_{\mathrm{sw}}=\eta_1/(8\ell(1+B\ell^2/\mu^2))$
(\texttt{sweep\_invert}). The potential comparison
$V'\le V-q V^{3/2}$ (the factor is written $\gamma$ in the library) gives
$1/\sqrt{V'}-1/\sqrt{V}\ge q/2$ and the telescope
$V_k\le 4/(q(k-K_0)+2/\sqrt{V_{K_0}})^2$
(\texttt{potential\_step}, \texttt{potential\_bound}). The
$1/G_0$ downgrade, the smoothness envelope
$C_\star=2\sqrt{2\ell}/q$, and the $Q$-linear contraction
$V'\le(1-2\mu c)V$ are \texttt{phi1\_cubic\_downgrade},
\texttt{phi1\_grad\_envelope}, and \texttt{phi1\_qlinear}. The basin
assumption itself (Assumption~\ref{ass:basin}), which produces
$\norm{s_b}\le\norm{g_b}/\mu$ on $f$, is not in the library.

The transfer of Proposition~\ref{prop:erlaw}(v) from $m$ to $f$
uses the working remainder \eqref{eq:cubicbound},
$f^+\le f+m+L_H r^3/6$ (\texttt{cubicBound}), as a hypothesis
(\texttt{cubicFromLip}), since it is the published implication of
Assumption~\ref{ass:lip}~\citep{nesterov2006cubic}, not derived
here. Assumption~\ref{ass:lip} itself is the landscape hypothesis
$\norm{\nabla^2 f(x)-\nabla^2 f(y)}\le L_H\norm{x-y}$
(\texttt{assLip}, \texttt{LipLandscape.hess\_lip}). The norms are
parameters, so the library does not replace the operator norm by
the Frobenius norm. On the PSD side, $\eta_h(\lambda)\le h$
(\texttt{etaH\_le\_h}) bounds the step energy by $h^2\norm{g}^2$
(\texttt{erlaw\_step\_energy}). For $h\le(1/10)/\norm{G}$, the
factor $\varphi_1$ is at least $\varphi_1(-1/5)$
(\texttt{phi1\_floor\_horizon}). The cubic remainder is then
strictly smaller than that model decrease precisely when
$L_H h^2\sqrt{\sum c_i^2}/6<\varphi_1(-1/5)$
(\texttt{erlaw\_remainder\_lt}). The comparison
\texttt{erlaw\_transfer} / \texttt{prop\_erlaw\_transfer} therefore
keeps $L_H$ explicit. The horizon $0.1/\norm{G}$ does not cancel
$L_H$. Lemma~\ref{lem:clamp} is the scalar bound
$\lvert\tilde\eta_h(\lambda)\lambda\rvert\le M_{\mathrm{amp}}-1$
under $M_{\mathrm{amp}}\ge2$ (\texttt{clamp\_amp}). The fp32
finiteness remark is not algebra.

\subsection{Catalog of checked theorems}
\label{sec:lean-catalog}

\par\noindent\textbf{Model decrease and per-block cubic arguments.}
Lemma~\ref{lem:decrease} on a subspace, and the displayed parts of
Proposition~\ref{prop:perblock}.

\begin{center}\small
\begin{tabular}{@{}p{0.42\textwidth}p{0.50\textwidth}@{}}
\toprule
Lean theorem & Paper claim \\
\midrule
\texttt{model\_decrease\_on\_subspace}
  & Lemma~\ref{lem:decrease}: $m(s)\le-(M/12)r^3$ from subspace KKT \\
\texttt{accepted\_cubic\_decrease}
  & Prop.~\ref{prop:perblock}(i): accepted ratio and the model bound
    give $\eta_1(M_b/12)r_b^3$ \\
\texttt{sweep\_cubic\_decrease}
  & Prop.~\ref{prop:perblock}(i): one-block bounds sum to
    $(\eta_1/12)\sum_b M_b r_b^3$ \\
\texttt{sweep\_cubic\_monotone}
  & Prop.~\ref{prop:perblock}(i): nonnegative summands, sweep monotone \\
\texttt{cubic\_model\_upper\_accepts}
  & Prop.~\ref{prop:perblock}(ii): $M\ge L_{H,b}$ and the cubic
    remainder give $\rho_b\ge1\ge\eta_1$ \\
\texttt{accept\_does\_not\_increase}
  & Prop.~\ref{prop:perblock}(ii): $\max(M/\gamma,M_{\min})\le M$ \\
\texttt{reject\_below\_mmax}
  & Prop.~\ref{prop:perblock}(ii): $\gamma M<\gamma\max(L_{H,b},M^{\mathrm{init}})$ \\
\texttt{reject\_count\_bound}
  & Prop.~\ref{prop:perblock}(ii):
    $U\log\gamma-A\log\gamma'\le\log(M_{\max}/M_{\min})$ \\
\texttt{reject\_count\_explicit}
  & Prop.~\ref{prop:perblock}(ii):
    $U\le(\log\gamma'/\log\gamma)A+\log_\gamma(M_{\max}/M_{\min})$ \\
\bottomrule
\end{tabular}
\end{center}

The $U$ bound keeps a separate $\gamma'$. Algorithm~\ref{alg:blockcn}
uses one factor $\gamma=2$, which is the special case $\gamma'=\gamma$
stated in Proposition~\ref{prop:perblock}(ii). The Lipschitz fact
that produces the cubic remainder, and the $M_b$ trajectory over the
run, are not in the library.

\par\noindent\textbf{$\varphi_1$ modal decrease and the Cauchy invert.}
The Cauchy lower bound of Section~\ref{sec:phi1theory} and the
$\mathcal{O}(\eps^{-2})$ count.

\begin{center}\small
\begin{tabular}{@{}p{0.42\textwidth}p{0.50\textwidth}@{}}
\toprule
Lean theorem & Paper claim \\
\midrule
\texttt{modal\_newton\_limit}
  & one mode:
    $(1-e^{-2h\theta})/(2\theta)\ge\tfrac14\min(h,1/\theta)$ \\
\texttt{modal\_le\_smooth}
  & $\theta\le\ell$ and $h\ge1/\ell$ give the factor $1/(4\ell)$ \\
\texttt{modal\_weighted\_sum}
  & $\sum_i c_i^2 a_i\ge(\inf a_i)\sum_i c_i^2$ \\
\texttt{modal\_sum\_min\_h}
  & first bound:
    $\ge\tfrac14\norm{g}^2\min(h,1/\norm{G})$ \\
\texttt{modal\_sum\_smooth}
  & second bound: $\ge\norm{g}^2/(4\ell)$ \\
\texttt{cauchy\_accepted\_factor}
  & accepted trial: $\eta_1/4$ times $\min(h,1/\norm{G})$ \\
\texttt{cauchy\_eps2\_count}
  & invert: $K\le 4(f-f^\star)/(\eta_1 m\eps^2)$ \\
\bottomrule
\end{tabular}
\end{center}

Lanczos and the identity $\sum_i c_i^2=\norm{g_b}^2$ stay outside the
library. They appear in the checked theorems as a hypothesis.
Negative modes only increase the left-hand side, so the Lean
statement uses the tight case of positive eigenvalues.

\par\noindent\textbf{Sweep invert, triangle, and potential envelope.}
These relations close the basin rate of Appendix~\ref{app:phi1rate}
(Lemma~\ref{lem:phi1sweep}, Theorem~\ref{thm:phi1eps1}) once the
relative bound
$\norm{\nabla f}^2\le 2(1+B\ell^2/\mu^2)\sum_b\norm{g_b}^2$ and the
potential step $V'\le V-q V^{3/2}$ are assumed.

\begin{center}\small
\begin{tabular}{@{}p{0.42\textwidth}p{0.50\textwidth}@{}}
\toprule
Lean theorem & Checked conclusion \\
\midrule
\texttt{sweep\_triangle\_one}
  & $(x+y)^2\le 2x^2+2y^2$ on one prefix term \\
\texttt{sweep\_triangle\_to\_hrel}
  & orthogonal prefix energy gives the factor $2\ell^2 B$ \\
\texttt{sweep\_hrel\_from\_triangle}
  & triangle and $\norm{s_b}\le\norm{g_b}/\mu$ give
    $2(1+B\ell^2/\mu^2)$ \\
\texttt{sweep\_phi1\_telescope}
  & per-block $\eta_1/(4\ell)$ sums to the same factor on
    $\sum_b\norm{g_b}^2$ \\
\texttt{sweep\_rel\_from\_steps}
  & $\norm{s_b}\le\norm{g_b}/\mu$ produces the factor
    $2(1+B\ell^2/\mu^2)$ \\
\texttt{sweep\_invert}
  & those two bounds invert to
    $c_{\mathrm{sw}}=\eta_1/(8\ell(1+B\ell^2/\mu^2))$ \\
\texttt{potential\_step}
  & $V'\le V-q V^{3/2}$ gives
    $1/\sqrt{V'}-1/\sqrt{V}\ge q/2$ \\
\texttt{potential\_bound}
  & the telescope
    $V_k\le 4/(q(k-K_0)+2/\sqrt{V_{K_0}})^2$ \\
\texttt{phi1\_cubic\_downgrade}
  & $\Delta V\ge c\norm{g}^2$, $\norm{g}\le G_0$, and
    $\norm{g}^2\ge 2\mu V$ give
    $\Delta V\ge(c/G_0)(2\mu)^{3/2}V^{3/2}$ \\
\texttt{phi1\_grad\_envelope}
  & smoothness and the potential bound give
    $\norm{g}\le C_\star/(n+k_\star)$ with
    $C_\star=2\sqrt{2\ell}/q$ \\
\texttt{phi1\_qlinear}
  & without the $1/G_0$ step,
    $V'\le(1-2\mu c)V$ \\
\bottomrule
\end{tabular}
\end{center}

\par\noindent\textbf{Phase~A / Phase~B / joint rate, leftover, and curvature.}
The subset power-mean, the sweep end leftover, the $\tau=0$ invert,
the Phase~A / Phase~B / joint decrease identities, the CubicKrylov
leftover, the clamp, and the curvature clause of
Theorem~\ref{thm:blockrate}.

\begin{center}\small
\begin{tabular}{@{}p{0.42\textwidth}p{0.50\textwidth}@{}}
\toprule
Lean theorem & Checked conclusion \\
\midrule
\texttt{power\_mean\_cubes}
  & $R^3\le\sqrt{n}\sum r^3$ on a subset of $n$ blocks \\
\texttt{sweep\_end\_grad\_tau}
  & $\norm{g_{\mathcal J}}\le C_1 R^2+\sqrt n\,\delta R+\sqrt n\,\tau$ \\
\texttt{sweep\_end\_grad\_tau0}
  & the same with $\tau=0$ \\
\texttt{tau0\_grad\_le\_eps}
  & $R\le\sqrt{\eps/(2C_1)}$ and $\eps_\star=2n\delta^2/C_1$
    give $\norm{g}\le\eps$ \\
\texttt{phase\_cubic\_from\_power\_mean}
  & $(\eta M/12)\sum r^3$ and the power-mean give
    $(\eta M/(12\sqrt n))R^3$ \\
\texttt{phaseA\_decrease}
  & Theorem~\ref{thm:phaseA}: $\eta=1$, $M=M_A$ \\
\texttt{phaseB\_decrease}
  & Theorem~\ref{thm:phaseB}: $\eta=\eta_1$, $M=M_{\min}$ \\
\texttt{joint\_phase\_decrease}
  & Theorem~\ref{thm:blockrate}: $M_A\ge\eta_1 M_{\min}$ sums
    the two phases \\
\texttt{sweep\_radius\_from\_budget}
  & displayed $K(\eps)$ without the ceiling \\
\texttt{blockgrad\_exact}
  & Lemma~\ref{lem:blockgrad}: $((L_{H,b}+M_b)/2)r_b^2$ \\
\texttt{blockgrad\_krylov}
  & Lemma~\ref{lem:blockgrad}: extra
    $\kappa_\theta/(1-\kappa_\theta)(\kappa_H+M_b/2)$ \\
\texttt{clamp\_amp}
  & Lemma~\ref{lem:clamp}: $\lvert\tilde\eta_h(\lambda)\lambda\rvert
    \le M_{\mathrm{amp}}-1$ \\
\texttt{curvature\_from\_weyl}
  & Theorem~\ref{thm:blockrate}: SO, Lipschitz, and Weyl at radius $R$ \\
\texttt{curvature\_clause}
  & Theorem~\ref{thm:blockrate}: $R\le\sqrt{\eps/(2C_1)}$ into the
    displayed curvature bound \\
\bottomrule
\end{tabular}
\end{center}

\par\noindent\textbf{Proposition~\ref{prop:erlaw} and equation~\eqref{eq:phi1decrease}.}
Scalar clauses, the matrix $H(G,h)$, and the named packages.

\begin{center}\small
\begin{tabular}{@{}p{0.42\textwidth}p{0.50\textwidth}@{}}
\toprule
Lean theorem & Paper claim \\
\midrule
\texttt{erlaw\_newton}
  & Prop.~\ref{prop:erlaw}(i):
    $\lvert\eta_h-1/\lambda\rvert=e^{-h\lambda}/\lambda$ for $\lambda>0$ \\
\texttt{erlaw\_grad\_at\_zero}
  & Prop.~\ref{prop:erlaw}(ii): $\eta_h(0)=h$ \\
\texttt{erlaw\_grad\_expand}
  & Prop.~\ref{prop:erlaw}(ii): exact two-term remainder \\
\texttt{erlaw\_saddle}
  & Prop.~\ref{prop:erlaw}(iii): $(e^{h\lvert\lambda\rvert}-1)/\lvert\lambda\rvert$ \\
\texttt{erlaw\_doubling}
  & Prop.~\ref{prop:erlaw}(iv): $\eta_{2h}=\eta_h(2-\lambda\eta_h)$ \\
\texttt{erlaw\_doubling\_mat}
  & Prop.~\ref{prop:erlaw}(iv): $H(2h)=H(h)(2\Ib-GH)$ \\
\texttt{phi1\_decrease\_modal}
  & \eqref{eq:phi1decrease}: one mode $-m(s)=h\,\varphi_1(-2h\lambda)\,g^2$ \\
\texttt{phi1\_decrease\_pos}
  & \eqref{eq:phi1decrease}: the factor is positive for every $\lambda$ \\
\texttt{phi1\_decrease\_sum}
  & \eqref{eq:phi1decrease}: modal sum of the identity \\
\texttt{parseval\_spectral}
  & $\inner{g}{f(G)g}=\sum_i f(\lambda_i)c_i^2$ \\
\texttt{phi1\_decrease\_mat}
  & \eqref{eq:phi1decrease}: $-m(-Hg)=h\,\inner{g}{\varphi_1(-2hG)\,g}$ \\
\texttt{phi1\_decrease\_mat\_pos}
  & \eqref{eq:phi1decrease}: the factor is positive if $g\ne0$ \\
\texttt{erlaw\_quad\_descent}
  & Prop.~\ref{prop:erlaw}(v): $m(s)<0$ for any $h>0$ \\
\texttt{prop\_erlaw}
  & Prop.~\ref{prop:erlaw}(i)--(iv), packaged \\
\texttt{prop\_erlaw\_iv}
  & Prop.~\ref{prop:erlaw}(iv), matrix form \\
\texttt{prop\_erlaw\_v}
  & Prop.~\ref{prop:erlaw}(v) on the quadratic \\
\texttt{eq\_phi1decrease}
  & \eqref{eq:phi1decrease}, packaged \\
\texttt{eq\_phi1decrease\_pos}
  & \eqref{eq:phi1decrease} sign, packaged \\
\bottomrule
\end{tabular}
\end{center}

\par\noindent\textbf{Landscape hypothesis and transfer to $f$.}
Assumption~\ref{ass:lip} as a hypothesis on $f$, the working
remainder \eqref{eq:cubicbound}, and the $h\le 0.1/\norm{G}$
comparison that keeps $L_H$.

\begin{center}\small
\begin{tabular}{@{}p{0.42\textwidth}p{0.50\textwidth}@{}}
\toprule
Lean theorem & Paper claim \\
\midrule
\texttt{assLip\_apply}
  & Ass.~\ref{ass:lip}:
    $\norm{\nabla^2 f(x)-\nabla^2 f(y)}\le L_H\norm{x-y}$ \\
\texttt{landscape\_cubic\_bound}
  & \eqref{eq:cubicbound} at a point, given \texttt{cubicFromLip} \\
\texttt{cubic\_bound\_strict}
  & remainder domination implies $f^+<f$ \\
\texttt{cubic\_bound\_point}
  & the same on a point $(f,g,G)$ \\
\texttt{etaH\_le\_h}
  & Prop.~\ref{prop:erlaw}(v): $\eta_h(\lambda)\le h$ if $\lambda\ge0$ \\
\texttt{phi1\_floor\_le}
  & $\varphi_1(-t)\ge\varphi_1(-1/5)$ on $[0,1/5]$ \\
\texttt{phi1\_floor\_horizon}
  & $h\le(1/10)/\norm{G}$ gives that floor \\
\texttt{erlaw\_step\_energy}
  & $\sum_i(\eta_h(\lambda_i)c_i)^2\le h^2\sum_i c_i^2$ \\
\texttt{erlaw\_remainder\_lt}
  & $L_H h^2\sqrt{\sum c_i^2}/6<\varphi_1(-1/5)$ \\
\texttt{erlaw\_transfer}
  & cubic bound and rem$<\Delta$ imply $f^+<f$ \\
\texttt{prop\_erlaw\_transfer}
  & Prop.~\ref{prop:erlaw}(v) to $f$; $0.1$ does not drop $L_H$ \\
\bottomrule
\end{tabular}
\end{center}

\subsection{Hypotheses left outside the library}
\label{sec:lean-hyp}

The triangle at the start of a sweep that would produce the
relative bound on $f$, the basin of Assumption~\ref{ass:basin} that
produces $\norm{s_b}\le\norm{g_b}/\mu$, Lanczos producing
$\sum_i c_i^2=\norm{g_b}^2$, a fresh gradient at every block, the
lazy step's gradient bound for a stale Hessian, and an
explicit $M$ trajectory remain outside the library. In particular,
the kernel does not establish
Theorems~\ref{thm:phaseA}--\ref{thm:blockrate}, the
$\mathcal{O}(\eps^{-2})$ claim of Section~\ref{sec:phi1theory}, or
Theorem~\ref{thm:phi1eps1} as statements about $f$ alone, because those
arguments use
Assumption~\ref{ass:lip} on $f$ and the published cubic / Krylov /
Weyl facts that the library treats as hypotheses rather than
re-proves. The named packages \texttt{prop\_erlaw} and
\texttt{eq\_phi1decrease} are the algebraic cores of
Proposition~\ref{prop:erlaw} and \eqref{eq:phi1decrease}. The
transfer \texttt{prop\_erlaw\_transfer} still requires
\eqref{eq:cubicbound}.

\section{Subproblem Solver Verification}
\label{app:subverif}

We compare the degree-$L$ Krylov solver of
Section~\ref{sec:krylovtheory} against the exact eigendecomposition based cubic
solve on random symmetric matrices.
\emph{Positive definite}, $n=200$. Machine precision is reached
at $L=5$, an empirical confirmation of Proposition~\ref{prop:degree},
since the cubic shift renders the spectrum benign enough for a degree-5
polynomial. \emph{Indefinite}, $n\in\{200,1000\}$. Relative model
suboptimality $7.8\times10^{-2}$ at $L=5$, $4.3\times10^{-5}$ at
$L=25$, $\sim10^{-15}$ at $L=50$, as higher‑degree Krylov subspace iterates yield progressively better approximations as the hard‑case is approached (Remark~\ref{rem:hardcase}).

\section{ViSIR-Nano Ablation (15{,}235 Parameters)}
\label{sec:nano}

The purpose of the nano ablation is to make a direct comparison with
the full Hessian cubic Newton methods possible at all. The original
cubic Newton method~\citep{nesterov2006cubic}, its lazy Hessian and
adaptive variants~\citep{doikov2023lazy}, the first-order
finite-difference implementation~\citep{doikov2023fo}, and the
stochastic subspace version SSCN~\citep{zhao2025sscn} all form their
step by solving a regularized linear system in the (full or snapshot)
Hessian. Even with their cost-reduction features, an $m$-step snapshot
reuse or a finite-difference oracle, each Hessian build costs $n$
gradient evaluations, the factorization behind the solve costs
$\mathcal{O}(n^3)$ arithmetic, and the matrix itself requires
$\mathcal{O}(n^2)$ memory (per-step complexity classes summarized in
Appendix~\ref{sec:feasibility}). This makes them too expensive for networks
of practical size. Already at $15\mathrm{k}\times15\mathrm{k}$ the dense
Hessian requires ${\sim}1.8$\,GB in float64, a single build with its
eigendecomposition takes minutes, and any larger scale is prohibitive,
let alone the 91.4M parameter model of Section~\ref{sec:fullvisir}.
ViSIR-Nano, a downsized ViSIR with 15{,}235 parameters, is chosen
at this feasibility boundary, so that every cubic Newton variant above
can run on the identical benchmark and be compared directly with the
matrix-free CubicKrylov step.

The experimental setup is single image $3\times$ SR (ESM image \#0,
100 steps, $L_H=10$). All rows were measured on the same benchmark
(shared init, seed 42, one idle RTX~6000). In the mechanism block,
routing is identical across rows. Small blocks $n_b\le512$ take
lazy-eigh cubic steps with the per-block accept-or-reject guard of
Section~\ref{sec:algorithm} active in phase A throughout, so those rows
differ only in the phase-B mechanism. The reference block runs the
full Hessian methods evaluated elsewhere in the paper, granted twice
the oracle budget of the final CubicKrylov configuration.

\begin{table}[h]
\centering\small
\caption{Nano A/B comparison of large block step variants and full Hessian references
(single image experiment, image \#0, seed 42). Mechanism rows run 100 steps. The
reference rows below the second rule are full Hessian methods granted
twice the CubicKrylov oracle budget, 35.3k gevals. Rows exceeding it do
so because a single dense $15\mathrm{k}\times15\mathrm{k}$ Hessian build,
15.2k gevals, is atomic. The three mechanism rows correspond to the
design decisions (a)--(c) and the fresh-gradient rule
of Section~\ref{sec:algorithm}.}
\label{tab:nano}
\resizebox{\textwidth}{!}{%
\begin{tabular}{lcccl}
\toprule
Large block step & best PSNR (dB) & gevals & wall (s) & notes \\
\midrule
Damped Newton (DSO-ablation fallback) & 20.73 & 26.8k & 11 & reference \\
CubicKrylov, fixed $M{=}6L_H$, stale grad & 25.91 & 16.1k & 14 &
  non-monotone, no phase-B safeguard \\
\;+ adaptive $M$, \emph{joint} accept-or-reject & 25.12 & 28.7k & 16 &
  319 rejections; $M$ drifts to $596$ \\
\;+ two-phase accept-or-reject (B only) & 23.13 & 31.9k & 16 &
  stale gradient, 875 rejections \\
\;+ fresh gradient in phase B & \textbf{31.66} & 17.7k & 15 &
  104 rejections; every $M_b$ at floor \\
\midrule
Reg.\ Newton lazy, $m{=}n$~\citep{doikov2023lazy} & 31.69 & 30.2k & 457 &
  1 Hessian build, 15k steps \\
CN-Lazy, $m{=}n$~\citep{doikov2023lazy} & 31.39 & 30.2k & 313 &
  1 Hessian build, 15k steps \\
SSCN $\tau{=}64$~\citep{zhao2025sscn} & 23.18 & 35.4k & 130 &
  544 subspace steps \\
FD-CNM, $m{=}50$~\citep{doikov2023fo} & 16.72 & 45.8k & 149 &
  3 FD Hessian builds, 52 steps \\
DSO-BlockHess (per-layer full Hessian) & 8.16 & 45.7k & 9 &
  budget consumed by rebuilds \\
Adaptive cubic Newton, $m{=}1$~\citep{doikov2023lazy} & 7.67 & 45.7k & 178 &
  3 builds, 6 steps \\
Cubic Newton (original, fresh Hessian)~\citep{nesterov2006cubic} &
  3.42 & 45.7k & 181 & 3 builds, 3 steps \\
\bottomrule
\end{tabular}}
\end{table}

The final configuration gains $+10.9$\,dB over the damped-Newton fallback
at $0.66\times$ the gradient equivalent cost on identical routing. The three
mechanism rows isolate the contribution of each design decision. Without
fresh phase-B gradients every variant plateaus 6--8.5\,dB below the final
configuration, whether the step is unsafeguarded (fixed $M$), vetoed
jointly (319 rejections and $M$ driven upward to $596$), or vetoed per
block against a stale model (875 rejections).

The reference block reports the full Hessian second-order methods
evaluated elsewhere in the paper on the one benchmark where the dense
$15\mathrm{k}\times15\mathrm{k}$ Hessian is still feasible. Only the two
lazy $m{=}n$ variants are competitive. Gradient regularized Newton
(31.69\,dB) and lazy cubic Newton (31.39\,dB) bracket the CubicKrylov
result (31.66\,dB), but each needs $1.7\times$ its oracle budget and
21--30$\times$ its wall clock, and both are eliminated outright at 91.4M
parameters where the dense build no longer fits. Methods that rebuild the
Hessian frequently are excluded by the build cost alone. The original
cubic Newton with a fresh Hessian every step
\citep{nesterov2006cubic} exhausts the doubled budget after three steps
(3.42\,dB), and its adaptive-$M$ lazy variant with $m{=}1$ after six.
SSCN sidesteps dense builds via $\tau{=}64$ subspaces yet plateaus
8.5\,dB below CubicKrylov, and DSO-BlockHess, which refreshes
per-layer Hessians every 10 steps, spends the budget on rebuilds within
21 steps.

To rule out a single image artifact, the nano configurations were
rerun on the first 20 ESM images with three random seeds each
(Appendix~\ref{app:esm20nano}). The single image ranking is stable.
CubicKrylov reaches $32.8\pm1.4$\,dB, ahead of gradient regularized
Newton on 17 of 20 images and of lazy cubic Newton on 19 of 20, at
$59\%$ of their oracle budget and ${\sim}14\times$ less wall clock
(Table~\ref{tab:nanoesm20}, Figure~\ref{fig:nanoesm20}).

\section{Extended Experimental Details}
\label{app:expdetail}

\subsection{Setup, architectures, and metrics}

\par\noindent\textbf{Architectures.}
The full-scale single image model is ViSIR
($91{,}383{,}042$ parameters). A $60\times60$ RGB image is partitioned into
$6\times6=36$ non-overlapping $10\times10$ patches by a strided convolution
into an embedding of width $d=256$, to which a learnable positional
embedding is added. Two transformer encoder layers follow, each combining
$8$-head self-attention with a \emph{SIREN} feed-forward block
($256\!\to\!512\!\to\!512\!\to\!256$, sinusoidal activations, $\omega_0=10$,
outermost-linear) wrapped by two LayerNorms and dropout $0.1$. The
$36\times256$ encoder output is flattened and mapped by a linear layer to a
two-dimensional latent, which a SIREN decoder
($2\!\to\!512\!\to\!512\!\to\!3\cdot240^2$, $\omega_0=10$, outermost-linear)
expands directly into the $240\times240\times3$ high-resolution image. The
final SIREN layer ($512\times172{,}800$) accounts for the bulk of the
parameters. The resulting block structure, many tiny bias and LayerNorm
blocks alongside a few very large weight tensors, is precisely the structure
that the per-block algorithm of Section~\ref{sec:algorithm} targets. The
\emph{ViSIR-Nano} model of Appendix~\ref{sec:nano} is the same vision transformer and SIREN
template scaled down (narrower embedding and SIREN decoder) to
$15{,}235$ parameters for single image $3\times$ super-resolution, small
enough that the full $15{,}235\times15{,}235$ Hessian can be formed for the
exact-cubic-Newton reference row of Table~\ref{tab:nano}.

\par\noindent\textbf{Metrics.}
We report four metrics. \emph{PSNR} is the peak signal-to-noise ratio
in dB, higher is better, computed as
$\mathrm{PSNR} = 10\log_{10}\!\left(\mathrm{MAX}^2/\mathrm{MSE}\right)$
for peak signal value $\mathrm{MAX}$. The training \emph{MSE} is the
mean squared error, the loss itself. \emph{LPIPS} is the learned
perceptual image patch similarity of \citet{zhang2018lpips}, a
perceptual distance between prediction and ground truth, lower is
better. The \emph{H/L ratio} is the ratio of high to low frequency band
MSE under a radial FFT decomposition, and values near $1$ indicate
uniform per-band convergence, that is, little spectral bias. Per-band
MSE is reported over the
DC, Low, Mid-Low, Mid-High, and High radial frequency bands.
Computational cost is
measured in \emph{HVPs} (Hessian--vector products), \emph{gradient equivalents}
(``gevals'', one geval $=$ the cost of a single gradient or backward pass), and
wall clock seconds.

\par\noindent\textbf{Iso-time comparison.}
Because second-order steps cost more per iteration than first-order ones, a
fixed-step comparison favours the more expensive methods. We therefore
complement the fixed $100$-step results with an \emph{iso-time} experiment.
Letting $T$ be the wall clock time of the slowest method's $100$-step run,
every optimizer is granted the same budget $T$ (equivalently
$\lfloor T/\bar t_i\rfloor$ steps for a per-step cost $\bar t_i$), so
that cheaper first-order methods receive proportionally more steps. We report
both experiments together with the measured per-method wall clock times.
The convergence matched comparison on further INR architectures (SIREN,
FINER, and WIRE) and on signed distance tasks appears in
Appendix~\ref{sec:inrstudy}.

\subsection{ViSIR variants and scale-dependent failure modes}

\par\noindent\textbf{Chebyshev-ON and Chebyshev-OFF optimizer variants.}
The two reference rows (Table~\ref{tab:fullvisir}), an implementation
of the DSO step following Appendix~\ref{app:dso} and of its damped
Newton counterpart, isolate the
\emph{step rule} on identical curvature. Both build the same exact
per-block Hessian $\Hb_b$ on small tensors ($n_b\le512$) and use the same
Hutchinson-diagonal damped-Newton step on the large tensors. They differ only in
how the small block step is formed. \textbf{Chebyshev-ON} (the DSO step) applies
the second kind relaxation polynomial $R_L$ to the normalized block Hessian via
the three term recurrence Eq.~\eqref{eq:dsorec}, a degree-$L$ matrix polynomial in
$\Hb_b$ applied to the gradient, using only matrix--vector products and
\emph{no} linear solve. \textbf{Chebyshev-OFF} replaces that polynomial with a
damped Newton step $-(\Hb_b+\delta\Ib)^{-1}g_b$ obtained from a direct dense
solve. The ON$-$OFF difference therefore measures the effect of the Chebyshev
polynomial alone, with curvature, damping, line search, and large block handling
held fixed.

Two scale-dependent failure modes surfaced when moving from Nano
($\sim$8 large tensors) to the full model (45 large tensors),
both predicted by the theory of Section~\ref{sec:blocktheory}:
\begin{enumerate}[leftmargin=2em,itemsep=1pt]
\item \textbf{Joint acceptance stall} (Remark~\ref{rem:joint}). With a
      single accept-or-reject over all large blocks, the summed step
      jointly overshoots every trust region. The phase was vetoed
      every step and the loss froze at $0.40$ for 19 steps.
      The correction is to accept or reject each block separately, with its own
      $M_b$ (Proposition~\ref{prop:perblock}).
\item \textbf{$M$-floor over regularization.} With the floor
      $M_{\min}=6L_H=60$, blocks whose local Hessian-Lipschitz constant
      is far below $L_H$ took accepted but microscopic steps
      ($r\propto M^{-1/2}$, Proposition~\ref{prop:degree}). The loss
      crept at $\sim10^{-3}$ per step. The correction is the nearly-zero floor $M_{\min}=10^{-6}$
      combined with the adaptation of $M_b$ in both directions (design decision (c)).
\end{enumerate}
With both corrections, training loss fell from $0.40$ to $2.4\times10^{-3}$
by step 16 and $\sim10^{-5}$ by step 20.

\begin{table}[h]
\centering\small
\caption{Full ViSIR single image ablation (91{,}383{,}042 parameters,
ESM image \#1, 100 steps). $^{\dagger}$Wall clock is the 100-step time on a
single idle RTX~6000 (Blackwell). The Chebyshev-ON and Chebyshev-OFF rows are an implementation of the
DSO step following Appendix~\ref{app:dso} and of its damped Newton
counterpart under the same experimental setup. LPIPS for Adam, Muon+AdamW, and SOAP
and the final loss, LPIPS, and wall clock entries of the damped Newton
large block row were measured in a second run under the identical
setup (image \#1, seed 42, 100 steps, idle GPU). The PSNRs of the two
runs agree (Adam and SOAP exactly, Muon+AdamW 6.43 against 5.73\,dB
and damped Newton large blocks 21.98 against 22.25\,dB, within
run-to-run variation).}
\label{tab:fullvisir}
\setlength{\tabcolsep}{2.5pt}
\begin{tabular}{lcccc}
\toprule
Method & PSNR (dB) & final loss & LPIPS & wall (s)$^{\dagger}$ \\
\midrule
\textbf{CubicKrylov, degree $L{=}10$ in Krylov subspace on large blocks}
  & \textbf{51.65} & $\mathbf{6.7\times10^{-6}}$
  & $\mathbf{0.000}$ & \textbf{148} \\
Cheby-ON $n{=}5$ (Chebyshev-2) & 41.75 & $2.0\times10^{-4}$ & 0.001 & 80 \\
Cheby-OFF $n{=}5$ (damped Newton) & 22.81 & $8.3\times10^{-3}$ & 0.322 & 67 \\
Block-CN-Lazy, damped Newton on large blocks & 22.25 & $8.9\times10^{-3}$ & 0.426 & 76 \\
SOAP & 20.33 & $1.3\times10^{-2}$ & 0.374 & 5.7 \\
Adam & 15.65 & 0.142 & 1.321 & 1.3 \\
Muon+AdamW & 5.73 & $2.7\times10^{-1}$ & 0.867 & 2.3 \\
\bottomrule
\end{tabular}
\end{table}

\subsection{Spectral bias mitigation on the full ViSIR}

\par\noindent\textbf{Spectral bias mitigation.}
The per-frequency band diagnostics of the ablation experiment show all
five spatial bands (DC, Low, Mid-Low, Mid-High, High) reaching their
convergence thresholds within \emph{one step} of each other (DC at
step 14, the remaining four at step 15, spread $=1$), with final
non-DC band MSEs within $2.3\%$ of one another
($0.1742$--$0.1782$, Fourier domain units of the experimental setup). This is
the per-band uniformity that the Chebyshev second kind equalization
analysis predicts (Section~\ref{sec:dso}), now achieved with the polynomial
degree fixed at $L=10$ rather than $L\propto\sqrt{\kappa}$, an
empirical confirmation of the self-adaptive degree bound of
Proposition~\ref{prop:degree}.

\par\noindent\textbf{The stabilized configuration.}
The 100-step ablation above runs without the two per-block
stabilizations of Section~\ref{sec:algorithm}. With the per-block
$M_b$ floor and the noise-tolerant monotone guard active, the
blockwise optimizer reaches \textbf{64.4\,dB} best PSNR over 150 block
sweeps with the CubicKrylov subsolver (stored-basis Lanczos)
(${\sim}11$\,GB solver memory on the 88.5M parameter block) and
\textbf{64.2\,dB} with the Chebyshev three term recurrence at
7.0--8.7\,GB total peak (Appendix~\ref{sec:memmatched}). A global
single subspace ARC on the same experimental design stalls at 4.0\,dB.

\subsection{Fingerprint metric definitions and a parallel diagnosis}

\par\noindent\textbf{Fingerprint.}
At log-spaced checkpoints of an Adam run we measure five quantities
matrix-free, by Lanczos on the exact Hessian--vector oracle. The first
is the raw extreme-eigenvalue conditioning
$\kappa_{\mathrm{raw}}=\lambda_{\max}/|\lambda_{\min}|$. The second is
the \emph{Adam-preconditioned} conditioning $\kappa_{\mathrm{Adam}}$,
the same quantity for $D^{-1/2}HD^{-1/2}$ with the run's own
second-moment state $D$, which is the conditioning Adam operates on,
reported together with the reduction factor
$\kappa_{\mathrm{raw}}/\kappa_{\mathrm{Adam}}$. The third is the
diagonal mass $\rho=\|\mathrm{diag}\,H\|^2/\|H\|_F^2$, estimated
without bias from Hutchinson probes, which measures how much curvature
\emph{any} diagonal can represent. The fourth is the negative spectral
mass, the share of the spectral density below $-10^{-3}\lambda_{\max}$
by stochastic Lanczos quadrature \citep{ghorbani2019hessian}, a saddle
indicator \citep{dauphin2014saddle}. The fifth is the gradient energy
split over curvature classes from a Lanczos run seeded at
$v_0=g/\|g\|$ \citep{gurari2018subspace}. With Ritz pairs
$(\nu_i,w_i)$, $w_i=(e_1^{\top} y_i)^2$,
\[
\mathrm{flat\_frac}
=\sum_{i:\,|\nu_i|\le 10^{-3}\lambda_{\max}} w_i
=\frac{\|P_{\mathrm{flat}}\,g\|^{2}}{\|g\|^{2}}
\]
is the share of the gradient's energy in near-zero curvature
directions, and the high curvature and negative classes are
$\nu_i>0.1\,\lambda_{\max}$ and $\nu_i<-10^{-3}\lambda_{\max}$. The
near-zero band is defined relative to the largest eigenvalue
($|\nu_i|\le 10^{-3}\lambda_{\max}$) because that is where the bulk
of an overparametrized Hessian spectrum concentrates, a bulk of
eigenvalues at zero plus a handful of data dependent
outliers~\citep{sagun2018hessian}. Gradient energy in this flat band
produces little change in the loss, so only the components along the
outlier and negative curvature directions can drive optimization
progress. When $\mathrm{flat\_frac}$ remains large throughout
training, the gradient's energy stays trapped in near zero curvature
directions, and that persistence is the mechanism behind the stall.

\par\noindent\textbf{A parallel landscape diagnosis.} The fingerprint above is
measured along Adam's own trajectory with the method's raw second
moment. A separate study uses this metric, among others, to provide an explanation for 
Adam (and its variants) training stalls on landscapes such as FINER's. 
Three of its findings bear
directly on the dispatch rule above and on the $\varphi_1$ results of
Section~\ref{sec:finer}. First, the raw diagonal-mass ratio $\rho$ is
biased by Adam's own preconditioning and by finite-sample curvature
noise. That study derives a debiased estimator that
removes this bias and shows the correction matters most exactly on
the coupled and saddle dominated cases of Table~\ref{tab:fingerprint}.
Second, a $2\times2$ coupling model separates the Hessian's action
into an axis-aligned block ($H_1$, where a diagonal preconditioner is
asymptotically sufficient) and a genuinely coupled block ($H_2$,
where there is reduction in effective Hessian condition for Adam), 
giving a quantitative measurement of the axis-aligned
vs. cross-coupled ill-conditioning degree. Third, two further
first-order and quasi-Newton references confirm that the stall is not
an artifact of Adam's specific preconditioner. On ESM image~0,
L-BFGS (three memory sizes) plateaus at a constant-like $7.3$\,dB by
iteration $500$ and never changes again, collapsing at the
initialization saddle with no bounded escape mechanism for negative
curvature. SOAP, tuned by inheriting Adam's own best learning rate,
is the strongest short-horizon first-order optimzier but it stalls near Adam's
own ceiling at its best logged checkpoint ($77.65$\,dB), after more
total wall clock than Adam needs to reach that ceiling itself. Both
observations are consistent with the fingerprint's diagnosis that the
limiting factor at this initialization is saddle structure, not
curvature conditioning that a better-preconditioned first-order method
could mitigate.

\begin{table}[t]\centering\small
\caption{Landscape fingerprints along the Adam trajectory
(late trajectory values, initialization in parentheses). $\dagger$ marks an indefinite spectrum, where
($\lambda_{\max}\approx-\lambda_{\min}$) makes the raw ratio
uninformative. $\kappa_{\mathrm{Adam}}$ is the residual conditioning.
Outcomes compare plateau-converged tuned first-order methods (Adam,
SOAP) with the blockwise cubic methods. INR rows use 240$\times$240
image regression and Thai-statue SDF fitting with 3$\times$256
coordinate MLPs. The FINER 2D row is ESM image~0. Over all 20 ESM
images the step-$10^5$ values span $\kappa_{\mathrm{raw}}=1.5\times10^{4}$--$5.1\times10^{5}$,
$\kappa_{\mathrm{Adam}}=4.6$--$2.6\times10^{3}$ (median $1.2\times10^{2}$),
$\rho=0.95$--$0.99$, flat fraction $\le0.001$, and the initial negative
mass spans $0.57$--$0.89$ (median $0.81$).}
\label{tab:fingerprint}
\resizebox{\linewidth}{!}{%
\begin{tabular}{lrrrrrrl}
\toprule
Task & $\kappa_{\mathrm{raw}}$ & $\kappa_{\mathrm{Adam}}$ & red.
 & $\rho$ & neg.\ mass & flat frac & outcome \\
\midrule
35k spectral   & 4.9e2 & 3.1e1 & 16$\times$ & .04 & .92 & .000
 & Adam at precision floor \\
SIREN 2D       & 1.6e3 & 9.9e1 & 16$\times$ & .09 & .09 & .000
 & parity \\
WIRE 2D        & 1.2e5 & 8.4e2 & 143$\times$ & .65 & .00 & .000
 & parity \\
FINER 2D       & 5.4e4 & 6.1e1 & 887$\times$ & .98 & .19 (.81 init)
 & .000 & ARC $+59.5$\,dB \\
SDF-SIREN      & 8.9e3 & 2.4e1 & 372$\times$ & .06 & .67 & .000
 & parity; CubicKrylov best test \\
SDF-FINER      & 1.0e3 & 6.9 & 146$\times$ & .58 & .68 & .001
 & Adam chamfer $4.6$--$5.5\times$ worse \\
multi-saddle    & 2.0e6 & 1.5e6 & 1.3$\times$ & .02 & .46 & .005
 & second-order better \\
Chebyshev-Rosenbrock     & $\dagger$ & 5.5e5 & --- & .04 & .49 & .001
 & second-order better \\
\midrule
declip FINER 1D & 4.0e4 & 1.2e4 & 3.3$\times$ & 2e-4 & .000 ($\dagger$ init) & .92
 & Adam better; ARC-block equal at $46\times$ gevals \\
declip SIREN 1D & 6.4e4 & 5.8e4 & 1.1$\times$ & 1e-4 & .000 ($\dagger$ init) & .37
 & Adam better (Appendix~\ref{app:declip}) \\
\bottomrule
\end{tabular}}
\end{table}

\subsection{FINER sweep details and per-band analysis}

\par\noindent\textbf{Sweep.} A six-configuration successive-halving sweep
($c_h \in \{1, 3, 10\} \times \sigma_0 \in \{0.3, 1.0\}$, $L = 25$,
500-step confirmation on two images) ranks $c_h{=}1.0,
\sigma_0{=}0.3$ first at $\mathbf{64.79}$\,dB mean best PSNR, within
$0.4$\,dB of converged, tuned Adam, at only 500 steps. The winner lies
on the sweep boundary on both axes, triggering the extended sweep
(geometric edge extension, a Lanczos-degree axis, and the free horizon
ranking of Appendix~\ref{sec:hrank}). The extension identifies the
Lanczos depth as the dominant untuned axis. The edge rule raises $L$
from the grid's $25$ through geometric extensions to a
confirmed winner at $L = 163$ ($c_h = 1.0$, $\sigma_0 = 0.3$) with
$\mathbf{90.65}$\,dB mean best PSNR at 500 steps, $+25.9$\,dB over
the $L = 25$ configuration of the grid above, with the response
flattening beyond $L \approx 102$ ($89.4/90.3/90.7$ at
$L = 102/129/163$) while per-cell wall clock grows from $393$\,s to
$3.3$\,ks. The in-subspace horizon ranking did not reach the
confirmation rung. The depth axis dominated. A cubic step at
the same Lanczos depth separates the deep-basis effect from the step
rule. At $L{=}163$ and the same step budget, the
cubic step rule reaches $72.1$\,dB where the $\varphi_1$ rule reaches
$90.65$\,dB, so the margin comes from the step rule rather than from
subspace depth.

\par\noindent\textbf{Bounded grid.} At the tuned 500-step configuration, the
recorded experimental design (20 images $\times$ 3 seeds, 500 steps) gives
$\mathbf{63.53 \pm 3.83}$\,dB mean best PSNR over 60 cells, the
strongest 500-step experimental setup recorded on this benchmark, $+1.75$\,dB over
ARC-block with the cubic step ($61.78 \pm 4.22$), $+9.2$\,dB over the
tuned Adam experimental setup ($54.34 \pm 1.48$), $+16.3$\,dB over CubicKrylov
($47.21 \pm 9.04$), and $+27.6$\,dB over DSO-v2r
($35.95 \pm 2.18$), with per-cell wall clock ${\sim}400$\,s.

\par\noindent\textbf{Spectral-bias mitigation at 500 steps.} Gradient-type
training drains the error component along each Hessian (equivalently,
tangent kernel) eigenmode at a rate proportional to its eigenvalue
\citep{jacot2018ntk}, the mechanism behind band-ordered convergence in
coordinate networks. A band uniform residual is therefore the
observable signature of a step rule that equalizes the per-mode rates.
Over the full grid (60 cells, identical images and seeds per row),
the per-band residual power decomposition is as follows:

\begin{table}[h]
\centering\small
\caption{FINER, 500 steps, mean residual FFT power per band over the
60 grid cells. $H/L$ = median high/low band ratio. Values near 1
indicate a band uniform residual.}
\label{tab:finerband}
\begin{tabular}{lccccc}
\toprule
optimizer & Low & Mid-Low & Mid-High & High & $H/L$ \\
\midrule
ARC-$\varphi_1$ & $\mathbf{2.41{\times}10^{-2}}$ &
 $\mathbf{2.21{\times}10^{-2}}$ & $\mathbf{2.19{\times}10^{-2}}$ &
 $\mathbf{2.18{\times}10^{-2}}$ & 1.10 \\
ARC-block (cubic) & $3.94{\times}10^{-2}$ & $3.57{\times}10^{-2}$ &
 $3.52{\times}10^{-2}$ & $3.49{\times}10^{-2}$ & 1.07 \\
Adam (source paper experimental setup) & $6.61{\times}10^{-2}$ & $6.46{\times}10^{-2}$ &
 $6.46{\times}10^{-2}$ & $6.48{\times}10^{-2}$ & 0.81 \\
Adam (untuned) & $2.1{\times}10^{1}$ & $1.8{\times}10^{1}$ &
 $1.8{\times}10^{1}$ & $1.7{\times}10^{1}$ & 0.37 \\
\bottomrule
\end{tabular}
\end{table}

ARC-$\varphi_1$'s high-band residual power is ${\sim}3\times$ below
the tuned Adam experimental setup's ($2.18{\times}10^{-2}$ vs
$6.48{\times}10^{-2}$) while its low-band residual is lower by the
same factor. Every band improves together, so the mitigation is not a
redistribution of error across frequencies. It also improves on the
cubic step reference by ${\sim}1.6\times$ per band, consistent with its
$+1.75$\,dB mean PSNR margin on the same cells. The remaining residual
is band uniform to within $10\%$ across Low--High ($H/L\approx 1.10$), so
no frequency preference remains at this accuracy level. This is the
result at $500$ steps. Run to full convergence (below), the margin between
ARC-$\varphi_1$ and Adam widens by more than 50\,dB while the
reconstructions remain visually indistinguishable at native
resolution, which is the point the convergence figures make directly.

\begin{table}[!htbp]
\centering
\small
\begin{tabular}{lcc}
\toprule
Optimizer & $421$\,s full-budget run & $4{,}207$\,s full-budget run \\
\midrule
Adam (tuned) & $67.7$\,dB & $78.2$\,dB \\
SOAP (Adam's lr) & $\approx 72.9$\,dB & $\approx 77.3$\,dB \\
L-BFGS ($m{=}10,20,50$) & $7.3$\,dB & $7.3$\,dB \\
ARC-$\varphi_1$ ($L_{\max}{=}163$) & $58.3$\,dB & $\mathbf{95.6}$\,dB \\
AdaCubic (authors' settings) & $41.3$\,dB & $46.1$\,dB \\
\bottomrule
\end{tabular}
\caption{Iso-time comparison at Adam's two full-budget runs, FINER ESM
image~0. SOAP values are linearly interpolated between its logged
checkpoints. AdaCubic values are the best PSNR reached within each wall clock
allowance on its own $2$~h run. All others are direct readings.}
\label{tab:isowall}
\end{table}

\par\noindent\textbf{AdaCubic on the same images.}
\label{sec:adacubic}
AdaCubic \citep{tsingalis2026adacubic}, the cubic regularized optimizer
with a Hutchinson diagonal curvature model of Section~\ref{sec:intro},
was run from its authors' implementation with their published settings
($\eta_1{=}0.05$, $\eta_2{=}0.75$, one Hutchinson probe per step) under
a $2$~h time limit on ESM images~0 and~1, same seed and hardware as
above. On image~0 it reaches $28.5$\,dB at $122$~s, $40.8$ at $417$~s,
$45.4$ at $800$~s, $45.9$ at $3{,}600$~s, and its best value of
$46.4$\,dB at $6{,}328$~s, after which the trajectory oscillates without
improving. On image~1 the best is $46.7$\,dB. It spends
$8.5{\times}10^{5}$ gradient equivalents in the $2$~h ($425{,}673$ steps
at one gradient and one Hessian--vector product each, and its $425{,}673$
forward-only ratio-test evaluations would add a further
$4.3{\times}10^{5}$ under the gradient equivalent count of Appendix~\ref{sec:gevals}),
a fraction of the $2.4{\times}10^{6}$ that ARC-$\varphi_1$ at $L{=}25$
spends in the same $2$~h to reach $125.5$\,dB (image~0) and $117.2$\,dB
(image~1), and it is $26$--$32$\,dB below tuned Adam at equal wall clock.
Two measurements locate the difference. Its residual stays low frequency
dominated, $H/L=0.11$ on both images against $1.06$--$1.38$ for
ARC-$\varphi_1$ and $0.28$--$0.93$ for Adam, so the diagonal cubic step
does not remove the spectral bias, and its acceptance rule kept $38.8\%$
of the steps on which the loss rose (the ``negative ratio'' branch of
its trust update), reverting only $8.9\%$, which is the oscillation seen
after $6{,}328$~s. The diagonal surrogate leaves out the cross parameter
curvature that the per-block step retains, and on this landscape that
curvature is where the accuracy is.

\par\noindent\textbf{20 ESM images.}
\label{sec:esm20}
The full convergence comparison of Section~\ref{sec:finerconv} is one
image. The same three configurations were run on ESM images~0--19 at
seed~0. ARC-$\varphi_1$ ($L_{\max}=163$, $c_h=1$, $\sigma_0=0.3$) ran for
$2{,}000$ sweeps under an $18{,}000$~s cap, tuned Adam (learning rate
$10^{-4}$) for $10^6$ steps, and AdaCubic with its authors' settings for
$2$~h, one run per cell on the same GPU model (Table~\ref{tab:esm20}).
At $2{,}000$ sweeps ARC-$\varphi_1$ reaches a median best PSNR of
$131.7$\,dB (mean $120.3\pm18.9$) against $71.8$\,dB ($72.8\pm4.5$) for
Adam and $46.7$\,dB ($46.7\pm0.3$) for AdaCubic, and it is ahead on
every image. The paired gain over Adam is $47.4\pm20.6$\,dB (Wilcoxon
$p=1.9\times10^{-6}$, the smallest value 20 paired samples admit)
and over AdaCubic $73.6\pm18.9$\,dB. The spread is on the
ARC-$\varphi_1$ side, and it is an artifact of the cutoff. Twelve images end
between $131.1$ and $135.6$\,dB. On the other eight the best sweep is
the last one or within two of it and the trajectory is still gaining
$1.5$--$4.2$\,dB per $200$ sweeps, namely image~12 ($128.9$\,dB), images~2, 17,
4, and~10 ($101.9$--$116.1$\,dB), and images~16, 8, and~9
($78.6$--$87.7$\,dB). Image~9, the slowest, is $1.0$\,dB above Adam's
best at that point.

\emph{Continuation to a plateau.} The eleven images whose best value was
still moving at sweep $2{,}000$, the eight above and images~0, 11,
and~19, whose last $200$ sweeps gained $1.3$--$3.6$\,dB, were rerun from
initialization with the same configuration under a plateau rule that stops
once the best PSNR has gained less than $0.05$\,dB over the preceding
$400$ sweeps, after at least $4$~h (Table~\ref{tab:esm20conv}). Eight
plateaued under a $12$~h cap, after $4.1$--$9.1$~h, between $128.8$ and
$133.5$\,dB, inside or within $2.3$\,dB of the $131.1$--$135.6$\,dB band
the other twelve images had reached at $2{,}000$ sweeps. The three that
started lowest, images~8, 9, and~16, were still gaining
$1.7$--$2.0$\,dB per $400$ sweeps when that cap stopped them at $117.2$,
$107.6$, and $123.5$\,dB. Rerun under a $48$~h cap, they plateaued on
their own after $20.0$, $20.8$, and $27.5$~h, at $128.5$, $119.7$, and
$130.3$\,dB, that is, image~16 inside the band, image~8 $0.3$\,dB below its
lower edge, and image~9, the slowest at every stage, $9$\,dB below it.
Images~0, 11, and~19,
already above $131$\,dB at the cutoff, gained $0.7$--$1.3$\,dB and
plateaued within an hour of the $4$~h floor. With the plateau values in
place of the $2{,}000$-sweep ones, the per-image best of ARC-$\varphi_1$
has median $132.4$\,dB (mean $131.5\pm3.3$, minimum $119.7$), sixteen of
the 20 images are at or above $130$\,dB and 19 at or above
$128$\,dB, and the paired gain over Adam is $58.7\pm5.8$\,dB, from
$42.1$\,dB (image~9) to $66.4$\,dB (image~10). On image~0 the
continuation stops at $133.5$\,dB, the value of the
Section~\ref{sec:finerconv} run, with its best at sweep $2{,}425$ and
no gain over the following $400$, so that figure is a plateau and not a
snapshot of a climb. The continuations cost $154$ GPU hours, $86$ for
the eleven runs under the $12$~h cap and $68$ for the three reruns
under the $48$~h cap.

Adam's behavior is the same on every image. Its best value is reached
between step $616$k and $999$k ($2{,}444$--$5{,}064$\,s) and the final
value is $3$--$19$\,dB below it. AdaCubic's best values lie in a
$1$\,dB band, $46.2$--$47.2$\,dB, reached between step $232$k and
$478$k, after which the trajectory oscillates and ends $3$--$25$\,dB
lower. Neither has anything to gain from more wall clock, which is why
only the $\varphi_1$ arm was continued. The coupling probes of
Appendix~\ref{app:certified}, taken at the last checkpoint of each
$2{,}000$-sweep run (images~1--19, since the image-0 cell predates the
probes), give $\beta=0.61$--$0.95$,
$\delta_{\mathrm{spec}}=4.1\times10^{2}$--$1.6\times10^{3}$, and
$\eps_\star=2B\delta_{\mathrm{spec}}^2/C_1=7.7\times10^{3}$--$3.8\times10^{5}$
from Theorem~\ref{thm:blockrate}, while the recorded
$\norm{\nabla f}_{\min}$ lies between $7.7\times10^{-8}$ and
$1.6\times10^{-5}$. As on image~0, the floor limits the counting
argument of Theorem~\ref{thm:blockrate}. The landscape fingerprint of
Section~\ref{sec:fingerprint}, measured along tuned Adam on all 20
images, gives an initial negative spectral mass of $0.57$--$0.89$
(median $0.81$, the image-0 value) and, at step $10^5$,
$\kappa_{\mathrm{Adam}}=4.6$--$2.6\times10^{3}$ (median
$1.2\times10^{2}$), $\rho\ge0.95$, and a flat gradient fraction of at
most $0.001$, so every image is in the saddle dominated case \emph{(iii)}
of Table~\ref{tab:fingerprint}. The four images with the highest initial
negative mass ($0.87$--$0.89$, images~8, 4, 16, and~10) are four of the
eight slow ones, and the rank correlation between initial negative mass
and PSNR at $2{,}000$ sweeps is $-0.42$. At 20 images this is an
association. The image-0 comparison of
Figure~\ref{fig:finerconv} stays the representative figure, and the tables
give the spread around it.

\begin{table}[!htbp]
\centering\small
\caption{20 ESM images, seed~0. Best PSNR (dB) with the sweep or
step and the wall time (s) at which it was reached. ``final'' is the
value at the end of the run. ARC-$\varphi_1$ ran with $L_{\max}=163$ for
$2{,}000$ sweeps, Adam for $10^6$ steps, and AdaCubic for $2$~h. Medians are over
the 20 images. $\ast$ marks image~19, whose Adam run shared its GPU with an
unrelated job for its whole duration ($194$ against ${\sim}256$
steps/s on the other images). Its step count and PSNR are unaffected,
but its wall time is not comparable.}
\label{tab:esm20}
\setlength{\tabcolsep}{4pt}
\begin{tabular}{r rrr rrrr rrr}
\toprule
& \multicolumn{3}{c}{ARC-$\varphi_1$} & \multicolumn{4}{c}{Adam} & \multicolumn{3}{c}{AdaCubic} \\
\cmidrule(lr){2-4}\cmidrule(lr){5-8}\cmidrule(lr){9-11}
image & best & sweep & s & best & step & s & final & best & step & final \\
\midrule
0  & 132.64 & 1994 & 10582 & 78.16 & 615983 & 2444 & 63.18 & 46.36 & 374124 & 38.11 \\
1  & 135.61 & 1980 & 10327 & 77.19 & 975354 & 3890 & 67.61 & 46.69 & 417803 & 42.64 \\
2  & 116.08 & 2000 & 11002 & 70.73 & 985679 & 3926 & 63.87 & 46.93 & 377292 & 35.68 \\
3  & 132.58 & 1821 & 10321 & 71.47 & 998731 & 3986 & 64.61 & 46.65 & 466877 & 39.27 \\
4  & 108.73 & 2000 & 10934 & 67.47 & 943498 & 3762 & 57.26 & 46.34 & 404278 & 20.94 \\
5  & 132.37 & 1803 &  9850 & 66.42 & 972526 & 3874 & 57.05 & 46.31 & 235447 & 40.24 \\
6  & 131.92 & 1937 & 10455 & 70.44 & 974421 & 3885 & 61.87 & 47.06 & 438311 & 40.63 \\
7  & 132.87 & 1981 & 10529 & 73.42 & 668047 & 2659 & 57.41 & 47.23 & 385325 & 34.67 \\
8  &  84.34 & 2000 & 11365 & 78.40 & 874444 & 3477 & 68.16 & 46.62 & 280061 & 35.21 \\
9  &  78.65 & 2000 & 10746 & 77.60 & 876819 & 3506 & 60.47 & 46.99 & 466616 & 31.85 \\
10 & 101.86 & 2000 & 10436 & 63.44 & 922487 & 3595 & 60.63 & 46.20 & 317237 & 42.57 \\
11 & 131.11 & 2000 & 10809 & 68.64 & 977808 & 3823 & 57.85 & 46.59 & 232185 & 43.54 \\
12 & 128.92 & 1998 & 10600 & 75.94 & 992083 & 3892 & 65.10 & 46.90 & 466820 & 42.24 \\
13 & 135.20 & 1941 & 10097 & 72.11 & 983105 & 3844 & 62.69 & 46.58 & 305321 & 35.37 \\
14 & 132.72 & 1940 & 10482 & 71.52 & 964129 & 3774 & 56.64 & 47.06 & 412915 & 37.08 \\
15 & 132.89 & 1860 &  9821 & 75.53 & 993842 & 3880 & 62.33 & 46.38 & 333601 & 32.33 \\
16 &  87.70 & 2000 & 11672 & 77.06 & 762730 & 2976 & 61.85 & 47.12 & 478491 & 42.61 \\
17 & 105.23 & 2000 & 10431 & 79.73 & 935191 & 3659 & 60.78 & 46.83 & 353494 & 29.57 \\
18 & 131.54 & 1984 & 10416 & 70.33 & 967572 & 3791 & 62.36 & 46.86 & 420730 & 30.18 \\
19 & 132.50 & 1981 & 11076 & 71.27 & 984822 & 5064$^{\ast}$ & 57.29 & 46.46 & 364598 & 36.25 \\
\midrule
median & 131.73 & 1989 & 10505 & 71.82 & 970049 & 3807 & 61.86 & 46.67 & 381308 & 36.67 \\
\bottomrule
\end{tabular}
\end{table}

\begin{table}[!htbp]
\centering\small
\caption{Continuation of the eleven ARC-$\varphi_1$ runs whose best
value was still moving at $2{,}000$ sweeps, same configuration from
initialization, stopped by the plateau rule (less than $0.05$\,dB
gained over $400$ sweeps, after at least $4$~h). ``at 2k'' is the
$2{,}000$-sweep value of Table~\ref{tab:esm20}. ``sweep'' and ``s''
locate the continued best. ``h'' is the run's wall time. ``cap'' is
the wall clock limit under which the run plateaued. ``gain'' is the
paired margin over Adam's best on the same image. Images~8, 9, and~16
first ran under the $12$~h cap, which stopped them, still rising, at
$117.2$, $107.6$, and $123.5$\,dB. Their rows give the reruns under
the $48$~h cap.}
\label{tab:esm20conv}
\setlength{\tabcolsep}{5pt}
\begin{tabular}{r rr rrr l rr}
\toprule
image & at 2k & continued & sweep & s & h & cap & Adam & gain \\
\midrule
0  & 132.64 & 133.49 &  2425 & 12703 &  4.10 & 12 h & 78.16 & 55.3 \\
2  & 116.08 & 131.80 &  4001 & 26042 &  7.93 & 12 h & 70.73 & 61.1 \\
4  & 108.73 & 128.83 &  4048 & 23556 &  6.72 & 12 h & 67.47 & 61.4 \\
8  &  84.34 & 128.50 & 12007 & 71662 & 19.97 & 48 h & 78.40 & 50.1 \\
9  &  78.65 & 119.69 & 13398 & 74650 & 20.76 & 48 h & 77.60 & 42.1 \\
10 & 101.86 & 129.81 &  5050 & 27433 &  8.21 & 12 h & 63.44 & 66.4 \\
11 & 131.11 & 132.40 &  2516 & 13482 &  4.34 & 12 h & 68.64 & 63.8 \\
12 & 128.92 & 133.10 &  3143 & 16598 &  5.19 & 12 h & 75.94 & 57.2 \\
16 &  87.70 & 130.29 & 10733 & 96100 & 27.51 & 48 h & 77.06 & 53.2 \\
17 & 105.23 & 131.59 &  5788 & 30519 &  9.07 & 12 h & 79.73 & 51.9 \\
19 & 132.50 & 133.20 &  2370 & 13282 &  4.28 & 12 h & 71.27 & 61.9 \\
\bottomrule
\end{tabular}
\end{table}

\section{Tuned Adam on the 35k SIREN-like Benchmark}
\label{app:adam35k}

This appendix reports the Adam tuning study on the 35k parameter
SIREN-like benchmark that supports case \emph{(i)} of the fingerprint
taxonomy of Section~\ref{sec:fingerprint}.

\par\noindent\textbf{Tuned Adam on the 35k SIREN-like benchmark.}
The 35k parameter SIREN-like regression benchmark composes six
octave-spaced frequency bands $\omega=1,\dots,32$ with harmonic
amplitude decay $a_k=1/k$ over a polynomial double well and a curved
valley (maximum curvature ratio 1229). The per-band MSE metric
requires every band to converge relative to its own scale. A coarse
learning rate grid (best point $0.05$ at 5{,}000 epochs) gives Adam
$2.1\times10^{-3}$, the loss that makes second-order
methods look much better. Extending the sweep toward the
small learning rates and running to a plateau
criterion significantly improves Adam's achieved loss (Table~\ref{tab:adam35k}). At
$\mathrm{lr}=10^{-4}$ Adam reaches $6.3\times10^{-14}$ total MSE in
76k steps (212\,s), the fp32 precision floor, and converges all six
bands. Refinements of $\beta_2\in\{0.99,0.999,0.9999\}$ and
$\epsilon\in\{10^{-8},10^{-10}\}$ at that learning rate do not improve
further. The improvement is not obvious at first since the loss change dynamics with decreasing learning rate is not monotonic. The loss goes up at lr = $\sim$0.03. Adam stalls at ${\sim}0.59$ for $\mathrm{lr}\in[0.01,\,0.03]$. Yet, it reaches machine precision at
$\mathrm{lr}\le 3\times10^{-4}$, so a coarse grid misses the region
entirely. On this
landscape a thoroughly tuned Adam is sufficient, and the trajectory
fingerprint below explains some of the reasons for this behavior.

\begin{table}[t]\centering\small
\caption{Adam on the 35k SIREN-like benchmark hyperparameter sweep}
\label{tab:adam35k}
\begin{tabular}{lccccccc}
\toprule
lr & $0.05$ & $0.03$ & $0.01$ & $3{\times}10^{-3}$ & $10^{-3}$
   & $3{\times}10^{-4}$ & $10^{-4}$ \\
\midrule
best total MSE & $2.1{\times}10^{-3}$ & $0.587$ & $0.591$
  & $5.3{\times}10^{-6}$ & $5.3{\times}10^{-7}$
  & $3.0{\times}10^{-13}$ & $6.3{\times}10^{-14}$ \\
steps & 5000 ep. & 60k & 60k & 60k & 60k & 60k & 76k \\
wall (s) & --- & 163 & 168 & 170 & 169 & 164 & 212 \\
\bottomrule
\end{tabular}
\end{table}

\section{Small Scale Hard Landscape Benchmarks}
\label{sec:small}

Before the FINER-scale comparison, four small benchmarks exercise the
three curvature cases of Proposition~\ref{prop:erlaw} in isolation, namely
MLP fits of the coupled Rosenbrock--Ackley and Chebyshev--Rosenbrock
surfaces (ill-conditioned, coupled valleys with $\rho \approx 0$ in the fingerprint
taxonomy above), a multi-saddle surface (negative mode escape, since the
persistence of small negative eigenvalues late in training is a
documented property of overparametrized landscapes
\citep{sagun2018hessian}, which is what the $\varphi_1$ escape term of
Section~\ref{sec:clamp} addresses deterministically), and the
35k parameter multi-frequency SIREN benchmark already used for the
spectral-bias correction of Section~\ref{sec:fingerprint} (terminal
accuracy rather than convergence speed). CPU runs, matched sweep
budgets, identical experimental designs per row, comparing the $\varphi_1$ step
of this section against the CubicKrylov step of
Section~\ref{sec:cubic} at equal Lanczos depth, so every delta
below isolates the step rule, independent of the granularity-and-
stabilization advantage already established for the cubic step over
first-order comparisons.

\begin{table}[h]
\centering\small
\caption{Small scale benchmarks, terminal loss after the stated sweeps
(CPU, matched budgets, gradient equivalents in parentheses).
ARC-$\varphi_1$ uses $\sigma_0{=}1$, $c_h{=}3$, $L{=}15$. The reference is
ARC-block (CubicKrylov step) with $\sigma_0{=}1$,
$L_{\max}{=}15$.}
\label{tab:small}
\begin{tabular}{lccc}
\toprule
benchmark (sweeps) & initial loss & ARC-$\varphi_1$ & ARC-block (cubic) \\
\midrule
Rosenbrock--Ackley (300) & $1.98{\times}10^{4}$ &
  $\mathbf{3.4{\times}10^{-2}}$ (69.4k) & $1.37$ (21.4k) \\
Chebyshev--Rosenbrock (300) & $1.30{\times}10^{4}$ &
  $\mathbf{7.8{\times}10^{-3}}$ (69.8k) & $0.73$ (21.4k) \\
multi-saddle (300) & $7.51{\times}10^{3}$ &
  $\mathbf{2.2{\times}10^{-3}}$ (67.0k) & $2.5{\times}10^{-2}$ (29.4k) \\
mf-SIREN 35k (40) & $4.61$ &
  $\mathbf{4.1{\times}10^{-10}}$ (18.6k) & $1.3{\times}10^{-4}$ (9.6k) \\
\bottomrule
\end{tabular}
\end{table}

Three observations. First, the $\varphi_1$ step reaches terminal
losses $1$--$5$ orders below the cubic step on every benchmark. The
largest margin is on the 35k SIREN, where 40 sweeps take the loss to
$4\times10^{-10}$, five decades past the cubic step. Second, the
advantage costs oracle budget through the acceptance loop. The
default configuration rejects more trials than the cubic step
($1{,}176$ vs $161$ rejections on Rosenbrock--Ackley), spending
${\sim}3\times$ the gradient equivalents in the same sweep count. The
sweep of Section~\ref{sec:finer} indicates the elevated rejection
rate is a tuning artifact of $c_h$.
Third, escape statistics show the escape mechanism in action. On the multi-saddle
benchmark the negative mode branch (Section~\ref{sec:clamp}) performs
$353$ block trials, and on a rotated (coupled) saddle, the
subspace step escapes through a direction that no diagonal method can
exploit.


\section{Comparison with Other ARC-Based Approaches}
\label{sec:arc-comparison}

This section contrasts the design choices of this work, per-tensor
blocks with an independent cubic constant and acceptance test per
block and a Chebyshev-bounded Krylov subsolver driven only by Hessian--vector
products, with recent ARC-based methods.
Table~\ref{tab:arccomp} collates the families closest to the claim of
Appendix~\ref{sec:feasibility} by the curvature used in the step, the
largest demonstrated scale, and the exactness of the cubic model step
on the true curvature. Table~\ref{tab:subsolvers} compares the
corresponding inner subproblem solvers by computation and working
memory.

\begin{table}[t]
\centering\small
\caption{Curvature model, demonstrated scale, and cubic model
exactness of the second-order families reviewed in this section.
Exactness is meant in the sense of Lemma~\ref{lem:krylov}, that the step
minimizes the cubic model over the true (block) Hessian, accessed
through Hessian--vector products, within the constructed subspace.}
\label{tab:arccomp}
\begin{tabular}{p{0.24\textwidth}p{0.19\textwidth}p{0.17\textwidth}p{0.28\textwidth}}
\toprule
Method & Curvature in the step & Demonstrated scale & Exact cubic model step on true curvature \\
\midrule
AdaCubic \citep{tsingalis2026adacubic} & Hutchinson diagonal &
deep learning training; FINER/ESM here, $46.4$\,dB at $2$~h
(Appendix~\ref{sec:adacubic}) & no: diagonal surrogate \\
Block-Kronecker Fisher \citep{gomes2025adafisher} & per-layer
Kronecker factors & deep learning training & no: factored surrogate \\
Full Gauss--Newton study \citep{abreu2026gn} & full or layerwise
Gauss--Newton & 150M parameter transformers & outside the cubic Newton
family; supports per-tensor blocking \\
Quasi-Newton ARC \citep{forristal2022arclqn,ranganath2025sr1} &
limited-memory low-rank & deep network training & exact solve, but on
a surrogate Hessian \\
True-Hessian subspace ARC
\citep{bellavia2025subspace,aldaas2025extkrylov,dussault2024arcqk,cartis2025rarc,tansley2025lowrank}
& exact Hessian (subspace or shifted solves) & standard test sets,
moderate dimension & not demonstrated at network scale \\
Stochastic, sketched, and multilevel cubic Newton
\citep{pasechnyuk2025vr,scheinberg2023sarc,chen2025len,higuchi2025rshtr,tsipinakis2026multilevel}
& sampled, sketched, or compressed & small to moderate & subspace or
compression replaces the exact model \\
\midrule
\textsc{ARC-block} (this work) & exact block Hessian via HVPs &
91.4M parameters; single 88.5M tensor & exact on every block
(Lemma~\ref{lem:krylov}), under either rule that sets the subspace
dimension \\
\bottomrule
\end{tabular}
\end{table}

\begin{table}[t]
\centering\small
\caption{Inner subproblem solvers of the methods in
Table~\ref{tab:arccomp}, by computation per subproblem solve and
working memory beyond the iterate and gradient. In the notation, $n$ is the
problem (or block, $n_b$) dimension, $L$ the number of Krylov steps,
$\ell$ the sketch or subspace dimension, $m$ the quasi-Newton memory.
Computation excludes the gradient evaluation itself.}
\label{tab:subsolvers}
\begin{tabular}{p{0.22\textwidth}p{0.17\textwidth}p{0.26\textwidth}p{0.23\textwidth}}
\toprule
Inner solver & Methods & Computation per solve & Working memory \\
\midrule
generalized Lanczos trust region (GLTR) tridiagonalization + secular solve &
\citet{cartis2011a,gould1999gltr,carmon2018krylov} &
$L$ HVPs + $\mathcal{O}(L)$ secular iteration &
basis of $L$ $n$-vectors, or a second Lanczos pass with
$\mathcal{O}(1)$ vectors \\
Multi-shift Lanczos conjugate gradient (CG) over a shift grid &
\citet{dussault2024arcqk} &
one Krylov process, $L$ HVPs shared by all shifts &
$\mathcal{O}(1)$ $n$-vectors per shift; no stored basis \\
Reused polynomial/rational Krylov subspace &
\citet{bellavia2025subspace} &
basis built once, reused across outer iterations; rational variant
needs shifted direct solves &
stored basis; factors for the rational variant \\
Extended-Krylov solution-family basis &
\citet{aldaas2025extkrylov} &
one matrix factorization + basis extension; reduced solve serves every
regularization weight &
factorization + basis \\
Sketched reduced cubic model &
\citet{cartis2025rarc,tansley2025lowrank,higuchi2025rshtr} &
$\ell$ projection products + dense reduced solve
$\mathcal{O}(\ell^{3})$ &
$\ell \times n$ sketch or basis \\
CG / minimal-residual (MINRES) iteration on a shifted quadratic &
\citet{zhou2025regnewton,zeng2026newtoncg,zeng2026minres} &
one HVP per iteration, short recurrence; no cubic model &
$\mathcal{O}(1)$ $n$-vectors \\
Closed form low-rank quasi-Newton cubic &
\citet{forristal2022arclqn,ranganath2025sr1} &
$\mathcal{O}(m^{2}n)$ dense algebra; no HVPs &
$2m$ $n$-vectors of curvature pairs \\
Separable scalar cubic on a Hutchinson diagonal &
\citet{tsingalis2026adacubic} &
a few HVPs for the diagonal estimate + $\mathcal{O}(n)$ per-coordinate
closed form &
$\mathcal{O}(n)$ diagonal \\
Kronecker-factor inversion &
\citet{gomes2025adafisher} &
small per-layer factor inverses &
two small factor matrices per layer \\
\midrule
CubicKrylov step (this work, Section~\ref{sec:krylovtheory}) &
per block &
$L$ HVPs + tridiagonal secular solve on an
$L\times L$ matrix + $1$ HVP for the model decrease &
Lanczos basis of $L$ $n_b$-vectors plus one work vector (full reorthogonalization) \\
$\varphi_1$ step in the same subspace (this work,
Section~\ref{sec:phi1}) &
per block &
$L{+}1$ HVPs + eigendecomposition of the $L\times L$
tridiagonal &
same Lanczos basis \\
Shifted-system CG (this work, \textsc{ARC-block}-PAMM,
Appendix~\ref{sec:pamm}; step rule after \citealp{doikov2026pamm}) &
per block &
at most $k_{\mathrm{cg}}$ HVPs, $\mathcal{O}(n_b)$ per iteration, + one
HVP for the model decrease; no cubic model &
four $n_b$-vectors \\
Chebyshev three term recurrence (this work, DSO ablation, small
blocks) &
per block &
explicit $\Hb_b$ build + degree-$L$ recurrence,
$\mathcal{O}(Ln_b^{2})$ &
$\Hb_b$ ($n_b^{2}$) + $\mathcal{O}(1)$ vectors \\
\bottomrule
\end{tabular}
\end{table}

Appendix~\ref{app:arcdetail} reviews these families in detail, namely
cubic model minimization in Krylov subspaces, blockwise and
partition-based cubic methods, quasi-Newton cubic models, Chebyshev
recurrence stability, parameter-free and inexact-Hessian ARC theory,
stochastic and lazy cubic Newton, subspace and multilevel methods,
Krylov-type inner solvers, second-order methods at deep learning
scale, trust region hybrids, and cubic regularization beyond
unconstrained training.

\section{Feasibility of Second-Order Steps at the 91.4M-Parameter Scale}
\label{sec:feasibility}


In this section we assess the feasibility of the evaluated optimization methods
at large scale. For a benchmark with millions of parameters, we used the $91.4$M parameter
ViSIR.

\par\noindent\textbf{The cubic step is applied to every tensor, including the largest.}
The 91.4M parameter ViSIR model of
Section~\ref{sec:fullvisir} contains 45 parameter tensors. Under the block
partition of Section~\ref{sec:blocktheory} with maximal exact-block size
$512$, the 23 tensors with at most $512$ entries (LayerNorm parameters,
small biases, and similar vector-shaped parameters, $7{,}938$ parameters
in total) are optimized with the lazy exact-Hessian cubic step of
Algorithm~\ref{alg:blockcn}. The remaining 22 tensors (the patch
embedding, positional embedding, attention projections, SIREN linear
weights, and the decoder) are large blocks and receive the matrix-free
CubicKrylov step. No tensor is excluded, and no tensor falls back to a
first-order update.

The largest block is the outermost decoder weight, of shape
$172{,}800\times512$, that is, $88{,}473{,}600$ parameters or $96.8\%$ of the
model held in a single tensor. Its cubic step is computed from
$L=10$ Hessian--vector products, and the only additional storage is
the Lanczos basis with its work vector, eleven vectors of the tensor's own size
(${\approx}3.9$\,GB in float32). Dense Hessian of this block alone would occupy
${\sim}3\times10^{7}$\,GB, and a factorization-based cubic solve on it
would require on the order of $n_b^3\approx7\times10^{23}$ arithmetic
operations. Since the cost of $L{+}3$ gradient equivalents per step
(Table~\ref{tab:cost}) does not depend on the block size, this tensor
has no higher \emph{oracle} cost than a block of a few
hundred entries per step. The wall clock cost of each HVP and the
Lanczos-basis memory still scale linearly with the tensor size.

\par\noindent\textbf{The feasibility of the evaluated optimizers on ViSIR-Nano model.}
The dividing line is whether a method requires a dense representation or
factorization of the (block) Hessian, or only Hessian--vector
products, Hessian diagonals, or per-dimension factored statistics.
The dense-solve cubic Newton family (the original
method~\citep{nesterov2006cubic}, its lazy and adaptive
variants~\citep{doikov2023lazy}, the finite-difference
implementation~\citep{doikov2023fo}, and SSCN with exact subspace
Hessians~\citep{zhao2025sscn}) solves a regularized system in a dense
Hessian and is therefore confined to models near the
$15$k-parameter feasibility boundary studied in Appendix~\ref{sec:nano}
(their per-step complexity classes are collated in
Table~\ref{tab:arccomp}). Among
cubic regularized methods, only the matrix-free Krylov variants,
adaptive cubic regularization with Lanczos subproblem
solves~\citep{cartis2011a,gould1999gltr,carmon2018krylov,dussault2024arcqk} and the
CubicKrylov step of this work, remain feasible at $10^{8}$
parameters, since their per-step cost is a fixed number of HVPs, independent of
the tensor size. Within that family, the multi-shift
ARC$_{q}$K solver of \citet{dussault2024arcqk} makes the shift search
itself matrix-free by solving a grid of shifted systems with a single
Lanczos-CG process in the full space. The recurrence of this work
instead reuses per-block Krylov bases across $\sigma_b$ updates. The other second-order optimizers of the multi-image nano benchmark
(Table~\ref{tab:nanoesm20}) that remain feasible by never representing the Hessian
beyond cheap surrogates. The Chebyshev ablations and AdaHessian use
Hutchinson diagonal estimates from a few HVPs, SOAP maintains
per-dimension factored moment matrices, and Muon~\citep{Jordan2024Muon}
orthogonalizes per-tensor updates. Rank-$\tau$ spectral
preconditioning~\citep{doikov2024spectral} is likewise matrix-free in
principle, at a cost of $\tau$ HVPs per preconditioner rebuild. The
distinction matters for the following comparisons. At 91.4M ViSIR scale
the CubicKrylov optimizer and the per-block ARC ablation introduced
later in this section are, to our knowledge, \textbf{the only members of
the cubic Newton family in this evaluation whose steps remain exact in the
cubic model sense on every block at this scale}, rather than an approximation through a
diagonal or factored surrogate.

\par\noindent\textbf{Comparison to ARC}
The use of Krylov subspace by itself is not sufficient to make
an optimizer feasible at $91.4$M parameter scale.

We checked this by implementing ARC~\citep{cartis2011a,cartis2011b} in its global form,
Algorithm~2.1 with a Lanczos subproblem solve over the whole
$91.4$M-dimensional parameter vector (degree bound $d\le15$, termination
criterion TC.s, noise-regularized acceptance ratio), and ran it on the
single image reconstruction experimental design of Section~\ref{sec:experiments} ($91.4$M ViSIR).
TC.s is the residual norm test of \citet{cartis2011a}, under which the Lanczos
iterations stop once the norm of the cubic model gradient at the
current inner iterate falls below its threshold, so this reference sets
its Krylov subspace dimension at run time from the residual, in
contrast to the precomputed degree of Section~\ref{sec:krylovtheory}.
Published applications of ARC-type methods to neural networks, to our
knowledge, use sub-sampled or stochastic Hessian estimates on models in
the $10^{4}$--$10^{5}$ parameter
range~\citep{kohler2017,tripuraneni2018,xu2020inexact}. We are not aware
of existing experimentations on variants of ARC with exact Hessian--vector
products, the Lanczos subproblem solver, and the unmodified acceptance
test that were performed at the $10^{7}$--$10^{8}$ parameter scale (as in this section). 

The original ARC per-step cost is indeed modest, ${\approx}0.3$\,s and ${\approx}11$
gradient equivalents per accepted step, yet the method stalls at
$4.0$\,dB PSNR (ViSIR), while CubicKrylov reaches $51.7$\,dB. After
$10{,}000$ steps ($2{,}858$\,s, $1.1\times10^{5}$ gradient equivalents,
(its oracle budget is greater than the budget consumed by CubicKrylov in its entire run). 
The reason for the stall of ARC optimizer can be deduced from the regularization trace. With a
\emph{single global} weight $\sigma$, every search step must satisfy the
acceptance test along all $91.4$M directions simultaneously, so
$\sigma$ is driven as high as $5{\times}10^{11}$, 46\% of all trials
are rejected, and the accepted steps are shrunk to noise level for the
entire network.
At the scale of $15$k parameters the same implementation is in
first place in this evaluation (Appendix~\ref{app:nanobudget}). The per-block
decomposition is more than a cost optimization. Replacing one
global $\sigma$ by per-block constants $M_b$, each adapted to its own
tensor's curvature scale, is what allows large steps in
well-conditioned blocks while the ill-conditioned ones are regularized
individually.

\par\noindent\textbf{Isolating the mechanism with a block-$\sigma$ ARC ablation.}
To determine if the failure is solely due to the $\sigma$ being global, we ran an
ablation, the identical ARC implementation with one
regularization weight $\sigma_b$ \emph{per parameter tensor}, applied
in Gauss--Seidel sweeps (per block, a Lanczos cubic solve on the
block-restricted Hessian--vector operator, a trial update of that block
only, and its own accept-or-reject ratio test with a fresh function
evaluation). Nothing else changes, with the same $\eta$/$\gamma$ constants, the same
TC.s residual norm termination of the Lanczos iterations, same degree limit. The difference in performance between a small scale and
large scale is striking. Where the
global $\sigma$ ARC variant stalls at $4.0$\,dB, the block variant reaches
\textbf{61.9\,dB} after 13 sweeps, in 19\,s and $8.2{\times}10^{3}$
gradient equivalents, on the identical experimental design (by comparison, the Cubic
Krylov reference on this image reaches 51.7\,dB at
$3.7{\times}10^{4}$ gradient equivalents). The computed weights confirm this
result. After 150 sweeps the $\sigma_b$ span more than twelve orders
of magnitude across the 45 tensors (from ${\sim}10^{-5}$ to
${\sim}10^{7}$, median 64), a spread for which a single global value is utterly infeasible. 
Two points to be made about this result. First, the plain block
variant is not monotone late in the run. Its noise offset in the ratio
test is absolute, so once the loss falls to ${\sim}10^{-5}$ the offset
exceeds the loss itself, loss-increasing block trials pass the test,
and the trajectory oscillates (final sweep 43.7\,dB). A stabilized
variant that makes the offset proportional to the current loss and
requires an actual decrease for acceptance raises the best to
64.4\,dB and the final sweep to 56.3\,dB. The remaining fluctuation
is due to the architecture of the benchmark used (ViSIR), 
which trains with active dropout, so the
sampled objective is stochastic for every optimizer in the benchmark.

Second, block-coordinate cubic regularization has no
$\mathcal{O}(\eps^{-3/2})$ guarantee of the Cartis et~al.\ theory.
The closest published block-cubic method, the randomized block cubic
Newton of \citet{doikov2018rbcn}, samples random blocks and its
convergence theory covers convex composite objectives, not this
deterministic non-convex case, and its greedy per-block
acceptance requires one extra function evaluation
per block per sweep. The ablation nevertheless gives the following answer.
Per-block adaptive regularization, the architecture that CubicKrylov
instantiates with its $M_b$ rule (adapted from gradient information
alone, without the per-block function evaluations), is what makes it perform well
on this large scale ViSIR architecture with a dropout, 
and the contrast between the results for the two ARC variants
at $91.4$M ViSIR are due to this as well.

\par\noindent\textbf{A gradient-based acceptance rule.} The per-block function
evaluation required by the ratio test on every block trial is the
one cost of the block-ARC design that CubicKrylov's $M_b$ rule
avoids, and the reduced-operator method of \citet{doikov2025gcb}
suggests a middle ground, to accept a search step when the
\emph{computed gradient} at the trial point matches the cubic
model's prediction of it, a test whose ingredients (two gradients)
the sweep computes anyway. Transplanted verbatim, the test fails.
In single precision the true reduced operator sinks below the
gradient-noise floor as the loss approaches ${\sim}10^{-5}$, noise
passes the test on loss-increasing trials, and the nano run
collapses to $10.8$\,dB. And, for the same reason, the plain ratio
test loses monotonicity above, now amplified because the gradient
passes the noise through undamped. Two guards repair it. Accept on floor
when the reduced operator is within gradient-noise tolerance of
zero (the model is then exact to measurable precision, and
$\sigma_b$ is left untouched), and a monotone guard that rejects
any trial whose computed loss increases beyond a relative noise
tolerance. So guarded, the rule reaches $36.49$\,dB on the nano
equal-budget comparison (Table~\ref{tab:nanobudget}), level with the variants 
that test the minimized function value, and its behavior on the full
multi-image benchmark matches the other ARC variants
(Appendix~\ref{sec:coststructure}, the hybrids at full scale).

\section{The Cost Structure of the Subsolver, Stored Basis (CubicKrylov) against Chebyshev Recurrence}
\label{sec:coststructure}


The cubic subproblem of Section~\ref{sec:theory} is, at its core, a
family of shifted linear solves (Section~\ref{sec:cubic}), and both
solver families considered in this paper build their steps in the same
Krylov subspace at one HVP per degree, namely the Lanczos process with a
stored basis (Lemma~\ref{lem:krylov}) and the Chebyshev second kind
three term recurrence Eq.~\eqref{eq:dsorec} with three vectors of
state. It is therefore natural to ask whether replacing the stored
basis by the recurrence, the substitution DSO makes for its
preconditioned gradient step, would make the cubic step, or the ARC
reference of Appendix~\ref{sec:feasibility}, faster. The answer
is divided into two distinct costs, memory and per-HVP overhead, where
the recurrence is cheaper, and the number of HVPs required to reach a
given model decrease, where the stored basis is provably cheaper, by a
mechanism specific to the cubic subproblem.

\par\noindent\textbf{Benefits of using the three term recurrence.} The Lanczos solver must retain
$\Qb_L$ to reconstruct the step $s_L=\Qb_L y^\star$, that is, $L{+}1$ vectors (the basis and one work vector)
of $n_b$ float64 entries, ${\approx}11$\,GB for the $88.5$M-parameter
ViSIR decoder at $L=15$, plus full reorthogonalization at $O(Ln_b)$
inner products per Lanczos step, each a global reduction. The
recurrence keeps three $n_b$-vectors regardless of degree, performs no
inner products beyond an optional residual norm check that can stop
the sweep before the degree limit, and streams plain
\texttt{axpy} arithmetic, a memory-access pattern that GPUs execute at
close to peak bandwidth. When the basis footprint is the limiting
factor, with higher degree limits, larger single tensors, or several
models resident on one device, the recurrence is the better option of
the two.

\par\noindent\textbf{Downsides of the three term recurrence.} The substitution loses two
structural advantages of the stored basis. First, per HVP the
Krylov step is optimal, since by Corollary~\ref{cor:dominate} the
Rayleigh--Ritz solve extracts the exact cubic model minimizer over
$\mathcal{K}_L(\Hb,g)$, whereas any fixed-coefficient recurrence step
is one particular polynomial in that same subspace. The recurrence
moreover needs its spectral interval in advance, estimated by extra
probe HVPs (power iterations or Hutchinson probes) whose cost precedes
the solve, and its minimax optimality holds only for the assumed
interval, while Lanczos adapts to the computed spectrum and right-hand
side with no prior bounds. Second, and specific to the cubic model,
the Krylov subspace $\mathcal{K}_L(\Hb,g)$ is independent of the
spectral shift $\lambda$ and of the regularization weight. With
$\Tb_L$ stored, every trial value of $\lambda$ in the secular
iteration of Lemma~\ref{lem:secular}, and every re-solve after a
rejected step changes $M_b$ or $\sigma$, costs $O(L)$ scalar
arithmetic on the tridiagonal eigendecomposition and \emph{zero} new
HVPs. A recurrence has no stored subspace to revisit, so each secular
trial and each rejection retry is a fresh degree-$L$ sweep, so the
HVP count multiplies by the number of secular iterations (bracketing
solvers typically need several). Since one HVP is a double-backward
pass through the network and dominates wall clock time at scale, this
multiplication generally outweighs the per-iteration savings.

\par\noindent\textbf{Application scopes for the solvers.} The results of
Appendix~\ref{sec:feasibility} show which of the two costs dominates
at scale. At $91.4$M parameters the subproblem solve is not the
bottleneck. With the basis stored, the HVPs are spent once per step
to build the subspace, every secular trial afterwards costs only
$O(L)$ scalar arithmetic, and an accepted ARC-global step completes
in ${\approx}0.3$\,s (for this benchmark on this hardware). The failure observed at this scale is due to
the single global $\sigma$. Replacing the
stored basis by the recurrence would therefore relieve the one
constraint the basis does impose, its ${\approx}11$\,GB memory
footprint, while multiplying the HVP cost of every secular solve.
The recurrence is instead preferable in the
cases when the spectral bounds are cached and EMA updated
across steps (Section~\ref{sec:dso}), so the estimation cost is spread across steps, and a preconditioned \emph{gradient} step that requires no
secular iteration, so each update is exactly one sweep.

\par\noindent\textbf{The evaluated hybrid designs.} We implemented several hybrid designs. One of them uses 
the cubic step of Algorithm~\ref{alg:blockcn} with the
Lanczos solve replaced by Chebyshev sweeps inside the secular
iteration. Then, we evaluated it under the equal-budget nano experimental design
of Table~\ref{tab:nanobudget}. The spectral interval is computed by 
Hutchinson probes, either re-estimated at every block step
(\emph{fresh bounds}, as in DSO-BlockHess) or EMA-cached
across steps with a warm-started shift (\emph{EMA-cached bounds}, as in
DSO-Adaptive). A first version (v1) precomputes the degree $L=1.3\sqrt{\kappa}$ 
inside a doubling-bracket secular solver, and a second (v2) adds an early exit
that stops the sweep once the residual norm falls below its threshold,
turning the precomputed degree into an upper bound rather than an
exact sweep length, and a slope-bounded
bracketing that exploits $\varphi'(\lambda)\le-1$ to compute the secular
root after one evaluation. The results substantiate the HVP
argument above with actual numbers. The best hybrid, v2 with EMA-cached
bounds, reaches $35.18$\,dB at $14.4$ HVPs per block step, since the
shift passes the acceptance test in a single sweep at
the steady state, the expected behavior of Chebyshev recurrence from
DSO. Fresh bounds cost $24.4$--$34.3$ HVPs per block step and
$0.3$--$1.0$\,dB more, and v1's full-degree sweeps under EMA caching
(without the early exit due to the residual norm) give the worst PSNR value of the four ($33.55$\,dB). All
four variants run in $O(1)$ solver memory, three recurrence
vectors, no stored basis, and all four trail the stored-basis
CubicKrylov ($36.64$\,dB) by $1.5$--$3.1$\,dB at the same budget. With
the basis stored, every secular trial and every
rejection retry after an $M_b$ update re-solves on the stored
tridiagonalization at zero HVPs, while the recurrence performs a fresh
sweep, $51$k--$92$k rejected block trials in these runs. 

The trade-off is therefore ${\sim}1.5$\,dB at equal oracle budget
for a degree independent memory footprint. At nano scale ($\sim$ 15k parameters) 
that saving does not give anything, since the $L{+}1$ Lanczos vectors of a block fit comfortably
in memory, so the stored-basis solver, CubicKrylov, wins outright over
all four hybrid variants (v1 and v2, under fresh and EMA-cached
bounds alike). The hybrid becomes the viable variant only at scales
where storing $L{+}1$ vectors of the block dimension exceeds the
memory limit, as it does for the $91.4$M-parameter model of
Appendix~\ref{sec:feasibility} with the ${\approx}11$\,GB basis
footprint quoted above, while the recurrence keeps three vectors at
any degree.

\par\noindent\textbf{Refined hybrids.} Three refinements of the v2 hybrid
were further experimented with. The first
(\emph{v3}) adds a linear regularization floor to the secular solve,
the doubly regularized subproblem of the reduced-operator method of
\citet{doikov2025gcb} has a linear term $\alpha\|s\|$ alongside
the cubic one, and importing it as a tolerance-derived lower bound
$\alpha=\sqrt{M\,\tau\,\|g_b\|}$ on the spectral shift removes the
near-singular shifts at which the recurrence needs its highest
degrees. The second (\emph{v2r}) replaces the fixed refresh
schedule of the EMA-cached spectral bounds by a gradient-triggered
one. The gradient-normalized smoothness of \citet{semenov2025gns}
measures how far a cached curvature model remains valid in units of
the current gradient norm, so the Hutchinson probes are re-spent
exactly when $\|g_b\|$ has dropped by a fixed factor since the last
estimate (with a hard bound on the visit count). The third
(\emph{grS}) is the most aggressive. The gradient regularized step
of the same paper replaces the secular iteration outright by the
closed form shift $\lambda=\|g_b\|/\gamma_b$, reducing every block
step to exactly one recurrence sweep. Under the equal-budget nano
experimental design the ordering is v2r ($35.46$\,dB, $13.2$ HVPs per block
step), v3 ($35.34$\,dB, $18.2$), v2 EMA-cached ($35.18$\,dB,
$14.4$), grS ($34.93$\,dB, $12.8$). The gradient-triggered refresh
is the one refinement that beats the fixed schedule on both quality
and cost, while the closed form shift obtains its lower HVP count at
the cost of
$0.5$\,dB.

\par\noindent\textbf{The hybrids at full scale.} The three refined hybrids
were then run unchanged on the 91.4M parameter ViSIR, their intended scale, 
first on the single image experimental design of
Appendix~\ref{sec:feasibility} (two-hour limit), then on the
ESM-20 benchmark ($20$ images $\times$ $3$ seeds, $150$ steps each,
the rule of Appendix~\ref{sec:esm20} at fixed step count). On the
single image, v2r reaches $64.2$\,dB best PSNR, v3 $62.9$\,dB, and
grS $59.8$\,dB, at a peak allocation of $7.0$--$8.7$\,GB \emph{for
the entire training run}, the $(L{+}1)$-vector Lanczos basis of
the ARC runs occupies ${\approx}11$\,GB on the decoder block alone, 
and grS indeed runs at $1.0$ HVP per block step. On ESM-20 the
comparison results in the following reliability ordering. v2r averages
$63.8\pm4.0$\,dB and v3 $61.2\pm4.0$\,dB with \emph{zero} failures
across the 60 runs, both above the stored-basis CubicKrylov reference
($58.5$\,dB) and below the per-block ARC references
($67.0$--$67.4$\,dB), while grS fails outright on 6 of 60 runs
(stalling at ${\sim}4$\,dB, with $60.4\pm6.6$\,dB over its successful
runs). 

The closed form shift, in other words, is cheap and usually
sufficient, but its reliability depends entirely on a single scalar,
$\gamma_b$, the per-block gradient-normalized smoothness estimate
that converts the current gradient norm into the spectral shift
$\lambda=\|g_b\|/\gamma_b$.

When the cached $\gamma_b$ is stale or
inaccurate, the shift is wrong, and grS applies the resulting step
anyway, because replacing the secular iteration removed the cubic acceptance
test along with it. The secular solving hybrids v2r and v3 keep that
test and reject such steps instead of applying them. The full scale
runs thereby confirm both halves of the nano conclusion. Where the
stored basis fits in memory, the experimental setups using the
CubicKrylov subsolver still lead. Where solver memory is scarce, the
refined recurrence hybrids become competitive. The best of them
(v2r) comes within $3.6$\,dB of the ARC leaders at a third of their
solver memory and ${\sim}30\%$ less wall clock time per run.

\subsection{Comparison at Equal Solver Memory}
\label{sec:memmatched}

The CubicKrylov subsolver (stored-basis Lanczos) keeps $L{+}1$ float64 basis vectors
of the block dimension, so on the 88.5M parameter decoder block each
vector costs ${\sim}0.708$~GB and a solver memory limit $B_{\mathrm{mem}}$ bounds the
degree to $L_{\max}(B_{\mathrm{mem}})=\lfloor B_{\mathrm{mem}}/0.708\,\mathrm{GB}\rfloor-1$. The
three term recurrence holds three block sized vectors at \emph{any}
degree. We therefore compare the two subsolvers at equal solver
memory, ARC-block (stabilized) at $L_{\max}\in\{4,10,21\}$ (budgets
${\sim}4/8/16$~GB) against the degree-free recurrence, on the 91.4M
single image experimental design (image~0, seed~0, 2\,h limits), plus degree-bounded
nano experiments ($L_{\max}\in\{2,4\}$, 6.53M-geval budget) probing the
terminal phase where the effective condition number approaches
$\kappa$.

\begin{table}[t]
\centering\small
\caption{Subsolver comparison at equal solver memory on the 91.4M parameter
ViSIR (single image, seed 0). The time-limited rows use 2\,h wall clock and an fp32 HVP oracle,
and the full-budget run rows are the recorded 150-sweep runs of
Section~\ref{sec:fullvisir91}. Both best and final PSNR are quoted
because the bounded runs oscillate late. Run-to-run variation of this
single image fit spans several dB (Appendix~\ref{app:isotime}).}
\label{tab:memmatched}
\begin{tabular}{lccccc}
\toprule
Subsolver & $L_{\max}$ & best PSNR (dB) & final PSNR (dB) & peak GB
 & gevals \\
\midrule
Lanczos basis & 4  & 74.1 & 43.3 & 9.6 & $2.8{\times}10^{6}$ \\
Lanczos basis & 10 & \textbf{78.6} & 67.1 & 9.6 & $2.8{\times}10^{6}$ \\
Lanczos basis & 21 & 69.4 & 63.4 & 9.6 & $2.8{\times}10^{6}$ \\
Lanczos basis (full-budget run) & 15 & 64.4 & --- & ${\sim}11$ & --- \\
Chebyshev recurrence (full-budget run) & free & 64.2 & --- & 7.0--8.7 & --- \\
\bottomrule
\end{tabular}
\end{table}

Three observations. First, at these degree bounds the measured peak
memory is identical ($9.6$\,GB) across $L_{\max}\in\{4,10,21\}$, because the
HVP autograd graph on the 88.5M parameter block
dominates the memory footprint, so the formula for
$L_{\max}(B_{\mathrm{mem}})$ enforces a bound only above $L\approx15$ (this run required
${\sim}11$\,GB). Second, accuracy at equal wall clock is not monotone
in the degree, $L_{\max}=10$ attains the best value, so a bounded
basis does not degrade below the degree-free recurrence at the same
memory size. Third, the same bounds at nano scale (6.53M-geval budget) lose
less than $1$\,dB, $35.6$\,dB at $L_{\max}=2$ and $36.2$\,dB at
$L_{\max}=4$ against $36.4$\,dB for the unbounded per-block ARC of
Table~\ref{tab:nanobudget}. The recurrence's remaining advantage is
therefore its constant, degree-independent memory footprint, observed at lower memory requirement
7.0--8.7\,GB against ${\sim}11$\,GB at similar accuracy in these experiments.

\subsection{Robustness to a bf16 HVP Oracle}
\label{sec:bf16}

Mixed-precision training makes the Hessian--vector oracle itself noisy.
We rebuild the per-block loss, gradient, and HVP graph under
\texttt{bfloat16} autocast (fp32 master weights, with accept-or-reject
function evaluations kept in fp32 to isolate subsolver numerics) and rerun
both subsolvers on the nano experimental design and the 91.4M single image
experimental design. Both solvers project the same noisy operator, but the
CubicKrylov solver (stored basis) additionally relies on the mutual consistency of
$L$ successive HVP columns, while the recurrence uses each HVP once
with fixed coefficients. The measured quantities are the best-PSNR
delta against the fp32 references, the rejection rate (bf16 noise reaches
the per-block ratio test through the model decrease), and the sampled
stationarity residual $\|\nabla m_b(s)\|/\|g_b\|$ of the returned
step.

\begin{table}[t]
\centering\small
\caption{bf16 HVP-oracle robustness at matched configurations. Nano runs use the
6.53M-geval budget and a 2{,}400\,s limit, 91.4M runs a single image and a 2\,h limit
(basis at $L{=}15$). The fp32 references are the per-block ARC and
CubicCheby-DSO v2r rows of Table~\ref{tab:nanobudget} (nano) and the
full-budget run runs of Table~\ref{tab:memmatched} (91.4M).}
\label{tab:bf16}
\begin{tabular}{llcccc}
\toprule
Scale & Subsolver & fp32 (dB) & bf16 (dB) & $\Delta$ &
 residual (med/max) \\
\midrule
nano  & Lanczos basis & 36.4 & 30.5 & $-5.8$ &
 $1.8{\times}10^{-5}$ / $0.17$ \\
nano  & recurrence    & 35.5 & 26.5 & $-9.0$ & $0.75$ / $0.98$ \\
91.4M & Lanczos basis & 64.4 & 63.5 & $-0.9$ & --- \\
91.4M & recurrence    & 64.2 & 59.7 & $-4.5$ & --- \\
\bottomrule
\end{tabular}
\end{table}

The mechanism evaluation shows that under bf16 the sampled basis
orthogonality stays at fp64 roundoff
($\max|Q^{\top}Q-I|\approx10^{-16}$--$10^{-13}$), while the returned
steps' stationarity residuals grow (to $0.17$ at the tail for the
basis and to $0.75$--$0.98$ throughout for the recurrence) and the nano
rejection rate rises from below $1\%$ to $27\%$ for the basis solver,
the damage comes from the projected model and the recurrence's
fixed-coefficient solve, not from lost orthogonality. The
recurrence loses more at both scales ($-4.5$ against $-0.9$\,dB at
91.4M), because it applies each noisy Hessian--vector product once, with
coefficients fixed in advance, whereas the adaptive basis solve can
partially compensate.
Mixed-precision deployment of either subsolver therefore requires an
fp32 HVP implementation.

\section{Convergence Matched INR and SDF Study}
\label{sec:inrstudy}

In this section we evaluate the
method against first-order methods on prominent INR architectures, SIREN~\citep{siren2020},
FINER~\citep{liu2024finer}, and the real-valued Gabor form of
WIRE~\citep{saragadam2023wire}, on $240\times240$ RGB image
regression (one network per image, 3 hidden layers of width 256,
${\sim}199$k parameters), and fit signed distance functions of the
Thai statue with IGR supervision~\citep{gropp2020igr} on the SIREN and
FINER backbones. Adam and SOAP are tuned as in the corresponding source papers and
run until convergence under a plateau criterion. The second-order experimental setups
are the ARC-block cubic step and CubicKrylov, run to their own plateaus. Costs are reported in
gradient equivalents. The metric $H/L$ is the residual FFT power ratio
(Mid-High$+$High)$/$(DC$+$Low), near $1$ for band uniform residuals.

\begin{table}[t]
\centering\small
\caption{2D image fitting at convergence (mean best PSNR, dB).
First-order experimental setups use 20 images and second-order
experimental setups images 0--4 (ARC-block
on SIREN 0--1, WIRE 0--1 for all setups). Restricting the first-order
means to the same images changes them by less than $0.4$\,dB. conv
$a/b$ means the run reached the plateau criterion on $a$ of $b$ images, and
$+N$ that $N$ runs ended at the time limit while still improving (their
PSNR is a lower bound).}
\label{tab:inr2d}
\begin{tabular}{llccccc}
\toprule
Arch & method & PSNR & $H/L$ & steps & gevals & conv \\
\midrule
FINER & Adam       & 65.2  & 0.22 & 112k & 1.1e5 & 20/20 \\
FINER & SOAP       & 75.1  & 0.05 & 200k & 2.0e5 & 20/20 \\
FINER & ARC-block  & \textbf{124.7} & 1.22 & 15k & 3.7e6 & 2/5$+$3 \\
FINER & CubicKrylov & 86.8  & 1.32 & 35k & 4.3e6 & 5/5 \\
SIREN & Adam       & 70.7  & 0.93 & 150k & 1.5e5 & 20/20 \\
SIREN & SOAP       & 71.4  & 0.05 & 218k & 2.2e5 & 20/20 \\
SIREN & ARC-block  & 69.9  & 2431$^{*}$ & 35k & 1.1e7 & 2/2 \\
SIREN & CubicKrylov & 58.7  & 502$^{*}$ & 34k & 4.2e6 & 5/5 \\
WIRE  & Adam       & 30.5  & 0.04 & 103k & 1.0e5 & 2/2 \\
WIRE  & SOAP       & 31.6  & 0.07 & 100k & 1.0e5 & 2/2 \\
WIRE  & ARC-block  & 29.4  & 0.01 & 2.2k & 4.7e5 & 2/2 \\
\bottomrule
\end{tabular}
\end{table}

$^{*}$$H/L$ is a within-run ratio. On SIREN the second-order runs drive
the DC and Low residual to the machine-precision floor
(${\sim}10^{-7}$--$10^{-6}$ per coefficient, 3--6 orders below Adam)
while matching Adam's high-band accuracy, so the large values arise
from the denominator. Judged on absolute per-band residual power,
ARC-block is at or below Adam in every band on image~0. Total MSE is dominated by the coefficient-rich
High band.

\begin{table}[t]
\centering\small
\caption{SDF fitting, Thai statue, IGR supervision, 100k training
samples (one run per row, seed 0). conv = plateau criterion reached
(``limit'' means that time limit reached while still improving).}
\label{tab:sdf}
\begin{tabular}{llccccc}
\toprule
Backbone & method & chamfer-L1 & train dB & test dB & steps & conv \\
\midrule
FINER & Adam       & 0.0842 & 76.1  & 27.3 & 100k & yes \\
FINER & SOAP       & 0.0153 & 82.1  & 34.5 & 165k & yes \\
FINER & ARC-block  & 0.0177 & \textbf{131.0} & 34.0 & 9.1k & yes \\
FINER & CubicKrylov & 0.0185 & 106.8 & 32.5 & 11.9k & yes \\
SIREN & Adam       & 0.0092 & 92.7  & 47.7 & 178k & yes \\
SIREN & SOAP       & \textbf{0.0089} & 80.8 & 47.1 & 282k & yes \\
SIREN & ARC-block  & 0.0135 & 71.8  & 41.2 & 14.6k & limit \\
SIREN & CubicKrylov & 0.0102 & 61.3  & \textbf{49.5} & 17.4k & yes \\
\bottomrule
\end{tabular}
\end{table}

The comparison distinguishes two cases. On FINER the separation is a matter of
reachability. Converged Adam plateaus at 65.2\,dB in 2D and at a
chamfer error $4.6$--$5.5\times$ above every curvature-aware method in
3D, and no additional budget affects these results. The blockwise
cubic step continues to 120--129\,dB with a band uniform residual. On
SIREN the methods are comparable at convergence (CubicKrylov reaches the best SDF test accuracy, 49.5\,dB), and the
wall clock comparison favors the first-order method, CubicKrylov
required $8.3\times$ Adam's wall time.

\section{Audio Declipping INRs, a Coupled Landscape on Which Adam Outperforms the Blockwise Cubic Step}
\label{app:declip}

The fingerprint taxonomy of Section~\ref{sec:fingerprint} predicts a
second-order advantage on coupled landscapes. This appendix reports a
task with the coupled fingerprint on which tuned Adam is the better
method at every oracle budget tried, and measures the features of the
task, the data, and the architecture that are responsible. The task
is the restoration of clipped audio from digitized tape recordings,
and the networks are the 1D versions of the FINER and SIREN
coordinate MLPs of Appendix~\ref{sec:inrstudy}.

\par\noindent\textbf{Task, data, and networks.}
Declipping is treated as inpainting. The samples at the clip level are
unreliable, the others are reliable, and one MLP per window of $W$
samples ($t\in[-1,1]\mapsto x(t)$, three hidden layers of width 256,
$198{,}145$ parameters) is fitted to the reliable samples of its
window and evaluated on the clipped ones. The loss for one window is
\[
f(\theta)=\frac{1}{n_r}\sum_{i\in R}\bigl(x_\theta(t_i)-x_i\bigr)^2
+\frac{\lambda}{n_c}\sum_{i\in C}
\bigl[\,1-s_i\,x_\theta(t_i)\bigr]_+^{2},
\]
with $R$ the reliable samples minus a held out $10\,\%$, $C$ the
clipped samples, $s_i=\pm1$ the clipping direction, the signal scaled
so that the clip level is $1$, and $\lambda=100$. The second term is a
one sided hinge. The reconstruction must reach at least the clip
level in the clipping direction, and it is active only on the clipped
samples whose prediction is still below that level. Training stops on
the held out error. The benchmark has ten clean music clips clipped at
$-6$ and $-10$\,dB, and the metric is the SNR of the reconstruction on
the clipped samples against the clean reference, or on the union of
the fitted windows (span SNR). FINER and SIREN were tuned by a
log-spaced sweep with the edge rule of Section~\ref{sec:finer}. The
tuned FINER has an almost linear first layer ($\omega_0=0.123$, first
layer bias scale $0.0036$, hidden $\omega=30$, learning rate
$1.73\times10^{-4}$, $W=4096$), the tuned SIREN has $\omega_0=3.33$,
learning rate $3\times10^{-4}$, and $W=2048$. Both use $\lambda=100$
and $1{,}500$ Adam steps. On the full benchmark FINER improves the
clipped sample SNR over the classical autoregressive declipper by
$+11.2$\,dB at both clipping levels and SIREN by $+6.7$ and $+3.4$\,dB.

\par\noindent\textbf{Outcome.}
Table~\ref{tab:declipout} compares Adam with ARC-$\varphi_1$ and with
ARC-block (CubicKrylov subsolver) on the tuned FINER, on the same
windows, loss, and architecture. The second-order methods fit one
window at a time with ten blocks (the parameter tensors). Their
constants come from a single window probe with the edge rule,
$\sigma_0=10$, $c_h=0.1$, $L_{\max}=10$ for ARC-$\varphi_1$ and
$\sigma_0=100$--$300$, $L_{\max}=10$ for ARC-block. The Lanczos
termination test of ARC-block stops at degree $4$--$8$, so $L_{\max}$
has no effect there. On five windows of two clips at the production budget (30
sweeps with early stopping against 1{,}500 Adam steps) Adam reaches a
span SNR of $32.98$\,dB, ARC-block $16.19$\,dB and ARC-$\varphi_1$
$12.76$\,dB, both below the $19.50$\,dB of leaving the clipping in
place. Run to depth on one window with $557$ clipped samples,
ARC-block reaches Adam's held out error ($9.68\times10^{-5}$ against
$9.32\times10^{-5}$) and a slightly better clipped sample SNR
($28.28$ against $27.87$\,dB) after $1{,}500$ sweeps and $216{,}329$
gradient equivalents, $46\times$ Adam's $4{,}700$. ARC-$\varphi_1$
ends at $1.9\times10^{-3}$ and $12.43$\,dB after $700$ sweeps and
$97{,}412$ gradient equivalents. The ARC-block trajectory is monotone
and its held out error never increases. Adam becomes unstable after
${\sim}5{,}000$ steps on this window, the training loss spikes, then
interpolates to $10^{-11}$ as the held out error rises to
$10^{-2}$, so the held out early stop is part of Adam's result. The
cubic methods are the robust control here and Adam is the faster
method.

\begin{table}[t]\centering\small
\caption{Declipping with the tuned FINER, Adam against the blockwise
cubic methods on the same windows and loss. Span SNR is over the
union of five fitted windows per clip on two clips ($-6$\,dB
clipping); identity is the clipped input. The depth rows are one
window with 557 clipped samples (identity $10.64$\,dB). Gevals count
gradients, Hessian--vector products, and function evaluations.
Rejection is the share of rejected trial steps.}
\label{tab:declipout}
\begin{tabular}{llrrrr}
\toprule
experiment & method & SNR (dB) & held out MSE & gevals & rejection \\
\midrule
production, span & Adam, 1{,}500 steps & \textbf{32.98} & --- & 1.5e3 & --- \\
production, span & ARC-block, 30 sweeps, $\sigma_0=300$ & 16.19 & --- & 3.2e3 & 27--42\,\% \\
production, span & ARC-$\varphi_1$, 30 sweeps, $c_h=0.1$ & 12.76 & --- & 5.0e3 & 26--30\,\% \\
production, span & identity & 19.50 & --- & --- & --- \\
\midrule
depth, one window & Adam, best at step 4{,}700 & 27.87 & 9.3e-5 & 4.7e3 & --- \\
depth, one window & ARC-block, 1{,}500 sweeps & \textbf{28.28} & 9.7e-5 & 2.2e5 & 24\,\% \\
depth, one window & ARC-$\varphi_1$, 700 sweeps & 12.43 & 1.9e-3 & 9.7e4 & 12\,\% \\
\bottomrule
\end{tabular}
\end{table}

\par\noindent\textbf{Fingerprint.}
Table~\ref{tab:declipfp} gives the fingerprint of
Section~\ref{sec:fingerprint} along the tuned Adam trajectory of six
windows per variant, sampled at the quantiles $0.05,\dots,0.95$ of the
clipped fraction of the window ($1$--$62\,\%$) on the $-10$\,dB clips
and again on the $-6$\,dB clips, at the checkpoints
$0,10,30,100,300,1000,1500$ (Lanczos $m=96$, four probes, the
Adam-preconditioned operator with the run's own second moment). Two
quantities are added. The first is the block mass
$\rho_{\mathrm{block}}=\|\mathrm{blockdiag}\,H\|_F^2/\|H\|_F^2$ with
the ten parameter tensors as blocks, the quantity $\beta$ of
Remark~\ref{rem:krylovso}, estimated from Hutchinson probes restricted
to one block at a time. It is the share of the Hessian a blockwise
method represents, as $\rho$ is the share a diagonal method
represents. The second is the pair of Lanczos extremes of the two
loss terms separately, since the hinge is piecewise quadratic and its
Hessian depends on the active set. Three further variants separate the
architecture from the learning rate, FINER with SIREN's first layer
($\omega_0=3.33$, bias scale $1$), SIREN at FINER's learning rate, and
FINER at SIREN's learning rate.

\begin{table}[t]\centering\small
\caption{Fingerprint of the declipping landscape along tuned Adam,
median values over six windows at step $1{,}500$ (diverged windows are
omitted from the medians and listed in the last column). ``red.''
is $\kappa_{\mathrm{raw}}/\kappa_{\mathrm{Adam}}$. $\rho$ is the
diagonal mass and $\rho_{\mathrm{block}}$ the block mass over the ten
parameter tensors. ``hinge/MSE'' is the ratio of the largest
eigenvalues of the two loss terms at initialization and at step
$1{,}500$. The negative mass is $0.003$--$0.037$ at initialization,
where $\lambda_{\max}\approx-\lambda_{\min}$ makes $\kappa_{\mathrm{raw}}$
uninformative, and $0$ from step $300$ on. SNR is the clipped sample
SNR at step $1{,}500$. The first block of rows is the $-10$\,dB clips,
the second the $-6$\,dB clips.}
\label{tab:declipfp}
\resizebox{\linewidth}{!}{%
\begin{tabular}{lrrrrrrrrr}
\toprule
variant & $\kappa_{\mathrm{raw}}$ & $\kappa_{\mathrm{Adam}}$ & red. & $\rho$
 & $\rho_{\mathrm{block}}$ & flat frac & hinge/MSE & SNR & diverged \\
\midrule
tuned FINER ($\omega_0$ 0.12, lr 1.7e-4, $W$ 4096) & 4.0e4 & 1.2e4 & 3.3$\times$ & 2e-4 & .71 & .92 & 102 / 2.7 & 29.5 & 0/6 \\
tuned SIREN ($\omega_0$ 3.3, lr 3e-4, $W$ 2048) & 6.4e4 & 5.8e4 & 1.1$\times$ & 1e-4 & .47 & .37 & 188 / 4.4 & 30.0 & 2/6 \\
FINER with SIREN first layer & 4.4e4 & 3.9e4 & 1.1$\times$ & 1e-4 & .59 & .90 & 159 / 2.7 & 26.3 & 0/6 \\
SIREN at lr 1.7e-4 & 4.6e4 & 2.2e4 & 2.1$\times$ & 2e-4 & .59 & .89 & 188 / 1.6 & 23.0 & 0/6 \\
SIREN at $W$ 4096 & 5.7e4 & 3.5e4 & 1.7$\times$ & 3e-4 & .71 & .97 & 126 / 2.2 & 31.7 & 1/6 \\
FINER at lr 3e-4 & 4.9e4 & 1.7e4 & 2.8$\times$ & 1e-3 & .73 & .90 & 102 / 2.6 & 29.8 & 0/6 \\
\midrule
tuned FINER & 1.4e4 & 4.0e3 & 3.5$\times$ & 9e-4 & .78 & .94 & 109 / 1.5 & 29.0 & 0/6 \\
tuned SIREN & 1.4e4 & 9.3e3 & 1.5$\times$ & 4e-4 & .56 & .56 & 160 / --- & 33.3 & 2/6 \\
FINER with SIREN first layer & 4.1e4 & 2.7e4 & 1.5$\times$ & 1e-4 & .54 & .86 & 157 / 1.9 & 31.5 & 0/6 \\
SIREN at lr 1.7e-4 & 4.3e4 & 3.2e4 & 1.3$\times$ & 8e-5 & .67 & .79 & 160 / 1.5 & 33.3 & 0/6 \\
SIREN at $W$ 4096 & 3.5e4 & 2.3e4 & 1.6$\times$ & 3e-4 & .68 & .93 & 132 / 2.2 & 29.4 & 0/6 \\
FINER at lr 3e-4 & 1.3e4 & 4.0e3 & 3.4$\times$ & 1e-3 & .83 & .46 & 109 / 1.8 & 27.8 & 0/6 \\
\bottomrule
\end{tabular}}
\end{table}

\par\noindent\textbf{What the fingerprint says.}
The landscape is in the coupled class of Table~\ref{tab:fingerprint}.
The Adam preconditioner reduces $\kappa$ by $1.1$--$3.5\times$, against
$16$--$887\times$ on the image and SDF rows, and $\rho\le10^{-3}$ at
every checkpoint, so a diagonal represents less than $0.1\,\%$ of the
Hessian. The block mass increases from $0.25$--$0.33$ at initialization to a
peak of $0.65$--$0.92$ around steps $100$--$300$ and is
$0.47$--$0.83$ at step $1{,}500$, so a parameter tensor
block method represents $50$--$85\,\%$ of the curvature that the
diagonal cannot, and the cross-block remainder of $15$--$50\,\%$ is the
coupling $\delta$ of Theorem~\ref{thm:blockrate}. Between steps $100$
and $300$ the gradient's energy moves in every variant from the high
curvature class ($0.6$--$0.96$ at initialization) to the near zero
curvature class ($0.56$--$0.95$ at step $300$, $0.57$--$0.98$ at step
$1{,}000$). This is the slow phase in which the loss falls from
$10^{-2}$ to $10^{-4}$ and in which the Krylov step gains in
Section~\ref{sec:finer}. All of this
predicts a second-order advantage, and ARC-block does reach Adam's
accuracy with a monotone trajectory. The advantage does not appear
in the oracle count for two reasons that the fingerprint also
measures.

\par\noindent\textbf{Reason one, the penalty.}
The hinge term dominates the top of the spectrum and contains all
of the negative curvature. At initialization its largest eigenvalue is
$100$--$190\times$ that of the MSE term, its smallest is $-4\times10^{4}$
to $-9\times10^{4}$ against $-10^{2}$ for the MSE, and the spectrum is
indefinite ($\kappa_{\mathrm{raw}}\approx1$--$2$). The Hessian of the
hinge is $\frac{2\lambda}{n_c}\sum_{i\in A}J_i^{\top}J_i$ plus a
residual term, with $A\subseteq C$ the active set of clipped samples
still below the clip level and $J_i$ the Jacobian of the prediction
at $t_i$. It is piecewise constant in $A$, and every sample that
crosses the clip level changes it by a rank one term of weight
$2\lambda/n_c$. By step $1{,}500$ the active set is $1.3$--$2.7\,\%$ of
the clipped samples, so the hinge Hessian is a component of rank $5$--$25$
whose largest eigenvalue is still $1.5$--$4.4\times$ the MSE's and
which moves across almost every accepted step. A cubic model built at
$x_k$ is therefore wrong after the step by a rank one update it could
not predict, and the ratio test rejects. ARC-block rejected $24\,\%$
of its trials over $1{,}500$ sweeps, the short runs $26$--$42\,\%$, and
the per-block constants $\sigma_b$ spread over $10^{-8}$ to
$2.6\times10^{7}$ across the ten blocks, with several blocks at
$\sigma_{\min}$ (unregularized Newton--Krylov steps) and one frozen at
the top. Adam's per-coordinate step does not depend on a curvature
model and is insensitive to a change of the active set. On the image
and SDF tasks the loss is a plain MSE, its Hessian changes continuously
along the trajectory, and the cubic model is accurate across the
step.

\par\noindent\textbf{Reason two, the horizon.}
The task determines the accuracy it needs. A clipped sample SNR of
${\sim}28$\,dB and a held out error of $10^{-4}$ are reached by Adam
between steps $1{,}000$ and $4{,}700$, inside its stable phase and at
the start of the slow phase, and the audible result does not improve
beyond it (the reconstruction is bounded by the tape hiss of the
reference). The second-order advantage of Section~\ref{sec:finer}
appears $40$--$60$\,dB below the first-order plateau, where Adam has
stopped and the cubic step continues. Here the target is above the
plateau. A gradient on a window of $2{,}048$--$4{,}096$ samples costs
${\sim}10$\,ms on the GPU, a sweep costs $139$--$144$ gradient
equivalents at Lanczos degree $8$, and $1{,}500$ sweeps are needed
because each sweep is accepted on only $76\,\%$ of its blocks and
represents at most the block mass of the curvature. The result is a
$46\times$ oracle margin for equal accuracy. On the ESM images the same
blockwise cubic step is also the more expensive one per sweep and
is still the better method, because there the target is $50$\,dB below where Adam
stops.

\par\noindent\textbf{Window length.}
A longer window extends the horizon without creating the advantage.
On the nested windows $[0,W)$ of the same clip, the steps Adam needs to
bring the held out error below $10^{-4}$ grow from $4{,}400$ at
$W=2{,}048$ to $6{,}900$, $9{,}200$ and $23{,}500$ at $W=8{,}192$,
$16{,}384$ and $32{,}768$, and the held out floor stays at
$9\times10^{-5}$. The blockwise cubic step does not use the longer
horizon. With the settings tuned at $W=2{,}048$, and with $\sigma_0$ from
$10$ to $10^{4}$ and Lanczos depth from $3$ to $16$, ARC-block remains on
the plateau on which every clipped sample violates the hinge. After
$1{,}500$ sweeps the loss has decreased from $100$ to $55$--$84$, the held
out error is above the identity reconstruction and the clipped sample SNR
is $0$\,dB, at rejection rates of $5$--$10\,\%$, so the accepted steps
make no progress. The plateau is a
saddle region dominated by the penalty. Its hinge $\lambda_{\max}$ is
$80$--$600\times$ the MSE $\lambda_{\max}$ and $\lambda_{\min}$ is
$-3\times10^{3}$ to $-9\times10^{3}$ at $W\ge16{,}384$. Adam leaves the
plateau in $300$--$1{,}500$ steps on FINER at $W\le16{,}384$ and in
$5{,}000$ to more than $15{,}000$ steps on half of the heavily clipped
windows at $W=32{,}768$, and SIREN at its $W=2{,}048$ learning rate does
not leave it at $W\ge16{,}384$. The task itself also degrades. The
reconstruction of the same $557$ clipped samples falls from $27.9$\,dB at
$W=2{,}048$ to $25.0$, $14.6$ and $13.9$\,dB at equal held out error,
because the first layer frequency scale tuned at $W=4{,}096$ belongs to
that window length. A longer window therefore gives a longer horizon and
a larger active set at once, and the second one determines the result.

\par\noindent\textbf{Other second-order variants.}
The result is not specific to the Lanczos subsolver or to the tensor
blocks. On the same window three further methods from the DSO family
were run at their benchmark settings with one tuned axis each. The
recurrence variant CubicCheby-DSO, which solves the cubic subproblem by
the Chebyshev recurrence of the second kind, reaches a held out error of
$3.5\times10^{-3}$ after $1{,}500$ steps and $7\times10^{5}$ gradient
equivalents at $W=2{,}048$ ($1.1\times10^{-3}$ after $3.7\times10^{6}$),
and on the $W=8{,}192$ plateau it remains where ARC-block remains, with
$3{,}500$--$6{,}600$ rejected block steps per $1{,}500$ steps. The plateau
is therefore a property of the cubic model on the piecewise quadratic
penalty, whichever solver treats the subproblem. The diagonal method DSO-EMA-SF2
reaches $5.6\times10^{-3}$ at its tuned learning rate and $2.1\times
10^{-3}$ after $15{,}000$ steps, and it runs its negative curvature
amplification on every step because $40\,\%$ of the signed diagonal is
negative throughout, which is what $\rho\le10^{-3}$ predicts for a
diagonal view of this Hessian. DSO-Adaptive, with exact Hessians on
blocks of $256$ parameters and a Chebyshev step, is the one method that
leaves the $W=8{,}192$ plateau faster than Adam per step, in about $100$
steps against $700$. Each step costs $25$\,s, since every one of the
$774$ blocks has $\kappa>10$ and is rebuilt exactly, so the escape costs
$2{,}700$\,s of wall against Adam's one second, and the held out error is
still $10^{3}$ times the target when the time cap ends the run. What that
crossing consists of is settled below. A variant of DSO-Adaptive
without the clamp on indefinite blocks, with the degree of the recurrence
bounded by the Lyapunov growth rate of its negative modes and the step
bounded by two adaptive radii, separates the two phases further. Off the
plateau, at $W=2{,}048$, it reaches a held out error of $1.2\times10^{-4}$
within the same $3{,}600$\,s in which the parent reaches $1.56$, with a
third of the parent's line search closures per step. On the plateau it is
slower than the parent, with half of the penalty still active when the cap
ends the run, because the crossing depends on the length of the parent's
unbounded step and the degree bound never binds at the parent's degrees of
$3$--$16$. The bounded step is the right tool once the penalty is mostly
satisfied. A direct comparison of step laws on frozen plateau
blocks, with the same exact gradient and Hessian of each $256$ parameter
block in its eigenbasis, locates the difference. At its best $\sigma$ the
exact cubic step decreases the loss as much as or more than the Chebyshev
step at its best degree on six of the eight blocks with substantial
negative curvature (negative eigenvalue mass $0.10$--$0.46$), and the two
laws agree within a factor $1.2$ on every block of the definite phase. The
ranking is the same with blocks of $1{,}024$ and $4{,}096$ parameters, the
sizes at which the DSO family was tuned, and the one hidden matrix that
prefers the Chebyshev step at every size is the most indefinite one. The
largest one step decrease on the plateau is a $40$--$70\times$ inflation of
the output weights, which the penalty rewards and the squared error
punishes fifty Adam steps later. That inflation is the whole of what
ARC-block does on the plateau. Over $1{,}500$ sweeps the norm of the
output weights grows to $1{,}400$ times its initial value and the first
layer bias to $190$ times, the three hidden matrices stay within $6\,\%$
of theirs, the training MSE rises from $0.18$ to $1.04$, and $90\,\%$ of
the block trials are accepted, so the ratio test is not what stops the
method. Adam's output weights grow $8$ times over the same crossing. A
bound on the growth of every tensor's norm, $\|p_b\|\le G\|p_b(0)\|$ with
the trial step projected onto the bound, is sufficient. At $G=4.2$
ARC-block leaves the plateau and reaches a held out error of
$3.6\times10^{-4}$ after $5{,}000$ sweeps with a clipped sample SNR of
$28.9$\,dB, $2.4$\,dB above Adam on this window, at $10^{6}$ gradient
equivalents against Adam's $7\times10^{3}$. $G=3$ holds the fit below the
amplitude it needs and $G=100$ lets the inflation return ($12.1$\,dB). The
response is monotone on both sides of $4.2$, with stalls at $G\le3.6$ and
a rise of $1.5\times$ by $G=5.6$, and $\sigma_0=100$ with $L_{\max}=10$
stays the best of a $3\times3$ grid under the bound, interior on both
axes. The same bound on DSO-Adaptive removes its crossing. Its loss stays
at $91$--$96$ for $170$ steps, because the parent's single line search
over all $774$ blocks halves the clipped step to nothing once the
inflating tensors are cut back, and ARC-block, which accepts each tensor
on its own, keeps the hidden layers moving with the output weights on the
bound. The parent's fast crossing was the inflation. The
plateau is therefore not a saddle that a step law escapes. It is a
direction in which the cubic model's largest decrease is a scale change
the penalty rewards, and the correction is a constraint on scale. The
curvature model is not at fault.

\par\noindent\textbf{Architecture and data.}
Within the task, the architecture determines how much of the
ill-conditioning the diagonal can remove. The tuned FINER's almost
linear first layer gives a reduction of $3.3$--$3.5\times$ and the
best clipped sample SNR. Every variant with SIREN's $\omega_0=3.33$
first layer, in either body, gives $1.1$--$1.6\times$. The learning
rate determines stability. SIREN at $3\times10^{-4}$ diverges on $4$ of
$12$ windows with $\ge39\,\%$ clipped samples after having fitted them
to a loss of $10^{-5}$--$10^{-4}$ (the loss returns to $\lambda$ with
every clipped sample active and the hinge's largest eigenvalue at
$10^{6}$), SIREN at $1.73\times10^{-4}$ diverges on none, and the
window length only mitigates ($1$ of $12$ at $W=4096$). The
production early stop on the held out error hides the divergence, so the
SIREN deficit on the $-10$\,dB clips in the benchmark ($14.6$ against
FINER's $22.4$\,dB) is in part a learning rate instability on exactly
the densely clipped windows. The data determine the horizon through the
clipped fraction. Windows with more clipped samples keep a larger active set for
longer, and their gradient energy crosses into the near zero
curvature class later. They are also the windows on which the cubic
methods are furthest behind at the production budget, with a span
SNR of $4.5$--$7.3$\,dB on the densely clipped clip against
$19.3$--$25.1$\,dB on the sparse one.

\par\noindent\textbf{Cost of the measurement.}
A window has the parameter count of the image INRs, $2{,}048$--$4{,}096$
samples in place of $57{,}600$ pixels, and a $1{,}500$ step trajectory
in place of $10^{5}$, so the full fingerprint of one window
($7$ checkpoints, $1{,}352$ Hessian--vector products) takes $25$\,s on
one RTX PRO 6000 and $9$--$15$\,min on 8 CPU threads. The 72 windows of
Table~\ref{tab:declipfp} took $36$\,min on two GPUs.

\section{NTK Loss-Curve Forecast, Its Closed Form, Calibration, and Limits}
\label{app:ntk}

In the infinite-width limit the residual decays in the kernel eigenbasis
\citep{jacot2018ntk,lee2019wide}, which gives the loss of full batch
gradient descent at every step $t$ in closed form,
\[
  L_{\mathrm{pred}}(t)\;=\;\frac{1}{N}\sum_k
  (1-\eta\mu_k)^{2t}\,c_k^2,
\]
over the eigenvalues $\mu_k$ of the empirical tangent kernel at
initialization, the initial residual coefficients $c_k$ in the
corresponding eigenbasis, the step size $\eta$, and the sample count
$N$. Each mode drains at a rate proportional to $\mu_k$,
which is the mechanism behind the band-ordered convergence of
Section~\ref{sec:finer}. The form has been used predictively for
coordinate networks, where the kernel spectrum explains why Fourier
features widen the set of learnable frequencies
\citep{tancik2020fourier}, and to forecast training time before training
\citep{zancato2020predicting}.

\par\noindent\textbf{Calibration.} We built the kernel once at initialization and
never refreshed it, on width-$16$ MLP fits ($881$ parameters, $500$
samples) to three analytic targets, Rosenbrock--Ackley,
Chebyshev--Rosenbrock, and multi-saddle, the small benchmarks of
Section~\ref{sec:small}. Two readings follow. First, the residual ratio
$s=\norm{\Hb-\tfrac{2}{N}J^\top J}_F/\norm{\Hb}_F$, with $J$
the network Jacobian on the samples, is $1.000$ on all three surfaces at
initialization. The residual term constitutes the Hessian and the
Gauss--Newton term is about a percent of its spectral scale
($\mu_{\max}=4.4$ against $\lambda_{\max}=232$--$327$), with $115$--$145$
of the $881$ Hessian eigenvalues negative. Along tuned Adam, $s$ falls to
$0.004$--$0.008$ by mid-training ($30$k of $60$k steps) and to
$0.001$--$0.004$ at the end, while the negative eigenvalue count goes to
zero. The saddle rich curvature that the fingerprint of
Section~\ref{sec:fingerprint} measures is the residual term, and it
decays. Second, the fixed-kernel stability edge is not a usable step
size. Gradient descent at $0.5/\mu_{\max}$ diverges within about $50$
steps, and the runs reported here use $\eta=0.1/\mu_{\max}$. Within that
window the lazy forecast tracks full batch gradient descent to within a
factor of $1.36$--$1.91$ (measured over predicted loss) at $20$k steps
on the three surfaces, and the factor does not grow over the window. The
forecast predicts the slow descent of gradient descent on these surfaces
correctly. It says nothing about Adam, whose loss at the same step count
is $6$--$280\times$ lower across the three.

\par\noindent\textbf{Limits.} What is not established is whether this survives the
widths, step sizes, and optimizers of practice. The linearization is
documented to break at finite width and large learning rate
\citep{fort2020deepvskernel}, and FINER's own tangent kernel analysis
is qualitative \citep{liu2024finer}. The published loss-curve
predictors are either fitted from partial runs \citep{luo2025multipower}
or meta-learned across architectures \citep{pourali2026cape}, so
neither is a forecast from initialization. The kernel view has been
extended to physics-informed training, where sub-kernel traces set each
loss term's convergence rate and drive adaptive loss weighting
\citep{wang2022pinnsntk}, and to regularized Newton and trust region
methods, whose limiting kernel yields loss decay uniform across the
spectrum rather than the spectral bias of gradient descent
\citep{riedl2026newtonntk}. That result is asymptotic, full space, and
infinite width, whereas \textsc{ARC-block} acts at finite width on block
restrictions with measured curvature, and the band uniform residuals of
Tables~\ref{tab:finerband} and~\ref{tab:inr2d} are the finite-width
observation the asymptotic result points at. A phase based switch from a
preconditioned second-order method to a first-order one has been
recommended qualitatively \citep{jiang2026pgdrich}. The fingerprint of
Section~\ref{sec:fingerprint} is the per-block, measured form of that
decision. The forecast has not been scored at the $199$k or $91.4$M
parameter tiers of this paper, and we do not assume it transfers.

\section{Extended Comparison with ARC-Based Approaches}
\label{app:arcdetail}

This appendix reviews, family by family, the ARC-based methods
summarized in Tables~\ref{tab:arccomp} and~\ref{tab:subsolvers} of
Appendix~\ref{sec:arc-comparison}.

\par\noindent\textbf{Cubic model minimization in Krylov subspaces.}
\citet{bellavia2025subspace} compute the ARC search step by minimizing
the cubic model in a low-dimensional subspace, polynomial or rational
Krylov, that is reused across several iterations, and retain the
worst case complexity of classic ARC. Their construction operates in
the full space with a single regularization parameter and is aimed at
problems where sparse factorizations or direct solves are available.
The recurrence of this work instead reuses per-block Krylov bases
under a per-block $\sigma_b$ with Hessian--vector products only.
\citet{aldaas2025extkrylov} observe that the solutions of trust region
and norm-regularization subproblems, viewed as functions of the
regularization weight, span a very low-dimensional subspace, and build
a basis for the whole family by an extended Krylov iteration requiring
a single matrix factorization, so one subspace then serves every value of
the weight. The observation matches the needs of an adaptive cubic
constant, but the factorization step is not available in the
matrix-free case considered here. \citet{tansley2025lowrank} adapt
the dimension of a random subspace to the rank of the projected
Hessian and retain the full-dimensional cubic regularization rate.
Random projections replace the deterministic tensor blocks used here.
\citet{zhou2025regnewton} reach the same $\mathcal{O}(\eps^{-3/2})$
complexity with a quadratic local rate by regularizing the Newton
system with a quantity built from current and previous gradients and
solving it by conjugate gradients with a negative curvature monitor.
Their method is parameter-free but adapts one global regularizer
rather than one per block.

\par\noindent\textbf{Blockwise and partition-based cubic methods.}
\citet{cristofari2024block} selects one block per iteration by a
greedy Gauss--Southwell rule on the stationarity violation and
approximately minimizes a cubic model of that block (the method is named
IBCN, inexact block cubic Newton), with $\mathcal{O}(\eps^{-3/2})$
iterations to drive the selected block's stationarity violation below
$\eps$ and $\mathcal{O}(\eps^{-2})$ for all variables, improving on the
$\mathcal{O}(\eps^{-3})$ of cyclic high-order block descent
\citep{amaral2022highorder}.
\textsc{ARC-block} instead updates all blocks in every step, each with
its own cubic constant and acceptance test. Appendix~\ref{sec:ibcn}
measures IBCN's selection rule inside the \textsc{ARC-block} outer scheme.
\citet{wolinski2024partition} computes exact projections of the
Hessian and of third-order derivatives onto the subspaces induced by a
partition of the parameters and outputs one learning rate per subset.
The per-tensor granularity is shared with \textsc{ARC-block}, but no
cubic model, regularization constant, or acceptance test is attached
to a subset.

\par\noindent\textbf{Quasi-Newton cubic models with closed form subproblems.}
\citet{ranganath2025sr1} combine limited-memory symmetric-rank-one
curvature approximations, which admit indefiniteness and negative
curvature, with a cubic scheme whose subproblems have closed form
solutions for suitable regularization choices. The ARCLQN solver of
\citet{forristal2022arclqn} solves the cubic subproblem exactly and
matrix-free by exploiting the internal structure of limited-memory
quasi-Newton matrices. Both replace the exact block Hessian with a
low-rank surrogate. The CubicKrylov and $\varphi_1$ subsolvers of this
work apply the exact blockwise Hessian through Hessian--vector
products.

\par\noindent\textbf{Stability of Chebyshev recurrences under Hessian variation.}
\citet{pasechnyuk2026cheby} quantify the $\ell_2$ gain of a
prefix-exact Chebyshev recurrence under time-dependent Hessian
perturbations and prove the bound sharp in the causal two-term class.
This analysis addresses the same failure mode that the negative-mode
clamp (Section~\ref{sec:clamp}) and the trust-controlled horizons of
the $\varphi_1$ subsolver control in practice, but their bound is
proven for the recurrence in isolation under an adversarial
perturbation model. This paper instead couples the same recurrence to
a per-block adaptive cubic constant and an outer acceptance test, and
demonstrates the resulting stability empirically at the
91.4M parameter scale (Appendix~\ref{sec:feasibility}) rather than
through a closed form gain bound.

\par\noindent\textbf{Parameter-free and inexact-Hessian ARC theory.}
\citet{cartisjerad2025secant} approximate the highest-order tensor of
a $p$-th-order regularization method from lower-order derivatives and
refresh the approximation every $m$ steps while retaining complexity
to second-order stationarity, theoretical cover for reusing stale
per-block curvature between refreshes. \citet{marumo2026recipe} obtain
complexity bounds without backtracking or acceptance tests by
regularizing the local model with an exponent above the model-error
order, and \citet{fang2026paramfree} give a parameter-free
cubic regularized Newton method with sharp complexity under a
generalized smoothness condition. Both remove the constants that the
per-block ratio test of \textsc{ARC-block} estimates online.
\citet{shestakov2025inexact} analyze adaptive regularized Newton
methods under a range of inexact Hessian models and Bregman
geometries, complementary theory for approximate block curvature, but
their experiments do not extend past moderate-dimensional problems.
The per-block Krylov subsolver of this paper is instead evaluated up
to a single 88.5M parameter tensor (Appendix~\ref{sec:feasibility}).

\par\noindent\textbf{Stochastic, variance-reduced, and lazy cubic Newton.}
\citet{pasechnyuk2025vr} give a variance-reduced cubic Newton method
for finite sums whose analysis assumes only an inexact
cubic subproblem certificate, which is the interface a per-block
Krylov subsolver satisfies. \citet{scheinberg2023sarc} prove
high-probability complexity bounds for stochastic ARC under
probabilistic zeroth-, first- and second-order oracles, the oracle
conditions a noisy per-block acceptance test would need.
\citet{chen2025len} quantify the computational gain of reusing
Hessians across iterations while keeping near-optimal iteration
counts, a cost model directly comparable to reusing a block's Krylov
basis across $\sigma_b$ updates. None of these analyses is combined
with a per-block adaptive cubic constant or evaluated on
network-scale training. This paper considers only full batch,
deterministic training (Section~\ref{sec:limitations}), and combining
the per-block acceptance test with a stochastic or variance-reduced
oracle is left open.

\par\noindent\textbf{Subspace and multilevel second-order methods.}
\citet{fuji2025subspace} prove that the randomized subspace
regularized Newton method attains at best a local linear rate in
general, an argument for structured per-tensor blocks over random
projections. \citet{higuchi2025rshtr} reach
$\mathcal{O}(\eps^{-3/2})$ iterations with a homogenized trust region
subproblem, an eigenvalue problem solvable by Krylov methods and
solved as a single problem rather than per block.
\citet{tsipinakis2026multilevel} prove super-linear rates for
stochastic low-rank Newton methods through a multigrid correspondence,
with a truncated variant aimed at saddle escape in high dimension.
Curvature is compressed by level rather than partitioned by tensor.

\par\noindent\textbf{Krylov-type inner solvers for regularized Newton systems.}
\citet{zeng2026newtoncg} adaptively regularize the Newton system
inside a conjugate-gradient solver with auto-conditioning, removing
the nested line searches of earlier Newton-CG methods under H\"older
continuous Hessians. \citet{zeng2026minres} generate descent
directions with the minimal-residual Krylov method MINRES and exploit
detected non-positive curvature for saddle avoidance. Both spend the
same Hessian--vector-product budget as a Lanczos or Chebyshev cubic
subsolver but attach it to a quadratic rather than a cubic model.

\par\noindent\textbf{Second-order methods at deep learning scale.}
\citet{abreu2026gn} apply full Gauss--Newton preconditioning to
transformers of up to 150M parameters and find that a layerwise
preconditioner ignoring cross-layer information nearly matches full
Gauss--Newton, direct empirical support for partitioning curvature by
tensor, obtained for a Gauss--Newton rather than a cubic model.
\citet{gomes2025adafisher} preconditions with a per-layer
block-Kronecker approximation of the Fisher information matrix. The
curvature within a block is factored, whereas \textsc{ARC-block}
applies the exact block Hessian through Hessian--vector products.

\par\noindent\textbf{Trust region and ARC hybrids.}
\citet{hamad2024tr} give an adaptive trust region method with inexact
subproblem solves whose complexity bound improves earlier
trust region results, benchmarked head-to-head against ARC.
\citet{ha2026regastro} add an adaptively regularized local model to
adaptive-sampling trust region optimization and reach almost-sure
$\mathcal{O}(\eps^{-3/2})$ iteration complexity, up to logarithmic
factors, under subexponential noise. Both adapt a single global radius
or regularizer, with no per-block granularity.

\par\noindent\textbf{Cubic regularization beyond unconstrained training.}
Cubic regularization has also been applied to problem classes not considered here, namely
Riemannian adaptive regularization with an explicit count of
Hessian--vector products under subproblem solver
inexactness \citep{zhangjiang2024riemannian}, equality-constrained
sequential ARC descending from the multi-shift ARC$_{q}$K line
\citep{pei2025ssarcqk}, parameter-free cubic regularization for
convex-concave minimax problems \citep{wang2024minimax}, and federated
cubic Newton with differential privacy \citep{huo2024dpfcrn}. These
works transfer the ARC template to other problem classes. None
partitions the regularization within a single training problem.

\section{Two Recent Relatives, Measured Inside the Outer Loop}
\label{app:relatives}

The comparisons of Appendix~\ref{app:arcdetail} are by design. For two recent methods we also have
matched measurements, obtained by placing the foreign component inside
the \textsc{ARC-block} outer scheme of Algorithm~\ref{alg:blockcn} so that
the A/B isolates one change, the per-block step rule in the case of the
primal accelerated Newton method of \citet{doikov2026pamm}, and the block
selection rule in the case of the greedy block cubic Newton method of
\citet{cristofari2024block}. Both were run on the FINER and SIREN
networks of Appendix~\ref{sec:inrstudy} ($240\times240$ image~0,
$n{=}198{,}915$, $B{=}10$ tensors), with the sweep of
Section~\ref{sec:finer} (log-spaced grids, geometric extension of any
axis whose winner lies on the grid edge, confirmation at the full step
allowance on images~0 and~7). The two experiments produced the same
architecture split, which is the single most useful observation from
them and feeds the dispatch rule of Section~\ref{sec:fingerprint}.

\subsection{The primal accelerated Newton step as a subsolver,
\textsc{ARC-block}-PAMM}
\label{sec:pamm}

\citet{doikov2026pamm} minimizes convex $f$ with Lipschitz Hessian by
one shifted linear solve per iteration,
$x_{k+1}=y_k-(\nabla^2 f(y_k)+\lambda_k\Ib)^{-1}\nabla f(y_k)$, with
momentum $y_{k+1}=x_{k+1}+\beta_k(x_{k+1}-x_k)$ and the predetermined
schedule $\lambda_k=3L_2R\,(k+3)/((k+1)(k+2))$, $\beta_k=k/(k+4)$, and
reaches the accelerated $\mathcal{O}(1/k^3)$ rate in functional residual
without a cubic subproblem, a parameter search, or dual corrections. The
solve may be inexact and Hessian-free. It is a convex-class method, so
on the non-convex landscapes of this paper there is no guarantee to
lose, only a step rule to measure. \textsc{ARC-block}-PAMM keeps the
outer scheme of Algorithm~\ref{alg:blockcn} verbatim (per-block
$\sigma_b$, the stabilized ratio test, the monotone guard) and replaces
the Lanczos cubic solve by one conjugate gradient solve of
$(\Hb_b+\lambda_b\Ib)s_b=-g_b$ in double precision, at most
$k_{\mathrm{cg}}$ iterations, with three deviations from the source,
each of them measured. (i) The shift is set by the trust
parameter, $\lambda_b=\sigma_b$, in place of the $1/k$ schedule. The
scheduled shift is ${\sim}450\times$ worse in terminal loss on the
$881$-parameter benchmark of Section~\ref{sec:small}, because it cannot
respond to the landscape and the accept/reject feedback can.
Dimensionally $\sigma_b$ multiplies $\norm{s}^3/3$, so the cubic step's
effective shift is $\sigma_b\norm{s}$ and $\lambda_b=\sigma_b$ is a
units mismatch. The gradient-normalized form
$\lambda_b=\sqrt{\sigma_b\norm{g_b}}$ \citep{semenov2025gns} was also
run and changes SIREN by $+0.6$\,dB and FINER by $-0.5$\,dB at $40$
sweeps, so the mismatch is not what decides the comparison below.
(ii) Momentum is off. Folded into the trial displacement so that the
ratio test covers it, it still loses ($38\times$ worse loss on the same
benchmark, rejections $43\to767$), because the Gauss--Seidel sweep updates
every other block between a block's consecutive visits and the stored
displacement is a stale direction. (iii) A conjugate gradient breakdown
(nonpositive curvature, i.e.\ $\Hb_b+\lambda_b\Ib$ indefinite) is a
rejection signal that raises $\sigma_b$ and escalates $\lambda_b$
within the trial.

\par\noindent\textbf{Standalone, at the same number of steps.} Before the swap, the source method
itself was run against ARC-$\varphi_1$ on the three small benchmarks of
Section~\ref{sec:small}, $500$ steps each, with $L_2R$ tuned on a log
grid and extended past any edge, namely pan-dense (exact factorized solve,
reference only, since a dense Hessian costs $n$ HVPs per step) and
pan-cg (the inexact solve of its Section~5). Table~\ref{tab:pamm_small}
gives the best loss. On the two non-convex surfaces the shifted solve
breaks down on most steps ($336$ and $52$ of $500$) and the method stays
within a factor $2$--$8$ of the initial loss while ARC-$\varphi_1$
improves it by six orders. On the $35$k SIREN, where every solve
completes ($0$ breakdowns, $79$ CG iterations per step on average),
pan-cg passes ARC-$\varphi_1$ late. ARC-$\varphi_1$ leads by four to
five orders at step $72$ ($3.8{\times}10^{-10}$ against
$1.4{\times}10^{-5}$), still leads at step $372$, and is overtaken
between steps $372$ and $492$. The breakdown count in the first sweeps
predicts which case applies, at no extra cost. Oracle counts differ by
arm (ARC-$\varphi_1$ $1.0$--$1.9{\times}10^{5}$ gevals against
$2.4$--$40{\times}10^{3}$ for pan-cg), so these rows compare the arms
at the same number of steps, and the comparison at equal oracle calls
is the one below.

\begin{table}[!htbp]\centering\small
\caption{Primal accelerated Newton (PAMM) against ARC-$\varphi_1$ on the
small benchmarks, $500$ steps, best loss, $L_2R$ tuned per row. pan-dense
uses an exact factorized solve and is a reference, and pan-cg is the
Hessian-free variant. ARC-$\varphi_1$ uses $\sigma_0{=}1$,
$L_{\max}{=}15$, $c_h{=}3$. Adam tuned on a log grid.}
\label{tab:pamm_small}
\resizebox{\textwidth}{!}{%
\begin{tabular}{lcccccc}
\toprule
benchmark ($n$, $f_0$) & ARC-$\varphi_1$ & pan-cg & pan-dense & Adam & CG breakdowns & mean CG its \\
\midrule
Rosenbrock--Ackley ($881$, $1.9{\times}10^{4}$) & $\mathbf{2.3{\times}10^{-2}}$ & $1.0{\times}10^{4}$ & $5.4{\times}10^{3}$ & $1.2{\times}10^{4}$ & $336/500$ & $3.8$ \\
multi-saddle ($881$, $6.7{\times}10^{3}$) & $\mathbf{2.0{\times}10^{-3}}$ & $3.1{\times}10^{3}$ & $8.9{\times}10^{2}$ & $3.9{\times}10^{3}$ & $52/500$ & $4.0$ \\
35k SIREN ($35{,}543$, $4.6$) & $3.3{\times}10^{-10}$ & $\mathbf{8.1{\times}10^{-13}}$ & -- & $1.4{\times}10^{-4}$ & $0/500$ & $78.8$ \\
\bottomrule
\end{tabular}}
\end{table}

\par\noindent\textbf{As a subsolver, at equal oracle calls and at the same number of steps.}
Table~\ref{tab:pamm_inr} reports \textsc{ARC-block}-PAMM against the
Lanczos cubic step (\textsc{ARC-block}) and the $\varphi_1$ step
(ARC-$\varphi_1$) under two experimental designs, every arm stopped at
$6{\times}10^{4}$ gradient equivalents on image~0, and the $500$-step
confirmation of the tuning sweep on images~0 and~7 with both tuned axes
extended until the optimum was interior ($k_{\mathrm{cg}}$ $100/200/300$
gives $50.61/49.49/47.08$\,dB on SIREN, and $\sigma_0$ $3/10/30$ gives
$55.81/56.57/56.37$ on SIREN and $69.75/72.05/58.76$ on FINER). Both
designs agree on direction. On SIREN the PAMM step wins,
$+1.35$\,dB over the Lanczos step and $+2.13$ over $\varphi_1$ at equal
gevals, $+2.25$ over the tuned Lanczos step at equal steps, and does so in
$252$~s against $1{,}724$~s for the Lanczos arm at identical gevals.
Equal gevals are not equal wall clock, since Lanczos also requires full
reorthogonalization, $\mathcal{O}(L^2 n_b)$ vector work and a
$k\times k$ eigendecomposition per trial, while conjugate gradients is
$\mathcal{O}(n_b)$ per iteration in four vectors. On FINER the result
inverts, $-7.37$ / $-7.50$\,dB at equal gevals and $-6.62$ at equal
steps, with $138$--$330$ rejections against $15$--$30$ on SIREN and a
non-monotone response to $k_{\mathrm{cg}}$ (optimum at $30$, breakdowns
$27/35/45$ at depth $10/30/100$), since deeper solves find the negative
curvature faster and the step itself does not improve. At the same
number of steps the PAMM arm also spends $2.2\times$ the gevals of the
$L{=}50$ Lanczos arm ($3.3{\times}10^{5}$ against $1.5{\times}10^{5}$),
so its SIREN margin there is obtained at the cost of oracle calls. The rows at
equal oracle calls remove that advantage.

\begin{table}[!htbp]\centering\small
\caption{\textsc{ARC-block}-PAMM against the Lanczos cubic and
$\varphi_1$ subsolvers in the same outer scheme, image~0 (rows at equal
oracle calls) and mean of images~0 and~7 (rows at the same number of
steps), best PSNR in dB.
$k_{\mathrm{cg}}$ is the conjugate gradient iteration limit, rej.\ the rejected
trials, and bd.\ the CG breakdowns. Wall clock on a shared GPU, so gevals and
steps are the exact axes.}
\label{tab:pamm_inr}
\resizebox{\textwidth}{!}{%
\begin{tabular}{llcccccc}
\toprule
experiment & arm & SIREN & wall & rej.\,/\,bd. & FINER & wall & rej.\,/\,bd. \\
\midrule
$6{\times}10^{4}$ gevals & \textsc{ARC-block}-PAMM, $k_{\mathrm{cg}}{=}100$ & $\mathbf{50.61}$ & $252$~s & $30/25$ & $48.29$ & $272$~s & $138/45$ \\
 & \textsc{ARC-block}-PAMM, $k_{\mathrm{cg}}{=}30$ & $49.02$ & $271$~s & $16/16$ & $49.69$ & $294$~s & $222/35$ \\
 & \textsc{ARC-block}, $L{=}25$ & $49.26$ & $1{,}724$~s & $30/$-- & $57.06$ & $842$~s & $138/$-- \\
 & ARC-$\varphi_1$, $L{=}25$ & $48.48$ & $410$~s & $16/$-- & $\mathbf{57.19}$ & $438$~s & $70/$-- \\
\midrule
$500$ steps, tuned & \textsc{ARC-block}-PAMM, $k_{\mathrm{cg}}{=}100$, $\sigma_0{=}10$ & $\mathbf{56.57}$ & $700$~s & & $72.05$ & $755$~s & \\
 & \textsc{ARC-block}, $L{=}50$ (tuned $\sigma_0$) & $54.32$ & $574$~s & & $\mathbf{78.67}$ & $636$~s & \\
\bottomrule
\end{tabular}}
\end{table}

\par\noindent\textbf{Mechanism.} A per-block diagnosis ($40$ sweeps, Lanczos
degree $40$ in double precision, stochastic Lanczos quadrature negative
mass over three probes) refutes the first explanation. All ten blocks
are positive definite on both architectures (negative mass $0.0$,
no indefinite block), so FINER is not more indefinite than SIREN, and
the residual breakdowns are rounding on ill-conditioned systems. What
differs is curvature scale. FINER's two scale-modulation layers (blocks
$0$ and $1$) have $\lambda_{\max}$ of $1.8{\times}10^{4}$ and
$4.5{\times}10^{4}$ against $1.4$ and $2.7$ for the corresponding SIREN
blocks. On SIREN $\sigma_b$ never leaves its $10^{-8}$ floor on any
block, so \textsc{ARC-block}-PAMM runs as near-exact Newton on a
positive definite problem, which is the whole SIREN win. On FINER the
ratio test escalates $\sigma_b$ on those two blocks by fourteen orders
and the PAMM step on block~$0$ comes out $3{,}052\times$ shorter than
the cubic step at the same $\sigma_b$. That damping is real but is not
what costs FINER its $7$\,dB. Routing blocks $\{0,1\}$ to the
$\varphi_1$ step recovers only $+0.86$\,dB, a comparison of the same size
routing two low curvature blocks $\{3,5\}$ recovers more ($+1.71$), and
routing all ten reproduces ARC-$\varphi_1$ exactly ($45.63$\,dB,
$43$ rejections, $11{,}564$ gevals in both, a correctness check on the
routing). The deficit is spread over nearly every block. The $\varphi_1$
spectral multiplier beats a constant-shift truncated conjugate gradient
step roughly uniformly on FINER's larger curvature blocks, and the
untested account is that $30$ iterations are uniformly too few on the
larger curvature operator. For the dispatch rule this is a second, cheap
predictor beside the fingerprint. A subsolver whose shifted solve
completes without breakdown and whose $\sigma_b$ remains at the floor is on
a landscape where the cheapest curvature step suffices.

\subsection{Greedy block selection (IBCN) inside the outer scheme}
\label{sec:ibcn}

The inexact block cubic Newton method of \citet{cristofari2024block}
picks, at each iteration, the coordinate with the largest gradient
entry, pads the block to a fixed dimension with uniformly random other
coordinates, rebuilds it every iteration, forms the dense block Hessian
explicitly, and minimizes a cubic model of that block under one global
regularizer that is doubled on rejection and never decreased. Its
published block dimensions are $1$ to $100$, and it affords them because
least squares and logistic regression expose the block Hessian in closed
form from $A^\top A$. A network does not. Our Hessian--vector product
costs one backward pass through the whole model regardless of block
size, so the number of blocks multiplies total cost. On FINER, blocks of
IBCN's size would need ${\sim}1{,}989$ of them to touch every coordinate
once, at ${\sim}28$ gevals each, about $5.6{\times}10^{4}$ gevals per
pass where a cyclic sweep spends $280$, a factor near $200$. The port
therefore keeps the ARC-$\varphi_1$ subsolver and the $\sigma_b$
adaptation (raised on rejection, lowered on a very successful trial)
unchanged and changes only the selection, behind
one code path with three modes. \emph{cyclic} is the sweep of
Algorithm~\ref{alg:blockcn}. \emph{greedy-tensor} is Gauss--Southwell
over the existing tensor partition, scoring from the gradient already in
hand and updating only the worst offender, by IBCN's own
$\max_i|g_i|$, by $\norm{g_b}_2$, or by the size-neutral
$\norm{g_b}/\sqrt{\dim b}$. One update costs $L{+}3$ gevals, a tenth of a
sweep, so greedy and cyclic are comparable at matched gevals.
\emph{greedy-pad} is IBCN's anchor-plus-random-padding construction with
the block dimension raised until the Hessian--vector product is
worthwhile. Its index set crosses tensor boundaries, so it is the only
\textsc{ARC-block} variant that sees cross-tensor coupling (restricting
$v$ to the block keeps $(\Hb v)$ on the block equal to
$\Hb[\mathrm{idx},\mathrm{idx}]\,v[\mathrm{idx}]$, verified against a
dense Hessian slice to $4.4{\times}10^{-16}$), with one global $\sigma$
because transient blocks give a per-block $\sigma$ nothing to attach to.

\begin{table}[!htbp]\centering\small
\caption{Block selection inside the \textsc{ARC-block} outer scheme with
the ARC-$\varphi_1$ subsolver ($L{=}25$, $c_h{=}1$, $\sigma_0{=}0.3$),
$500$-step confirmation, mean best PSNR (dB) over images~0 and~7. Screen
values ($120$ steps, image~0) in parentheses where no confirmation was
run. greedy-pad block dimension in parameters, where $n{=}198{,}915$ is the
whole vector, i.e.\ no blocking.}
\label{tab:ibcn}
\begin{tabular}{lccccc}
\toprule
 & cyclic & greedy-pad, $N$ & greedy-pad, $60{,}000$ & greedy-pad, $20{,}000$ & greedy-tensor, $\norm{g_b}_2$ \\
\midrule
FINER & $\mathbf{64.79}$ & $42.12$ & ($38.84$) & $39.03$ & ($41.83$) \\
SIREN & $50.14$ & $\mathbf{52.30}$ & $51.88$ & ($47.95$) & $50.92$ \\
\bottomrule
\end{tabular}
\end{table}

Table~\ref{tab:ibcn} gives the confirmation. On FINER cyclic wins
outright, by $23$--$26$\,dB, and is $3.4\times$ faster in wall clock at
matched gevals ($99$~s against $333$~s for greedy-pad $20{,}000$ at the
$120$-step screen). The mechanism is starvation, and it is quantitative.
FINER's three $65{,}536$-parameter hidden weight tensors hold $98.8\%$
of the parameters, and the updates landing on them per $1{,}200$ block
updates order the outcomes monotonically, cyclic $360$ ($54.72$\,dB at
the screen), $\norm{g_b}_2$ $91$ ($41.83$), $\max_i|g_i|$ $17$
($27.66$), size-neutral $16$ ($26.82$). IBCN's own $\max_i|g_i|$ rule
gave $1{,}014$ of $1{,}200$ updates to a single $256$-parameter tensor,
because FINER's scale-modulation layers have gradient and curvature two
to three orders larger than the rest (the same property that decides
Appendix~\ref{sec:pamm}). Coverage has to be counted in parameters.
The size-neutral score touched all ten tensors with the least
concentrated histogram and scored worst. greedy-pad fails for a different
reason. Its single global $\sigma$ runs away when blocks mix curvature
scales (at block dimension $2{,}000$, $\sigma$ reaches $2.5{\times}10^{3}$ with
$730$ escape trials and $13.60$\,dB, and at $20{,}000$ the same quantities are $1.2{\times}10^{-3}$,
$244$, $44.15$). Cyclic copes with the same heterogeneity through its
per-tensor $\sigma_b$, which runs to the floor on eight tensors while
holding $9.8{\times}10^{3}$ and $3.1{\times}10^{2}$ on the two
modulation layers. Greedy never gives a tensor enough visits for that
adaptation to form. On SIREN the result inverts, as in
Appendix~\ref{sec:pamm}, where greedy-pad beats cyclic by $2.16$\,dB and
greedy-tensor by $0.78$. Two caveats keep this from being a
greedy-selection result. The block dimension trend is monotone up to the
grid edge and the edge is the whole parameter vector, so the axis cannot
be extended, and at that cell the greedy anchor is irrelevant since the
block is everything. The winning configuration is full-parameter
ARC-$\varphi_1$, and the finding is that on SIREN, whose curvature is
near uniform across tensors and whose $\sigma_b$ remains at the floor on
every block, no blocking beats blocking at matched gevals. And equal
gevals are again not equal wall clock ($289$~s against $85$~s at the
$120$-step screen). All confirmation cells are two images and one seed.

\FloatBarrier

\section{Certified Krylov Steps and the Gradient-Regularized $\varphi_1$ Horizon}
\label{app:certified}

Theorem~\ref{thm:blockrate} covers a CubicKrylov step on the trials where
the step-form termination criterion~\eqref{eq:tc} holds
\citep[TC.s]{cartis2011b}. Its left side is free for a Lanczos step,
$\norm{\nabla m_b(s_L)}=\beta_{L+1}|e_L^\top y^\star|$, as is the
subspace second-order certificate, the smallest Ritz value of $\Tb_L$
against the shift $\tfrac{M_b}{2}\norm{s_L}$. The implementation records
both on every block trial, together with the realized decrease against
the bound $\eta_1\tfrac{M_b}{12}\norm{s_b}^3$ of
Proposition~\ref{prop:perblock}, and offers two stopping rules for the
adaptive Lanczos depth, the \emph{gradient form}
$\norm{\nabla m_b(s_k)}\le\min(1,\norm{g_b}^{1/2})\norm{g_b}$, which all
recorded runs use, and the \emph{certified} step form~\eqref{eq:tc} with
$\kappa_\theta=0.1$. \citet{cartis2011a} report the same pair, where the step
form has the improved complexity, the gradient form gives the better
numerical performance. Table~\ref{tab:certified} measures both on FINER
image~0 ($B=10$, seven tensors of at most $768$ parameters and three
$256\times256$ weights) over $500$ sweeps at $\sigma_0=1$, $\eta_1=0.1$.

\begin{table}[!htbp]
\centering\small
\caption{CubicKrylov stopping rules on ESM image~0, $500$ sweeps.
``Certified'' is the fraction of accepted trials meeting~\eqref{eq:tc}
at $\kappa_\theta=0.1$. ``at cap'' is the fraction of accepted trials that reached $L_{\max}$
without meeting the active stopping rule. ``bound met'' is the fraction whose
realized decrease reached $\eta_1\tfrac{M_b}{12}\norm{s_b}^3$ outright
(the rest were accepted inside the guard tolerance $t$, where the bound is
below fp32 loss resolution). The slope is the least-squares exponent of the
running minimum $\norm{\nabla f}$ against sweep count over the last three
quarters of the run. Theorem~\ref{thm:blockrate} guarantees $-2/3$.}
\label{tab:certified}
\begin{tabular}{llrrrrrrr}
\toprule
Stop rule & $L_{\max}$ & PSNR (dB) & gevals & rejected & certified & at cap & bound met & slope \\
\midrule
gradient form & 25 & 60.77 & $1.03{\times}10^{5}$ & 3.4\% & 10\% & 47\% & 84\% & $-0.92$ \\
step form & 25 & 60.27 & $1.42{\times}10^{5}$ & 38.7\% & 34\% & 66\% & 94\% & $-1.02$ \\
gradient form & 50 & 73.16 & $1.53{\times}10^{5}$ & 12.5\% & 10\% & 29\% & 84\% & $-1.27$ \\
step form & 50 & 71.40 & $2.08{\times}10^{5}$ & 46.9\% & 58\% & 42\% & 94\% & $-1.36$ \\
gradient form & 163 & \textbf{82.65} & $2.87{\times}10^{5}$ & 8.9\% & 10\% & 19\% & 84\% & $-1.68$ \\
step form & 163 & 75.89 & $3.76{\times}10^{5}$ & 48.1\% & 68\% & 32\% & 92\% & $-1.41$ \\
\bottomrule
\end{tabular}
\end{table}

Three readings. \emph{(i) Which blocks certify.} Under the step form the
seven small tensors certify on all of their accepted trials at
$L_{\max}=163$ (six of the seven at $50$), because Lanczos exhausts their Krylov space (breakdown
at degree $3$--$35$) and the residual identity then vanishes. The three
$256\times256$ weights reach the cap on $96$--$98\%$ of their trials at
every $L_{\max}$ including $163$, with median relative residual
$4{\times}10^{-6}$ but a $90$th percentile of $0.6$. On those blocks the
step is a truncated Krylov step whose cubic decrease
(Proposition~\ref{prop:perblock}) holds at any degree, while the
gradient bound of Lemma~\ref{lem:blockgrad} is not certified. The
theorem therefore applies to the exact subsolver and to the small blocks
as run. For the large weights it states what a certified step would
give. \emph{(ii) The certified rule costs accuracy per oracle call.}
Pushing every block to the step-form residual raises the rejection rate
from $3$--$13\%$ to $39$--$48\%$. The more exact subspace step at the
adapted $M_b$ is the step the ratio test rejects more often, and the
truncated gradient-form step acts as a second regularizer. At matched
$2.87{\times}10^{5}$ gevals the step form at $L_{\max}=163$ reaches
$72.1$\,dB against $82.65$. This is the numerical side of the
\citet{cartis2011a} remark, measured at $199$k parameters. \emph{(iii)
Rates.} The running minimum gradient decays with exponents between
$-0.9$ and $-1.7$ in sweeps, all steeper than the $-2/3$ of the theorem
and steeper at larger $L_{\max}$, so the worst case bound is not tight
here. The per-block decrease bound was met outright on $84$--$95\%$ of
accepted trials, the remainder being acceptances inside the tolerance
$t$ of Algorithm~\ref{alg:blockcn}. Negative Ritz values appeared on
$0.3$--$1.6\%$ of accepted trials, and the subspace second-order
certificate $\theta_{\min}(\Tb_L)+\tfrac{M_b}{2}\norm{s_L}\ge0$ held on
all of them.

\par\noindent\textbf{Coupling floor on FINER.}
Remark~\ref{rem:krylovso} estimates
$\delta\le\norm{\nabla^2f}_F\sqrt{1-\beta}$ from Hutchinson HVPs on the
ten parameter tensors. At initialization of the recorded image-0 run
($n=198915$, $B=10$), $\beta=0.165$,
$\norm{\nabla^2f}_F=2.31{\times}10^{3}$, $\delta_F=2.11{\times}10^{3}$,
and a power iteration on $P_b\nabla^2f\,P_{\neq b}$ gives
$\delta_{\mathrm{spec}}=135$, with a Hessian-Lipschitz secant
$L_H=6.7$. After $50$ and $250$ blockwise sweeps
(Table~\ref{tab:coupling}) the mass has shifted onto the blocks
($\beta=0.84$ then $0.78$) while $\delta_{\mathrm{spec}}$ rises to
$\sim 6{\times}10^{2}$. That operator-norm bound is loose against the
leftover the sweep actually leaves, since
$\hat\delta=\norm{\nabla f}/(\sqrt B\,R)$ settles at $0.05$--$0.16$ on
the E3 logs, four orders below $\delta_{\mathrm{spec}}$. The E3
running minimum slope never leaves $-2/3$ down to
$\norm{\nabla f}=1.4{\times}10^{-5}$ at $L_{\max}=163$
($1.4{\times}10^{-4}$ at $L_{\max}=25$), so the recorded cubic-looking
case lies below the conservative $\eps_\star$ and no plateau was
observed. A two- and ten-block coupled quartic with known
$\norm{\nabla^2f-\mathrm{blkdiag}\,\nabla^2f}=\delta$ recovers
$\delta_{\mathrm{spec}}$ to three digits. ARC-block reaches
stationarity there, so $\eps_\star$ limits the $\eps^{-3/2}$ counting
argument alone.

\begin{table}[!htbp]
\centering\small
\caption{Tensor-block coupling on ESM image~0 along the $L_{\max}=25$
gradient-form run. $\beta=\norm{\mathrm{blkdiag}\,\nabla^2f}_F^2/\norm{\nabla^2f}_F^2$
and $\delta_F=\norm{\nabla^2f}_F\sqrt{1-\beta}$ from $16$ Hutchinson
probes. $\delta_{\mathrm{spec}}=\max_b\sigma_{\max}(P_b\nabla^2f\,P_{\neq b})$.
$L_H$ is the max HVP secant. $\hat\delta=\norm{\nabla f}/(\sqrt B\,R)$ is computed
from the same sweep's $(R,\norm{\nabla f})$.}
\label{tab:coupling}
\begin{tabular}{lrrrrrr}
\toprule
Iterate & $\beta$ & $\delta_F$ & $\delta_{\mathrm{spec}}$ & $L_H$ & $\norm{\nabla f}$ & $\hat\delta$ \\
\midrule
init & $0.165$ & $2.11{\times}10^{3}$ & $135$ & $6.7$ & $5.62$ & --- \\
sweep $50$ & $0.842$ & $1.13{\times}10^{3}$ & $617$ & $20$ & $1.16{\times}10^{-3}$ & $0.076$ \\
sweep $250$ & $0.785$ & $1.45{\times}10^{3}$ & $634$ & $37$ & $2.63{\times}10^{-4}$ & $0.130$ \\
\bottomrule
\end{tabular}
\end{table}

\par\noindent\textbf{Gradient-regularized horizon for $\varphi_1$.}
On positive modes the $\varphi_1$ multiplier $(1-e^{-h\lambda})/\lambda$
lies within a factor $1.3$ of $1/(\lambda+1/h)$, a Levenberg--Marquardt
step with shift $1/h$. Setting $1/h_b=\sqrt{\sigma_b\norm{g_b}}$, with
$\sigma_b$ the adaptive Hessian-Lipschitz proxy ($M_b=2\sigma_b$), places
the positive-mode step in the $\sqrt{L\norm{g}}$-regularized Newton
framework, whose first-order complexity is $\mathcal O(\eps^{-3/2})$
\citep{mishchenko2023,doikov2024gradreg,gratton2024yet}. Negative and flat
modes keep the clamped rule of Section~\ref{sec:phi1}. We ran this
horizon, $h_b=c_h/\sqrt{\sigma_b\norm{g_b}}$, against the recorded trust
horizon $h_b=c_h/\sigma_b$ on FINER with $L=25$, $\sigma_0=0.3$, under the
sweep of Section~\ref{sec:finer}, a $120$-sweep screen over
$c_h\in\{0.3,1,3,10\}$ on image~0 (interior winner $c_h=3$), geometric
midpoints $1.7$ and $5.5$, and $500$-sweep confirmation on images~0
and~7 (Table~\ref{tab:phi1gr}). A capped variant
$h_b=\min(c_h/\sigma_b,1/\sqrt{\sigma_b\norm{g_b}})$ was $2.3$\,dB below
the trust horizon at the screen and was not taken further.

\begin{table}[!htbp]
\centering\small
\caption{Trust against gradient-regularized horizon for ARC-$\varphi_1$,
FINER, $500$ sweeps, $L=25$, $\sigma_0=0.3$, mean of images~0 and~7. The
last column shows each run's best PSNR at the trust run's geval total
($1.33{\times}10^{5}$).}
\label{tab:phi1gr}
\begin{tabular}{llrrrr}
\toprule
Horizon & $c_h$ & PSNR (dB) & gevals & rejected & PSNR at $1.33{\times}10^{5}$ gevals \\
\midrule
trust $c_h/\sigma_b$ & 1.0 & 64.78 & $1.32{\times}10^{5}$ & 2\% & \textbf{64.78} \\
regularized $c_h/\sqrt{\sigma_b\norm{g_b}}$ & 1.7 & 66.61 & $2.37{\times}10^{5}$ & 48\% & 61.13 \\
regularized & 3.0 & \textbf{66.69} & $2.34{\times}10^{5}$ & 47\% & 62.12 \\
regularized & 5.5 & 66.19 & $2.46{\times}10^{5}$ & 50\% & 61.02 \\
\bottomrule
\end{tabular}
\end{table}

The regularized horizon gains $1.4$--$1.9$\,dB per sweep and loses
$2.7$--$3.8$\,dB per gradient equivalent, since it spends $1.8\times$ the
oracle calls, almost half of them on rejected trials, because
$1/\sqrt{\sigma_b\norm{g_b}}$ grows as the gradient shrinks and turns
every mode Newton-like faster than the ratio test accepts. Its
running minimum gradient exponent against gevals ($-0.7$ to $-1.5$) is
not steeper than the trust horizon's ($-0.9$ to $-1.0$), and the winner
flips between images ($c_h=1.7$ on image~0, $3.0$ on image~7) and
between screen and confirmation, which the sweep rules treat as a flat
response. A cubic-type rate for the $\varphi_1$ step would therefore
still be proved by the analysis of regularized Newton methods, and the
trust horizon remains the method. The Cauchy type $\mathcal O(\eps^{-2})$
bound of Section~\ref{sec:phi1theory} is what the paper claims for it.
Appendix~\ref{app:phi1rate} records the same trust rule on ESM image~0
through $2{,}000$ instrumented sweeps.

\par\noindent\textbf{Why $\varphi_1$ leads CubicKrylov on FINER.}
\label{par:phi1vscubic}
Section~\ref{sec:finer} records $90.65$\,dB for trust $\varphi_1$ against
$72.1$\,dB for the cubic step at the same $L_{\max}=163$ and $500$
sweeps. Tables~\ref{tab:certified} and~\ref{tab:phi1gr} supply the
reject rates that comparison omitted, and Table~\ref{tab:phi1vscubic}
places them on one line. At $L=25$, trust $\varphi_1$ rejects $2\%$ of
trials and the cubic gradient form $3.4\%$. The cubic step form, the
rule that has the $\mathcal O(\eps^{-3/2})$ bound, rejects $39\%$. At
$L_{\max}=163$ the cubic gradient form rejects $8.9\%$ ($82.65$\,dB)
and the step form $48\%$ ($75.89$\,dB, and $72.1$\,dB at the gradient
form's geval total). Recorded trust $\varphi_1$ runs at the same degree
on image~0 reject $1$--$2\%$.

The exponents do not contradict that ranking. The cubic
$\mathcal O(\eps^{-3/2})$ bound is for an accepted model minimizer
that meets~\eqref{eq:tc}. On FINER the three $256\times256$ weights
never meet that test (reading~(i) above), so the cubic arm is a
truncated Krylov step with adapted $M_b=2\sigma_b$, not the object of
Theorem~\ref{thm:blockrate}. The $\varphi_1$ bound of
Section~\ref{sec:phi1theory} is the Cauchy estimate
$\Omega(\norm{g}^2\min(h,1/\norm{\Hb_b}))$ on a step whose exact
quadratic decrease~\eqref{eq:phi1decrease} is strictly positive for
every sign pattern. That estimate is loose.

The cubic multiplier $1/(\lambda+\sigma\norm{s})$ ($M=2\sigma$) damps
flat modes, where an INR residual is concentrated and no extra regularizer is
needed. The $\varphi_1$ step is Newton-like on high curvature Ritz values,
gradient descent with horizon $h$ on flat ones, and a clamped
exponential on $\lambda<0$.
FINER starts with $81\%$ negative spectral mass. After rejects, large
$M_b$ makes cubic steps smaller still, while $h_b=c_h/\sigma_b$ still
steps on those modes.

Table~\ref{tab:phi1gr} is the check that this is not a hidden
$\eps^{-3/2}$ mechanism. The gradient regularized horizon is the
$\sqrt{L_H\norm{g}}$-regularized Newton step, whose first-order
complexity is $\mathcal O(\eps^{-3/2})$, and on FINER it gains
$1.4$--$1.9$\,dB per
sweep and loses $2.7$--$3.8$\,dB per geval at $47$--$50\%$ rejects.
The empirically best $\varphi_1$ rule is the trust horizon, i.e.\ the
Cauchy class. Worst case stationarity exponents hide constants and
are not a PSNR forecast at a fixed sweep budget.

\begin{table}[!htbp]
\centering\small
\caption{FINER reject rates and PSNR for trust $\varphi_1$ against
CubicKrylov. $L=25$ $\varphi_1$ is the mean of images~0 and~7 from
Table~\ref{tab:phi1gr}. Cubic rows and $L=163$ PSNR are image~0 from
Table~\ref{tab:certified}. The $72.1$\,dB figure is the step form at
the gradient form's $2.87{\times}10^{5}$ gevals. Trust $\varphi_1$ at
$L_{\max}=163$ is the $500$-sweep confirmation of
Section~\ref{sec:finer}. Its $1$--$2\%$ reject range is the recorded
image-0 runs at the same $(c_h,\sigma_0,L)$.}
\label{tab:phi1vscubic}
\begin{tabular}{llrrr}
\toprule
Method & $L_{\max}$ & PSNR (dB) & gevals & rejected \\
\midrule
$\varphi_1$ trust & 25 & 64.78 & $1.32{\times}10^{5}$ & $2\%$ \\
CubicKrylov, gradient form & 25 & 60.77 & $1.03{\times}10^{5}$ & $3.4\%$ \\
CubicKrylov, step form & 25 & 60.27 & $1.42{\times}10^{5}$ & $38.7\%$ \\
$\varphi_1$ trust & 163 & $90.65$ & --- & $1$--$2\%$ \\
CubicKrylov, gradient form & 163 & $82.65$ & $2.87{\times}10^{5}$ & $8.9\%$ \\
CubicKrylov, step form & 163 & $75.89$ ($72.1$) & $3.76{\times}10^{5}$ & $48.1\%$ \\
\bottomrule
\end{tabular}
\end{table}

\section{First-Order Rate of Trust ARC-$\varphi_1$ on FINER, Measured
Slopes and a Basin Rate}
\label{app:phi1rate}

Section~\ref{sec:phi1theory} bounds accepted trust steps by the Cauchy
estimate $\Omega(\norm{g}^2\min(h,1/\norm{\Hb_b}))$, hence
$\mathcal{O}(\eps^{-2})$ first-order iterations under a horizon floor
\citep{cartis2010steepest}. Section~\ref{sec:finerconv} reports PSNR
on ESM image~0 without a per-sweep gradient. This appendix logs
$\norm{\nabla f}$ after every sweep on that image and fits the
running minimum envelope
$\norm{\nabla f}_{\min,k}:=\min_{j\le k}\norm{\nabla f(x_j)}$.
Section~\ref{sec:phi1basin} then proves the sharper rate that the
measured slopes match once the trust horizons have reached their
Newton limit state.

\par\noindent\textbf{Experiment.}
The architecture, image, seed, and trust parameters are those of
Section~\ref{sec:finer}, namely three hidden layers of width $256$
($198{,}915$ parameters), ESM image~0, $c_h=1$, $\sigma_0=0.3$,
$h_b=c_h/\sigma_b$. Two Lanczos depths, $L_{\max}=25$ and
$L_{\max}=163$, each run for $2{,}000$ sweeps with a $18{,}000$\,s
time limit that neither run meets. After each sweep the implementation
records $\norm{\nabla f}$, $f$, $\sigma_{\min}$, $h_b$, accepts, and
rejects. Ordinary least squares of
$\log\norm{\nabla f}_{\min}$ against $\log k$ on a window $[k_0,k_1]$
gives a slope $\alpha$. Cauchy is $\alpha=-1/2$, cubic regularization
is $\alpha=-2/3$, and $\norm{\nabla f}_{\min}\asymp k^{-1}$ is
$\alpha=-1$. The same fit on $\log f_{\min}$ is reported beside it.

\begin{table}[!htbp]
\centering\small
\caption{Trust ARC-$\varphi_1$ on ESM image~0, $2{,}000$ sweeps.
Best PSNR for $L_{\max}=163$ is at sweep $1{,}994$. Rejected is the
fraction of block trials. Median trial horizon $h_b$ is $10^{8}$ on
both runs ($\sigma_b$ at $\sigma_{\min}=10^{-8}$).}
\label{tab:phi1rate}
\begin{tabular}{lrrrrr}
\toprule
$L_{\max}$ & PSNR (dB) & sweeps & gevals & rejected & $\norm{\nabla f}_{\min}$ \\
\midrule
$25$ & $85.67$ & $2{,}000$ & $5.27{\times}10^{5}$ & $1.58\%$ & $1.11{\times}10^{-5}$ \\
$163$ & $132.64$ & $2{,}000$ & $3.03{\times}10^{6}$ & $0.70\%$ & $1.85{\times}10^{-7}$ \\
\bottomrule
\end{tabular}
\end{table}

\begin{table}[!htbp]
\centering\small
\caption{OLS slope of $\log\norm{\nabla f}_{\min}$ (and
$\log f_{\min}$) against $\log k$ on ESM image~0. The windows are
all $k\in[3,2000]$, early $[101,700]$, mid $[501,2000]$,
late $[1001,2000]$. $R^2$ in parentheses.}
\label{tab:phi1rateslope}
\begin{tabular}{llcccc}
\toprule
$L_{\max}$ & series & all & early & mid & late \\
\midrule
$25$ & $\norm{\nabla f}_{\min}$
  & $-1.28$ ($0.98$) & $-1.04$ ($1.00$)
  & $-1.62$ ($1.00$) & $-1.80$ ($1.00$) \\
$25$ & $f_{\min}$
  & $-2.41$ ($0.98$) & $-2.06$ ($1.00$)
  & $-3.19$ ($0.99$) & $-3.63$ ($1.00$) \\
$163$ & $\norm{\nabla f}_{\min}$
  & $-2.42$ ($0.97$) & $-2.05$ ($0.98$)
  & $-3.12$ ($0.97$) & $-2.27$ ($0.90$) \\
$163$ & $f_{\min}$
  & $-5.18$ ($0.93$) & $-4.01$ ($0.98$)
  & $-8.24$ ($0.99$) & $-8.92$ ($0.99$) \\
\bottomrule
\end{tabular}
\end{table}

\par\noindent\textbf{$L_{\max}=25$.}
Every window after sweep $100$ has slope at most $-1.04$
($R^2\ge0.996$). Sliding blocks of length $80$ pass through
$-0.89$ on $k\in[103,182]$ and $-1.12$ on $k\in[403,482]$, then
steepen to $-1.80$ on the last thousand sweeps. The ratio
$\norm{\nabla f}^2/f$ is constant at $0.045$ (median and last fifty
sweeps), so the loss exponent is twice the gradient exponent to the
reported precision. The median $h_b$ is $10^{8}$, so on positive Ritz values the
step is already the Newton multiplier $1/\theta$. The
$\mathcal{O}(\eps^{-2})$ Cauchy class of
Section~\ref{sec:phi1theory} is therefore loose on this image. The
measured envelope is $\norm{\nabla f}_{\min,k}\lesssim C\,k^{-1}$
with $C\approx 5{\times}10^{-2}$ on $k\in[125,500]$, and steeper
after that, on one image and one seed.

\par\noindent\textbf{$L_{\max}=163$.}
The same parameters at the confirmed Lanczos depth of
Section~\ref{sec:finer} reach $132.64$\,dB (the $2{,}824$-sweep run of
Section~\ref{sec:finerconv} continues to $133.5$\,dB without this
gradient log). The gradient slope is $-2.05$ on $[101,700]$ and
$-3.12$ after sweep $500$, faster than $k^{-1}$. Sliding windows peak
near $-4.5$ around $k=1{,}100$ and flatten after $k\approx 1{,}600$
as $\norm{\nabla f}$ remains at $1.9{\times}10^{-7}$. Loss keeps falling
(late slope $-8.92$), so $\norm{\nabla f}^2/f$ rises from $0.058$
(median) to $0.69$ on the last fifty sweeps, so the floor is in the
gradient. Mean realized Lanczos depth is $147.1$. Rejects fall
from $1.58\%$ at $L_{\max}=25$ to $0.70\%$.

\par\noindent\textbf{What the paper claims.}
The trust rule remains in the Cauchy class of
Section~\ref{sec:phi1theory}. Table~\ref{tab:phi1rateslope} is the
observed envelope on ESM image~0 under that rule. It is consistent
with the $k^{-1}$ bound of Theorem~\ref{thm:phi1eps1} after the
iterate has entered the $h_b=c_h/\sigma_{\min}$ state, and with the
$Q$-linear contraction of the same theorem at $L_{\max}=163$ until
$\norm{\nabla f}$ meets the logged floor.
At the recorded $h_b=10^{8}$ the overflow clamp $h\nu\le 60$
already makes every Ritz value $\nu\gtrsim 6{\times}10^{-7}$ an
exact Newton multiplier. The gradient regularized horizon of
Table~\ref{tab:phi1gr} is a different rule and is not used here.

\subsection{A basin rate for the trust horizon}
\label{sec:phi1basin}

The recorded state after sustained acceptance, $h_b=c_h/\sigma_{\min}$
on every block and a reject rate of $1$--$2\%$, is a Newton limit.
Positive Ritz values receive the multiplier $1/\nu$, and the block
step is a Krylov Newton step on a subspace that contains $g_b$. Under
a local Polyak--{\L}ojasiewicz hypothesis that state gives a
$1/\sqrt V$ potential, an explicit $\mathcal O(\eps^{-1})$ envelope,
and a $Q$-linear contraction of the functional residual. The algebraic steps are
the eleven library theorems of the third table in
Appendix~\ref{sec:lean-catalog}. Throughout, $\ell$ is a gradient
Lipschitz constant (the library writes it \texttt{L}) and
$V_k=f(x_k)-f(x^\star)$.

\begin{assumption}[Newton limit basin]
\label{ass:basin}
There exist $K_0\in\mathbb N$, $\mu\in(0,\ell]$, and a convex
neighbourhood $U$ of a local minimizer $x^\star$ such that, for every
sweep $k\ge K_0$:
\begin{enumerate}[leftmargin=2.4em,itemsep=1pt,topsep=2pt]
\item[(B1)] $x_k\in U$, and on $U$ the loss is $\ell$-smooth,
  $\norm{\nabla f(x)-\nabla f(y)}\le\ell\norm{x-y}$, and satisfies
  $\norm{\nabla f(x)}^2\ge 2\mu\,(f(x)-f(x^\star))$;
\item[(B2)] $\sigma_b=\sigma_{\min}$ on every block, hence
  $h_b=h_\star:=c_h/\sigma_{\min}$ with $h_\star\ge1/\ell$;
\item[(B3)] every accepted Krylov spectrum satisfies
  $\nu_{\min}\ge\mu$ (the negative mode clamp is inactive);
\item[(B4)] the first trial on every block is accepted
  ($\rho_b\ge\eta_1$).
\end{enumerate}
$\ell$-smoothness on $U$ also gives $\norm{\nabla f(x)}^2\le
2\ell\,(f(x)-f(x^\star))$ and $\norm{\Hb_b}\le\ell$.
\end{assumption}

(B1) holds under local strong convexity, $\nabla^2f\succeq\mu\Ib$ on
$U$. (B2)--(B4) are the recorded trust horizon state after sustained
acceptance. On ESM image~0 the median $h_b$ is $10^8$ by sweep $500$
and the reject rate is $1$--$2\%$. The Cauchy bound does not depend on the Lanczos depth $L$, because
$g_b\in\mathcal K_1$, so it uses the full block
gradient at every $L\ge1$.

\begin{lemma}[Newton limit model decrease]
\label{lem:phi1nl}
Under (B2)--(B3), an accepted $\varphi_1$ step on block $b$ satisfies
\[
  m_b(0)-m_b(s_b)
  \;\ge\;
  \tfrac14\,\norm{g_b}^2\min(h_\star,1/\norm{\Hb_b})
  \;\ge\;
  \tfrac{1}{4\ell}\,\norm{g_b}^2 .
\]
\end{lemma}
\begin{proof}
Lanczos starts at $q_1=g_b/\norm{g_b}$ and
$\Qb_L^\top\Hb_b\Qb_L=\Tb_L$, so the quadratic decrease is the modal
sum $\sum_i c_i^2(1-e^{-2h\nu_i})/(2\nu_i)$ with
$\sum_ic_i^2=\norm{g_b}^2$ (equation~\eqref{eq:phi1decrease} in Ritz
coordinates). For $\nu>0$, $(1-e^{-2h\nu})/(2\nu)\ge\tfrac14\min(h,1/\nu)$.
If $h\nu\le1$ then $1-e^{-2h\nu}\ge h\nu/2$ and the left side is at
least $h/4$. If $h\nu\ge1$ then $1-e^{-2h\nu}\ge1-e^{-2}$ and the left
side is at least $(1-e^{-2})/(2\nu)\ge1/(4\nu)$
(\texttt{modal\_newton\_limit}). Parseval gives the first inequality
(\texttt{modal\_sum\_min\_h}). (B3) gives $\nu_i\le\norm{\Hb_b}\le\ell$,
so $\min(h_\star,1/\nu_i)\ge1/\ell$ once $h_\star\ge1/\ell$, which is
(B2) (\texttt{modal\_le\_smooth}, \texttt{modal\_sum\_smooth}).
\end{proof}

\begin{lemma}[Sweep decrease in the basin]
\label{lem:phi1sweep}
Under Assumption~\ref{ass:basin}, up to the stabilized tolerance,
\[
  f(x_k)-f(x_{k+1})
  \;\ge\;
  c_{\mathrm{sw}}\,\norm{\nabla f(x_k)}^2,
  \qquad
  c_{\mathrm{sw}}
  \;=\;
  \frac{\eta_1}{8\ell\bigl(1+B\ell^2/\mu^2\bigr)} .
\]
\end{lemma}
\begin{proof}
(B4) and Lemma~\ref{lem:phi1nl} give the per-block decrease
$f(x^{(b-1)})-f(x^{(b)})\ge\eta_1\norm{g_b(x^{(b-1)})}^2/(4\ell)$, and
the telescope gives
$f(x_k)-f(x_{k+1})\ge(\eta_1/(4\ell))\sum_b\norm{g_b(x^{(b-1)})}^2$
(\texttt{sweep\_phi1\_telescope}). The Newton limit multiplier is at
most $1/\nu\le1/\mu$, so $\norm{s_b}\le\mu^{-1}\norm{g_b(x^{(b-1)})}$.
$\ell$-smoothness relates the start of sweep gradient to the
intermediate ones,
$\norm{g_b(x_k)}\le\norm{g_b(x^{(b-1)})}+\ell\norm{x_k-x^{(b-1)}}$, and
the block steps $s_j$, $j<b$, have disjoint supports, so
$\norm{x_k-x^{(b-1)}}^2=\sum_{j<b}\norm{s_j}^2\le\sum_j\norm{s_j}^2$.
Squaring and summing over the $B$ blocks,
\[
  \norm{\nabla f(x_k)}^2
  \;\le\;
  2\sum_b\norm{g_b(x^{(b-1)})}^2+2\ell^2B\sum_j\norm{s_j}^2
  \;\le\;
  2\Bigl(1+\frac{B\ell^2}{\mu^2}\Bigr)\sum_b\norm{g_b(x^{(b-1)})}^2
\]
(\texttt{sweep\_triangle\_to\_hrel}, \texttt{sweep\_rel\_from\_steps}).
Inverting and substituting the telescope gives $c_{\mathrm{sw}}$
(\texttt{sweep\_invert}).
\end{proof}

\begin{theorem}[$\varphi_1$ stationarity in the basin]
\label{thm:phi1eps1}
Under Assumptions~\ref{ass:lip} and~\ref{ass:basin}, the
ARC-$\varphi_1$ iterates with the trust horizon and the ratio test of
Algorithm~\ref{alg:arcphi1} satisfy, for every $k\ge K_0$,
\begin{equation}
  \norm{\nabla f(x_k)}
  \;\le\;
  \frac{C_\star}{k-K_0+k_\star},
  \qquad
  C_\star=\frac{2\sqrt{2\ell}}{q},\quad
  k_\star=\frac{2}{q\sqrt{V_{K_0}}},\quad
  q=\frac{c_{\mathrm{sw}}(2\mu)^{3/2}}{\sqrt{2\ell V_{K_0}}},
  \label{eq:phi1oneoverk}
\end{equation}
so at most $C_\star/\eps=\mathcal O(\eps^{-1})$ sweeps after $K_0$
reach $\norm{\nabla f}\le\eps$. The same hypotheses give the
$Q$-linear bound $V_{k+1}\le(1-2\mu\,c_{\mathrm{sw}})V_k$, which is
sharper and implies \eqref{eq:phi1oneoverk}.
\end{theorem}
\begin{proof}
The potential argument is that of \citet{nesterov2006cubic},
Section~6, for gradient dominated functions.
Let $G_0:=\sqrt{2\ell V_{K_0}}$. Nested sublevel sets and
$\ell$-smoothness give $\norm{\nabla f(x_k)}\le G_0$ for $k\ge K_0$.
Lemma~\ref{lem:phi1sweep} and (B1) then yield
$V_k-V_{k+1}\ge c_{\mathrm{sw}}\norm{\nabla f(x_k)}^2
\ge(c_{\mathrm{sw}}/G_0)\norm{\nabla f(x_k)}^3
\ge(c_{\mathrm{sw}}/G_0)(2\mu)^{3/2}V_k^{3/2}=qV_k^{3/2}$
(\texttt{phi1\_cubic\_downgrade}). If $V_{k+1}\le V_k-qV_k^{3/2}$ then
$1/\sqrt{V_{k+1}}-1/\sqrt{V_k}\ge q/2$, and summing from $K_0$,
$V_k\le4/(q(k-K_0)+2/\sqrt{V_{K_0}})^2$
(\texttt{potential\_step}, \texttt{potential\_bound}). The smoothness
bound $\norm{\nabla f}^2\le2\ell V$ turns this into
\eqref{eq:phi1oneoverk} (\texttt{phi1\_grad\_envelope}). The
$Q$-linear claim is the same decrease without the $1/G_0$ step,
$V_k-V_{k+1}\ge c_{\mathrm{sw}}\cdot2\mu V_k$ (\texttt{phi1\_qlinear}).
\end{proof}

The $1/G_0$ step is the only slack that turns a linear rate into a
polynomial one. That is the envelope a log-log fit on a short window
reports as slope $-1$. On ESM image~0 at $L_{\max}=25$ the measured
slope of $\log\norm{\nabla f}_{\min,k}$ against $\log k$ is $-1.04$
on $k\in[101,700]$ and steepens to $-1.80$ on the last thousand sweeps
(Table~\ref{tab:phi1rateslope}), which is the $Q$-linear claim
becoming visible on a log-log plot. The global $\mathcal O(\eps^{-2})$
class of Section~\ref{sec:phi1theory} is not removed. It still governs
the escape to the basin ($k<K_0$) and any iterate that leaves
(B2)--(B4). Assumption~\ref{ass:basin} itself is not a Lean statement
(Appendix~\ref{sec:lean-hyp}).

\section{Iso-Time and Oracle-Budget Comparisons at 91.4M Parameters ViSIR}
\label{app:isotime}

\subsection{Iso-time results}
\label{sec:isotimeresults}
In the iso-time experiment, every optimizer is granted the same wall clock
budget $T=148$\,s (the 100-step time of the CubicKrylov optimizer on the
idle GPU). Cheaper
first-order and diagonal methods therefore run many more steps.
The results are summarized in Table~\ref{tab:isotime} and Figure~\ref{fig:isotime}. 

First, given equal wall clock time, the CubicKrylov optimizer still performs better than
the first-order and the variants with the diagonal Hutchinson approximation by a wide margin.
It also leads in the experiment that lets it perform 
$100$ steps and allows all other methods to run for the same wall clock time that those $100$ steps took. Every first-order or diagonal method plateaus far below
it (Adam and SOAP near $20$--$21$\,dB, Muon and AdaHessian lower) despite running
thousands of steps in the same amount of time. Second, we note that this is a
\emph{single image} reconstruction task, so high PSNR reflects the
quality of the one image reconstruction rather than generalization. 

\begin{table}[h]
\centering\small
\caption{Iso-time comparison at $T=148$\,s per method ($91.4$M ViSIR,
image \#1). ``steps'' is the number of steps in time T}
\label{tab:isotime}
\begin{tabular}{lcc}
\toprule
Method & PSNR (dB) & steps in $T$ \\
\midrule
\textbf{CubicKrylov} & \textbf{67.28} & 100 \\
Cheby-2 $n{=}5$ & 52.46 & 187 \\
Cheby-1 $n{=}3$ & 33.84 & 151 \\
Cheby-2 $n{=}3$ & 25.69 & 201 \\
Cheby-1 $n{=}5$ & 25.19 & 141 \\
Newton $n{=}3$ & 24.26 & 243 \\
Newton $n{=}5$ & 22.67 & 217 \\
Adam & 20.71 & 17{,}830 \\
SOAP & 20.61 & 3{,}311 \\
Muon+AdamW & 11.43 & 9{,}962 \\
AdaHessian & 5.89 & 7{,}427 \\
\bottomrule
\end{tabular}
\end{table}

\begin{figure}[h]
\centering
\includegraphics[width=0.49\textwidth]{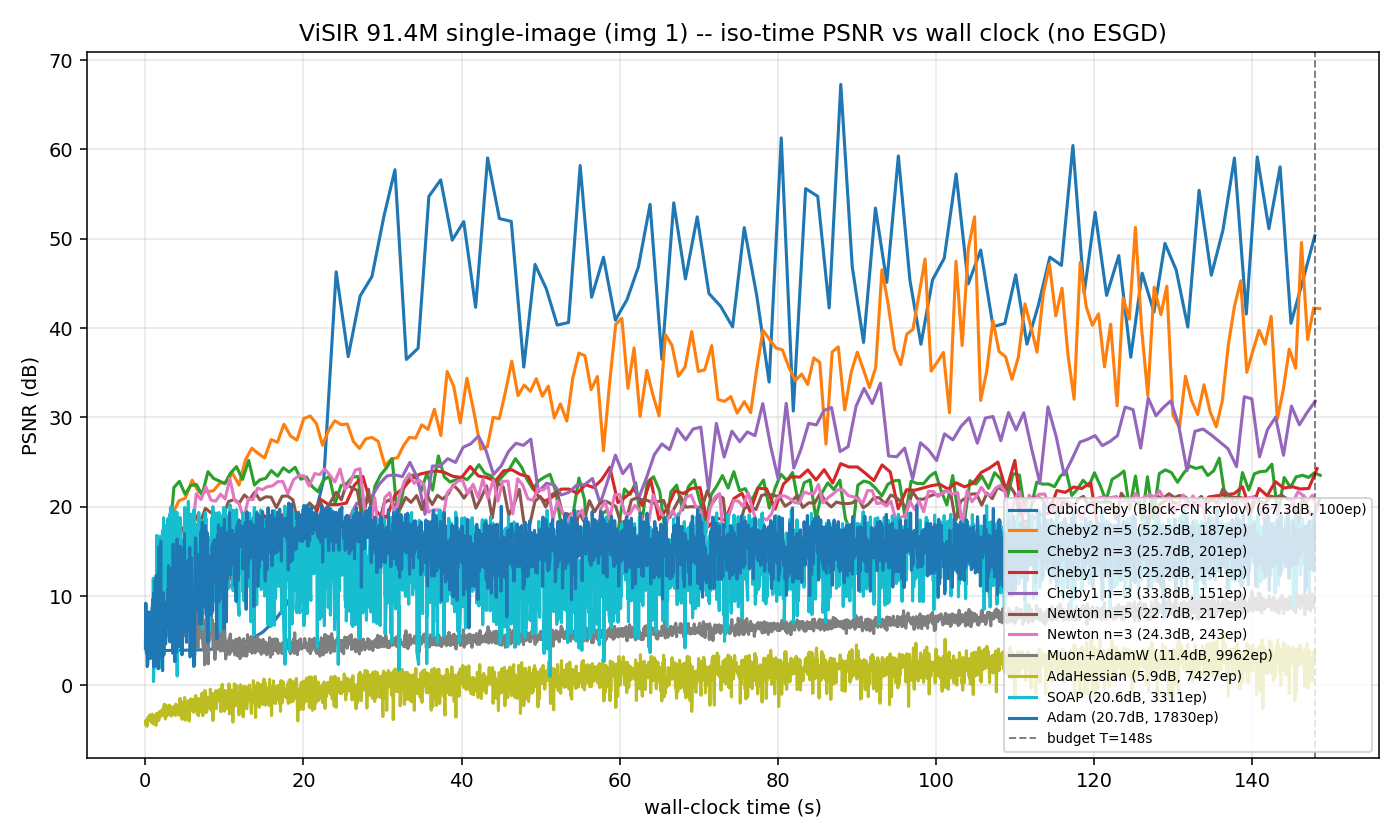}
\includegraphics[width=0.49\textwidth]{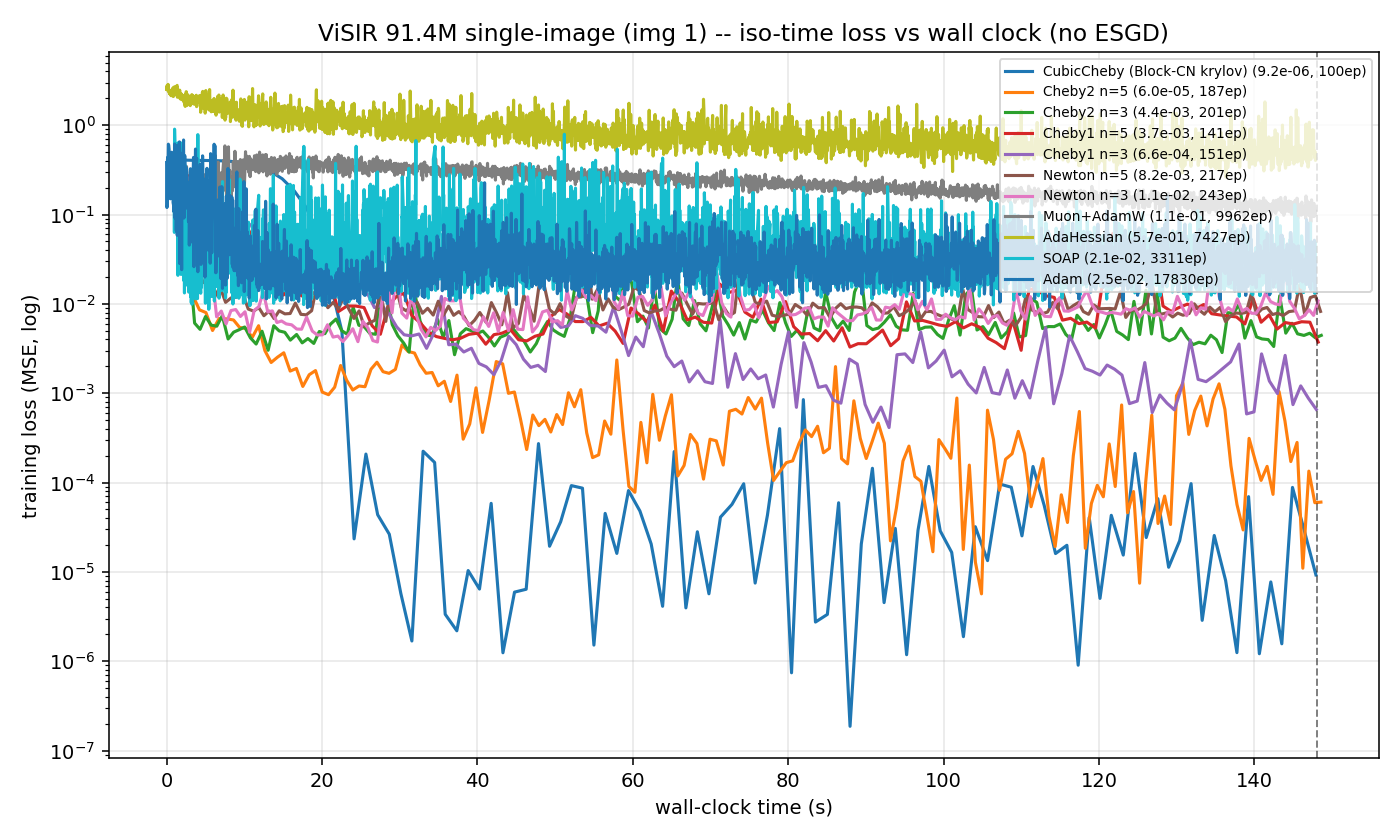}
\caption{Iso-time training curves (PSNR and MSE loss vs.\ wall clock time) at the
$T=148$\,s budget.}
\label{fig:isotime}
\end{figure}

\subsection{Convergence vs.\ oracle cost (gradient equivalents)}
\label{sec:gevals}
The comparison metric common in the second-order optimization literature is
the number of \emph{oracle calls}, gradient
and Hessian(-vector) evaluations, counted here in \emph{gradient equivalents}
(gevals, one HVP $\approx$ one backward pass) \citep{doikov_thesis,doikov2023poly} rather than
the number of steps or wall clock time. One function evaluation
likewise costs one gradient equivalent. Each block trial costs one
extra Hessian--vector product for the model decrease, and rejected trials
are included at full cost. The totals are summed over all blocks, Krylov
iterations, and sweeps of a run.
Note that this metric is hardware-independent and correctly credits methods
that reuse curvature information.
Table~\ref{tab:gevals} and Figure~\ref{fig:gevals} report PSNR against cumulative
gevals.

The \emph{per-step} oracle cost is evaluated by this experiment. The matrix-free
Krylov step of CubicKrylov costs ${\approx}368$ gevals/step (lazy small block
Hessians plus a degree-$10$ Krylov solve on the large tensors), whereas the
full Hessian Chebyshev-ON/OFF/Newton steps cost ${\approx}8000$ gevals/step
because they rebuild the exact small block Hessians ($\sum_{n_b\le512} n_b\approx7900$
gevals) \emph{every} step. This is a ${\sim}22\times$ reduction. Consequently, within
the $100$-step oracle budget of CubicKrylov ($36{,}836$ gevals) the full Hessian
methods complete only ${\sim}5$ steps and reach $\le19$\,dB, while CubicKrylov
reaches $67$\,dB. First-order methods (1 geval/step) plateau near $20$\,dB.
In this experiment the block size independence of the Krylov
step (Section~\ref{sec:algorithm}) is crucial. It reduces the cost of a
curvature-aware cubic step from ${\sim}8000$ to ${\sim}368$ gevals.

\begin{table}[h]
\centering\small
\caption{Best PSNR reached within a matched oracle budget of $36{,}836$
gradient equivalents (CubicKrylov's 100-step cost) on $91.4$M ViSIR.
$^{\ddagger}$See Table~\ref{tab:isotime}.}
\label{tab:gevals}
\begin{tabular}{lcc}
\toprule
Method & gevals/step & PSNR at ${\le}36{,}836$ gevals \\
\midrule
\textbf{CubicKrylov} & \textbf{368} & \textbf{67.3}$^{\ddagger}$ \\
SOAP & 1 & 20.6 \\
Adam & 1 & 20.4 \\
Cheby-ON $n{=}5$ (Chebyshev-2) & 8049 & 19.1 \\
Newton $n{=}5$ (damped) & 8049 & 13.8 \\
Muon+AdamW & 1 & 11.3 \\
Cheby-1 $n{=}5$ & 8049 & 7.5 \\
AdaHessian & 2 & 5.8 \\
\bottomrule
\end{tabular}
\end{table}

\begin{figure}[h]
\centering
\includegraphics[width=0.7\textwidth]{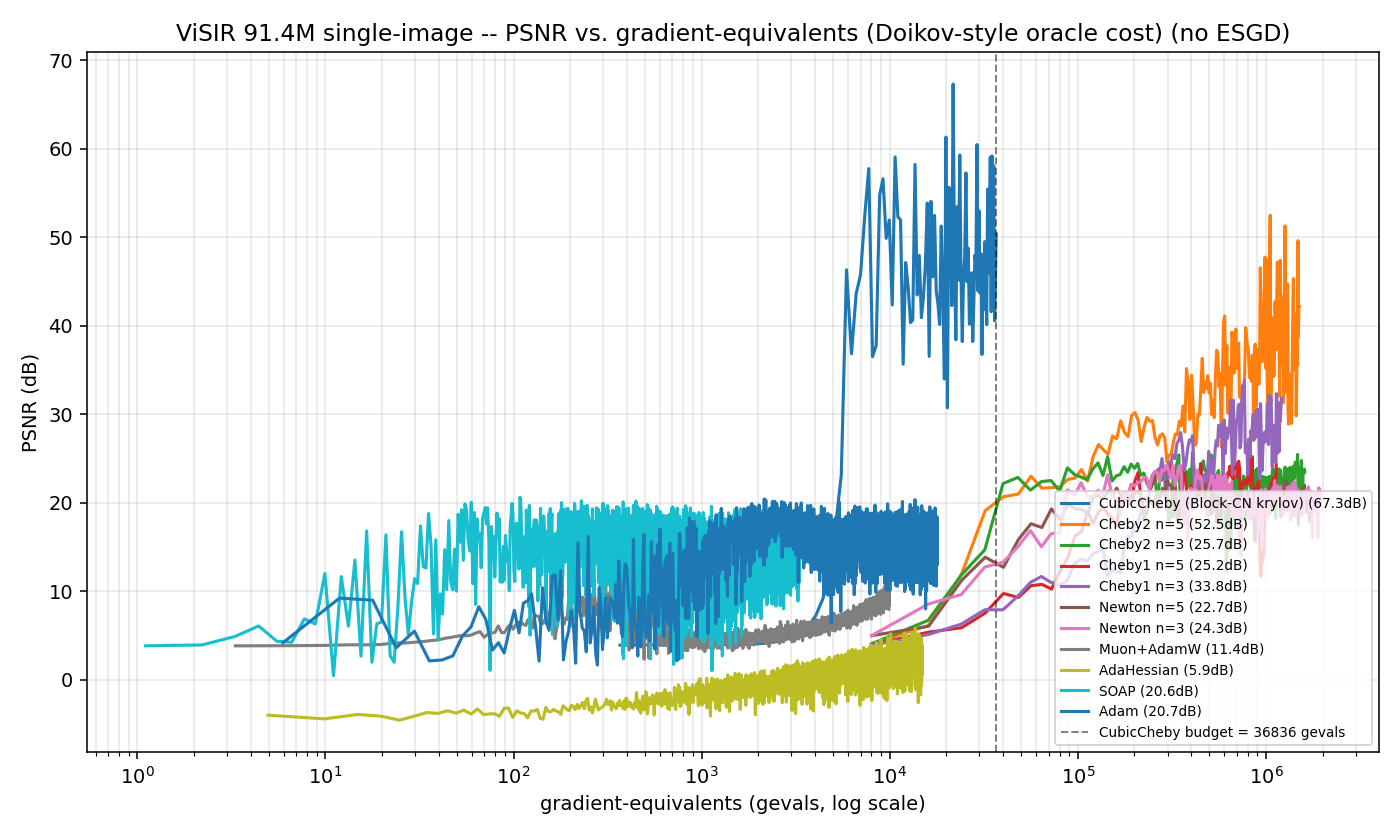}
\caption{PSNR vs.\ cumulative gradient equivalents (log scale). At matched oracle budget CubicKrylov (named CubicCheby in the plot)
dominates the full Hessian Chebyshev steps by ${\sim}48$\,dB (they exhaust
the budget in ${\sim}5$ steps), because its matrix-free Krylov step is
${\sim}22\times$ cheaper per step.}
\label{fig:gevals}
\end{figure}

\section{Multi-Image, Multi-Seed ViSIR-Nano Benchmark (ESM-20)}
\label{app:esm20nano}
To rule out a single image artifact, the nano configurations of Appendix~\ref{sec:nano} were
rerun in the style of the ESM-20 benchmark of
Appendix~\ref{app:nanowall}, the first 20 ESM images with three random seeds
each. Two reference methods are added at this scale, rank-$\tau$ spectral
preconditioned gradient descent with $\tau{=}10$
\citep{doikov2024spectral}, and a second implementation of
DSO-BlockHess that eigendecomposes every per-layer Hessian and
preconditions with absolute eigenvalues, for ten methods and 600 runs
total. The budget rule differs from the 91.4M ESM-20
experiment, which grants every optimizer the same 150 optimization steps.
Here, for each (image, seed) pair, CubicKrylov runs its 100-step ViSIR-Nano
experiment first and every reference method is then granted
\emph{twice} the gradient equivalents CubicKrylov spent on that pair. A
fixed step count would be meaningless across this set, since a single
dense Hessian build costs 15.2k gradient equivalents, and the doubled oracle
budget rules out exhaustion of that budget as the explanation for the ranking.

\begin{table}[h]
\centering\small
\caption{ESM-20-style nano benchmark (20 images $\times$ 3 seeds,
ViSIR-Nano, 15{,}235 parameters). Best PSNR per run, aggregated over the
60 runs per method, with mean wall clock and total gradient equivalents per run.
The right part is the paired per-image comparison against CubicKrylov (mean over seeds), with the
mean PSNR advantage of CubicKrylov, its per-image win rate, and the
Wilcoxon signed-rank $p$-value ($N{=}20$ images).}
\label{tab:nanoesm20}
\resizebox{\textwidth}{!}{%
\begin{tabular}{lcccc|ccc}
\toprule
 & \multicolumn{4}{c|}{aggregate over 60 runs} &
   \multicolumn{3}{c}{paired vs.\ CubicKrylov} \\
Method & PSNR (dB) & median & wall (s) & gevals &
  $\Delta$PSNR & win\% & $p_W$ \\
\midrule
\textbf{CubicKrylov} (nano) & $\mathbf{32.8\pm1.4}$ & \textbf{32.9} & 21 &
  $1.8{\times}10^4$ & --- & --- & --- \\
Reg.\ Newton lazy, $m{=}n$~\citep{doikov2023lazy} & $32.4\pm1.2$ & 32.5 &
  298 & $3.0{\times}10^4$ & $+0.3$ & 85\% & $3.9{\times}10^{-4}$ \\
CN-Lazy, $m{=}n$~\citep{doikov2023lazy} & $32.1\pm1.1$ & 32.1 &
  305 & $3.0{\times}10^4$ & $+0.7$ & 95\% & $3.8{\times}10^{-6}$ \\
Spectral-precond.\ GD $\tau{=}10$~\citep{doikov2024spectral} &
  $27.7\pm3.3$ & 28.5 &
  116 & $3.6{\times}10^4$ & $+5.1$ & 100\% & $1.9{\times}10^{-6}$ \\
SSCN $\tau{=}64$~\citep{zhao2025sscn} & $24.4\pm1.2$ & 24.4 &
  126 & $3.6{\times}10^4$ & $+8.4$ & 100\% & $1.9{\times}10^{-6}$ \\
FD-CNM, $m{=}50$~\citep{doikov2023fo} & $14.8\pm4.2$ & 14.6 &
  124 & $4.6{\times}10^4$ & $+17.9$ & 100\% & $1.9{\times}10^{-6}$ \\
DSO-BlockHess (per-layer full Hessian) & $9.0\pm0.7$ & 9.0 &
  9 & $4.6{\times}10^4$ & $+23.8$ & 100\% & $1.9{\times}10^{-6}$ \\
DSO-BlockHess (abs-eig, rebuild every step) & $8.2\pm0.8$ & 8.2 &
  5 & $4.6{\times}10^4$ & $+24.6$ & 100\% & $1.9{\times}10^{-6}$ \\
Adaptive cubic Newton, $m{=}1$~\citep{doikov2023lazy} & $7.9\pm1.3$ & 8.1 &
  182 & $4.6{\times}10^4$ & $+24.9$ & 100\% & $1.9{\times}10^{-6}$ \\
Cubic Newton (fresh Hessian)~\citep{nesterov2006cubic} & $3.9\pm0.9$ & 4.4 &
  183 & $4.6{\times}10^4$ & $+28.9$ & 100\% & $1.9{\times}10^{-6}$ \\
\bottomrule
\end{tabular}}
\end{table}

\begin{figure}[h]
\centering
\includegraphics[width=0.95\textwidth]{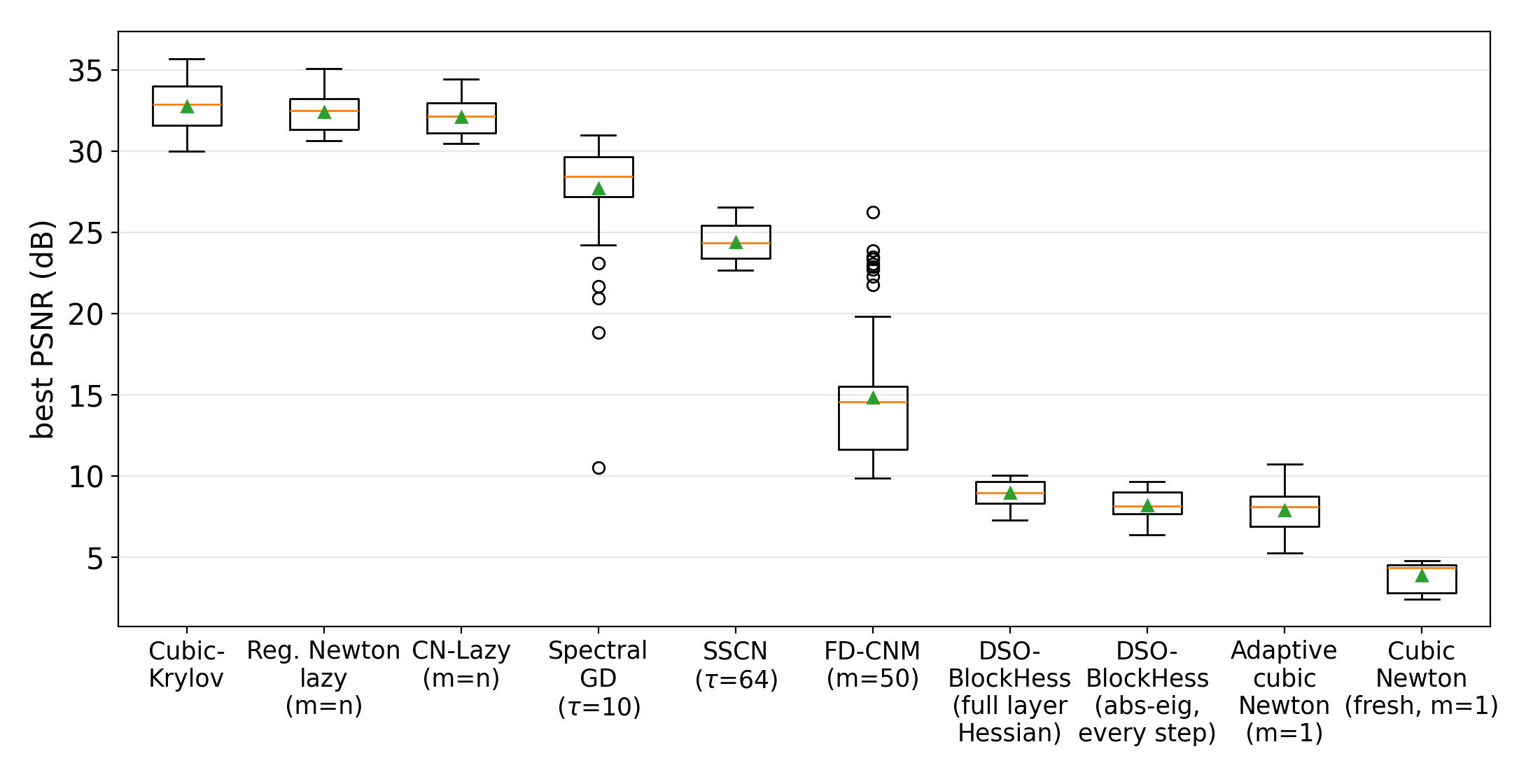}
\caption{Best-PSNR distributions over the 60 runs per method (20 ESM
images $\times$ 3 seeds, ViSIR-Nano). Every reference method is
granted twice the CubicKrylov oracle budget of the same (image, seed)
pair. Boxes show quartiles and triangles the means.}
\label{fig:nanoesm20}
\end{figure}

The single image ranking is stable across images and seeds
(Table~\ref{tab:nanoesm20}, Figure~\ref{fig:nanoesm20}). CubicKrylov
reaches $32.8\pm1.4$\,dB and beats gradient regularized Newton on 17 of
20 images ($+0.3$\,dB mean, Wilcoxon $p=3.9\times10^{-4}$) and lazy cubic
Newton on 19 of 20 ($+0.7$\,dB, $p=3.8\times10^{-6}$), while consuming
$59\%$ of their oracle budget and ${\sim}14\times$ less wall clock. The
margins differ from those of the 91.4M ESM-20 benchmark. At nano
scale, where the dense $15\mathrm{k}\times15\mathrm{k}$ Hessian is
affordable, the lazy $m{=}n$ variants are genuine competitors whose
distributions overlap CubicKrylov's, and the advantage of the matrix-free
block step is its cost. The quality of the steps is on par. Rank-$\tau$ spectral preconditioned
gradient descent~\citep{doikov2024spectral} falls between the lazy
variants and SSCN at $27.7\pm3.3$\,dB. Its rank-10 preconditioner is
cheap to maintain, but ten eigendirections cannot equalize the
contraction rates of a SIREN spectrum. Every other configuration
loses on all 20 images at the smallest attainable $p$-value, for the
same reasons as in the single image reference block. Frequent dense
rebuilds exhaust the budget within a handful of steps, and the
$\tau{=}64$ subspace of SSCN cannot match full space curvature. The
small DSO-BlockHess wall clock entries deserve a caveat. Per-layer
rebuilds are vectorized, hence cheap in wall clock and expensive in
oracle calls (one rebuild of all layers costs $n$ gradient
equivalents), so the budget
is spent within ${\sim}21$ steps in ${\sim}9$\,s for the
rebuild-every-10 variant and within 3 steps for the abs-eig variant
that rebuilds every step. For the plain-Newton variant the cause is not an exhausted oracle
budget, since an equal-time comparison on image \#0 (300\,s, 670 steps,
${\approx}10^{6}$ gradient equivalents) still plateaus at $9.2$\,dB. The
abs-eig variant, by contrast, recovers to 33.1\,dB when the budget rule
is lifted. Appendix~\ref{app:nanobudget} reports an equal-budget comparison
that separates the rebuild cost from genuine step quality for the
whole reference panel.

\section{Wall-Clock Statistics of the Multi-Image Nano Benchmark}
\label{app:nanowall}


The wall clock column of Table~\ref{tab:nanoesm20} reports the
\emph{mean} per-run time over the 60 runs of each method (20 images
$\times$ 3 seeds). Table~\ref{tab:nanowall} reports the spread behind those
means.

\begin{table}[h]
\centering\small
\caption{Per-run wall clock statistics of the multi-image ViSIR-Nano benchmark
(Table~\ref{tab:nanoesm20}), with mean, standard deviation, extremes, and
coefficient of variation over the 60 runs per method.}
\label{tab:nanowall}
\begin{tabular}{lccccc}
\toprule
Method & mean (s) & std (s) & min (s) & max (s) & CV \\
\midrule
CubicKrylov  & 20.5 & 8.7 & 14.5 & 60.4 & 43\% \\
Reg.\ Newton lazy, $m{=}n$ & 298.3 & 35.6 & 274.9 & 424.3 & 12\% \\
CN-Lazy, $m{=}n$ & 304.9 & 27.6 & 284.8 & 390.1 & 9\% \\
Spectral-precond.\ GD $\tau{=}10$ & 115.7 & 5.9 & 107.4 & 136.8 & 5\% \\
SSCN $\tau{=}64$ & 125.7 & 7.1 & 114.1 & 144.8 & 6\% \\
FD-CNM, $m{=}50$ & 123.9 & 22.5 & 106.3 & 232.5 & 18\% \\
DSO-BlockHess & 8.7 & 2.6 & 6.0 & 16.7 & 30\% \\
DSO-BlockHess (abs-eig) & 4.8 & 0.1 & 4.8 & 5.5 & 2\% \\
Adaptive cubic Newton, $m{=}1$ & 182.0 & 3.6 & 176.7 & 195.8 & 2\% \\
Cubic Newton (fresh Hessian) & 182.5 & 5.1 & 175.7 & 198.4 & 3\% \\
\bottomrule
\end{tabular}
\end{table}

Methods with fixed size dense linear solves are
nearly independent of the target image, because the cost of building a dense
Hessian for a $15\mathrm{k}\times15\mathrm{k}$ parameter block and of
its eigendecomposition depends only on the block size, not on the
image being fit, so the fresh Hessian cubic Newton variants, the
spectral preconditioner, and SSCN vary by only 2--6\%. The lazy
$m{=}n$ methods vary moderately (coefficient of variation (CV)
${\sim}$9--12\%), because their time is dominated by one dense build plus
15{,}000 cheap iterations whose timing fluctuates from run to run,
apparently with GPU load.
CubicKrylov shows the largest \emph{relative} spread (CV 43\%,
per-image means from 14.9 to 31.1\,s, worst single run 60.4\,s)
precisely because its work does depend on the target image. Harder
images cause more rejected steps, each rejection forces a fresh lazy
Hessian rebuild in phase A, and phase B spends additional
adaptive $M_b$ retries. The two DSO-BlockHess variants run at a small
absolute time scale (5--17\,s) because of the early oracle budget
exhaustion discussed in Appendix~\ref{sec:nano}. At these short times
the abs-eig variant, the implementation that eigendecomposes every
per-layer Hessian and preconditions with the absolute values of the
eigenvalues (Appendix~\ref{sec:nano}), is nearly deterministic (CV
2\%), while the plain Newton variant varies more (CV 30\%).

The comparison in Table~\ref{tab:nanoesm20} corresponds to these results.
The \emph{worst} CubicKrylov run (60.4\,s) is still ${\sim}4.5\times$
faster than the \emph{fastest} run of either lazy cubic Newton reference
with $m{=}n$ (274.9\,s). It has the ${\sim}14\times$ mean wall clock advantage.

\FloatBarrier    
\clearpage       

\section{Equal Budget Comparison at Nano Scale (15k parameters ViSIR replica), Hessian Recomputation Cost
versus Step Quality}
\label{app:nanobudget}


The oracle budget rule of the nano benchmark (twice the
gradient equivalents CubicKrylov spends, Appendix~\ref{sec:nano}) affects the results of the methods with a dense per-layer Hessian computation very much. The methods 
that compute Hessian frequently are stopped after a handful of steps.
This raises a fairness question. Does the low ranking of the
DSO-BlockHess variants (and of cubic Newton with Hessian recomputation every step) reflect
poor \emph{steps} or merely the \emph{rebuild cost}? On other
benchmarks the same second kind family is strong. The DSO-Adaptive variant
that recomputes the full block Hessian only every $K$ steps, EMA-smooths
it ($\beta{=}0.9$), runs a Hutchinson-diagonal Chebyshev step on
well-conditioned blocks, and adds an element-wise negative curvature
escape performs better than every
Chebyshev \emph{first}-kind run by two to three orders of magnitude in
MSE on a saddle rich 2-D regression benchmark, a ranking that was
re-verified for this appendix with a fresh three-seed rerun accurately reproducing
the results achieved earlier on this same benchmark. 
The experiment below checks whether
that ordering is preserved if the methods are run until convergence.

\par\noindent\textbf{Experimental design.}
All runs use the single image nano mode experimental design (ESM image \#0, seed 42,
15{,}235 parameters). Gradient equivalents for DSO-Adaptive are computed
conservatively at the harness level, without modifying the optimizer, as
one per closure call (including line-search evaluations), one per
Hutchinson Hessian--vector product (3 per block per step), and the block
size per full block Hessian computation. Three stages: (i)~DSO-Adaptive inside the
benchmark budget of $2\times17.7\mathrm{k}$ gradient equivalents, with
Hessian re-computation frequency of $K\in\{5,10,20\}$; (ii)~DSO-Adaptive unconstrained,
to locate its convergence plateau; (iii)~every optimizer of the nano
benchmark, plus the global-Krylov and per-block ARC implementations of
Appendix~\ref{sec:feasibility}, granted that plateau budget, which is equal to
$6.53{\times}10^{6}$ gradient equivalents (${\sim}185\times$ the benchmark
rule). Additionally, a 2{,}400\,s wall clock limit was used to have a reasonable termination condition for this case.

\par\noindent\textbf{Stage (i), inside the benchmark budget.}
DSO-Adaptive does not have sufficient budget to reach its reasonable performance level exactly like the other methods that recompute Hessian frequently. $K{=}5$ gives 11 steps (12.6\,dB), $K{=}10$ 21 steps
(16.6\,dB), $K{=}20$ 41 steps (19.9\,dB). Doubling the Hessian recomputation frequency
doubles the usable steps and adds ${\sim}3.5$\,dB, but even $K{=}20$
remains far below the rank-$\tau$ spectral preconditioner (28\,dB),
because the curvature based mode dispatcher of DSO-Adaptive routes ${\sim}96\%$ of the parameters
through full block Hessian mode (instead of diagonal Hutchinson mode) of DSO-Adaptive and the method keeps expending its geval budget at a great rate.

\par\noindent\textbf{Stage (ii), the plateau.}
Unconstrained, DSO-Adaptive ($K{=}5$) converges to $30.0$\,dB in
${\sim}2{,}400$\,s and ${\sim}3{,}000$ steps, consuming
$6.53{\times}10^{6}$ gradient equivalents (8{,}867 block rebuilds, with ten
steps skipped where \texttt{eigvalsh} failed to converge on a
near-singular EMA Hessian).

\begin{table}[h]
\centering\small
\caption{Equal-budget comparison, every method granted DSO-Adaptive's
plateau budget of $6.53{\times}10^{6}$ gradient equivalents
(${\sim}185\times$ the benchmark rule) and a 2{,}400\,s wall clock limit, on
ESM image \#0, seed 42. ``Limited by'' is the stopping constraint.}
\label{tab:nanobudget}
\resizebox{\textwidth}{!}{%
\begin{tabular}{lcccl}
\toprule
Method & best PSNR (dB) & steps & gevals used & limited by \\
\midrule
ARC, global Krylov~\citep{cartis2011a,cartis2011b} & \textbf{37.21} &
  16{,}391 & $4.6{\times}10^5$ & wall clock \\
\textbf{CubicKrylov} (nano) & 36.64 & 7{,}334 & $5.8{\times}10^6$ &
  wall clock \\
ARC, per-block $\sigma_b$, reduced-operator
  acceptance~\citep{doikov2025gcb} & 36.49 & 2{,}790 sweeps &
  $1.4{\times}10^6$ & wall clock \\
ARC, per-block $\sigma_b$ (Appendix~\ref{sec:feasibility}) & 36.37 &
  1{,}795 sweeps & $1.2{\times}10^6$ & wall clock \\
Reg.\ Newton lazy, $m{=}n$~\citep{doikov2023lazy} & 36.36 & 120{,}000 &
  $2.4{\times}10^5$ & wall clock \\
CN-Lazy, $m{=}n$~\citep{doikov2023lazy} & 36.18 & 115{,}000 &
  $2.4{\times}10^5$ & wall clock \\
CubicCheby-DSO v2r, grad-triggered refresh
  (Appendix~\ref{sec:coststructure}) & 35.46 & 8{,}113 & $6.5{\times}10^6$
  & gevals \\
CubicCheby-DSO v3, $\alpha$-floor (Appendix~\ref{sec:coststructure}) &
  35.34 & 7{,}889 & $6.5{\times}10^6$ & gevals \\
CubicCheby-DSO v2, EMA-cached bounds (Appendix~\ref{sec:coststructure}) &
  35.18 & 7{,}997 & $6.5{\times}10^6$ & gevals \\
CubicCheby-DSO grS, gradient regularized shift
  (Appendix~\ref{sec:coststructure}) & 34.93 & 8{,}018 & $6.5{\times}10^6$
  & gevals \\
CubicCheby-DSO v1, fresh bounds (Appendix~\ref{sec:coststructure}) &
  34.92 & 5{,}350 & $5.0{\times}10^6$ & wall clock \\
FD-CNM, $m{=}50$~\citep{doikov2023fo} & 34.70 & 893 & $5.4{\times}10^5$ &
  own convergence \\
CubicCheby-DSO v2, fresh bounds (Appendix~\ref{sec:coststructure}) &
  34.20 & 6{,}358 & $5.6{\times}10^6$ & wall clock \\
CubicCheby-DSO v1, EMA-cached bounds (Appendix~\ref{sec:coststructure}) &
  33.55 & 5{,}919 & $5.3{\times}10^6$ & wall clock \\
DSO-BlockHess (abs-eig, every step) & 33.07 & 429 & $6.5{\times}10^6$ &
  gevals \\
DSO-Adaptive, $K{=}5$ (reference) & 29.95 & 2{,}496 & $6.5{\times}10^6$ &
  own plateau \\
Adaptive cubic Newton, $m{=}1$~\citep{doikov2023lazy} & 29.53 & 80 &
  $6.1{\times}10^5$ & wall clock \\
SSCN $\tau{=}64$~\citep{zhao2025sscn} & 28.87 & 10{,}000 &
  $6.5{\times}10^5$ & step limit \\
Spectral-precond.\ GD $\tau{=}10$~\citep{doikov2024spectral} & 28.00 &
  32{,}000 & $7.7{\times}10^5$ & saturated \\
DSO-BlockHess (plain Newton, rebuild every 10) & 8.20 & 4{,}281 &
  $6.5{\times}10^6$ & algorithmic \\
Cubic Newton (fresh Hessian)~\citep{nesterov2006cubic} & 3.42 & 9 &
  $1.4{\times}10^5$ & \texttt{eigh} failure \\
\bottomrule
\end{tabular}}
\end{table}

\par\noindent\textbf{Stage (iii), interpretation.}
Table~\ref{tab:nanobudget} distinguishes three failure modes that are not clearly separated in the benchmark with the fixed oracle budget, Table~\ref{tab:nanoesm20}.

\emph{Rebuild cost.} The abs-eig DSO-BlockHess (exact block Hessian with absolute eigenvalues), which scores
$8.2$\,dB under the benchmark rule (three affordable steps within the geval budget), reaches \textbf{33.1\,dB}, it needed greater budget to give a good result. The same holds in
attenuated form for DSO-Adaptive ($12.6\to30.0$\,dB). The second kind
DSO family therefore does beat rank-$\tau$ spectral preconditioning at
nano scale too, by $2$--$5$\,dB, consistent with its results on the
saddle rich regression benchmark and on the 91.4M parameter fidelity
study. 

\emph{Algorithmic failure.} The plain Newton DSO-BlockHess variant stays at $8.2$\,dB after
$6.5{\times}10^{6}$ gradient equivalents because its unregularized Newton
step on severely indefinite layer Hessians is attracted to saddles (since the eigenvalues are absolute, the saddle escape mechanism due to the Chebyshev second kind preconditioning does not function), the
failure already diagnosed in Appendix~\ref{sec:nano}. Increasing the budget does not help it because it cannot escape saddles. Plain cubic Newton (with exact Hessian recomputations every step) crashes at step~9 when its own trajectory drives the dense
Hessian into a state whose eigendecomposition fails to converge in
LAPACK. 

\emph{CubicKrylov performance.} Among the methods evaluated on these benchmarks, each under a fixed gradient-equivalent
budget, CubicKrylov leads at \emph{both} budget levels, 31.66\,dB at
$17.7\mathrm{k}$ gradient equivalents and 36.64\,dB at $6.5{\times}10^{6}$. 
The lazy cubic Newton variants with $m{=}n$ are close to
within $0.3$--$0.5$\,dB here, but only by spending the full 2{,}400\,s
runtime budget on $10^5$ iterations, and the strongest DSO variant with exact Hessian recomputation remains $3.6$\,dB behind at $150\times$ the
spent oracle cost.

\emph{ARC-global and the scale boundary.} The one method that edges
past CubicKrylov in this comparison is global-Krylov ARC, at \textbf{37.21\,dB}
after 16{,}391 accepted steps (21 rejections, and $\sigma$ settles at its
$10^{-8}$ floor) at only $4.6{\times}10^{5}$ gradient equivalents. At $n=15$k the loss landscape is still benign enough that a
\emph{single} global regularization weight is adequate. Nearly every
search step is accepted, the multiplicative update drives $\sigma$ down
to its lower clamp $\sigma_{\min}=10^{-8}$, the cubic term becomes
negligible, the whole-space Lanczos step is nearly an unregularized
Newton step, and the method converges rapidly. The per-block ARC
variants of Appendix~\ref{sec:feasibility} reach the same PSNR range
(36.37\,dB plain, 36.44\,dB stabilized, 36.49\,dB with the guarded
gradient-based acceptance, and every $\sigma_b$ ends at the same
clamp), so at this scale it does not matter whether the regularization
weight is global or per-block, and the block bookkeeping only adds
sweep overhead. This makes the full-scale outcome the more instructive
one. The identical global implementation stalls at $4.0$\,dB on the
91.4M parameter model (Appendix~\ref{sec:feasibility}), where a single
$\sigma$ must satisfy the acceptance test along all $91.4$M directions
at once, so rejections escalate it to ${\sim}10^{10}$ and the accepted
steps shrink to noise level, while the per-block variant reaches
61.9\,dB there. The nano comparison and the full-scale runs therefore lie
on opposite sides of the model size at which the global-$\sigma$
formulation stops scaling and per-block regularization constants
become the enabling mechanism.

Among the methods of the benchmark with the fixed oracle budget,
Table~\ref{tab:nanoesm20}, the matrix-free CubicKrylov step leads at
both budget levels tested, 31.66\,dB at $17.7$k gradient equivalents and
36.64\,dB at $6.5{\times}10^{6}$, and the only entry of
Table~\ref{tab:nanobudget} ahead of it in this comparison, global-Krylov
ARC at 37.21\,dB, leads by less than $0.6$\,dB. The methods of
Table~\ref{tab:nanobudget} that have also been run unchanged at the
91.4M parameter scale, CubicKrylov, the per-block ARC variants of
Appendix~\ref{sec:feasibility}, and the refined recurrence hybrids of
Appendix~\ref{sec:coststructure}, all remain effective at that scale
through the same mechanism of per-block regularization
(Appendix~\ref{sec:feasibility} and the ESM-20 runs of
Appendix~\ref{sec:coststructure}). These results are consistent with the
nano ablation of Appendix~\ref{sec:nano}, which evaluates the same
methods without the equal oracle budget imposed in this comparison.



\begin{thebibliography}{99}
\small

\bibitem[Nesterov and Polyak(2006)]{nesterov2006cubic}
Y.~Nesterov and B.~T. Polyak.
Cubic regularization of the Newton method and its global performance.
\emph{Mathematical Programming}, 108(1):177--205, 2006.

\bibitem[Cartis et~al.(2010)]{cartis2010steepest}
C.~Cartis, N.~I.~M. Gould, and P.~L. Toint.
On the complexity of steepest descent, Newton's and regularized
Newton's methods for nonconvex unconstrained optimization problems.
\emph{SIAM Journal on Optimization}, 20(6):2833--2852, 2010.

\bibitem[Nocedal and Wright(2006)]{nocedal2006}
J.~Nocedal and S.~J. Wright.
\emph{Numerical Optimization}.
Springer, New York, second edition, 2006.

\bibitem[Cartis et~al.(2011a)]{cartis2011a}
C.~Cartis, N.~I.~M. Gould, and P.~L. Toint.
Adaptive cubic regularisation methods for unconstrained optimization.
Part I: motivation, convergence and numerical results.
\emph{Mathematical Programming}, 127(2):245--295, 2011.

\bibitem[Cartis et~al.(2011b)]{cartis2011b}
C.~Cartis, N.~I.~M. Gould, and P.~L. Toint.
Adaptive cubic regularisation methods for unconstrained optimization.
Part II: worst-case function- and derivative-evaluation complexity.
\emph{Mathematical Programming}, 130(2):295--319, 2011.

\bibitem[Cartis et~al.(2025)]{cartis2025rarc}
C.~Cartis, Z.~Shao, and E.~Tansley.
Random subspace cubic-regularization methods, with applications to
low-rank functions.
arXiv preprint arXiv:2501.09734, 2025.

\bibitem[Bellavia et~al.(2025)]{bellavia2025subspace}
S.~Bellavia, D.~Palitta, M.~Porcelli, and V.~Simoncini.
Regularized methods via cubic model subspace minimization for
nonconvex optimization.
\emph{Computational Optimization and Applications}, 90(3):801--837, 2025.
doi:10.1007/s10589-025-00655-2.

\bibitem[Al~Daas and Gould(2025)]{aldaas2025extkrylov}
H.~Al~Daas and N.~I.~M. Gould.
Extended-Krylov-subspace methods for trust-region and
norm-regularization subproblems.
arXiv preprint arXiv:2511.11135, 2025.

\bibitem[Cristofari(2026)]{cristofari2024block}
A.~Cristofari.
Block cubic Newton with greedy selection.
\emph{Optimization Methods and Software}, 2026.
Published online 28 June 2026, doi:10.1080/10556788.2026.2692465;
arXiv:2407.18150.

\bibitem[Wolinski(2025)]{wolinski2024partition}
P.~Wolinski.
Gathering and exploiting higher-order information when training large
structured models.
arXiv preprint arXiv:2312.03885, 2025.

\bibitem[Ranganath et~al.(2025)]{ranganath2025sr1}
A.~Ranganath, M.~Singhal, and R.~Marcia.
Symmetric rank-one quasi-Newton methods for deep learning using cubic
regularization.
arXiv preprint arXiv:2502.12298, 2025.

\bibitem[Forristal et~al.(2022)]{forristal2022arclqn}
J.~Forristal, J.~Griffin, W.~Zhou, and S.~A. Yektamaram.
A novel fast exact subproblem solver for stochastic quasi-Newton cubic
regularized optimization.
arXiv preprint arXiv:2204.09116, 2022.

\bibitem[Pasechnyuk-Vilensky and Tak\'a\v{c}(2026)]{pasechnyuk2026cheby}
D.~Pasechnyuk-Vilensky and M.~Tak\'a\v{c}.
Chebyshev-exact acceleration under Hessian variation, I: Sine-Jacobi
method.
arXiv preprint arXiv:2606.16671, 2026.

\bibitem[Tansley and Cartis(2025)]{tansley2025lowrank}
E.~Tansley and C.~Cartis.
Scalable second-order optimization algorithms for minimizing low-rank
functions.
In \emph{OPT 2024: Workshop on Optimization for Machine Learning}, 2024.
arXiv:2501.03718.

\bibitem[Zhou et~al.(2025)]{zhou2025regnewton}
Y.~Zhou, J.~Xu, B.~Li, C.~Bao, C.~Ding, and J.~Zhu.
A regularized Newton method for nonconvex optimization with global
and local complexity guarantees.
In \emph{Advances in Neural Information Processing Systems (NeurIPS)},
38:81676--81731, 2025.

\bibitem[Cartis et~al.(2025)]{cartisjerad2025secant}
C.~Cartis, S.~Jerad, and K.~Welzel.
On global rates for regularization methods based on secant derivative
approximations.
arXiv preprint arXiv:2509.07580, 2025.

\bibitem[Marumo and Takeda(2026)]{marumo2026recipe}
N.~Marumo and A.~Takeda.
A general recipe for parameter-free nonconvex optimization via
higher-order regularization.
arXiv preprint arXiv:2605.30891, 2026.

\bibitem[Fang et~al.(2026)]{fang2026paramfree}
S.~Fang, N.~Marumo, and A.~Takeda.
Parameter-free cubic-regularized Newton method: Sharp complexity and
generalized smoothness.
arXiv preprint arXiv:2607.10741, 2026.

\bibitem[Shestakov et~al.(2025)]{shestakov2025inexact}
A.~Shestakov, N.~Bashirov, A.~Semenov, A.~Gasnikov, M.~Tak\'a\v{c},
A.~Beznosikov, and D.~Kamzolov.
Adaptive regularized Newton method with inexact Hessian.
arXiv preprint arXiv:2512.08775, 2025.

\bibitem[Pasechnyuk-Vilensky et~al.(2025)]{pasechnyuk2025vr}
D.~Pasechnyuk-Vilensky, D.~Kamzolov, and M.~Tak\'a\v{c}.
Cubic regularized Newton method with variance reduction for finite-sum
non-convex problems.
arXiv preprint arXiv:2510.08714, 2025.

\bibitem[Scheinberg and Xie(2023)]{scheinberg2023sarc}
K.~Scheinberg and M.~Xie.
First- and second-order stochastic adaptive regularization with
cubics: High probability iteration and sample complexity.
arXiv preprint arXiv:2308.13161, 2023.

\bibitem[Chen et~al.(2026)]{chen2025len}
L.~Chen, C.~Liu, L.~Luo, and J.~Zhang.
Faster Newton methods for convex and nonconvex optimization in
gradient complexity.
In \emph{Conference on Learning Theory (COLT)}, PMLR, 2026.
arXiv:2501.17488.

\bibitem[Fuji et~al.(2025)]{fuji2025subspace}
T.~Fuji, P.-L. Poirion, and A.~Takeda.
Theoretical analysis of the randomized subspace regularized Newton
method for non-convex optimization.
\emph{Open Journal of Mathematical Optimization}, 6:article no.~8,
2025.

\bibitem[Higuchi et~al.(2025)]{higuchi2025rshtr}
R.~Higuchi, P.-L. Poirion, and A.~Takeda.
Improving convergence guarantees of random subspace second-order
algorithm for nonconvex optimization.
In \emph{International Conference on Learning Representations (ICLR)},
2025.

\bibitem[Tsipinakis et~al.(2026)]{tsipinakis2026multilevel}
N.~Tsipinakis, P.~Tigas, and P.~Parpas.
A multilevel low-rank Newton method with super-linear convergence rate
and its application to non-convex problems.
\emph{Transactions on Machine Learning Research}, 2026.

\bibitem[Zeng et~al.(2026a)]{zeng2026newtoncg}
Z.~Zeng, J.~Zhang, and C.~He.
Adaptive Newton-CG methods with global and local analysis for
unconstrained optimization with H\"older continuous Hessian.
arXiv preprint arXiv:2604.02763, 2026.

\bibitem[Zeng et~al.(2026b)]{zeng2026minres}
H.~Zeng, Y.~Liu, W.~Ouyang, and A.~Milzarek.
A MINRES-based linesearch algorithm for nonconvex optimization with
non-positive curvature detection.
arXiv preprint arXiv:2601.01575, 2026.

\bibitem[Abreu et~al.(2026)]{abreu2026gn}
N.~Abreu, N.~Vyas, S.~M. Kakade, and D.~Morwani.
The potential of second-order optimization for LLMs: A study with full
Gauss-Newton.
arXiv preprint arXiv:2510.09378, 2026.

\bibitem[Martins Gomes(2025)]{gomes2025adafisher}
D.~Martins Gomes.
Towards practical second-order optimizers in deep learning: Insights
from Fisher information analysis.
M.A.Sc.\ thesis, Concordia University, 2025.
arXiv:2504.20096.

\bibitem[Hamad and Hinder(2024)]{hamad2024tr}
F.~Hamad and O.~Hinder.
A simple and practical adaptive trust-region method.
arXiv preprint arXiv:2412.02079, 2024.

\bibitem[Ha et~al.(2026)]{ha2026regastro}
Y.~Ha, S.~Shashaani, and Q.~Tran-Dinh.
Adaptive regularization within trust region methods for stochastic
nonconvex optimization.
arXiv preprint arXiv:2604.15457, 2026.

\bibitem[Zhang and Jiang(2025)]{zhangjiang2024riemannian}
C.~Zhang and R.~Jiang.
Riemannian adaptive regularized Newton methods with H\"older
continuous Hessians.
\emph{Computational Optimization and Applications}, 92(1):29--79, 2025.
doi:10.1007/s10589-025-00692-x.

\bibitem[Pei et~al.(2025)]{pei2025ssarcqk}
Y.~Pei, Y.~Lin, S.~Shao, M.~Silva~Louzeiro, and D.~Zhu.
A scalable sequential adaptive cubic regularization algorithm for
optimization with general equality constraints.
arXiv preprint arXiv:2503.11254, 2025.

\bibitem[Wang et~al.(2025)]{wang2024minimax}
J.-L. Wang, Z.~Xu, and H.-L. Zhang.
A fully parameter-free second-order algorithm for convex-concave
minimax problems.
\emph{Journal of Machine Learning Research}, 2025.
arXiv:2407.03571.

\bibitem[Huo et~al.(2024)]{huo2024dpfcrn}
W.~Huo, C.~Liu, K.~Ding, K.~H. Johansson, and L.~Shi.
Federated cubic regularized Newton learning with
sparsification-amplified differential privacy.
arXiv preprint arXiv:2408.04315, 2024.

\bibitem[d'Aspremont et~al.(2021)]{daspremont2021acceleration}
A.~d'Aspremont, D.~Scieur, and A.~Taylor.
Acceleration methods.
\emph{Foundations and Trends in Optimization}, 5(1--2):1--245, 2021.

\bibitem[Kohler and Lucchi(2017)]{kohler2017}
J.~M. Kohler and A.~Lucchi.
Sub-sampled cubic regularization for non-convex optimization.
In \emph{International Conference on Machine Learning (ICML)},
PMLR 70:1895--1904, 2017.

\bibitem[Dussault et~al.(2023)]{dussault2024arcqk}
J.-P. Dussault, T.~Migot, and D.~Orban.
Scalable adaptive cubic regularization methods.
\emph{Mathematical Programming}, 207(1--2):191--225, 2023.
doi:10.1007/s10107-023-02007-6.

\bibitem[Tsingalis et~al.(2026)]{tsingalis2026adacubic}
I.~Tsingalis, C.~Kotropoulos, and C.~Briat.
AdaCubic: an adaptive cubic regularization optimizer for deep learning.
\emph{Transactions on Machine Learning Research}, 2026.

\bibitem[Tripuraneni et~al.(2018)]{tripuraneni2018}
N.~Tripuraneni, M.~Stern, C.~Jin, J.~Regier, and M.~I. Jordan.
Stochastic cubic regularization for fast nonconvex optimization.
In \emph{Advances in Neural Information Processing Systems~31
(NeurIPS)}, pages 2899--2908, 2018.

\bibitem[Xu et~al.(2019)]{xu2020inexact}
P.~Xu, F.~Roosta, and M.~W. Mahoney.
Newton-type methods for non-convex optimization under inexact Hessian
information.
\emph{Mathematical Programming}, 184(1--2):35--70, 2019.
doi:10.1007/s10107-019-01405-z.

\bibitem[Carmon and Duchi(2019)]{carmon2018}
Y.~Carmon and J.~C. Duchi.
Gradient descent finds the cubic-regularized nonconvex Newton step.
\emph{SIAM Journal on Optimization}, 29(3):2146--2178, 2019.

\bibitem[Carmon and Duchi(2018)]{carmon2018krylov}
Y.~Carmon and J.~C. Duchi.
Analysis of Krylov subspace solutions of regularized non-convex
quadratic problems.
In \emph{Advances in Neural Information Processing Systems~31
(NeurIPS)}, 2018.

\bibitem[Carmon et~al.(2020)]{carmon2020lower}
Y.~Carmon, J.~C. Duchi, O.~Hinder, and A.~Sidford.
Lower bounds for finding stationary points I.
\emph{Mathematical Programming}, 184(1):71--120, 2020.

\bibitem[Nemirovski and Yudin(1983)]{nemirovski1983}
A.~S. Nemirovski and D.~B. Yudin.
\emph{Problem Complexity and Method Efficiency in Optimization}.
Wiley-Interscience, 1983.

\bibitem[Masiha et~al.(2022)]{masiha2022}
S.~Masiha, S.~Salehkaleybar, N.~He, N.~Kiyavash, and P.~Thiran.
Stochastic second-order methods improve best-known sample complexity of
SGD for gradient-dominated functions.
In \emph{Advances in Neural Information Processing Systems~35
(NeurIPS)}, 2022.

\bibitem[Doikov et~al.(2024)]{doikov2024spectral}
N.~Doikov, S.~U. Stich, and M.~Jaggi.
Spectral preconditioning for gradient methods on graded non-convex
functions.
In \emph{International Conference on Machine Learning (ICML)}, 2024.

\bibitem[Doikov et~al.(2023)]{doikov2023lazy}
N.~Doikov, E.~M. Chayti, and M.~Jaggi.
Second-order optimization with lazy Hessians.
In \emph{International Conference on Machine Learning (ICML)}, 2023.

\bibitem[Doikov and Nesterov(2025)]{doikov2025gcb}
N.~Doikov and Y.~Nesterov.
Universal reduced-operator method and high-order global curvature
bounds.
\emph{arXiv preprint arXiv:2511.07341}, 2025.

\bibitem[Semenov et~al.(2025)]{semenov2025gns}
A.~Semenov, M.~Jaggi, and N.~Doikov.
Gradient-normalized smoothness for optimization with approximate
Hessians.
\emph{arXiv preprint arXiv:2506.13710}, 2025.

\bibitem[Doikov and Richt\'arik(2018)]{doikov2018rbcn}
N.~Doikov and P.~Richt\'arik.
Randomized block cubic Newton method.
In \emph{International Conference on Machine Learning (ICML)}, 2018.

\bibitem[Doikov and Grapiglia(2023)]{doikov2023fo}
N.~Doikov and G.~N. Grapiglia.
First and zeroth-order implementations of the regularized Newton
method with lazy approximated Hessians.
\emph{arXiv preprint arXiv:2309.02412}, 2023.

\bibitem[Doikov and Rodomanov(2023)]{doikov2023poly}
N.~Doikov and A.~Rodomanov.
Polynomial preconditioning for gradient methods.
In \emph{International Conference on Machine Learning (ICML)}, 2023.

\bibitem[Doikov(2021)]{doikov_thesis}
N.~Doikov.
\emph{New Second-Order and Tensor Methods in Convex Optimization}.
PhD thesis, Universit\'e catholique de Louvain (UCLouvain), 2021.

\bibitem[Chayti et~al.(2025)]{chayti2025momentum}
E.~M. Chayti, N.~Doikov, and M.~Jaggi.
Improving stochastic cubic Newton with momentum.
In \emph{International Conference on Artificial Intelligence and Statistics
(AISTATS)}, PMLR 258, 2025.

\bibitem[Chayti et~al.(2024)]{chayti2024unified}
E.~M. Chayti, M.~Jaggi, and N.~Doikov.
Unified convergence theory of stochastic and variance-reduced cubic Newton
methods.
\emph{Transactions on Machine Learning Research (TMLR)}, 2024.

\bibitem[Agafonov et~al.(2024)]{agafonov2024inexact}
A.~Agafonov, D.~Kamzolov, P.~Dvurechensky, A.~Gasnikov, and M.~Tak\'a\v{c}.
Inexact tensor methods and their application to stochastic convex
optimization.
To appear in \emph{Optimization Methods and Software}.

\bibitem[Zhu and Chang(2022)]{zhu2022acrl}
Z.~Zhu and J.~Chang.
Solving the adaptive cubic regularization sub-problem using the Lanczos
method.
\emph{Symmetry}, 14(10):2191, 2022.

\bibitem[Zhao et~al.(2025)]{zhao2025sscn}
J.~Zhao, A.~Lucchi, and N.~Doikov.
Cubic regularized subspace Newton for non-convex optimization.
In \emph{International Conference on Artificial Intelligence and Statistics
(AISTATS)}, 2025.

\bibitem[Jarre(2013)]{jarre2011chebyrosen}
F.~Jarre.
On Nesterov's smooth Chebyshev--Rosenbrock function.
\emph{Optimization Methods and Software}, 28(3):478--500, 2013.
(Preprint, University of D\"usseldorf, 2011.)

\bibitem[Podorozhny(2026)]{podorozh2026siam}
R.~M. Podorozhny.
Curvature-aware optimization via {Chebyshev} polynomials of the second
kind for deep learning.
In \emph{SIAM Conference on Optimization (OP26), Book of Abstracts},
session CP4, Edinburgh, June 2026. Talk abstract,
\url{https://www.siam.org/media/r0be0xtr/op26_abstracts_v3.pdf}.

\bibitem[Sitzmann et~al.(2020)]{siren2020}
V.~Sitzmann, J.~N.~P. Martel, A.~W. Bergman, D.~B. Lindell, and
G.~Wetzstein.
Implicit neural representations with periodic activation functions.
In \emph{Advances in Neural Information Processing Systems (NeurIPS)},
2020. Oral presentation.

\bibitem[Zeraatkar et~al.(2025)]{zeraatkar2025}
E.~Zeraatkar, S.~A. Faroughi, and J.~Te\v{s}i\'c.
{ViSIR}: Vision transformer single image reconstruction method for
{Earth System Models}.
\emph{IEEE Pulse}, 16(1):23--25, 2025.

\bibitem[Rahaman et~al.(2019)]{rahaman2019}
N.~Rahaman, A.~Baratin, D.~Arpit, F.~Draxler, M.~Lin, F.~Hamprecht,
Y.~Bengio, and A.~Courville.
On the spectral bias of neural networks.
In \emph{International Conference on Machine Learning (ICML)}, 2019.

\bibitem[Tancik et~al.(2020)]{tancik2020fourier}
M.~Tancik, P.~P. Srinivasan, B.~Mildenhall, S.~Fridovich-Keil,
N.~Raghavan, U.~Singhal, R.~Ramamoorthi, J.~T. Barron, and R.~Ng.
Fourier features let networks learn high frequency functions in low
dimensional domains.
In \emph{Advances in Neural Information Processing Systems~33
(NeurIPS)}, 2020.

\bibitem[Saragadam et~al.(2023)]{saragadam2023wire}
V.~Saragadam, D.~LeJeune, J.~Tan, G.~Balakrishnan, A.~Veeraraghavan,
and R.~G. Baraniuk.
WIRE: wavelet implicit neural representations.
In \emph{IEEE/CVF Conference on Computer Vision and Pattern
Recognition (CVPR)}, 2023.

\bibitem[Liu et~al.(2024)]{liu2024finer}
Z.~Liu, H.~Zhu, Q.~Zhang, J.~Fu, W.~Deng, Z.~Ma, Y.~Guo, and X.~Cao.
FINER: flexible spectral-bias tuning in implicit neural representation
by variable-periodic activation functions.
In \emph{IEEE/CVF Conference on Computer Vision and Pattern
Recognition (CVPR)}, 2024.

\bibitem[Kania et~al.(2025)]{kania2025fresh}
A.~Kania, M.~Mihajlovi\'c, S.~Prokudin, J.~Tabor, and P.~Spurek.
FreSh: frequency shifting for accelerated neural representation
learning.
In \emph{International Conference on Learning Representations (ICLR)},
2025.

\bibitem[Chng et~al.(2025)]{chng2025precond}
S.-F. Chng, H.~Saratchandran, and S.~Lucey.
Preconditioners for the stochastic training of neural fields.
In \emph{IEEE/CVF Conference on Computer Vision and Pattern
Recognition (CVPR)}, pages 27222--27232, 2025.

\bibitem[Shi et~al.(2025)]{shi2025iga}
K.~Shi, H.~Chen, L.~Zhang, and S.~Gu.
Inductive gradient adjustment for spectral bias in implicit neural
representations.
In \emph{International Conference on Machine Learning (ICML)},
PMLR 267:54864--54891, 2025.

\bibitem[Ling et~al.(2025)]{ling2025stochprecond}
S.~Ling, M.~Nimier-David, A.~Jacobson, and N.~Sharp.
Stochastic preconditioning for neural field optimization.
\emph{ACM Transactions on Graphics (Proc.\ SIGGRAPH)}, 44(4), 2025.

\bibitem[Gropp et~al.(2020)]{gropp2020igr}
A.~Gropp, L.~Yariv, N.~Haim, M.~Atzmon, and Y.~Lipman.
Implicit geometric regularization for learning shapes.
In \emph{International Conference on Machine Learning (ICML)},
PMLR 119:3789--3799, 2020.

\bibitem[Kingma and Ba(2015)]{kingma2015adam}
D.~P. Kingma and J.~Ba.
Adam: a method for stochastic optimization.
In \emph{International Conference on Learning Representations (ICLR)},
2015.

\bibitem[Loshchilov and Hutter(2019)]{loshchilov2019}
I.~Loshchilov and F.~Hutter.
Decoupled weight decay regularization.
In \emph{International Conference on Learning Representations (ICLR)},
2019.

\bibitem[Jordan et~al.(2024)]{Jordan2024Muon}
J.~Jordan, K.~Keller, Y.~Jin, V.~Boza, J.~You, F.~Cesista, L.~Newhouse, and J.~Bernstein.
Muon: An Optimizer for Hidden Layers in Neural Networks.
\url{https://kellerjordan.github.io/posts/muon/}, 2024.

\bibitem[Zhang et~al.(2018)]{zhang2018lpips}
R.~Zhang, P.~Isola, A.~A. Efros, E.~Shechtman, and O.~Wang.
The unreasonable effectiveness of deep features as a perceptual metric.
In \emph{IEEE/CVF Conference on Computer Vision and Pattern Recognition
(CVPR)}, 2018.

\bibitem[Gould et~al.(1999)]{gould1999gltr}
N.~I.~M. Gould, S.~Lucidi, M.~Roma, and P.~L. Toint.
Solving the trust-region subproblem using the Lanczos method.
\emph{SIAM Journal on Optimization}, 9(2):504--525, 1999.

\bibitem[Saad(2003)]{saad2003}
Y.~Saad.
\emph{Iterative Methods for Sparse Linear Systems}.
SIAM, 2nd edition, 2003.

\bibitem[Golub and Van~Loan(2013)]{golub2013}
G.~H. Golub and C.~F. Van~Loan.
\emph{Matrix Computations}.
JHU Press, 4th edition, 2013.

\bibitem[Mishchenko(2023)]{mishchenko2023}
K.~Mishchenko.
Regularized Newton method with global $\mathcal{O}(1/k^2)$
convergence.
\emph{SIAM Journal on Optimization}, 33(3):1440--1462, 2023.

\bibitem[Doikov and Nesterov(2023)]{doikov2024gradreg}
N.~Doikov and Y.~Nesterov.
Gradient regularization of Newton method with Bregman distances.
\emph{Mathematical Programming}, 204(1--2):1--25, 2023.
doi:10.1007/s10107-023-01943-7.

\bibitem[Gratton et~al.(2023)]{gratton2024yet}
S.~Gratton, S.~Jerad, and P.~L. Toint.
Yet another fast variant of Newton's method for nonconvex optimization.
arXiv preprint arXiv:2302.10065, 2023.

\bibitem[Jiang et al.(2024)]{jiang2024}
R.~Jiang, P.~Raman, S.~Sabach, A.~Mokhtari, M.~Hong, and V.~Cevher.
Krylov cubic regularized Newton: a subspace second-order method with
dimension-free convergence rate.
In \emph{International Conference on Artificial Intelligence and
Statistics (AISTATS)}, PMLR 238, 2024.

\bibitem[Das et~al.(2024)]{das2024precond}
R.~Das, N.~Agarwal, S.~Sanghavi, and I.~S. Dhillon.
Towards quantifying the preconditioning effect of Adam.
\emph{arXiv preprint arXiv:2402.07114}, 2024.

\bibitem[Jiang et~al.(2023)]{jiang2023geometry}
K.~Jiang, D.~Malik, and Y.~Li.
How does adaptive optimization impact local neural network geometry?
In \emph{Advances in Neural Information Processing Systems~36
(NeurIPS)}, 2023.

\bibitem[Zhang et~al.(2025)]{zhangmaes2025rotation}
T.~H. Zhang, L.~Maes, A.~Milligan, A.~Jolicoeur-Martineau,
I.~Mitliagkas, D.~Scieur, S.~Lacoste-Julien, and C.~Guille-Escuret.
Understanding Adam requires better rotation dependent assumptions.
In \emph{Advances in Neural Information Processing Systems~38
(NeurIPS)}, 2025.

\bibitem[Xie et~al.(2024)]{xie2024linf}
S.~Xie, M.~A. Mohamadi, and Z.~Li.
Adam exploits $\ell_\infty$-geometry of loss landscape via
coordinate-wise adaptivity.
\emph{arXiv preprint arXiv:2410.08198}, 2024.

\bibitem[Zhang et~al.(2024)]{zhang2024transformers}
Y.~Zhang, C.~Chen, T.~Ding, Z.~Li, R.~Sun, and Z.-Q. Luo.
Why transformers need Adam: a Hessian perspective.
In \emph{Advances in Neural Information Processing Systems~37
(NeurIPS)}, 2024.

\bibitem[Dong et~al.(2025)]{dong2025hessianstruct}
Z.~Dong, Y.~Zhang, J.~Yao, and R.~Sun.
Towards quantifying the Hessian structure of neural networks.
\emph{arXiv preprint arXiv:2505.02809}, 2025.

\bibitem[Sagun et~al.(2018)]{sagun2018hessian}
Sagun, L., Evci, U., G\"uney, V.~U., Dauphin, Y., and Bottou, L.
\newblock Empirical analysis of the {H}essian of over-parametrized
neural networks.
\newblock \emph{ICLR Workshop}, 2018.

\bibitem[Jacot et~al.(2018)]{jacot2018ntk}
Jacot, A., Gabriel, F., and Hongler, C.
\newblock Neural tangent kernel: Convergence and generalization in
neural networks.
\newblock \emph{NeurIPS}, 2018.

\bibitem[Gur-Ari et~al.(2018)]{gurari2018subspace}
G.~Gur-Ari, D.~A. Roberts, and E.~Dyer.
Gradient descent happens in a tiny subspace.
\emph{arXiv preprint arXiv:1812.04754}, 2018.

\bibitem[Ghorbani et~al.(2019)]{ghorbani2019hessian}
B.~Ghorbani, S.~Krishnan, and Y.~Xiao.
An investigation into neural net optimization via Hessian eigenvalue
density.
In \emph{International Conference on Machine Learning (ICML)},
PMLR 97:2232--2241, 2019.

\bibitem[Dauphin et~al.(2014)]{dauphin2014saddle}
Y.~N. Dauphin, R.~Pascanu, C.~Gulcehre, K.~Cho, S.~Ganguli, and
Y.~Bengio.
Identifying and attacking the saddle point problem in high-dimensional
non-convex optimization.
In \emph{Advances in Neural Information Processing Systems~27
(NeurIPS)}, 2014.

\bibitem[Cohen et~al.(2022)]{cohen2022adaptive}
J.~Cohen, B.~Ghorbani, S.~Krishnan, N.~Agarwal, S.~Medapati,
M.~Badura, D.~Suo, D.~Cardoze, Z.~Nado, G.~E. Dahl, and J.~Gilmer.
Adaptive gradient methods at the edge of stability.
\emph{arXiv preprint arXiv:2207.14484}, 2022.

\bibitem[van~der~Sluis(1969)]{vandersluis1969}
A.~van~der~Sluis.
Condition numbers and equilibration of matrices.
\emph{Numerische Mathematik}, 14(1):14--23, 1969.

\bibitem[Griewank(1981)]{griewank1981modification}
A.~Griewank.
The modification of Newton's method for unconstrained optimization by
bounding cubic terms.
\emph{Technical Report NA/12}, DAMTP, University of Cambridge, 1981.

\bibitem[Hochbruck and Lubich(1997)]{hochbruck1997krylov}
M.~Hochbruck and C.~Lubich.
On Krylov subspace approximations to the matrix exponential operator.
\emph{SIAM Journal on Numerical Analysis}, 34(5):1911--1925, 1997.

\bibitem[Hochbruck and Ostermann(2010)]{hochbruck2010exponential}
M.~Hochbruck and A.~Ostermann.
Exponential integrators.
\emph{Acta Numerica}, 19:209--286, 2010.

\bibitem[Saad(1992)]{saad1992krylov}
Y.~Saad.
Analysis of some Krylov subspace approximations to the matrix
exponential operator.
\emph{SIAM Journal on Numerical Analysis}, 29(1):209--228, 1992.

\bibitem[Vyas et~al.(2024)]{vyas2024soap}
N.~Vyas, D.~Morwani, R.~Zhao, M.~Kwun, I.~Shapira, D.~Brandfonbrener,
L.~Janson, and S.~Kakade.
SOAP: Improving and stabilizing Shampoo using Adam.
\emph{arXiv preprint arXiv:2409.11321}, 2024.

\bibitem[Zhu et~al.(2024)]{zhu2024finerpp}
H.~Zhu, Z.~Liu, Q.~Zhang, J.~Fu, W.~Deng, Z.~Ma, Y.~Guo, and X.~Cao.
FINER++: building a family of variable-periodic functions for
activating implicit neural representation.
\emph{arXiv preprint arXiv:2407.19434}, 2024.

\bibitem[Doikov(2026)]{doikov2026pamm}
N.~Doikov.
Primal acceleration of Newton's method.
\emph{arXiv preprint arXiv:2608.21359}, 2026.

\bibitem[Amaral et~al.(2022)]{amaral2022highorder}
V.~S. Amaral, R.~Andreani, E.~G. Birgin, D.~S. Marcondes, and J.~M. Mart\'{\i}nez.
On complexity and convergence of high-order coordinate descent algorithms
for smooth nonconvex box-constrained minimization.
\emph{Journal of Global Optimization}, 84(3):527--561, 2022.

\bibitem[Lee et~al.(2019)]{lee2019wide}
J.~Lee, L.~Xiao, S.~S. Schoenholz, Y.~Bahri, R.~Novak, J.~Sohl-Dickstein,
and J.~Pennington.
Wide neural networks of any depth evolve as linear models under gradient
descent.
In \emph{Advances in Neural Information Processing Systems (NeurIPS)}, 2019.

\bibitem[Zancato et~al.(2020)]{zancato2020predicting}
L.~Zancato, A.~Achille, A.~Ravichandran, R.~Bhotika, and S.~Soatto.
Predicting training time without training.
arXiv preprint arXiv:2008.12478, 2020.

\bibitem[Fort et~al.(2020)]{fort2020deepvskernel}
S.~Fort, G.~K. Dziugaite, M.~Paul, S.~Kharaghani, D.~M. Roy, and
S.~Ganguli.
Deep learning versus kernel learning: an empirical study of loss landscape
geometry and the time evolution of the neural tangent kernel.
In \emph{Advances in Neural Information Processing Systems (NeurIPS)}, 2020.

\bibitem[Luo et~al.(2025)]{luo2025multipower}
K.~Luo, H.~Wen, S.~Hu, Z.~Sun, Z.~Liu, M.~Sun, K.~Lyu, and W.~Chen.
A multi-power law for loss curve prediction across learning rate
schedules.
In \emph{International Conference on Learning Representations (ICLR)}, 2025.

\bibitem[Pourali et~al.(2026)]{pourali2026cape}
A.~Pourali, B.~Boukani, and H.~Khazaei.
{CAPE}: generalized convergence prediction across architectures without
full training.
\emph{Transactions on Machine Learning Research}, 2026.

\bibitem[Wang et~al.(2022)]{wang2022pinnsntk}
S.~Wang, X.~Yu, and P.~Perdikaris.
When and why {PINNs} fail to train: a neural tangent kernel perspective.
\emph{Journal of Computational Physics}, 449:110768, 2022.

\bibitem[Riedl et~al.(2026)]{riedl2026newtonntk}
K.~Riedl, K.~Spiliopoulos, and J.~Sirignano.
Convergence analysis of Newton's method for neural networks in the
overparameterized limit.
\emph{arXiv preprint arXiv:2605.08352}, 2026.

\bibitem[Jiang et~al.(2026)]{jiang2026pgdrich}
S.~Jiang, A.~Voronin, E.~C. Cyr, and B.~Southworth.
On the convergence behavior of preconditioned gradient descent toward the
rich learning regime.
In \emph{International Conference on Learning Representations (ICLR)}, 2026.
arXiv:2601.03162.

\bibitem[de~Moura and Ullrich(2021)]{demoura2021lean}
L.~de~Moura and S.~Ullrich.
The {Lean} 4 theorem prover and programming language.
In \emph{Automated Deduction -- CADE 28}, Lecture Notes in Computer
Science, pages 625--635, 2021.
doi:10.1007/978-3-030-79876-5\_37.

\bibitem[mathlib Community(2020)]{mathlib2020}
The mathlib Community.
The {Lean} mathematical library.
In \emph{Certified Programs and Proofs (CPP)}, 2020.

\end{thebibliography}
\end{document}